\documentclass[a4paper, 11pt]{article}
\pdfoutput=1  
\usepackage{etoolbox,verbatim}
\newtoggle{arxiv}
\toggletrue{arxiv} 
\newtoggle{colt}
\togglefalse{colt} 
\newtoggle{nspace}
\togglefalse{nspace} 
\newcommand{\nonumsection}[1]{\section*{#1}%
\addcontentsline{toc}{section}{#1}}
\newcommand{\nonumpart}[1]{\part*{#1}
\addcontentsline{toc}{part}{#1}}

\usepackage[T1]{fontenc}

\usepackage{amssymb,upgreek}
\usepackage{bm}
\usepackage{mathtools}

\iftoggle{colt}{
\DeclareRobustCommand{\qed}{%
\usepackage{thmtools}
\usepackage{thm-restate}
  \ifmmode \mathqed
  \else
    \leavevmode\unskip\penalty9999 \hbox{}\nobreak\hfill
    \quad\hbox{\qedsymbol}%
  \fi
}
}{\usepackage{amsthm}
\usepackage{thmtools}
\usepackage{thm-restate}}

\iftoggle{arxiv}{
\usepackage{mathrsfs}

\usepackage[
a4paper,
top=1in,
bottom=1in,
left=1in,
right=1in]{geometry}

\usepackage[bitstream-charter]{mathdesign}
 
\usepackage[scaled=0.92]{PTSans}
}
{
\usepackage{mathrsfs}
}

\usepackage[english]{babel}  
\usepackage{multirow}
\usepackage{tcolorbox}
\newcommand{\linkcolor}{blue!70!black}
  
    \newcommand{\thmcolordark}{red!30!black}
\usepackage{tikz-cd} 
\usepackage{turnstile}
\usepackage{xcolor}
\usepackage{xspace}
\usepackage{xparse}
\usepackage[normalem]{ulem}
\usepackage{verbatim}  
\usepackage{tabu}

\usepackage{relsize}
\usepackage{thm-restate}
\usepackage{appendix}
\usepackage{breakcites}

\usepackage{longtable}
\usepackage[shortlabels]{enumitem}
\usepackage{tablefootnote}
\usepackage{verbatim} 

\usepackage{booktabs}   
\usepackage{float}
\usepackage{makecell}

\usepackage{footnote}

\usepackage{boxedminipage}
\usepackage{multirow,nicefrac}
\usepackage{xcolor}

\iftoggle{arxiv}
{
\usepackage{subfig}
  \usepackage[breaklinks=true,linkcolor=\linkcolor]{hyperref}

  \usepackage{fullpage}
  \usepackage[noend]{algpseudocode}
  \usepackage{algorithm}
  \hypersetup{colorlinks,citecolor=blue,linkcolor = \linkcolor}
}
{
\usepackage{hyperref}
  
}
  \usepackage{natbib}

\DeclareMathSymbol{\shortminus}{\mathbin}{AMSa}{"39}

\usepackage{prettyref}

\usepackage[capitalise,nameinlink]{cleveref}
\Crefname{equation}{Eq.}{Eqs.}
\Crefname{assumption}{Assumption}{Assumptions}
\Crefname{condition}{Condition}{Conditions}
\Crefname{claim}{Claim}{Claims}
\Crefname{property}{Property}{Properties}
\Crefname{construction}{Construction}{Constructions}

\usepackage{xcolor}

\newcommand{\N}{\mathbb{N}}
\newcommand{\R}{\mathbb{R}}

\numberwithin{equation}{section}
\newcommand\numberthis{\addtocounter{equation}{1}\tag{\theequation}}

\newcommand{\eye}{\mathbf{I}}

\newcommand{\nablatwo}{\nabla^{\,2}}

\newcommand{\rmd}{\mathrm{d}}

\newcommand{\bzero}{\ensuremath{\mathbf 0}}
\newcommand{\y}{\ensuremath{\mathbf y}}

\def\bB{\mathbf{B}}

\def\bu{\mathbf{u}}

\def\y{\mathbf{y}}

\def\by{\mathbf{y}}
\def\bz{\mathbf{z}}

\def\bu{\mathbf{u}}
\def\bv{\mathbf{v}}

\def\bw{\mathbf{w}}
\def\bA{\mathbf{A}}

\def\bI{\mathbf{I}}

\def\bP{\mathbf{P}}
\def\bQ{\mathbf{Q}}

\newcommand{\ignore}[1]{}

\DeclareMathOperator{\BigOm}{{O}}

\newcommand{\BigOh}[1]{\BigOm\left({#1}\right)}

\newcommand{\iidsim}{\overset{\mathrm{i.i.d}}{\sim}}

\newcommand{\op}{\mathrm{op}}
\newcommand{\fro}{\mathrm{F}}

\iftoggle{colt}{}{
\declaretheoremstyle[
    headformat=\normalfont\textcolor{\thmcolordark}{\bfseries\NAME\,\NUMBER}\NOTE,%
    notefont={\normalfont\textcolor{\thmcolordark}{\bfseries}}, 
    notebraces={}{},
    bodyfont=\normalfont\itshape,
    spaceabove = 6pt,
    spacebelow = 6pt,
    ]{coloredthmversion}

\declaretheoremstyle[
    headformat=\normalfont\textcolor{\thmcolordark}{\bfseries\NAME\,\NUMBER}\NOTE,%
    bodyfont=\normalfont\itshape,
    spaceabove = 6pt,
    spacebelow = 6pt,
    ]{coloredthm}

\declaretheoremstyle[
    headformat=\normalfont\textcolor{\thmcolordark}{\bfseries\NAME\,\NUMBER}\NOTE,%
    bodyfont=\normalfont,
    spaceabove = 6pt,
    spacebelow = 6pt,
    ]{coloreddef}
}

\iftoggle{arxiv}{\newcommand{\colorbold}[1]{\textbf{\textcolor{\thmcolor}{#1}}}}
{
\newcommand{\colorbold}[1]{\textbf{#1}
}
}
\iftoggle{arxiv}
{
\newcommand{\colorpar}[1]{\paragraph{\textcolor{\thmcolor}{#1}}}
}
{
\newcommand{\colorpar}[1]{\paragraph{#1}}
}

\iftoggle{colt}{}{
\theoremstyle{coloredthmversion}
}

  \newtheorem*{thminformalversion}{Informal Version of}
\iftoggle{colt}{}{
  \theoremstyle{coloredthm}
}
\iftoggle{colt}{

}{
  \newtheorem{theorem}{Theorem}
  \newtheorem{lemma}{Lemma}[section]
  \newtheorem{corollary}{Corollary}[section]
  \newtheorem{proposition}[lemma]{Proposition}
}
\newtheorem{claim}[lemma]{Claim}

\newtheorem*{thminformal*}{Informal Theorem}

\iftoggle{colt}{
  \newtheorem{property}{Property}[section]
}{
\theoremstyle{coloreddef}
\newtheorem{definition}{Definition}[section]
\newtheorem{example}{Example}[section]
\newtheorem{remark}{Remark}[section]
\newtheorem{property}{Property}[section]

}
\newtheorem{observation}[lemma]{Observation}
\newtheorem{construction}{Construction}[section]

  \newtheorem{assumption}{Assumption}[section]
  \newtheorem{condition}{Condition}[section]

\makeatletter
\newcommand{\neutralize}[1]{\expandafter\let\csname c@#1\endcsname\count@}
\makeatother

\newenvironment{thmmod}[3]
  {\renewcommand{\thetheorem}{\ref*{#1}#2}%
   \neutralize{theorem}\phantomsection
   \begin{theorem}[#3]}
  {\end{theorem}}

\newtheorem*{theorem*}{Theorem}
\newtheorem*{lemma*}{Lemma}
\newtheorem*{corollary*}{Corollary}
\newtheorem*{proposition*}{Proposition}
\newtheorem*{claim*}{Claim}
\newtheorem*{fact*}{Fact}
\newtheorem*{observation*}{Observation}

\newtheorem*{definition*}{Definition}
\newtheorem*{remark*}{Remark}
\newtheorem*{example*}{Example}

\iftoggle{colt}{}{
\newtheoremstyle{named}{}{}{\itshape}{}{\bfseries}{}{.5em}{\Cref{#3} {\normalfont (informal)} }
{}
\theoremstyle{named}

\theoremstyle{plain}
}

\DeclareMathAlphabet{\mathbfsf}{\encodingdefault}{\sfdefault}{bx}{n}

\DeclareMathOperator*{\argmax}{arg\,max}

\let\Pr\relax
\DeclareMathOperator{\Pr}{\mathbb{P}}

\newcommand{\floor}[1]{\lfloor #1 \rfloor}

\newcommand{\Exp}{\mathbb{E}}

\newcommand{\trace}{\mathrm{tr}}

\newcommand{\poly}{\mathrm{poly}}
\newcommand{\I}{\mathbf{I}}

\def\ddefloop#1{\ifx\ddefloop#1\else\ddef{#1}\expandafter\ddefloop\fi}
\def\ddef#1{\expandafter\def\csname bb#1\endcsname{\ensuremath{\mathbb{#1}}}}
\ddefloop ABCDEFGHIJKLMNOPQRSTUVWXYZ\ddefloop

\def\ddefloop#1{\ifx\ddefloop#1\else\ddef{#1}\expandafter\ddefloop\fi}
\def\ddef#1{\expandafter\def\csname frak#1\endcsname{\ensuremath{\mathfrak{#1}}}}
\ddefloop ABCDEFGHIJKLMNOPQRSTUVWXYZ\ddefloop

\def\ddefloop#1{\ifx\ddefloop#1\else\ddef{#1}\expandafter\ddefloop\fi}
\def\ddef#1{\expandafter\def\csname fr#1\endcsname{\ensuremath{\mathfrak{#1}}}}
\ddefloop ABCDEFGHIJKLMNOPQRSTUVWXYZ\ddefloop

\def\ddefloop#1{\ifx\ddefloop#1\else\ddef{#1}\expandafter\ddefloop\fi}
\def\ddef#1{\expandafter\def\csname eul#1\endcsname{\ensuremath{\EuScript{#1}}}}
\ddefloop ABCDEFGHIJKLMNOPQRSTUVWXYZ\ddefloop

\def\ddefloop#1{\ifx\ddefloop#1\else\ddef{#1}\expandafter\ddefloop\fi}
\def\ddef#1{\expandafter\def\csname scr#1\endcsname{\ensuremath{\mathscr{#1}}}}
\ddefloop ABCDEFGHIJKLMNOPQRSTUVWXYZ\ddefloop

\def\ddefloop#1{\ifx\ddefloop#1\else\ddef{#1}\expandafter\ddefloop\fi}
\def\ddef#1{\expandafter\def\csname b#1\endcsname{\ensuremath{\mathbf{#1}}}}
\ddefloop ABCDEFGHIJKLMNOPQRSTUVWXYZ\ddefloop

\def\ddefloop#1{\ifx\ddefloop#1\else\ddef{#1}\expandafter\ddefloop\fi}
\def\ddef#1{\expandafter\def\csname bhat#1\endcsname{\ensuremath{\hat{\mathbf{#1}}}}}
\ddefloop ABCDEFGHIJKLMNOPQRSTUVWXYZ\ddefloop

\def\ddefloop#1{\ifx\ddefloop#1\else\ddef{#1}\expandafter\ddefloop\fi}
\def\ddef#1{\expandafter\def\csname btil#1\endcsname{\ensuremath{\tilde{\mathbf{#1}}}}}
\ddefloop ABCDEFGHIJKLMNOPQRSTUVWXYZ\ddefloop

\def\ddefloop#1{\ifx\ddefloop#1\else\ddef{#1}\expandafter\ddefloop\fi}
\def\ddef#1{\expandafter\def\csname bst#1\endcsname{\ensuremath{\mathbf{#1}^\star}}}
\ddefloop ABCDEFGHIJKLMNOPQRSTUVWXYZ\ddefloop

\def\ddefloop#1{\ifx\ddefloop#1\else\ddef{#1}\expandafter\ddefloop\fi}
\def\ddef#1{\expandafter\def\csname bst#1\endcsname{\ensuremath{\mathbf{#1}^\star}}}
\ddefloop abcdeghijklmnopqrstuvwxyz\ddefloop

\def\ddefloop#1{\ifx\ddefloop#1\else\ddef{#1}\expandafter\ddefloop\fi}
\def\ddef#1{\expandafter\def\csname bhat#1\endcsname{\ensuremath{\hat{\mathbf{#1}}}}}
\ddefloop abcdefghijklmnopqrstuvwxyz\ddefloop

\def\ddefloop#1{\ifx\ddefloop#1\else\ddef{#1}\expandafter\ddefloop\fi}
\def\ddef#1{\expandafter\def\csname b#1\endcsname{\ensuremath{\mathbf{#1}}}}
\ddefloop abcdeghijklnopqrstuvwxyz\ddefloop

\def\ddefloop#1{\ifx\ddefloop#1\else\ddef{#1}\expandafter\ddefloop\fi}
\def\ddef#1{\expandafter\def\csname barb#1\endcsname{\ensuremath{\bar{\mathbf{#1}}}}}
\ddefloop abcdefghijklmnopqrstuvwxyz\ddefloop

\def\ddef#1{\expandafter\def\csname c#1\endcsname{\ensuremath{\mathcal{#1}}}}
\ddefloop ABCDEFGHIJKLMNOPQRSTUVWXYZ\ddefloop
\def\ddef#1{\expandafter\def\csname h#1\endcsname{\ensuremath{\widehat{#1}}}}
\ddefloop ABCDEFGHIJKLMNOPQRSTUVWXYZ\ddefloop
\def\ddef#1{\expandafter\def\csname hc#1\endcsname{\ensuremath{\widehat{\mathcal{#1}}}}}
\ddefloop ABCDEFGHIJKLMNOPQRSTUVWXYZ\ddefloop
\def\ddef#1{\expandafter\def\csname t#1\endcsname{\ensuremath{\widetilde{#1}}}}
\ddefloop ABCDEFGHIJKLMNOPQRSTUVWXYZ\ddefloop
\def\ddef#1{\expandafter\def\csname tc#1\endcsname{\ensuremath{\widetilde{\mathcal{#1}}}}}
\ddefloop ABCDEFGHIJKLMNOPQRSTUVWXYZ\ddefloop

\newcommand{\gst}{g^\star}
\newcommand{\Z}{\mathbb{Z}}
\newcommand{\Gclass}{\cG}
\newcommand{\ost}{o_{\star}}

\newcommand{\Unif}{\Dist_{\mathrm{unif}}}
\newcommand{\ballk}[1][k]{\cB_{#1}}
\newcommand{\ballkr}[1][r]{\cB_{k}(r)}

\newcommand{\Gsmooth}{\cG_{\mathrm{smooth}}}
\newcommand{\fbar}{\bar{f}}
\newcommand{\beps}{\bm{\epsilon}}

\newcommand{\minprobcost}[1][\cost]{\minmax_{#1,\mathrm{prob}}}
\newcommand{\sampreg}[1][n]{S_{#1,\mathrm{reg}}}
 
\newcommand{\kappaconc}{\kappa}
\newcommand{\deltaconc}{\delta}

\newcommand{\Var}{\mathrm{Var}}
\newcommand{\Disz}[1][z]{\Dist_{\{Z=#1\}}}
\newcommand{\angs}[2]{\left \langle #1,#2 \right \rangle}
\newcommand{\absangs}[2]{\left|\left \langle #1,#2 \right \rangle\right|}
\newcommand{\epssmall}{\epsilon_{\mathrm{small}}}

\newcommand{\TV}{\mathrm{TV}}
\newcommand{\typical}{typical\xspace}
\newcommand{\typicall}{typical}
\newcommand{\Avan}{\bbA_{\mathrm{simple}}}

\newcommand{\Areason}{\bbA_{\mathrm{gen},\mathrm{smooth}}}
\newcommand{\midd}{~\big{|}~}
\newcommand{\trajj}{\bx_{1:H},\bu_{1:H}}

\newcommand{\chard}{\cost_{\mathrm{hard}}}
\newcommand{\chardt}[1][t]{\cost_{\mathrm{hard},#1}}

\newcommand{\bump}{\mathrm{bump}}
\newcommand{\restrict}{\mathrm{restrict}}
\newcommand{\Lcost}{C_{\mathrm{cost}}}

\newcommand{\polyost}{\mathrm{poly}\text{-}o^\star}

\newcommand{\Phat}{\hat{P}}
\newcommand{\Dzone}{D_{\{Z=1\}}}
\newcommand{\Asmooth}{\bbA_{\mathrm{smooth}}}

\newcommand{\Proj}{\mathrm{Proj}}
\newcommand{\Rsl}{\Risk_{\mathrm{reg}}}
\newcommand{\Rtr}{\Risk_{\mathrm{expert},L_1}}
\newcommand{\Rtrone}{\Risk_{\mathrm{train},h=1}}
\newcommand{\delu}{\updelta\bu}

\newcommand{\bSigma}{\bm{\Sigma}}
\newcommand{\law}{\mathrm{law}}
\newcommand{\lawtil}{{\mathrm{law}}_{\mathrm{reg}}}
\newcommand{\projj}{\mathrm{proj}}
\newcommand{\kernn}{\mathcal{K}}

\newcommand{\Ball}{\mathcal{B}}
\newcommand{\mincost}{\minmax_{\cost}}
\newcommand{\btilu}{\tilde{\bu}}
\newcommand{\btilx}{\tilde{\bx}}
\newcommand{\bzeta}{\bm{\upzeta}}
\newcommand{\pibar}{\bar{\pi}}
\newcommand{\mintrainone}{\minmax_{\mathrm{expert},h=1}}
\newcommand{\delx}{\updelta \bx}
\renewcommand{\cI}{\cP}
\newcommand{\bxtil}{\tilde{\bx}}
\newcommand{\proj}{\mathbf{P}}

\newcommand{\nhat}{\hat{n}}

\newcommand{\bbarA}{\bar{\bA}}
\newcommand{\bbarK}{\bar{\bK}}

\newcommand{\bDel}{\bm{\Delta}}

\newcommand{\Aproper}{\bbA_{\mathrm{proper}}}
\newcommand{\Rlpcost}[1][p]{\Risk_{\cost,L_{#1}}}

\newcommand{\mincostlp}[1][p]{\minmax_{\cost,L_{#1}}}
  \newcommand{\mintrajlp}[1][p]{\minmax_{\mathrm{traj},L_{#1}}}
\newcommand{\pihat}{\hat{\pi}}

\newcommand{\Rzeroone}{\Risk_{\mathrm{expert},\{0,1\}}}

\newcommand{\Cliptil}{{\cC}_{\mathrm{lip},\mathrm{max}}}
\newcommand{\minprob}{\minmax_{\mathrm{eval},\mathrm{prob}}}
\newcommand{\Rtrain}{\Risk_{\mathrm{expert},L_2}}

\newcommand{\pist}{\pi^{\star}}

\newcommand{\dirac}{\bm{\updelta}}

\newcommand{\Dreg}{\Dist_{\mathrm{reg}}}

\newcommand{\ghat}{\hat{g}}
\newcommand{\estreg}{\est_{\mathrm{reg}}}

\newcommand{\Alg}{\mathrm{alg}}

\newcommand{\est}{\Alg}

\newcommand{\minmax}{\bm{\mathsf{M}}}
\newcommand{\Aclk}[1][k]{\bA_{\mathrm{cl},#1}}

\newcommand{\minbctrain}{\minmax_{\mathrm{expert},L_2}}

\newcommand{\minsl}{\minmax_{\mathrm{reg},L_2}}
\newcommand{\spn}{\mathrm{span}}
\newcommand{\Rsubcost}[1][\cost]{\Risk_{#1}}

\newcommand{\minbcevalhb}[1][h]{\minmax_{\mathrm{eval},#1,B}}
\newcommand{\Clip}{\mathcal{C}_{\mathrm{Lip}}}

\newcommand{\inst}{\cI}
\newcommand{\unifsim}{\overset{\mathrm{unif}}{\sim}}
\newcommand{\bxi}{\bm{\xi}}
\newcommand{\laws}{\bm{\Updelta}}
\newcommand{\Pnot}{\Dist}
\newcommand{\traj}{\bm{\mathrm{traj}}}

\newcommand{\mean}{\mathrm{mean}}

\newcommand{\sample}{\mathrm{S}}

\newcommand{\Samp}[1][n]{\sample_{#1,H}}

\newcommand{\bmsf}[1]{\bm{\mathsf{#1}}}

\newcommand{\Risk}{\bmsf{R}}
\newcommand{\Rcost}{\Risk_{\mathrm{cost}}}

\newcommand{\minslprob}{\minmax_{\mathrm{reg},\mathrm{prob}}}
\newcommand{\Rtraj}{\Risk_{\mathrm{traj},L_1}}
\newcommand{\cost}{\mathrm{cost}}
\newcommand{\algg}{\Alg}
\newcommand{\Dist}{D}
\newcommand{\Xspace}{\mathbb{X}}
\newcommand{\Uspace}{\mathbb{U}}

\usepackage{preamble/color-edits}
\addauthor{ms}{magenta}

\title{The Pitfalls of Imitation Learning when Actions are Continuous }

\author{
  Max Simchowitz\footnote{\texttt{msimchow@andrew.cmu.edu}}\\
  CMU
  \and
  Daniel Pfrommer\footnote{\texttt{dpfrom@mit.edu}} \\
  MIT
  \and
  Ali Jadbabaie\footnote{\texttt{jadbabai@mit.edu}}\\
  MIT
}

\date{\today}

\begin{document}

\maketitle

\begin{abstract}


We study the problem of imitating an expert demonstrator in a discrete-time, continuous state-and-action control system. We show that, even if the dynamics satisfy a control-theoretic property called exponential stability (i.e. the effects of perturbations decay exponentially quickly), and the expert is smooth and deterministic, any smooth, deterministic imitator policy  necessarily suffers  error on execution that is  exponentially larger, as a function of problem horizon, than  the error under the distribution of expert training data. Our negative result applies to \textbf{any algorithm} which learns solely from expert data, including both behavior cloning and offline-RL algorithms, unless the algorithm produces highly ``improper'' imitator policies --- those which are non-smooth, non-Markovian, or which  exhibit highly state-dependent stochasticity --- or unless the expert trajectory distribution is sufficiently ``spread.'' We provide experimental evidence of the benefits of these more complex policy parameterizations, explicating the benefits of today's popular policy parameterizations in robot learning (e.g. action-chunking and diffusion policies). We also establish a host of complementary negative and positive results for imitation in control systems.




\end{abstract}

\section{Introduction}




Imitation Learning (IL), or learning a multi-step behavior from demonstration, encompasses both the earliest-introduced and most currently popular methodologies for training autonomous robotic systems with machine learning techniques \citep{ross2011reduction, ho2016generative, teng2023motion, zhao2023learning}. These successes have been buoyed by a host of new innovations: the uses of generative models (e.g. Diffusion policies \citep{chi2023diffusion}) to represent robotic behavior, the practice of ``chunking'' sequences of predicted actions, and various means of data augmentation beyond raw expert demonstrations \citep{ke2021grasping, jia2023seil}.  At the same time, with the rise of large language models (LLMs), IL also has become increasingly more prevalent in settings in which the agent predicts \emph{discrete tokens}, such as steps on a chess board, lines on a math proof, or words in a sentence \citep{chen2021decision}. For robot applications, in contrast,  the state and action variables are continuous (but for convenience, time may still be treated discretely). \iftoggle{arxiv}{Hence we ask,}{We ask,}
\begin{quote}
\emph{What are the fundamental differences between imitating continuous actions and discrete behaviors? How do these differences explain the necessity of common techniques observed in today's robot learning pipelines?} 
\end{quote}

\newcommand{\Dagger}{\textsc{Dagger}}

 We consider control systems with continuous-valued states $\bx \in \Xspace = \R^d$, control inputs $\bu \in \Uspace = \R^m$ and dynamics $\bx_{t+1} = f(\bx_t,\bu_t)$, where $t$ denotes timestep. We assume $f$ is \textbf{unknown} to the learner. The key parameter in our study is the task horizon, denoted by $H \in \mathbb{N}$, or number of steps of behavior to be imitated. The expert provides $n$ length-$H$ demonstration trajectories $(\bx_1,\bu_1,\dots,\bx_H,\bu_H)$, determined by the \emph{expert policy} $\pist:\Xspace \to \Uspace$ via $\bu_t = \pist(\bx_t)$ with some initial state distribution of $\bx_1$. The learner observes these trajectories and selects a policy $\pihat:\Xspace \to \Uspace$, deployed under the same dynamics, with the goal of emulating the expert: $\pihat \approx \pist$. For clarity, we will use  \textbf{imitation learning} (IL) to refer to learning from expert demonstration  in which the agent cannot interact further with its environment or the expert after demonstrations are given.  We will (colloquially) refer to as \textbf{behavior cloning} (BC)  those methods which perform IL by fitting the data with pure supervised learning.

As learning is imperfect, the learner makes small errors which may add together over time, forcing the learner to stray off-course.
Ultimately, the difference between the trajectories deployed by the learner and the expert trajectories may be much larger than the errors of learning the expert's actions under the distribution of demonstration trajectories, typically by a multiplicative factor  depending on $H$. 
This is the \textbf{compounding error problem.} 
\iftoggle{arxiv}
{

}
{}
While much attention has been devoted to circumventing compounding error via additional interaction with the expert \citep{ross2010efficient} or with the environment \citep{ho2016generative},   we  aim to understand when imitation learning is possible \textbf{without interactive access to either the environment or the expert}; what we deem the ``non-interactive setting.''
\iftoggle{arxiv}
{

\colorbold{The Continuous vs. Discrete Settings.} 
}
{}
Even without interaction, existing theoretical literature shows that  compounding error is benign in \textbf{discrete} problem domains: it scales at most polynomially in the problem horizon, $H$ \citep{ross2010efficient} and can even be eliminated entirely in some situations, via  an appropriate loss function (e.g. the log-loss,  \citep{fosterbehavior}). However, these results are contingent on being able to estimate expert behavior in certain very strong error metrics (e.g. the $\{0,1\}$-loss) which, while feasible for discrete problems, we show are \textbf{unattainable when actions are continuous.} Prior theoretical literature studying IL in continuous-action control systems has required additional assumptions and algorithmic modifications (e.g. expert-interactions, stabilization oracles and score-matching oracles).
Hence, a systematic theoretical understanding of the difficulty of non-interactive, continuous-action IL remains absent. \iftoggle{arxiv}{}
{Due to space constraints, a complete discussion of \textbf{related work is deferred to \Cref{sec:related}.}}

\iftoggle{arxiv}{
\begin{figure}
\label{fig:main_fig}
\begin{center}
\includegraphics[width=0.45\linewidth]{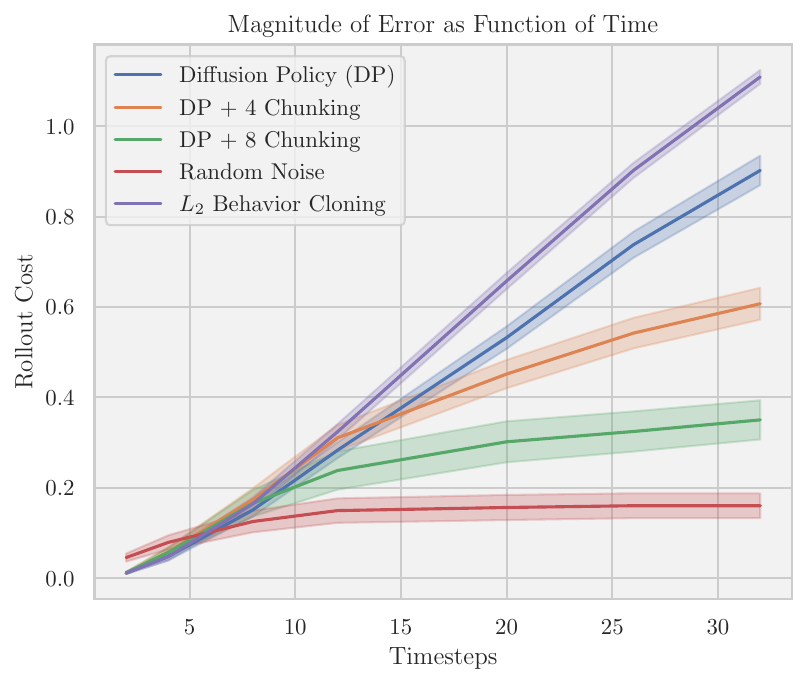}
\includegraphics[width=0.35\linewidth]{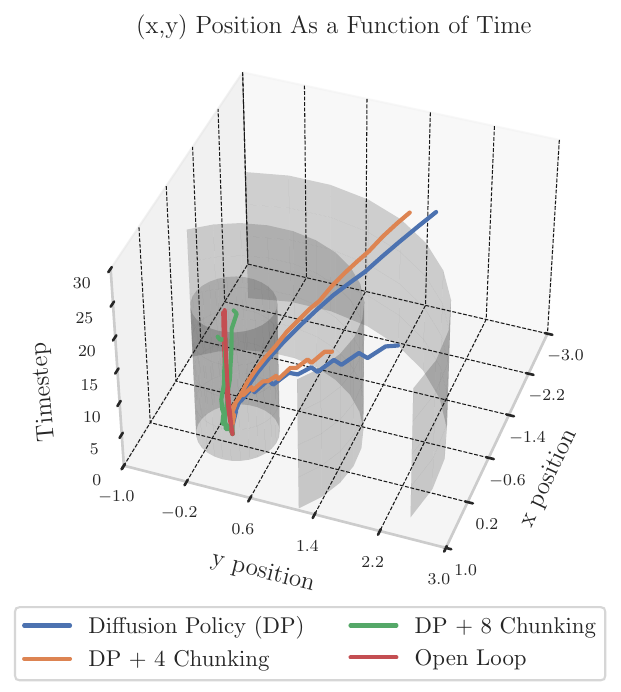}
\end{center}
    \caption{We benchmark the performance of different methods on a ``hard''  stable dynamical system, whose construction is given formally in \Cref{const:stable}. Further experimental details are given in \Cref{app:experiments} . \textbf{Left}: The expert policy should remain identically zero in the $e_1$ direction for all times $t$,  we measure rollout cost $\max_{s \le t } |\langle x_s, e_1 \rangle|$ for all imitation policies. We compare behavior cloning with a unimodal MLP policy using the $L_2$ loss, random noise ($\bu_h \sim \mathcal{N}(\bzero, \frac{1}{6}\bI)$), and Diffusion policy (DP) \citep{chi2023diffusion}. DP refers to a policy predicting one action at a time, ``4-chunking'' and ''8-chunking''  execute sequences (``chunks'', \cite{zhao2023learning}) of  4 and 8 actions in open loop. We observe that $L_2$ behavior cloning suffers from significant compounding error, whereas Diffusion policy fares better, and action chunking further improves the performance. \textbf{Right}: Trajectory visualization of behavior cloning trajectories in the $x-y$ plane, comparing Diffusion Policy \citep{chi2023diffusion} with and without action-chunking, and open loop rollouts (zero control input). 
     In both figures, we notice that random control inputs (random noise, left plot) or zero inputs (open loop, right plot) leads to greater errors at first, but the errors do not accumulate with time due to the open-loop stability of the underlying system and the resulting policy performs better than all learned policies. 
}
\end{figure}
}{}

\subsection{Contributions}\label{subsec:contributions}

We show that imitation learning where both the expert $\pist$ and learned policy $\pihat$ are ``simple'' suffers from  \textbf{exponential-in-horizon compounding error, even in seemingly benign continuous-state-and-action control systems.} This contrasts  discrete-token behavior cloning,  in which compounding error is at most polynomial in horizon. We also provide evidence that exponential compounding can be mitigated by more sophisticated policy representations. While it has been popular to motivate more sophisticated policies (e.g. action-chunked Transformers and Diffusion policies) by the need to fit ``{multi-modal} '' expert data (expert demonstrations with multiple \emph{modes}, or strategies, to solve a given task), this suggests that \textbf{even the imitation of simple, deterministic, and hence uni-modal experts may benefit from complex policy parameterizations.}
\iftoggle{arxiv}{Importantly, our }{Our} negative results depend only on the structure of $\pihat$, but are agnostic to the learning algorithm \iftoggle{arxiv}{which produces it}{}. In particular, our results apply to \textbf{behavior cloning}, to any \textbf{any inverse reinforcement approach} \citep{ho2016generative} which does not use additional interaction with the environment, and to  offline reinforcement learning (e.g. \cite{kumar2020conservative,kostrikov2021offline}) approaches.  More specifically, we find the following:

\iftoggle{arxiv}{ \paragraph{Contribution 1.}}{\colorbold{1.}} We consider  smooth, deterministic expert policies and smooth, deterministic dynamical systems that satisfy a control-theoretic property called \textbf{exponential stability} (\Cref{def:diss}), which stipulates that the effects of perturbations to the system decay exponentially quickly. We colloquially refer to such systems as \emph{stable}. We assume that stability holds both for the dynamics themselves (open-loop stability), and the dynamics in feedback with the expert policy (closed-loop stability). We show that, if the imitator policy is also smooth and deterministic, or more generally, can be written as the sum of a smooth deterministic policy with state-independent noise (we call these ``simply-stochastic''), then the learner's execution error is exponentially-in-$H$ larger than the training error under the expert distribution.

\begin{thminformalversion}[\Cref{thm:main_stable,thm:constant_prob}]There exists a family of imitation learning problems with exponentially stable, smooth and deterministic experts and dynamics as described above, for which optimal execution error attained by any algorithm constrained to returning simply stochastic policies $\pihat$ with smooth means is at least $\exp(\Omega(H))$ times the optimal  expert-distribution error of \emph{any} (possibly unrestricted) learning algorithm. This holds for  a cost of interest that is \textbf{bounded} in $[0,1]$ and \textbf{Lipschitz}.
\end{thminformalversion}

 As noted above, the lower bound depends only on the parameterization of the learned policy $\pihat$, but not on the learning algorithm used to produce it. Therefore, offline reinforcement learning, behavior cloning, and inverse reinforcement learning without further environmental interaction (e.g. \cite{ho2016generative}) fail to circumvent the lower bound.

\iftoggle{arxiv}{\paragraph{Contribution 2.}}{\colorbold{2.}}  We show that large compounding error occurs for more general stochastic policies, but potentially substantially less than for the ``simply-stochastic'' policies described above. 

\begin{thminformalversion}[\Cref{thm:non_simple_stochastic}] For classes of stable, smooth and deterministic experts and dynamics described above, imitation with a general class of smooth, stochastic, but perhaps multi-modal Markovian policies still incurs either exponential error, or else the rate of execution error scales strictly worse than the rate of expert-distribution error.
\end{thminformalversion} 

As described below, we show that a host of more complex policy classes suffice to ameliorate compounding error for our lower bound and validate this finding with numerical simulations (\Cref{sec:nonvanilla}).

\iftoggle{arxiv}{\paragraph{Contribution 3.}}{\colorbold{3.}} We show that exponential compounding error is unconditionally unavoidable if system dynamics  may be unstable (even if the dynamical system is smooth, Lipschitz and deterministic). Consequently, \textbf{observation of expert trajectories alone does not suffice for learning in these control systems, no matter what algorithm or policy class are used.} 
\begin{thminformalversion}[\Cref{thm:unstable}] When the system dynamics are permitted to be unstable, exponential compounding error is unavoidable by \emph{any} non-interactive IL procedure.
\end{thminformalversion}

\iftoggle{arxiv}{ \paragraph{Contribution 4.}}{\colorbold{4.}} We show that, if expert data are sufficiently ``spread'' or anti-concentrated,  even pure {behavior cloning} avoids compounding error. Hence, certain conditions on {data quality} suffice to avoid the pathologies above. 
\begin{thminformalversion}[\Cref{thm:smoothgen}] Suppose that the expert demonstrations are smooth and stabilize the dynamics in closed-loop (but dynamics need not be open-loop stable). Then, if the distribution over expert trajectories has a sufficiently ``spread'' probability density, \textbf{simple behavior cloning} yields low execution error, with limited compounding error.
\end{thminformalversion}


\iftoggle{arxiv}{ \subsection{The benefits of Action-Chunking and Diffusion Policies?}}{\colorbold{5.}}  \Cref{sec:nonvanilla} illustrates that our lower bound can be circumvented by using policies that are either non-smooth, non-Markovian, or non simply-stochastic.  This provides informal evidence that popular practices in modern robotic imitation learning, such as the use of \textbf{Diffusion models} as  policy parameterizations \citep{chi2023diffusion} (which are \emph{non simply-stochastic}) and predicting multiple actions per time-step (\textbf{action-chunking}, \citep{zhao2023learning,chi2023diffusion}) can circumvent these pathologies. In particular, this suggests that multi-modal policies such as diffusion policies can better imitate certain uni-modal expert demonstrations. 

We corroborate the findings in Section 5 with numerical simulations (see \Cref{fig:main_fig}).   For the challenging open-loop stable construction used in \Cref{thm:main_stable}, we demonstrate in \Cref{sec:experiments}  the poor performance of different imitation learning methods. Our experiments validate the core tenet of this paper: that continuous-action imitation learning is difficult even when the dynamics are open-loop exponentially stable. Furthermore, our experiments suggest that the aforementioned more complex techniques, such as action-chunking, can successfully circumvent our lower bounds. 

\iftoggle{arxiv}{\subsection{Proof Intuition}}{\colorbold{Proof Intuition.}} Though the formal proof of our main result (\Cref{thm:main_stable}) involves numerous technical subtleties, the core idea is intuitive and sketched in \Cref{sec:proof_intuition}:  the learner is faced with two candidate pairs of policies and dynamics, $(f_i,\pi_i^\star)$, $i \in \{1,2\}$ (recall; dynamics are unknown). While each $\pi_i^\star$ stabilizes its corresponding $f_i$, it does not stabilize the alternative system $f_{j},j \ne i$. The learner is given insufficient data to determine the true index $i$. If the goal was simply to stabilize the unknown $f_i$, the zero policy $\pihat(\bx) = \bzero$ would suffice because each $f_i$ is exponentially stable. We show, however, that there is no way for the learner to both stabilize $f_i$ for unknown $i \in \{1,2\}$, and to simultaneously emulate the expert policy under the expert's demonstration data distribution (for example, acting according to $\pihat(\bx) = \bzero)$ would cause large imitation error).  Thus, for whatever the learner chooses, one of the systems is destabilized, and small errors compound exponentially into large ones.


As described in \Cref{rem:significance_unknown}, the lack of learner's knowledge of dynamics is essential; otherwise, there exist (possibly computationally inefficient) procedures leveraging dynamical knowledge to avoid compounding error pathologies. For those familiar with the theoretical RL literature, our result can be interpreted in terms of the Lipschitz constants of certain classes of Q-functions (see \Cref{rem:Lipschitz}). For the control theorist, our argument is related to, but differs importantly from, the celebrated gap metric \citep{zames1981uncertainty}, as discussed in  \Cref{rem:gap_metric}. The core of our construction, outlined in \Cref{sec:gesture_towards_formal}, is based on 2-dimensional linear systems, but the full argument relies on a number of technical innovations, overviewed in \Cref{sec:extended_proof_notes}.

\subsection{Related Work}

\label{sec:related}
Imitation from expert demonstration has emerged as a pre-eminent technique for learning in robotic control tasks;  applications have included self-driving vehicles \citep{hussein2017imitation,bojarski2016end,bansal2018chauffeurnet}, visuomotor policies \citep{finn2017one,zhang2018deep}, and navigation tasks \citep{hussein2018deep}, and large-scale robotic decision making models \citep{pmlr-v229-zitkovich23a,black2024pi_0}. These advances have been accelerated by the introduction of generative neural network architectures parameterizing the robotic policy, including diffusion and flow-based models \citep{janner2022planning,chi2023diffusion,pearce2023imitating,hansen2023idql,black2024pi_0}, and Transformer architectures with appropriate tokenization of actions \citep{zhao2023learning,chen2021decision,shafiullah2022behavior}. The common rationale for these models is that they may represent a rich and varied distribution of expert strategies, or \emph{modes}, for solving a given task \citep{chi2023diffusion,shafiullah2022behavior}. Our contributions suggests that these models may enjoy benefits even for deterministic and smooth (i.e., uni-modal!) expert policies.

The compounding error problem --- that is, the possibility that execution error can be significantly larger than error on the training data distribution --- has been widely acknowledged in imitation learning \citep{ross2010efficient,ross2011reduction}. The seminal work of \citet{ross2010efficient} proposes the \textsc{Dagger} algorithm for \emph{interactive data} collection to circumvent this challenge, an algorithm which has seen widespread adoption \citep{sun2023mega,kelly2019hg}. Other approaches have focused on modifying the distribution of data collected by the expert to provide sufficient coverage of failure modes \citep{laskey2017dart,ke2021grasping}. 

On the theoretical side, however, the challenge of compounding error appears more benign: for example, \cite{ross2010efficient} show that without interventions, the discrepancy between training and execution error is at most polynomial in the horizon. Further, recent work by \cite{fosterbehavior} demonstrates that, by minimizing the log-loss (as is common in discrete imitation learning applications, such as text),  horizon length may have no adverse effect on  the performance of imitation learning. However, both of these works operate in settings that are not well-suited for control settings: \cite{ross2010efficient} and \cite{fosterbehavior} assume the ability to learn the expert policy in the total-variation and Hellinger distances, respectively, which is not feasible for deterministic policies in continuous action spaces (see \Cref{prop:cannot_est_TV}). Though these purely probabilistic distances can be relaxed to integral probability metrics (IPMs) induced by relevant classes of Q-functions (see e.g. \cite{swamy2021moments} or the discussion in \Cref{sec:RL_perspective}), we explain how these metrics may be too stringent in the worst case as well (\cref{sec:rl_vs_control}). 

 Recent work has established mathematical guarantees for imitation specifically for control-theoretic settings. Unfortunately, these required either interactive access to the expert demonstrator \cite{pfrommer2022tasil}, multiple steps of environment interaction \citep{wu2024diffusing} or a complex recipe of hierarchical trajectory stabilization, and targeted data augmentation \citep{block2024provable}.  Hence, the theoretical understanding of \textbf{non-interactive imitation learning in control systems} has remained entirely open.

  \citet{pfrommer2022tasil,block2024provable} propose incremental input-to-state stability \citep{agrachev2008input} as the natural regularity condition governing the possibility of imitation in these settings. \Cref{sec:rl_vs_control} connects this notion to formalisms more commonly studied in the theoretical reinforcement learning literature, arguing how traditional assumptions in the latter community may be insufficiently delicate for control-theoretic settings. Our negative results draw connections to a yet more classical principle in control theory, namely the gap metric due to \cite{zames1981uncertainty} (\Cref{rem:gap_metric}). Finally, our lower bounds also involve a range of other technical tools, including log-concave anti-concentration \citep{carbery2001distributional}, nonparametric regression in the zero-noise (interpolation) setting \citep{kohler2013optimal,bauer2017nonparametric,krieg2022recovery}, and quantitative variants of the unstable manifold theorem applied in the study of saddle-point escape in non-convex optimization \citep{jin2017escape}.

\subsection{Organization}
Our paper is organized so that the casual reader can extract all main takeaways from the first few sections, whilst readers more familiar with the statistical learning and reinforcement learning literature can find more systematic treatments of findings in the sections that follow. \Cref{sec:prelim} contains all preliminaries and notation. \Cref{sec:results} provides formal statements of all main results, namely those stated in the informal theorems in \Cref{subsec:contributions} above. \Cref{sec:proof_intuition} provides the broad brushstrokes of the proof of our most surprising result: the lower bounds against imitation in stable systems with ``simple'' experts (\Cref{thm:main_stable,thm:constant_prob}). Finally, \Cref{sec:nonvanilla} describes how our lower bound construction can be circumvented by more complex policy parameterizations, and provides experimental evidence to this effect. The main body of the paper concludes with a discussion in \Cref{sec:discussion}. 

The remainder of the paper contains two parts. First, the \hyperlink{targ_addendum}{Addendum}, targeted at experts,  reformulates our results (\Cref{sec:minmax}) and provides more detailed theorem statements (\Cref{sec:minmax_lb}) in the language of minimax risks favored by the statistical learning community \citep{wainwright2019high}. These results are followed by a more detailed proof schematic in \Cref{sec:schematic}. Following the \hyperlink{targ_addendum}{Addendum} is a more traditional \hyperlink{targ_appendix}{Appendix}, which contains the full proofs of all claims made throughout the manuscript, and whose organization is outlined at its beginning.


\section{Preliminaries}\label{sec:prelim}

\newcommand{\Eshorthand}{\Exp_{[\Alg,\pist,f,n,H]}}
\iftoggle{arxiv}{We consider}{Consider} a control system with states $\bx \in \Xspace := \R^d$ and control actions $\bu \in \Uspace := \R^m$. \iftoggle{arxiv}{The dynamics}{Dynamics}  evolve deterministically, via dynamical maps $f:\Xspace \times \Uspace \to \Xspace$,  $\bx_{t+1} = f(\bx_t,\bu_t)$, $t \ge 1$.  Unless otherwise stated, we consider time-invariant, Markovian (static) policies that are mappings of states to distributions over actions $\pi: \Xspace \to \laws(\Uspace)$. When $\pi$  is deterministic, we simply write $\bu = \pi(\bx)$.  
\iftoggle{arxiv}
{

}{}
A triple $(\pi,f,\Pnot)$ of policy $\pi$, dynamics $f$, and  initial distribution $\Pnot \in \laws(\Xspace)$ over states $\bx$, define a distribution $\Pr_{\pi,f,\Dist}$ over trajectories where
\iftoggle{arxiv}
{
	\begin{align*}
	 \bx_{t+1} = f(\bx_t,\bu_t), \quad \bu_t \mid \bx_{t} \sim \pi(\bx_{t}), \quad \bx_1 \sim \Pnot.
	\end{align*}
}
{ $\bx_1 \sim \Pnot$, $\bu_t \mid \bx_{t} \sim \pi(\bx_{t})$, and $\bx_{t+1} = f(\bx_t,\bu_t)$.}
 An {imitation learning problem} is specified by a tuple $(\pist,f,\Dist,H)$ with {expert policy} $\pist$, dynamics $f$, and initial state distribution, and problem {horizon} $H \in \mathbb{N}$. Throughout, we take the expert {$\pist$ to be deterministic.}
\iftoggle{arxiv}
 {
 
 }
 {}
 The learner has access to a {sample} $\Samp$ consisting of $n$ trajectories $\traj_{i,1:H} = (\bx_{i,1:H},\bu_{i,1:H})$, $1 \le i \le n$, drawn i.i.d. from $\Pr_{\pist,f,\Dist}$. A (non-interactive) IL algorithm, denoted $\algg$, is a possibly randomized mapping from $\Samp$ to the space of imitator policies  $\pihat$. We denote these as $\Samp \sim [\pist,f,\Dist]$ and $\pihat \sim \algg(\Samp)$, and let $\Eshorthand$ denote expectation over both of these sources of randomness (suppressing dependence on $\Dist$ for simplicity). Importantly, the dynamics $f$ are \textbf{not known} to the learner. 
\iftoggle{arxiv}
{

}
{}
 Given a $\cost(\cdot): \Xspace^H \times \Uspace^H \to \R$,
  the \textbf{execution error} under  $\cost(\cdot)$  is the difference
\begin{align*}
\Risk_{\cost}(\pihat; \pist, f,\Dist,H) := \Exp_{\pihat,f,\Dist}\left[\cost(\bx_{1:H},\bu_{1:H})\right]- \Exp_{\pist,f,\Dist}\left[\cost(\bx_{1:H},\bu_{1:H})\right].
\end{align*}
We  focus on the class of additive costs $\cC_{\mathrm{lip}}$ comprised of 
$\cost(\bx_{1:H},\bu_{1:H}) = \sum_{h=1}^H \tilde\cost(\bx_h,\bu_h)$, where  $\tilde \cost(\cdot,\cdot)$ is $1$-Lipschitz and bounded in $[0,1]$. 
 Our lower bounds will show impossibility of imitating in $\Rsubcost$ for a fixed cost.  \iftoggle{arxiv}{Our upper}{Upper} bounds extend to a stronger metric, $\Rtraj$, defined in \Cref{app:traj_dist}, that upper bounds $\sup_{\cost \in \Clip} \Rsubcost$. 
\newcommand{\RtrainLp}[1][p]{\Risk_{\mathrm{expert},L_{#1}}}

\begin{remark}[Bounded cost] We stress that our costs of interests are \textbf{bounded}. Hence, our lower bounds do not rely on an unbounded growth on magnitude of the cost. 
\end{remark}

\iftoggle{arxiv}{\begin{remark}[Imitation learning v.s.  behavior cloning] For clarity, we will refer to \textbf{imitation learning} (IL) as the general problem setting described above. More precisely, this is the \textbf{non-interactive} imitation learning setting, as the agent cannot interact further with its environment or the expert after demonstrations $\Samp$ are given. There are a number of popular IL methodologies. We will colloquially refer to \textbf{behavior cloning} (BC) as those methods which train $\pihat$ via pure supervised learning; e.g. selecting $\pihat$ to minimize the empirical version of $\RtrainLp$ (defined below) on the sample $\Samp$. Our lower bounds apply to all imitation learning algorithms, while our upper bound is realized by behavior cloning. 
\end{remark}
}{}
\textbf{Further Notation.} Throughout, $\|\cdot\|$ denotes the Euclidean norm. A function $g:\R^{d_1} \to \R$ is \emph{$L$-Lipschitz} if $|g(\bz) - g(\bz')| \le L\|\bz - \bz'\|$; $g$ is \emph{$M$-smooth} if it is twice-continuously differentiable and $\|\nabla^{\,2} g\|_{\op} \le M$; $g: \R^{d_1} \to \R^{d_2}$ is $L$-Lipschitz (resp. $M$-smooth) if  $\langle \bv,g\rangle$  is $L$-Lipschitz (resp. $M$-smooth) for all unit vectors $\bv \in \R^{d_2}$.  The mean of stochastic policy $\pi$ is  the deterministic policy $\mean[\pi](\bx) := \Exp_{\bu \sim \pi(\bx)}[\bu]$\iftoggle{arxiv}{; note that if $\pi$ is deterministic, $\pi(\bx) \equiv \mean[\pi](\bx)$.}{} We use $\be_i$ as shorthand for the $i$-th canonical basis vector, where dimension is clear from context.  $\cB_{d}(r)$ denotes the ball of radius $r$ in $\R^d$, and $\cS^{d-1}$ the sphere. $C$ is a ``universal constant'' if it does not depend on any problem parameters;  $a = \BigOh{b}$ if $a \le C b$ for a universal constant $C$, and $a = \ost(b)$ means ``$a \le c \cdot b$ for a \emph{sufficiently small} universal constant $c$.''


\iftoggle{arxiv}{\subsection{Compounding Error}}
{\paragraph{Compounding Error}}
Compounding error is the phenomenon by which small errors in estimation of $\pist$ during training {compound}, leading to deviations between $\pist$ and $\pihat$ when deployed on horizon $H$. We measure this by comparing  $\Rsubcost$ to a natural measure of  error under the distribution of expert data collected:
     \begin{align}
    \RtrainLp(\pihat ; \pist, f, \Pnot,H) =\sum_{t=1}^H \Exp_{\pi^\star,f,\Pnot} \Exp_{\bhatu_t \sim \pihat(\bx_t)} \left[\|\bhatu_t - \pist(\bx_t)\|^p\right]^{1/p}. \label{eq:bc_train}
\end{align}

Note that, while $\pist$ is deterministic, $\RtrainLp$ is well-defined even if  $\pihat$ is stochastic.
For reasons of technical convenience, we focus on $\RtrainLp[2]$ (see \Cref{rem:why_l_two}), but qualitatively similar results hold for other choices of $p$ (e.g. $\RtrainLp[1]$). 
\iftoggle{arxiv}{

}
{}
Our paper argues that there exist natural, seemingly benign settings for continuous action IL where, for some choice of $\cost$, imitating a simple expert with a simple policy renders $\Rsubcost$  exponentially larger than $\Rtrain$
\iftoggle{arxiv}{:
\begin{align*}
\exists \cost \in \Clip \text{ s.t.} \,\,  \left(\text{ worst-case optimal } \Rsubcost \text{ }\right) \ge e^{\Omega(H)} \cdot \left(\text{  worst-case optimal } \RtrainLp[2] \text{ }\right).
\end{align*}
Above, ``worst-case optimal'' means the minimal value attained by a suitable IL algorithm $\Alg$, on the worst-case  problem instance (formally, the minimax risk, \Cref{sec:minmax}).
}
{.}
\iftoggle{arxiv}{

}
{}
\iftoggle{arxiv}
{
\subsection{Control-Theoretic Stability}
}
{\paragraph{Control-Theoretic Stability.}}

Adopting a control-theoretic perspective (e.g. \iftoggle{arxiv}{\citep{pfrommer2022tasil,block2024provable}}{\cite{pfrommer2022tasil}}), our notion of ``benign-ness'' is defined in terms of \textbf{exponential incremental stability}. In general, stability is a control-theoretic property of a dynamical system that describes  the sensitivity of the dynamics to perturbations of the state or input (c.f. \cite{kirk2004optimal}). We focus on an incredibly strong form of stability that we call exponential incremental stability, which corresponds to a dynamical system in which the effects of perturbations on future dynamics diminish exponentially in time. 

\begin{definition}[Exponential Incremental Input-to-State Stability]\label{def:diss} Let $C \ge 1$ and $\rho \in [0,1)$. We say  $f: \Xspace \times \Uspace  \to \Xspace$ is $(C,\rho)$-exponentially incrementally input-to-state stable (E-IISS) if for any two states $\bx_1,\bx_1'$ and sequences of inputs $\{\bu_k,\bu_k'\}_{k=1}^{t}$, the resulting trajectories $\bx_{t+1} = f(\bx_t,\bu_t)$ and $\bx_{t+1}' = f(\bx_t',\bu_t')$ satisfy
\iftoggle{arxiv}
{
\begin{align}
    \|\bx_{t+1} - \bx'_{t+1}\| \leq C\rho^{t}\|\bx_1 - \bx_1'\| + \sum_{1 \le k \le t}C\rho^{t-k}\|\bu_k - \bu_k'\|. \label{eq:exp_ISS}
\end{align}
}
{$\|\bx_{t+1} - \bx'_{t+1}\| \leq C\rho^{t}\|\bx_1 - \bx_1'\| + \sum_{1 \le k \le t}C\rho^{t-k}\|\bu_k - \bu_k'\|$.}
 Let $\pist: \Xspace \to \Uspace$ be a deterministic policy. We say $(\pist,f)$ are $(C,\rho)$-E-IISS
 if the ``closed loop'' system $f^{\pi}(\bx,\bu):=f(\bx,\pi(\bx) + \bu)$ is $(C,\rho)$-E-IISS.

\end{definition}

In all of our lower bound constructions, the origin will be a fixed point of both the open-loop and closed-loop system dynamics: $f^{\pi}(\bzero,\bzero) = f(\bzero,\bzero) = \bzero$. Thus, for these cases, \Cref{def:diss} stipulates that the dynamics {exponentially contract towards the origin}.
E-IISS is essentially the strongest form of incremental stability, a term originally due to \cite{agrachev2008input}. Yet, \textbf{despite its strength}, we demonstrate that exponential-in-horizon compounding error can occur 
even when the dynamics $f$ and expert $(\pist,f)$ satisfy E-IISS. We complement these results with upper bounds that hold under much weaker conditions.

\newcommand{\Qfun}{Q}
\newcommand{\Vfun}{V}

\subsection{The RL Perspective on Imitation Learning}\label{sec:RL_perspective}

Given that policies can be measured in terms of total-cost incurred, it has been popular to adopt the formalism of reinforcement learning to study performance of IL methods. 
Focusing on additive costs $\cost(\bx_{1:H},\bu_{1:H}) = \sum_{h=1}^H \cost_h(\bx_h,\bu_h)$, define the $Q$-function
\iftoggle{arxiv}{
\begin{align*}
\Qfun_{h;f,\pihat,\cost,H}(\bx,\bu) &:= \cost_h(\bx,\bu) + \sum_{h'> h}^H  \Exp_{\pihat,f}\left[\cost_{h'}\left(\bx_{h'},\bu_{h'}\right) \mid (\bx_h,\bu_h) = (\bx,\bu)\right]. 
\end{align*}
}
{
$\Qfun_{h;f,\pihat,\cost,H}(\bx,\bu) := \cost_h(\bx,\bu) + \sum_{h'> h}^H  \Exp_{\pihat,f}\left[\cost_{h'}\left(\bx_{h'},\bu_{h'}\right) \mid (\bx_h,\bu_h) = (\bx,\bu)\right]$.
}
The $Q$-function formalism gives two natural conditions under which $\Rsubcost$ can be controlled by training risk. 

 First, if  $\bu \mapsto \Qfun_{h;f,\pihat,\cost,H}(\bx,\bu)$ is $L$-Lipschitz for each $h$, then  \Cref{lem:perf_diff_Lipschitz} yields
 \iftoggle{arxiv}
 {
\begin{align}
\Rsubcost(\pihat;\pist,f,\Dist, H) \le  L \cdot \RtrainLp[1](\pihat;\pist,f,\Dist, H) \le L \cdot \Rtrain(\pihat;\pist,f,\Dist,H).  \label{eq:perff_diff_Lip}
\end{align}
}
{\\
$\Rsubcost(\pihat;\pist,f,\Dist, H) \le  L \cdot \RtrainLp[1](\pihat;\pist,f,\Dist, H)$.
}
Second, if each $\Qfun_{h;f,\pihat,\cost,H}(\bx,\bu) \in [0,B]$, then \Cref{lem:perf_diff_zero_one} ensures
\iftoggle{arxiv}
{
\begin{align}
\Rsubcost(\pihat;\pist,f,\Dist,H) &\le B\cdot\Rzeroone(\pihat;\pist,f,\Dist,H), \label{eq:R01}
\end{align}
}
{
$\Rsubcost(\pihat;\pist,f,\Dist,H) \le B\cdot\Rzeroone(\pihat;\pist,f,\Dist,H)$.
}
where 
\iftoggle{arxiv}{$\Rzeroone(\pihat;\pist,f,\Dist, H) := \sum_{h=1}^H\Exp_{\pist,f,\Dist} \Exp_{\bhatu \sim \pihat(\bstx_h)}I
\left\{\bstu_h \ne \bhatu_h\right\}$ }
{ $\Rzeroone(\pihat;\pist,f,\Dist, H)$, defined as $\sum_{h=1}^H\Exp_{\pist,f,\Dist} \Exp_{\bhatu \sim \pihat(\bstx_h)}I
\left\{\bstu_h \ne \bhatu_h\right\}$,
}
is the $\{0,1\}$-loss
analogue of $\RtrainLp$. Both statements are proved in \Cref{app:Q_functions} via the celebrated performance difference lemma \citep{kakade2003sample}. Thus, IL (even pure behavior cloning!) exhibits limited compounding error provided that either (a) relevant $Q$-functions are Lipschitz, or (b) it is feasible to minimize the $\{0,1\}$ risk $\Rzeroone$.
\iftoggle{arxiv}
{
\begin{remark} \Cref{eq:perff_diff_Lip,eq:R01} are special cases of a more general principle that the imitation learning error can be related to a certain integral probability metric (IPM) induced by the class of possible $Q$-functions \citep{swamy2021moments} : $L$-Lipschitz $Q$-functions induce an IPM which scales the $L_1$ expert-distribution error by a factor of $L$, whereas the condition that the $Q$-functions are bounded by $B$ scales the resulting $\{0,1\}$ loss bound by that same factor.
\end{remark}
}
{

}

\subsection{The Control vs.  RL Perspectives, and Limitations of the Latter}
\label{sec:rl_vs_control}
The control perspective focuses on the properties of the dynamical map $f$, and the closed-loop dynamics between $f$ and the expert policy $\pist$. The RL perspective places assumptions directly on the $Q$-functions; these depend implicitly on the dynamics and choice of cost, and, when arguing via the performance difference lemma, on the learner policy $\pihat$. One connection between the two viewpoints is that, when  the learned policy $\pihat$ is such that $(\pihat,f)$ is E-IISS, the resulting $Q$-functions are Lipschitz, and hence compounding error is avoided (see \Cref{sec:proof_lem_EIISS} for proof):

\begin{restatable}{lemma}{eiiss}\label{lem:EIISS} Suppose that $(f,\pihat)$ is $(C,\rho)$-E-IISS and $\pihat$ is $L_{\pihat}$-Lipschitz. Then, for any $\cost \in \cC_{\mathrm{lip}}$, $\Qfun_{h;f,\pihat,\cost,H}$ is $\frac{C}{1-\rho}(2 + L_{\pihat})$-Lipschitz. Moreover, for any $D \in \laws(\Xspace)$ and $H \ge 1$,  and any $\cost \in \Clip$,
\iftoggle{arxiv}
{
\begin{align*}
\Rsubcost(\pihat;\pist,f,D,H) \le \frac{C}{1-\rho}(2 + L_{\pihat}) \cdot \RtrainLp[1](\pihat;\pist,f,D,H).
\end{align*}
}
{
$\Rsubcost(\pihat;\pist,f,D,H) \le \frac{C}{1-\rho}(2 + L_{\pihat}) \cdot \RtrainLp[1](\pihat;\pist,f,D,H)$.
}
\end{restatable}
The above bound extends to $\RtrainLp[2]$ via  H\"older's inequality, and resembles classical equivalences between controllability to the origin and existence of Lyapunov functions \citep{sontag1983lyapunov}.
\iftoggle{arxiv}
{

}
{}
Still, when the dynamics $f$ are unknown, it is not clear how to ensure that the closed-loop learned dynamics with $\pihat$ are E-IISS. Indeed:
\begin{lemma}\label{lem:EISS_hard} There exists a pair of linear, deterministic time-invariant policies and dynamics  $(f_i,\pi_i)$, $i \in \{1,2\}$, such that $f_1$, $f_2$, $(\pi_1,f_1)$ and $(\pi_2,f_2)$ are all $(C,\rho)$-E-IISS for some $C \ge 1$ and $\rho \in (0,1)$. However,  neither $(\pi_1,f_2)$ nor  $(\pi_2,f_1)$ are E-IISS for any choice of $C' \ge 1,\rho' \in (0,1)$. 
\end{lemma}  

The above  follows from \Cref{lem:chall_pair}, and this insight lies at the heart of our lower bound. \Cref{lem:EISS_hard} cautions against placing overly optimistic assumptions on the class of learners' $Q$-functions, or claiming that, if such assumptions fail, the problem faced is unrealistically pathological. Instead, the control-theoretic lens suggests that there are {seemingly benign} problem regimes in which uniform Lipschitzness of $Q$-functions is itself \textbf{too coarse} an assumption.

\iftoggle{arxiv}
{
  
\paragraph{The limitations of prior work.} Recall from \iftoggle{arxiv}{\Cref{eq:R01}}{\Cref{sec:RL_perspective}} that imitation in the $\{0,1\}$ loss (as considered in \Cref{eq:R01}) yields at most $\mathrm{poly}(H)$ compounding error, a now-classical argument present, e.g., in the the seminal  \textsc{Dagger}  paper \cite{ross2010efficient}. Recent work by \cite{fosterbehavior} shows improved dependence on horizon when the imitation error is measured in the trajectory-wise Hellinger distance, which can be achieved algorithmically by minimizing a $\log$-loss. \Cref{rem:hellinger_tv} discusses the classical fact that the Hellinger distance is qualitatively equivalent to the Total Variation distance,  which, when specialized to per-timestep imitation of deterministic  experts, is equal to the $\{0,1\}$-loss considered in \Cref{eq:R01}.  Hence,  the findings in both \cite{ross2010efficient} and \cite{fosterbehavior} implicitly require that it be feasible to imitate in  the binary, $\{0,1\}$ sense. 

However, in 
\Cref{sec:zero_one_loss}, we show that non-vacuous $\{0,1\}$ imitation is  impossible in continuous action spaces, exposing the limitations of this analysis in such settings.
\begin{proposition}[Informal] \label{prop:cannot_est_TV_informal} Consider the problem of estimating $1$-Lipschitz functions $z \in [0,1] \mapsto \pi_{\star}(z)$, from noiseless samples $(z,\pi_{\star}(z))$, where $z$ is drawn uniformly on $[0,1]$. For any $n$, such functions can be learned up to error $1/n$ in $\ell_2$-error, but the expectation of the $\{0,1\}$-loss is equal to $1$. 
\end{proposition}
A formal statement and proof is given in \Cref{sec:zero_one_loss}. As a consequence, the statistical primitives on which both \cite{ross2010efficient} and \cite{fosterbehavior} rely---supervised learning in a loss (equivalent to) the $\{0,1\}$-loss---fail to hold for imitation of continuous-action deterministic policies.  
Of course, it is possible to avoid the pathology by replacing the $\{0,1\}$-loss at the cost of a discretization error. But, as is implied by our lower bounds,  the resulting discretization error can compound exponentially in the horizon. }
{\begin{remark}
Recall from \iftoggle{arxiv}{\Cref{eq:R01}}{\Cref{sec:RL_perspective}} that imitating in the $\{0,1\}$-loss yields at most $\mathrm{poly}(H)$ compounding error, a now-classical argument present, e.g., in the  seminal  \textsc{Dagger}  paper \cite{ross2010efficient}.  In
\Cref{sec:zero_one_loss}, we show that non-vacuous $\{0,1\}$-loss imitation is  impossible in continuous action spaces, exposing the limitations of this analysis in such settings. 
\end{remark}}


\section{Main Results}\label{sec:results}
\newcommand{\Piinst}{\Pi(\inst)}
\newcommand{\Finst}{\cF(\inst)}

\colorbold{Organization.} This section presents our main results in their most concrete forms: \Cref{thm:main_stable,thm:constant_prob} are lower bounds against ``simple'' policies (defined below), \Cref{thm:non_simple_stochastic} is a lower bound for more general policies, and \Cref{thm:unstable}  lower bounds arbitrary policies when dynamics are unstable. Finally, \Cref{thm:smoothgen} shows that compounding error can be avoided when expert data provides sufficient coverage. Each theorem has a corresponding result, labeled as ``Theorem \#.A'' given in \Cref{sec:minmax_lb}, which is more granular and formulated in the language of minimax risks better suited to the expert reader. These show that arbitrary families of $L_2$ regression problems can be embedded into imitation learning problems which witness the same degree of compounding error. All lower bounds instantiate a common proof schematic, given in \Cref{sec:schematic}.

\iftoggle{arxiv}{\colorbold{Setup.}}{} Recall from \Cref{lem:EIISS} that if $(\pihat,f)$ is E-IISS, compounding error is avoided. Hence, 
our negative results necessarily leverage that the learner has uncertainty over the true dynamics $f$, and thus cannot ensure $(\pihat,f)$ is incrementally stable. Indeed, if $f$ is known and $(\pist,f)$ is guaranteed to be stable, then  the (possibly inefficient) algorithm which optimizes only over policies $\pihat$ for which $(\pihat,f)$ is stable avoids compounding error. To this end, we establish lower bounds against problem families  defined as follows.
\begin{definition}[Problem Class] An $(\R^d,\R^m)$-\textbf{IL problem family} $(\cI,\Dist)$ is specified by state space $\Xspace = \R^d$, input space $\Uspace = \R^m$,  and an \textbf{instance class} $\inst = \{(\pist,f): (\pist,f) \in \inst\}$ of pairs of candidate expert policies $\pist$ and ground-truth dynamics $f$, as well as a distribution $\Dist$ over initial states.
\end{definition}

Given an instance class $\inst$, we define its constituent  policies $\Pi(\inst) := \{\pist : \exists f \text{ for which } (\pist,f) \in \cI\}$ and dynamics $\cF(\inst) := \{f : \exists \pist \text{ for which } (\pist,f) \in \cI\}$. Our lower bound constructions are ``regular,'' with experts and dynamics being deterministic, Lipschitz, and smooth.
\begin{definition}[Regularity Conditions]\label{def:regular}
 We say $(\cI,\Dist)$ is $(R,L,M)$-regular for all  $(\pist,f) \in \cI$, if (a) $\pist$ is deterministic, (b) $\bx \mapsto \pist(\bx)$ and $(\bx,\bu) \mapsto f(\bx,\bu)$ are $L$-Lipschitz and $M$-smooth, and (c) with probability 1 under $\Pr_{\pist,f,\Pnot}$, it holds that $\max_t\max\{\|\bx_t\|,\|\bu_t\|\} \le R$.  We say that $(\cI,\Dist)$ is $O(1)$-regular if we can take $R,L,M$ to be at most universal constants.
 \end{definition}

\subsection{``Simple'' Policies and Algorithms}

We define \textbf{simple IL policies} as a slight generalization of the smooth, deterministic expert policies considered above. \textbf{Simple algorithms} are those that return simple policies.


\begin{definition}[Simple Policies and Algorithms]\label{defn:simple} A policy $\pihat$ is \textbf{simply-stochastic} if the distribution of deviations from the mean, $\pihat(\cdot \mid \bx) - \mean[\pihat](\bx)$, does not depend on $\bx$. We say $\pihat$ is \textbf{$(L,M)$-simple} if $\pihat$ is simply-stochastic, and $\mean[\pihat]$ is $L$-Lipschitz and $M$-smooth.
\iftoggle{arxiv}
{

}{}
\label{defn:vanilla} An IL algorithm $\est$ is \textbf{$(L,M)$-simple} if, for any sample $\Samp$, with probability one over $\pihat \sim \est(\Samp)$, $\pihat$ is $(L,M)$-simple.  We let $\Avan(L,M)$ denote the class of $(L,M)$-simple IL algorithms, and denote by $\Avan(\BigOh{1})$ a class $\Avan(L,M)$ for some sufficiently large $L,M = \BigOh{1}$.
\end{definition}

The simply-stochastic requirement permits both deterministic policies, as well as popular {Gaussian policies}, where $\pihat(\bx) = \cN(\bm{\mu}(\bx), \bSigma)$, where $\bSigma$ is fixed for all $\bx$. For the ``regular'' IL problem families above, restricting to simple policies subsumes the classical \iftoggle{arxiv}{learning-theoretic notions}{notion} of {proper learning}. 
\begin{definition}[Proper Algorithms]
\label{def:proper} Given an instance class $\cI = \{(\pist,f)\}$, we say that $\est$ is $\cI$-proper if,  for any sample $\Samp$, with probability one over $\pihat \sim \est(\Samp)$, it holds that $\pihat \in \Pi(\inst)$. We denote the set of $\cI$-proper algorithms $\Aproper(\cI)$. 
\end{definition}
In particular, if $(\cI,\Dist)$ is $\BigOh{1}$-regular\iftoggle{arxiv}{, and as we consider only deterministic experts, all}{, all deterministic} expert policies are simple: $\Aproper(\cI) \subset \Avan(\BigOh{1})$. Thus, lower bounds against simple algorithms imply lower bounds against proper algorithms as a special case. At the same time, they rule out the potential benefits of adding state-independent noise, or of inflating smoothness and Lipschitzness constraints by constant factors.

\subsection{Simple Policies Fail to Imitate Simple Experts}

\newcommand{\variance}{\mathsf{var}}

 Recall the notation $\Eshorthand$ denoting expectation under a sample $\Samp$ drawn from $[\pist,f,\Dist]$, and policy $\pihat \sim \Alg(\Samp)$. Our main result states that, for any desired fractional rate of estimation, there exist regular problem families with  open- and closed-loop stable dynamics for which execution error is exponentially larger than training error. 
\begin{theorem}\label{thm:main_stable} Fix a $k, s \in \N$ with $s \ge 2$ and define $\beps_n = n^{-s/k}$, and let   $C \ge 1$, $\rho \in (0,1)$ be universal constants, and let $C_{1},C_{2}$ be constants depending only on $(k,s)$.
 There exists a  $(\R^d,\R^d)$-IL problem family $(\cI,\Dist)$, with $d = k+2$ and a  $\cost = \cost(\cdot,\cdot) \in \Clip$,  such  that  $(\cI,\Dist)$ is $O(1)$-regular, $f$ and $(\pist,f)$ are $(C,\rho)$-E-IISS for all $(\pist,f) \in \cI$, and for all  $H \ge 2, n \ge 1$: 
\begin{itemize}
	\item[(a)] There exists an IL $\algg \in \Aproper(\cI) \subset \Avan(O(1))$ such that for all  $(\pist,f) \in \cI$, it holds that
    \iftoggle{arxiv}
    {
	\begin{align}
	 \Eshorthand[\RtrainLp[2](\pihat;\pist,f,\Dist,H)] \le C_{1} \beps_n \label{eq:RtrajL2p:bound}.
	\end{align}
    }{$\Eshorthand[\RtrainLp[2](\pihat;\pist,f,\Dist,H)] \le C_{1} \beps_n$.}
	\item[(b)] Let $L,M \ge 1$. For any $\algg \in \Avan(L,M)$,  there exists $(\pist,f) \in \cI$ for which
    \iftoggle{arxiv}
    {
	\begin{align}
\Eshorthand[\Rsubcost(\pihat;\pist,f,\Dist,H)] \ge C_{2}\min\left\{1.05^H \beps_n,~ 1/(ML^2) \right\}. \label{eq:bad_lb}
	\end{align}
    }
{it holds that $\Eshorthand[\Rsubcost(\pihat;\pist,f,\Dist,H)]$ is at least $ C_{2}\min\left\{1.05^H \beps_n,~ 1/(ML^2) \right\}$.}
\end{itemize}
\end{theorem}

In words, our bound states there are problem instances where it is possible to make an $L_2$ supervised learning loss small, but the error incurred on deployment is exponential in horizon, up to a threshold which shrinks gracefully as the smoothness and Lipschitz constants grow. By comparing $\RtrainLp[2]$ and $\Rsubcost$, we  show that the imitation learning problem is challenging even if the underlying supervised learning problem is statistically tractable. 

Crucially, our lower bound does not prescribe \emph{how} the learner uses the expert demonstration data, only that the returned policy $\pihat$ is simple. In particular, the learner need not attempt to minimize \Cref{eq:bc_train}, or conduct any form of behavior cloning. Indeed, our lower bound is entirely agnostic to the learning algorithm. Therefore, our lower bound applies to algorithms which attempt to imitate in some integral probability metric via inverse reinforcement learning \citep{ho2016generative}, provided that they do not interact further with the dynamics $f$.  Moreover, because our bound holds for a  fixed  cost across all instances, it applies \textbf{even to cost/reward-aware algorithms, such as those based on offline reinforcement learning}.  

The above result is strengthened as \Cref{thm:stable_detailed} in \Cref{sec:minmax_lb}, where we further show that (a) all optimal algorithms which minimize $\RtrainLp[2]$ are proper (and thus simple) algorithms; that is, non-simple policies confer \emph{no advantage} on the expert data distribution; and (b) the dynamics $f \in \cF(\cI)$ are  one-step controllable with good condition number.    Finally,  our lower bound can be strengthened so that exponential compounding occurs on a constant-probability event.\footnote{Note that this doesn't immediately follow from the boundedness of $\cost$, because $2^H \beps_n$ can still be much less than one, in which case its possible that the cost incurred is $\Omega(1)$ on an event of probability $2^H \beps_n \ll 1$. }

\begin{theorem}\label{thm:constant_prob} Consider the setting of \Cref{thm:main_stable}, with the same  cost $\cost \in \Clip$. There exist constants $C_3,C_4 > 0$ depending only on $(k,s)$ for which $\Pr_{\pist,f,\Dist}[\cost(\bx_{1:H},\bu_{1:H}) = 0] =1$ for all $(\pist,f) \in \inst$, but for any $\algg \in \Avan(L,M)$, there exists \iftoggle{arxiv}{some}{} $(\pist,f) \in \inst$ such that
\begin{align} \Eshorthand\Pr_{\pihat,f,\Dist}\left[\cost(\bx_{1:H},\bu_{1:H}) \ge C_{3}\min\left\{1.05^H \beps_n, L^{-2}M^{-1}\right\}\right] \ge C_{4}. \label{eq:bad_lb_very}
	\end{align}
\end{theorem}
\Cref{thm:constant_prob,thm:main_stable} are derived from \Cref{thm:stable_detailed}, whose proof is given in \Cref{app:stable}. These  results rely on three properties of simple policies: smoothness, simple-stochasticity, and (tacitly) that $\pihat:\Xspace \to \laws(\Xspace)$ is Markovian (static). 
In \Cref{sec:nonvanilla}, we illustrate (and in \Cref{label:app_non_vanilla}, we formally prove) that removing any of the three restrictions breaks our lower bound construction. 

\iftoggle{arxiv}
{
\begin{remark}[Significance of unknown dynamics]\label{rem:significance_unknown} From \Cref{lem:EIISS}, if $\pihat$ stabilizes $f$, the resulting Q-function is Lipschitz. And we know from \Cref{eq:perff_diff_Lip} that Lipschitzness  of the $Q$-functions prevents compounding error.  Crucially, \Cref{eq:perff_diff_Lip} consider the $Q$ function induced by the learned policy $\pihat$ and the \emph{actual dynamics}, which we will denote $f_{\star}$. But while $\pihat \in \Pi(\cI)$ stabilizes every $f$ such that $(\pihat,f) \in \cI$, it does not stabilize every possible $f_{\star} \in \cF(\cI)$.  Stated otherwise, the product instance class $\tilde \cI := \Pi(\cI) \times \cF(\cI)$ contains pairs $(\pihat,f_{\star})$ which are \emph{not} closed loop stable. This may be interpreted as follows: the expert will always act in a way that stabilizes the actual dynamics, but not in a way that stabilizes \emph{every possible dynamics. } Thus, if we cannot resolve the true dynamics $f_\star$, we cannot ensure $\pihat$ stabilizes it, and thus cannot ensure the resulting $Q$-function is Lipschitz. If the true dynamics $f_\star$ were known, then we could just restrict only to the set of $\pihat$ which stabilize it, and avoid compounding error. 
\end{remark}
\begin{remark}[RL interpretation of open-loop stability]\label{rem:Lipschitz} The property that $f$ is open-loop stable implies that, under dynamics $f$ and the \emph{zero policy} $\pi_0(\bx) \equiv \bzero$, the resulting $Q$-function is Lipschitz. In other words, the open-loop stability of all $f \in \cF(\cI)$  is equivalent to the existence of a known, single reference policy $\pi_0$ which renders all $Q$-functions associated with $(\pi_0,f)$ Lipschitz. Thus, our results say that the existence of such a single known stabilizing/Lipschitz-inducing policy is insufficient (given the simplicity requirements).
\end{remark}
}
{
\begin{remark}\label{rem:Lipschitz}\label{rem:significance_unknown} From an RL-theoretic perspective, open loop stability ensures the existence of a nomimal policy $\pi_0(\bx) \equiv \bzero$ which ensures that the Lipschitz constant of all $Q$ functions, induced by $\pi_0$ and the unknown dynamics $f$, is bounded.  Thus, our lower bound ensures the mere existence of \emph{some known policy} with bounded Lipschitz $Q$-functions is not sufficient to avoid exponential compounding error. Importantly, our result uses a class $\inst$ which is \emph{not} a product-set (i.e. $\cP \ne \Pi(\cP) \times \cF(\cP)$); otherwise, either (a) $\inst$ would contain some pair $(\pist,f)$ which is either unstable in closed loop or (b) all pairs $(\pist,f)$ would be stable, obviating a lower bound.
\end{remark}}
\begin{remark}[Error amplification, not unbounded costs] Crucially, our lower holds for costs that are bounded in $[0,1]$, and as show in \Cref{thm:constant_prob}, the probability of the event on which error is magnified by $\exp(H)$ is at least a universal constant. Thus, our results state that it is error amplification, rather than unbounded growth of the costs, that accounts for the lower bound. 
\end{remark}

\subsection{Lower Bounds Against More Complex Policies}
We now give two lower bounds for possibly non-simple policies. The first relaxes the simply-stochastic requirement, at the expense of a weaker result, and the second holds unconditionally, but considers unstable open-loop  (as opposed to E-IISS) dynamics. 
\iftoggle{arxiv}{\Cref{sec:nonvanilla} presents {very preliminary evidence that non-simple policies may indeed be more powerful for imitating simple experts}\iftoggle{arxiv}{; an observation which the authors find quite surprising.}{.}
}
{}

The next result requires two new objects. First, the class  $\Areason(L,M,\alpha,p) \supset \Avan(L,M)$ of algorithms which return policies with $L$-Lipschitz, $M$-smooth means, and whose stochasticity  satisfies a mild anti-concentration condition parameterized by $\alpha,p \in (0,1]$. For suitable constants $\alpha,p$ bounded away from zero, this class includes all simply-stochastic, Gaussian, and most mixture-policies as special cases. The second is an $L_2$-variant of $\Rsubcost$, denoted $\Rlpcost[2]$.  Formal definitions are deferred to \Cref{sec:minmax_anticonc}. Once supplied, the following theorem is entirely formal.

\begin{theorem}[Lower Bound beyond Simple-Stochasticity]\label{thm:non_simple_stochastic} Consider the setting and problem family $(\cI,\Dist)$ of \Cref{thm:main_stable}, with integers $k, s \ge 2$, and $\beps_n := n^{-s/k}$. Given parameters $\alpha, p \in (0,1]$,  consider the algorithm class $\bbA = \Areason(L,M,\alpha,p)$.  Moreover, suppose that $n$ is larger than a sufficiently large polynomial in $(LMk/\alpha p)^{k/s}$.  Then, there exists a  constant $C$ depending only on $(k,s)$ and  universal $C' \ge 1$ such that for any $\Alg \in \bbA$ and $n,H \ge 2$,
\iftoggle{arxiv}
{
\begin{align*}
\Eshorthand[\Rlpcost[2](\pihat;\pist,f,\Dist,H)] \ge  C \cdot \min\left\{ \beps_n \cdot 1.05^{H}, \beps_n^{1-\frac{1}{C'(1+\log(1/(\alpha p)))}}\right\}. 
\end{align*}
}
{
 $\Eshorthand[\Rlpcost[2](\pihat;\pist,f,\Dist,H)]$ is at least $C$ times the minimum of $\beps_n \cdot 1.05^{H}$ and $\beps_n^{1-\frac{1}{C'(1+\log(1/(\alpha p)))}}$. 
}
\end{theorem}

\Cref{thm:non_simple_stochastic} is a consequence of \Cref{thm:stable_general_noise}, proven in \Cref{app:general_noise}. Because $\bbA = \Areason(L,M,\alpha,p) \supset \Avan(L,M)$,    and \Cref{thm:non_simple_stochastic}
uses the same problem family as \Cref{thm:main_stable}, $\beps_n$ upper bounds the best attainable expert-distribution error in \Cref{thm:non_simple_stochastic} as well. Under $\Rlpcost[2](n;\cI,\Dist,H)$, we suffer exponential compounding error  to a threshold that is $\beps_n^{1-\Omega(1)}$. That is, the \textbf{scaling} of the error for large $H$ has a strictly worse exponent than linear in $\beps_n$.  $\Rlpcost[2]$ is an $L_2$ analogue of $\Rsubcost$ which places greater emphasis on the upper tails of the cost. We suspect that, with a sharper analysis, we may be able to obtain the same bound on expected (i.e. $L_1$) cost, $\Rsubcost$.
\iftoggle{arxiv}
{

}
{

}
Unlike \Cref{thm:constant_prob}, compounding error in \Cref{thm:non_simple_stochastic} occurs on a very low probability event and this is responsible for the at most $\beps_n^{1-\Omega(1)}$ rate of error. In \Cref{label:app_non_vanilla}, we show that both the low-probability of compounding error and $\beps_n^{1-\Omega(1)}$ error rates are qualitatively unimprovable for the construction used in \Cref{thm:main_stable}/\Cref{thm:non_simple_stochastic}. This is a reflection of the Benign Gambler's Ruin phenomenon described in \Cref{sec:benevolent_ruin}.  
\iftoggle{arxiv}
{

}{

}
Lastly,  if the dynamics are stable in closed-loop but possibly unstable in open-loop, exponential compounding occurs with zero restriction on the learned policies $\pihat$.
\begin{theorem}[Unstable Dynamics]\label{thm:unstable} Fix  integers $1 \le k$ and $s \ge 2$; set $\beps_n = n^{-s/k}$. For  $d \ge k$, there exists an $O(1)$-regular IL class $\cI$ such that each $(f,\pist) \in \cI$ is $(1,0)$-E-IISS, constants $C_1,C_2$ depending only $k,s$, and a $\cost \in \Clip$ such that for all $2 \le H \le \frac{1}{2}e^{d/8}$ and $n \ge 1$
\iftoggle{arxiv}
{
 \begin{itemize}
 	\item[(a)] There is an algorithm $\Alg \in \Aproper(\cI)$ such that, for all instances $(\pist,f) \in \cI$,
 	\begin{align*}
 	\Eshorthand[\RtrainLp[2](\pihat;\pist,f,\Dist,H)] \le C_{1} \beps_n.\end{align*}  
 	\item[(b)] For \emph{any} IL algorithm $\Alg$, including those permitted to return policies $\pihat$ which are arbitrarily non-smooth, stochastic, and even history-dependent, there exist some $(\pist,f) \in \cI$ for which
\begin{align}
\Eshorthand[\Rsubcost(\pihat;\pist,f,\Dist,H)] \ge C_{2} \min\{2^H\beps_n,1\}. \numberthis \label{eq:numberthis_exp_compounding}
\end{align}
\end{itemize}
}
{
There is an algorithm $\Alg \in \Aproper(\cI)$ such that, for all instances $(\pist,f) \in \cI$,
 	$\Eshorthand[\RtrainLp[2](\pihat;\pist,f,\Dist,H)] \le C_{1} \beps_n$, but (b) for \emph{any} IL algorithm $\Alg$, including those permitted to return policies $\pihat$ which are arbitrarily non-smooth, stochastic, and even history-dependent, there exist some $(\pist,f) \in \cI$ for which
$\Eshorthand[\Rsubcost(\pihat;\pist,f,\Dist,H)] \ge C_{2} \min\{2^H\beps_n,1\}
$
}
Moreover, a constant-probability variant \iftoggle{arxiv}{of \Cref{eq:numberthis_exp_compounding}}{} analogous to that in \Cref{thm:constant_prob} holds.

\end{theorem}
\Cref{thm:unstable} is derived from \Cref{thm:unstable_detailed}, whose proof is given in \Cref{app:unstable}.
While exponential compounding is intuitive  when dynamics are unstable, a  rigorous and unconditional proof is subtle. For example, in a certain unstable \emph{scalar} system, exponential compounding error can be mitigated by a number of strategies (\Cref{sec:nonvanilla}). Indeed, our bound requires sufficiently large dimension $d = \Omega(\log H)$, and this is likely sharp if one combines the concentric stabilization strategy (\Cref{sec:concentric_stabilization}) with a covering argument. Importantly, the dimension $d$ in \Cref{thm:unstable} can be made arbitrarily large, while the constants depend only on $s$ and $k$, which can be taken to be fixed.

\subsection{Simple Policies Avoid Compounding Error with Sufficient  Coverage}

 \Cref{thm:main_stable} relies on the indistinguishability of different stabilizing system dynamics from the perspective of the learner. In \Cref{thm:smoothgen}, which follows, we show that this can be circumvented via E-IISS in addition to a strong data coverage requirement, which we term \emph{well-spreadness} (\Cref{def:dist_smooth}). Well-spread distributions can arise naturally, for instance via additive Gaussian exploration noise in the context of fully controllable systems. Our result, proved in \Cref{app:ub_proofs}, can be interpreted as a polynomial upper bound for  experts whose own trajectories induce sufficient exploration (see \Cref{rem:explore_general} for a more careful explanation).

\newcommand{\dist}{\mathrm{dist}}
\begin{definition}[Well-Spread Distribution]\label{def:dist_smooth}
A distribution $P$ over $\R^d$ is $(L,\epsilon,\sigma_0)$-\textbf{well-spread} if $P$ has a density $p(\cdot)$ with respect to the Lebesgue measure, and if there exists a convex, compact set $\cK \subset \R^d$ such that (a) the score function $ \bx \mapsto \log p(\bx) $  
is $L$-Lipschitz on $\cK$, and (b) $\Pr_{\bx \sim P}[\dist(\bx, \cK^c) \le \sigma_0 ] \le \epsilon$.
\end{definition}

\begin{restatable}[Smooth Training Distribution]{theorem}{smoothgen}\label{thm:smoothgen}
Consider any $(d,m)$-IL problem family $(\inst, \Pnot)$.
Provided for any $(\pist, f) \in \inst$, $h \in [H]$, the distribution of $\bx_h^\star$ under $\Pr_{\pist, f, \Pnot}$ is $(L,\epsilon, \sigma_0)$-well-spread \iftoggle{arxiv}{(\Cref{def:dist_smooth})}{} for $h > 1$ and $\pist,\pihat$ are deterministic, $M$-smooth, $L_{\pi}$-Lipschitz, and $B$-bounded, and $\pist$ is $(C, \rho)$ incrementally input-to-state stabilizing (\Cref{def:diss}). Then, provided that $\Rtrain(\pihat; \pist, f, \Pnot, H) \leq \min\{\rho_0, 1/L\}$, for some $\rho_0$ inverse polynomial in relevant problem parameters, it holds that
\iftoggle{arxiv}
{
\begin{align*}
\Rcost(\pihat;\pist, f, D, H) &\leq c H d \frac{C^2}{(1 - \rho)^2}\left[\Rtrain(\pihat; \pist, f, \Pnot, H) + \sqrt{\epsilon}\right].
\end{align*}
}
{ $\Rcost(\pihat;\pist, f, D, H)$ is upper bounded by 
$c H d \frac{C^2}{(1 - \rho)^2}\left[\Rtrain(\pihat; \pist, f, \Pnot, H) + \sqrt{\epsilon}\right]$,}
where $c := d(8 + 16B^2 + 16M^2)$.
\end{restatable}


\newcommand{\Dest}{D_{\mathrm{est}}}
\section{Proof Overview}\label{sec:proof_intuition}
\newcommand{\cmuconst}{c_{\mu}}

Here, we focus on providing the core intuitions behind  the proof of \Cref{thm:main_stable,thm:constant_prob}. The formal proof is  given in \Cref{app:stable}, which instantiates  a general schematic in \Cref{sec:schematic}.

The crux of the proof is to construct two pairs of policies and dynamical systems $(\pi_i,f_i)_{i \in \{1,2\}}$ such that (a) both pairs are  open- and closed-loop stable, (b) $\pi_i$ destabilizes $f_j$ for $i \ne j$, and (c) $(\pi_i,f_i)$ look indistinguishable on the distribution of expert data. We accomplish this first by constructing a pair of $2 \times 2$ linear dynamical systems with these properties:
\begin{definition}[Challenging Pair]\label{defn:chall_pair} Fix a parameter $\mu \in (0,1/2]$. Define $c_{\mu} := \frac{3}{2}{\mu}$. The \emph{challenging pair} of instances $(\bA_i,\bK_i)_{i \in \{1,2\}}$ are the  matrices in $\R^{2\times 2}$ given by 
\iftoggle{arxiv}{
\begin{align*}
    &\bA_1 := \begin{bmatrix} 1+ \mu & \cmuconst \\ 
    -  \cmuconst & 1 - 2\mu \end{bmatrix}, \quad \bA_2 := \begin{bmatrix} -(1- \frac{1}{4}\mu) &  \cmuconst \\ 
    0 & 1 - 2\mu \end{bmatrix}\\
    & \bK_1 := \begin{bmatrix} -(1+\mu) & - \cmuconst \\
      \cmuconst & 0 
    \end{bmatrix}, \quad \bK_2 := \begin{bmatrix} (1- \frac{1}{4}\mu) & - \cmuconst \\
        0 & 0 
    \end{bmatrix}, \quad 
    \end{align*}
    }
    {$\bA_1 = \begin{bmatrix} 1+ \mu & \cmuconst \\ 
    -  \cmuconst & 1 - 2\mu \end{bmatrix}$, $\bA_2 = \begin{bmatrix} -(1- \frac{1}{4}\mu) &  \cmuconst \\ 
    0 & 1 - 2\mu \end{bmatrix}$, $\bK_1 = \begin{bmatrix} -(1+\mu) & - \cmuconst \\
      \cmuconst & 0 
    \end{bmatrix}$, $ \bK_2 = \begin{bmatrix} (1- \frac{1}{4}\mu) & - \cmuconst \\
        0 & 0 \end{bmatrix}$.}
\end{definition}
Defining the linear dynamics $\fbar_i(\bx,\bu) := \bA_i \bx + \bu$ and policies $\pibar_i(\bx) := \bK_i \bx$, \iftoggle{arxiv}{we verify that}{} $(\pibar_i,\fbar_i)_{i \in \{1,2\}}$ satisfy points (a) and (b) above, and satisfy (c) for expert data on the $\mathrm{span}(\be_2)$ subspace of $\R^2$.
\iftoggle{arxiv}{
For linear systems, $(C,\rho)$-E-IISS reduces to the following definition:}{}
\begin{definition}\label{defn:stab_matrix} A matrix $\bA \in \R^{d \times d}$ is $(C,\rho)$-stable if $\|\bA^s\|_{\op} \le C\rho^s$. 
\end{definition}

\begin{proposition}[The Challenging Pair induces exponential compounding error.] \label{lem:chall_pair}Consider the challenging pair as in \Cref{defn:chall_pair} with parameter $\mu \in (0,\frac{1}{2}]$. Also, we set 
    $\Aclk[i] := \bA_i + \bK_i$, noting that $\fbar_i^{\pibar_i}(\bx,\bu) = \Aclk[i]\bx + \bu$.  Then
\begin{enumerate}[topsep=0em,itemsep=0em,partopsep=0em, parsep=0em]
    \item[(a)] \iftoggle{arxiv}{The matrices}{} $\bA_1,\bA_2,\Aclk[1],\Aclk[2]$ are all $(C_{\mu},\rho_{\mu})$ stable for some  $C_{\mu} > 0$, $\rho_{\mu} \in (0,1)$ depending on $\mu$. 
    \item[(b)] Fix $\bhatK \in \R^{2 \times 2}$ satisfying $\bhatK \be_2 = \bK_1 \be_2 ( = \bK_2 \be_2)$. Then \iftoggle{arxiv}{$\max_{i \in \{1,2\}}\left\|(\bA_i +  \bhatK)^H \be_1\right\| \ge \left(1+\frac{\mu}{4}\right)^H$}{$\max_{i \in \{1,2\}}\|(\bA_i +  \bhatK)^H \be_1\| \ge (1+\frac{\mu}{4})^H$}. 
    \item[(c)] The values of $\bA_i \be_2,\bK_i \be_2,\Aclk[i]\be_2$ do not depend on $i \in \{1,2\}$, and   $V := \spn(\be_2)$ is an invariant subspace of both $\Aclk[1]$ and $\Aclk[2]$. Hence, $(\pibar_i,\fbar_i)$ yield indistinguishable trajectories for any starting state $\bx_1 \in \spn(\be_2)$. 
\end{enumerate}
\end{proposition}
Importantly, note that for any linear $\pihat(\bx) = \bhatK\bx$, $\fbar_i^{\pihat}(\bx,\bu) = (\bA_i + \bhatK_i)\bx + \bu$. Thus, part (c) of the proposition ensures that the closed loop dynamics between a linear learner policy $\pihat$ and $\bA_i$, for at least one index $i \in \{1,2\}$, is unstable.

\iftoggle{arxiv}
{
\begin{remark}[Connection to the gap metric] \label{rem:gap_metric} The gap-metric \citep{zames1981uncertainty} in control theory allows one to measure the extent to which two different dynamical systems can be stabilized by the same control law. In our case, both transition matrices $\bA_i$ are stable in the classical sense (see also \Cref{defn:stab_matrix} above), and thus, as noted above, are simultaneously stabilized by the  indentically-zero control law. However,  neither system can be stabilized by any linear feedback which coincides with the $\bK_i$'s on the subspace $V = \mathrm{span}(\be_2)$. 
\end{remark}
}
{}
\iftoggle{arxiv}
{
\begin{proof}[Proof of \Cref{lem:chall_pair}] We prove  the stability and instability, properties (a) and (b). Property (c) follows directly from observation.  

\textbf{Proof of (a).} A standard fact is that $\bA$ is $(C,\rho)$-stable for some $C> 0$ and $\rho < 1$ if and only if $\rho(\bA) < 1$, where $\rho(\bA)$ denotes the spectral radius, or largest-magnitude real part of an eigenvalue, of $\bA$. The eigenvalues of $\bA_1, \bA_2, \Aclk[i]$, $i \in \{1,2\}$ are $\{1 - \frac{\mu}{2}\}$, $\{1 - 2\mu, -(\frac{1}{4}\mu - 1)\}$, and $\{0, 1 - 2\mu\}$ respectively, and are strictly less than one. The eigenvalues for $\bA_2,\Aclk[1],\Aclk[2]$ can be read directly off their upper triangular form. For $\bA_1$, we observe that its spectrum is the set of the roots of the characteristic polynomial $(1+\mu - \lambda)(1- 2\mu - \lambda) + \frac{9}{4}\mu^2  =(\lambda - (1- \frac{\mu}{2}))^2 $, both of which are $1 - \mu/2 \in (0,1)$. 


\textbf{Proof of (b).}  Any $\bhatK$ satisfying the stipulated constraint is of the form $\bhatK = \begin{bmatrix}a & -c_\mu \\ b & 0\end{bmatrix}$. Then, 
\begin{align*}
    \bA_1 +  \bhatK = \begin{bmatrix} 1+ \mu + a & 0\\ 
   b - c_{\mu} & 1 - 2\mu \end{bmatrix}, \quad \bA_2 +   \bhatK = \begin{bmatrix} -(1- \frac{1}{4}\mu) + a & 0 \\ 
    b & 1 - 2\mu \end{bmatrix}
\end{align*}
Using the lower triangular structure of the above matrices, we can check the quantity of interest is at least the maximum of $|1+ \mu + a|^H$ and $|-(1- \frac{1}{4}\mu)+a|^H$. 
Since $\min_{a} \max\{|1+ \mu + a|,|-(1- \frac{1}{4}\mu)+a|\} \ge 1+\frac{\mu}{4}$, the bound follows.
\end{proof}
}
{\Cref{lem:chall_pair} is proven in \Cref{sec:proof:lem_chall_pair}.
} 
Recall that  we consider \textbf{noiseless} expert demonstrations. Thus, purely linear dynamics and policies do not suffice for a lower bound as such policies can be imitated \emph{exactly} given a sufficient number of expert trajectories. 
Instead, we embed a \emph{nonlinear} supervised learning problem into our linear construction. Our constructions are parametrized by the unknown index $i$ of the challenging pair, and an unknown, nonlinear function $\gst: \R^{d} \to \R $, belonging to a class $\cG$ whose rate of estimation in $L_2$ under a suitable distribution $\Dest$ matches $\beps_n$.

Our construction combines both $g^\star$ and $A_i$ by carving the state space $\mathbb{X}$ into two distinct regions, $R_1, R_2 \subset \mathbb{X}$, where $R_1$ is a unit ball around $0$ and $R_2$ is a unit ball around $3\be_1$. For $\bx \in R_1$, we let the dynamics and policy be given by $f(\bx,\bu) = A_i\bx + \bu$, $\pi^{\star}(\bx) = K_i \bx$. Conversely for $\bx \in R_2$ we use the aforementioned $\gst \in \cG$ to define $f(\bx,\bu) = g^{\star}(\bx)\be_1 - \bu$, $\pi^{\star}(\bx) = g^{\star}(\bx)\be_1$. For our initial state distribution, we consider a mixture which samples half of the initial states from $\Dest$ over $R_2$ and half drawn uniformly on the segment between $-\be_2$ and $+\be_2$.

This construction ensures that the learner errors on $R_2$ scale with $\beps_n$ by the choice of $(\cG, \Dest)$, meaning that $\pihat, \pist$ must diverge at $t = 2$ when the initial condition is sampled from $R_2$, with the learner perturbed in the $\be_1$ direction. For the chosen initial state distribution, trajectories under $\pist$ give no information regarding $\bK_i\be_1$, as $\bx_t = 0$ under $\pist$ for $t \geq 2$ and hence the learner has no way of disambiguating between $\bK_1,\bK_2$. Furthermore, by our choice of distribution over $R_1$, any learner with smooth mean that matches the expert $\bK_i$ on $\mathrm{span}(\{\be_2\})$ can be written $\mean[\pihat](\bx) \approx \bhatK \bx$, where $\bhatK$ satisfies the conditions of \Cref{lem:chall_pair}(b). In this case, by \Cref{lem:chall_pair}(b), the $\beps_n$-magnitude errors in the $\be_1$ space are then magnified exponentially in the horizon $H$. Crucially, the argument requires the \emph{simply-stochastic} noise as more intelligent noise distributions can cancel the compounding error. Alternatively, $\pihat$ may deviate from the expert on the $\be_2$ subspace, as we do not restrict $\algg$ to behavior-cloning-like algorithms. However, in this case, the learner will incur at least $C_2/(ML^2)$ error from $\pist$ in order to prevent the exponential instability.

\subsection{Overview of Additional Proof  Techniques.}

\label{sec:gesture_towards_formal}
\iftoggle{arxiv}{Our formal proof is based on a minimax framework introduced in \Cref{sec:minmax}. Then,  more detailed statements of results are given in \Cref{sec:minmax_lb}, with a detailed proof schematic in \Cref{sec:schematic}.  Full proof details are given in \Cref{app:stable}. All repurposable technical tools  are given in \Cref{app:technical}. Here, we summarize some of the essential technical ingredients.
A key  theme is the need for compounding error with good enough probability, which is necessary due to the boundedness of costs (if all errors are concentrated on rare events, then any bounded cost must be small in expectation.)}

\label{sec:extended_proof_notes}

\colorbold{Statistical Learning.}  The functions $\gst$ defining the ``$R_2$ policy'' $\pist(\bx) = \gst(\bx)\be_1$  must be chosen from a class that is (a) smooth (to preserve overall system smoothness), and (b) has non-trivial statistical error when learned from \emph{noiseless training examples} $(\bx_0,\gst(\bx_0))$. In particular, linear $\gst$ does not suffice. A key subtlety is that (c) we require the estimation error of $\gst$ to be large with constant probability; otherwise, the large errors can only compound by a limited amount before  saturating the bound on the cost magnitude. In \Cref{prop:typical_true}, we show that non-parametric function classes of $\{g\}$ satisfy requirements (a), (b), and (c). This requires operating in the ``interpolation,'' or noise-free, setting of nonparametric regression \citep{kohler2013optimal}.

\colorbold{Bump functions.} We use bump functions to stitch together the aforementioned $R_1, R_2$ regions in a smooth manner. Doing so requires care to ensure that the system remains globally stable, and we accomplish this by making the magnitude of the nonlinear terms sufficiently small, so that they are dominated by the stable linear dynamics. 

\colorbold{Enforcing $\bhatK\bv \approx \bK_i\bv$ for $\bv \perp \be_1$.} With some probability, the initial state distribution randomizes over $\bv \sim \Delta \cB_V$,  the uniform distribution on the unit ball of radius $\Delta$ on the subspace $V = \mathrm{span}(\be_1)^\perp$. Thus for any policy with low imitation error, $\Exp_{\bv \sim \Delta \cB_V}\|\pihat(\bx) - \bK_i \bx\|^2$ is small. By smoothness of $\pihat$, and by making $\Delta$ sufficiently small, classical arguments for zero-order gradient estimation ensure $\nabla \mean[\pihat(\bzero)]\bP_V \approx \bK_i \bP_V$, where $\bP_V$ is the projection onto $V$ (note: $\bK_1 \bP_V = \bK_2 \bP_V$). To ensure compounding error with constant probability, we leverage anti-concentration due to the Carbery-Wright \citep{carbery2001distributional} and Paley-Zygmund inequalities; these use the convenient fact that the uniform distribution on the unit ball is log-concave.

\colorbold{Compounding error in nonlinear systems.} The most significant technical obstacle is generalizing our compounding error argument from linear to nonlinear systems. First, consider deterministic policies $\pihat$. Define the autonomous dynamical system $F_i(\bx)= f_i(\bx,\pihat(\bx)) = \bA_i \bx + \pihat(\bx)$, and $\bhatK := \nabla \pihat(\bzero)$, we see that $\nabla F_i(\bzero) = \bA_i + \bhatK$ must be unstable along the $\be_1$ direction for one $i \in \{1,2\}$, by \Cref{lem:chall_pair}. To show that the resulting \emph{nonlinear} system is unstable,  we adopt an  argument  due to \cite{jin2017escape} to bound the rate at which gradient-based optimizers escape  strict saddle points. \iftoggle{arxiv}{This can be viewed as a quantitative analogue of the classical unstable manifold theorem.}{}   When policies are simply-stochastic, their randomness can be coupled such that the joint distribution $(\bhatu,\bhatu') \sim (\pi(\bx),\pi(\bx'))$ ensures the differences $\bhatu - \bhatu' = \mean[\pi](\bx) - \mean[\pi](\bx')$ are deterministic. Beyond simply-stochastic policies, as in the proof of \Cref{thm:non_simple_stochastic}, we use a considerably more subtle coupling to  witness our stipulated anti-concentration condition, described in \Cref{sec:minmax_anticonc}.


\iftoggle{arxiv}{\section{Potential Benefits of Complex Policy Parameterizations.}}{\section{Potential benefits of complex policy parameterizations.}}\label{sec:nonvanilla}
\iftoggle{arxiv}{In this section, we evaluate the extent to which  policies that violate the simplicity condition (\Cref{defn:simple}) can improve over those which abide by it.}{}


\iftoggle{colt}{}{

\iftoggle{arxiv}{\subsection{Experimental Findings}}{}
\label{sec:experiments}

\begin{figure*}
\begin{center}
\includegraphics[width=1\linewidth]{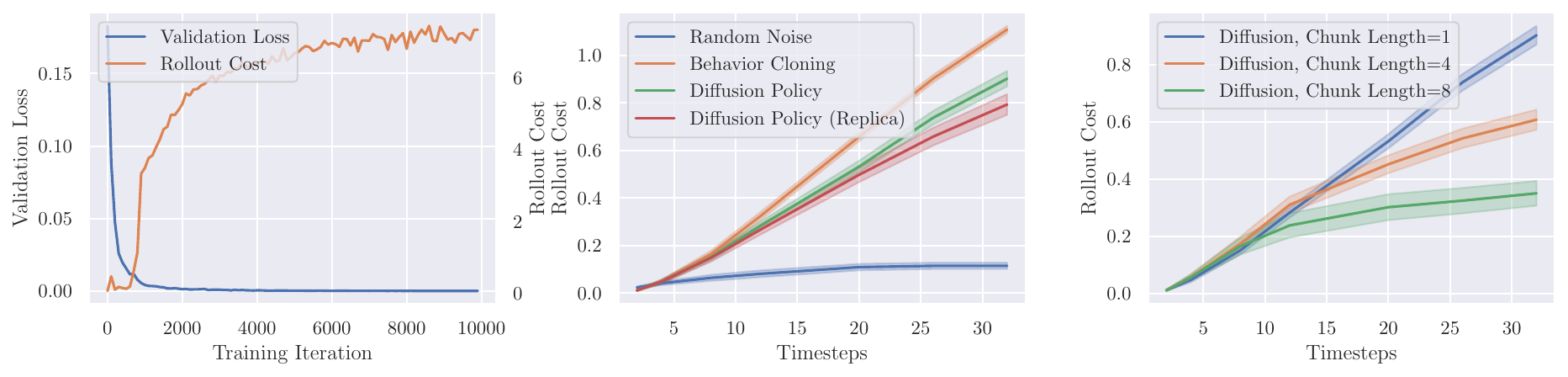}
\end{center}
    \caption{We benchmark the performance of different methods on \Cref{const:stable}. See \Cref{app:experiments} for details. \textbf{Left}: Validation loss and Rollout Cost ($\max_t \langle \be_1, \bx_t\rangle$) of behavior cloning using $H = 32$. Center: Performance of behavior cloning, Diffusion policy \citep{chi2023diffusion}, replica noising \citep{block2024provable}, and random noise $\bu_h \sim \mathcal{N}(\bzero, \frac{1}{6}\bI)$. \textbf{Right}: Diffusion policy with action-chunking.
}
\label{fig:experiments}
\end{figure*}

\iftoggle{arxiv}{
First, we}{We} conduct a series of experiments using the open-loop stable construction \Cref{const:stable} underlying \Cref{thm:main_stable}, demonstrating that our construction can be used as a benchmark for common behavior cloning pipelines. See \Cref{app:experiments} for details.  We visualize in \Cref{fig:experiments} the cost $\max_t \langle \be_1, \bx_t \rangle$ (a) for different checkpoints over the course of a single training run, (b) as a function of the number of rollout timesteps for different methods, and (c) on Diffusion policy with larger action-chunks. The experiments highlight several counterintuitive aspects of our construction:

\begin{enumerate}
\item The rollout cost increases although validation loss decreases throughout training. 
\item Random noise outperforms all policy learning methods and avoids exponential-in-time error, due to the E-IISS open-loop stability of the dynamics.

\item More complex techniques such as Diffusion Policy \citep{chi2023diffusion}, replica noising \citep{block2024provable}, and action-chunking \iftoggle{arxiv}{outperform regular behavior cloning.}{improve performance.} 
\end{enumerate}
Notably, action-chunking does not suffer from exponential error, which we attribute to the open-loop stability of each chunk. These results affirm our theory and suggest that imitators must be non-simple in order to avoid exponential error.
}
\iftoggle{arxiv}{\subsection{Three Stylized, Non-Simple Policies}}{}

\iftoggle{arxiv}{Next, we}{Here, we } provide an informal discussion of how non-simple policies can circumvent exponential compounding error. Each strategy can be applied to the construction underpinning our main theorem, \Cref{thm:main_stable}, and we show that the lower bounds based on that construction can be circumvented (see \Cref{label:app_non_vanilla}). In particular, this shows that \Cref{thm:non_simple_stochastic} is qualitatively unimprovable without appealing to a different construction. 
\iftoggle{colt}{In \Cref{sec:experiments}, we provide experimental evidence that more sophisticated policy parameterizations do indeed ameliorate compounding error. }{}


For simplicity, consider a  scalar analogue of the  construction of \Cref{sec:proof_intuition} with  dimension $d = m = 1$. We take $f_0(\bx,\bu) = \bu - \gst(\bx)$ , but the dynamics at steps $t \ge 1$ are 
\begin{align}
\bx_{t+1} = \xi \cdot \rho \bx_{t} + \bu_t,  \quad \bx_1 = \beps, \label{eq:gambler_dyn}
\end{align}
where $\rho > 1$ is known \iftoggle{arxiv}{to the learner}{}, but $\xi \in \{-1,1\}$ is an unknown sign. We begin in state $\bx_1 = \beps$,  representing some initial learner error. The goal is to select a policy $\pi$ which keeps $|\bx_t|$ small, without prior knowledge of the sign $\xi$.

\iftoggle{arxiv}{

}{}
Any smooth, deterministic policy $\pihat$ suffers from exponential compounding error on this problem: approximating $\pihat(\bx) \approx k \bx + \bu_0$, where $k \in \R$, around the origin $\bx \approx \bzero$,  we see that the dynamics compound with either $(\rho + k)^t$ or $(\rho - k)^t$, one of which must have an exponent of base $> 1$. This same pathology extends to simply-stochastic policies by considering \emph{differences} in trajectories, and coupling them so that their randomness cancels.  We further recall from \Cref{thm:unstable} that  in $d = \Omega(\log H)$-dimensions,  compounding error is unconditionally unavoidable.  
However, for the one-dimensional case described here, removing the constraints of either Markovianity, simple-stochasticity, or smoothness can evade this challenge.

\newcommand{\pibgr}{\pi_{\textsc{bgr}}}
\newcommand{\pics}{\pi_{\textsc{cs}}}

\iftoggle{arxiv}{\subsubsection{Action-Switching}}{\colorbold{Action-Switching.}} 
Consider a time-dependent strategy $\pi(\bx,t)$, \iftoggle{arxiv}{which alternates between}{with} $\pi(\bx,t) = -\rho \bx$ if $t$ is odd and $\pi(\bx,t) = \rho \bx$ if $t$ is even. By time-step $t = 3$, the system will have converged to state $\bx_3 = 0$, and will remain at rest there. 
This strategy uses time-dependence to hedge over dynamical uncertainty; time-dependence can be replaced by stochasticity as shown below. 

\iftoggle{arxiv}
{In addition, the only system unknown is $\xi$. Hence, one can consider a history dependent policy $\pi(\bx_{1:t},\bu_{1:t-1})$. Then $\pi$ can selects, for $t \ge 2$, $\bu_t = - \left(\frac{\bx_2 - \bu_1}{\bx_1}\right) \bx_t$. The above is always equal to $-\xi \rho \bx_t$, sending $\bx_t$ to zero for $t \ge 2$. This is essentially an adaptive control strategy: learning the unknown underlying dynamics to stabilize it (and it succeeds for $\rho$ unknown!).
}
{}

\iftoggle{arxiv}{\subsubsection{Benevolent Gambler's Ruin}}{\colorbold{Benevolent Gambler's Ruin}.}\label{sec:benevolent_ruin} 
\iftoggle{arxiv}{Gambler's Ruin is the classical paradox where  a gambler's wealth $W_t \in \R$ either doubles or is made zero at successive time steps $t\ge 1$ with equal probability $1/2$. In expectation, $\Exp[W_t] = W_1$ for all times $t$. But, with probability one, there exists a finite $t_\star$ for which the gambler loses their funds: $W_{t_\star} = 0$. Concretely, $\Pr[t_{\star} > t] = \Pr[W_{t+1} \ne 0] = 2^{-t} \to 0$. }
{Gambler's ruin is the classical paradox where a Gambler's wealth either doubles or is zeroed over rounds $t$; the expected wealth remains constant, but is non-zero with vanishing-in-$t$ probability. }
While gambling ultimately ruins the gambler in finite time, a stochastic policy can enact the same strategy to its benefit. 

For $t \ge 1$, we consider the Benevolent Gambler's Ruin policy $\pibgr(\bx)$ which selects $\rho \bx$ with probability $1/2$ and $-\rho \bx$ with the remaining probability. Crucially, such a policy's randomization depends on the state, and therefore is \textbf{not simply-stochastic}. This policy has an identically  zero, and therefore smooth, mean $\mean[\pibgr](\bx) \equiv \bzero$.

Under this policy and dynamics \Cref{eq:gambler_dyn},  $\Pr[\bx_{t+1} \ne 0] = 2^{-t} \to 0$. With remaining probability $|\bx_{t+1}| = (2\rho)^t\beps$, so in expectation, $\Exp[|\bx_{t+1}|] = \rho^t\beps$. The clipped error is then $\Exp[\min\{1,|\bx_{t+1}|\}] = 2^{-t}\min\{1,(2\rho)^t\beps\}$. By balancing these two terms over $t$, this quantity is only ever at most $\beps^{p}$,  where $p = \log \rho/\rho(2\rho) \in (0,1)$. In other words, the clipped expectation of state magnitude $\Exp[\min\{1,|\bx_{t+1}|\}] \le \min\{\beps^p,(2\rho)^t \beps\}$ grows at most \emph{sublinearly in the initial error} $\beps$. 

\iftoggle{arxiv}
{
\begin{remark}[Other notions of smoothness] We notice that (in dimension $1$) the Gambler's Ruin policy satisfies $W_1(\pibgr(\bx),\pibgr(\bx')) = W_2(\pibgr(\bx),\pibgr(\bx')) = ||\bx| - |\bx'||$, which is \emph{non-smooth} in the Wasserstein distance. This suggests that more stringent notions of smoothness may preclude these strategies.  
\end{remark}
}{}

\iftoggle{arxiv}{\subsubsection{Concentric Stabilization}}{\colorbold{Concentric Stabilization}.}\label{sec:concentric_stabilization}
 Concentric stabilization  swaps randomization/alternation for non-smoothness.
For \iftoggle{arxiv}{integers}{} $j  \in \Z$, define intervals $\mathcal{I}_j = ((2\rho)^{-2j},(2\rho)^{-2(j-1)}]$. For any $\bx \in \R \setminus \{0\}$, there exists a unique $j(\bx)$ such that $|\bx| \in \mathcal{I}_{j(\bx)}$.  We define the concentric stabilization policy, $\pics(\bx)$ which selects $\rho \bx$  if $j(\bx)$ is even,  and $-\rho \bx$ if $j(\bx)$ is odd. This policy is deterministic, but highly non-smooth as $\bx \to 0$.\iftoggle{arxiv}{\footnote{For $\bx$ bounded away from zero, it can be smoothed out via bump functions, with smoothness proportional to $1/|\bx|^2$}.}{}
\iftoggle{arxiv}{

}{}

 Let $f^{\pics}(\bx) = \xi \rho \bx + \pics(\bx)$ denote the induced closed-loop dynamics. 
For any $\bx_1$, and  either choice of $\xi \in \{-1,1\}$, consider the sequence induced by $\bx_{t+1} = f^{\pics}(\bx_t)$. One can compute that $\bx_t = 0$ for $t > 3$, and that $\max\{|\bx_1|,|\bx_2|,|\bx_3|\} \le (2\rho)^2 |\bx_1|$. By leveraging non-smoothness,  concentric stabilization limits the state's growth to at most a constant factor.


\section{Discussion}\label{sec:discussion}
We demonstrate that imitation learning in a continuous-action control system can exhibit exponential-in-horizon compounding error, even if the dynamics are stable in both open- and closed-loop. We  provide preliminary evidence that more complex policy parameterizations may be able to avoid this pitfall, and that expert data with good coverage avoids compounding error even under unstable  dynamics. There are many  exciting questions for future work: (a) When precisely can complex policies mitigate compounding error? (b) How can the expert provide optimal agents from suboptimal states? (c) What is the sample complexity of offline RL, e.g. from  \emph{suboptimal data}, in control systems. A final pressing question is understanding the benefits and limitations of online environment interaction (e.g. RL finetuning) in continuous-action control. 

Lastly, our work corroborates a provocative empirical finding from \cite{block2023butterfly}: what makes  behavior cloning challenging is not instability in the dynamics themselves, but rather instabilities arising from the closed-loop feedback between dynamics and an imperfect imitation policy. As shown in \Cref{sec:nonvanilla}, the design choices in the behavior cloning policy  (Diffusion, data-augmentation, action-chunking) lead to meaningful differences in performance; \cite{block2023butterfly} finds similarly that the choice of \emph{optimizer} can have similar effects on downstream performance as well. Thus, better understanding the interactions between the design space of algorithms, optimizers, and data is an important direction for future theoretical, empirical, and methodological work.


\newpage
\nonumsection{Acknowledgements}
The authors are deeply grateful to a number who contributed to shaping this manuscript: John Miller for the incredibly invaluable conversations regarding framing and scoping the formalism; Nati Srebro for the insightful early discussions about the results; Dylan Foster, for the initial encouragement to pursue a more systematic treatment of control-theoretic decision making; and Reese Pathak, for helping us navigate the  noise-free-regression literature.
\newline
\newline
DP and AJ acknowledge support from the Office of Naval Research under ONR grant N00014-23-1-2299. DP additionally acknowledges support from a MathWorks Research Fellowship.
\newpage

\newpage

\addcontentsline{toc}{section}{References}
\tableofcontents 
\newpage
\bibliographystyle{plainnat}
\bibliography{refs}
\newpage 

\hypertarget{targ_addendum}{\nonumpart{Addendum: Minimax Formulations}}

The next three sections that follow are intended for readers either familiar with, or curious about, the statistical learning and decision-making literature. The main focus is framing all results in the language of minimax risks which, while natural and expedient to those familiar with statistical learning, may be cumbersome for those less familiar (hence, the decision to defer these sections).

\Cref{sec:minmax} introduces a systematic treatment of compounding error in imitation learning problems via minimax risks. Compounding error is then the distuation when the minimax risk under the expert distribution (we focus on $\minbctrain$) and the risk under evaluation on a cost ($\mincost$ in expectation, and $\minprobcost$ in probability) differ by large amounts. This section also introduces language for minimax risks of standard supervised learning problems with the $L_2$ loss.

\Cref{sec:minmax_lb} provides more general, more detailed versions of the results in \Cref{sec:results}, stated in terms of the minimax risks from the preceding section. The idea is to show that  \textbf{any $L_2$ regression problem} satisfying the appropriate regularity conditions can be \textbf{embedded} into an imitation learning problem in which compounding error occurs. The results in \Cref{sec:results} can be instantiated by using \Cref{prop:typical_true} in the previous section, which guarantees that for any rational $q \in \mathbb{Q}$, there exist a sufficiently regular $L_2$-regression problem whose error decays like $n^{-q}$. This section also  describes additional features and strengthenings of the results.

Finally, \Cref{sec:schematic} describes the formal, general schematic used to  prove of the aforementioned results. It then briefly discusses how this schematic is specialized to each particular result. The proof of the main results for simple policies, \Cref{thm:stable_detailed,thm:constant_prob}, adopts the strategy already described in \Cref{sec:proof_intuition}. \Cref{thm:non_simple_stochastic}, our guarantee for non-simple policies, uses the same construction but a somewhat more intricate proof strategy, and the lower bound with unstable dynamics, \Cref{thm:unstable}, uses a construction based on random orthogonal matrices.

\newcommand{\Rtrajltwo}{\Risk_{\mathrm{traj},L_2}}
\newcommand{\Rltwocost}{\Risk_{\cost,L_2}}

\newcommand{\Cvan}{\cC_{\mathrm{vanish}}}
\newcommand{\Rtrajlp}[1][p]{\Risk_{\mathrm{traj},L_{#1}}}
\newcommand{\minmaxltwo}{\minmax_{\mathrm{eval},L_2}}
\newcommand{\minmaxunderlp}[1][p]{\underbar{\minmax}_{\,\mathrm{eval},L_{#1}}}
 \newcommand{\minmaxlp}[1][p]{{\minmax}_{\,\mathrm{eval},L_{#1}}}

\section{Minimax Imitation Learning Risks } \label{sec:minmax}

 In this section, we introduce a more systematic formulation of the results stated in \Cref{sec:results}. We adopt the language of \emph{minimax risks}, which cast statistical decision problems as zero-sum games between learning algorithms (the ``$\min$ player'') and adversaries selecting the unknown problem parameter (the $\max$ player). The cost (or negative payoff) in the game is the risk function to be minimized by the learner. The minimax risk thus characterizes the best attainable expected value of the function, over all randomness involved, on the worst-case problem instance. For a comprehensive treatment of minimax risks in statistical estimation and decision making,  consult the works  \cite{wainwright2019high,geer2000empirical,gyorfi2006distribution,tsybakov1997nonparametric}, and references therein. Furthermore, all proofs in this section are deferred to \Cref{app:minmax}.

The remainder of the section has the following organization. First,  in \Cref{sec:BC_minimax_risks}, we introduce the standard formalism of the minimax risk specialized to IL problems. Next, we introduce a notion of ``in-probability'' minimax risk in \Cref{ssec:in_prob_risks}, which gives provides a more granular characteriztion of the compounding error behavior. Our lower bounds follow by embedding supervised learning problems. To this end, we introduce minimax risks for standard supervised learning problems in \Cref{sec:reg_minimax}. This includes the stipulation of an important \emph{typicality assumption}, which we show holds for a very general family of regression problems (\Cref{prop:typical_true}).  

\subsection{IL Minimax Risks}\label{sec:BC_minimax_risks}
 Recall that a  $(\R^d,\R^m)$-\textbf{IL problem family} is a tuple $(\inst,\Pnot)$ of instances $\inst = \{(\pi^\star,f)\}$ with $\Xspace = \R^d$ and $\Uspace = \R^m$, and initial distribution $\Pnot$ on $\Xspace$. .

 \begin{definition}[IL  Minimax Risk]\label{defn:minmax} Let $(\inst,\Dist)$ be an $(\R^d,\R^m)$-IL problem family. Further, let $\bbA$ be a class of IL estimation algorithms mapping samples $\Samp$ to (distributions over) policies.  For $n \in \N$ trajectories and horizon $H \in \N$, the minimax risk of $\bbA$ under a risk function $\Risk(\pihat;\pist,f,\Pnot,H)$ is 
\begin{align}
    \minmax^{\bbA}(n,\Risk; \inst,\Pnot,H):= \inf_{\est \in \bbA} \sup_{(\pist,f) \in \inst}\Eshorthand\left[\Risk(\pihat;\pi_{\star},f,\Pnot,H)\right].\label{eq:defn:minmax}
\end{align}
\end{definition}
As described above, the minimax risk admits a game-theoretic interpretation: a learner's move is their selection of algorithm $\Alg$, and an \emph{adversary} selects an instance $(\pist,f) \in \cI$. The learner's penalty is then the expected risk over all sources of randomness $\Exp_{\Samp \sim [\pist,f,\Dist] }\Exp_{\pihat \sim \est(\Samp)}\left[\Risk(\pihat;\pi_{\star},f,\Pnot,H)\right]$. Minimax risk thus measures the minimal penalty the learner can suffer in such a game. Notice that our formalism treats $\Dist$ as fixed, which can be interpreted as given the learner foreknowledge of the initial state distribution.  This foreknowledge only makes lower bounds stronger. 

 Thoughout, we adopt the shorthand for validation and evaluation risks:
\begin{equation}
\begin{aligned}
     \minbctrain^{\bbA}(n;\inst,\Dist, H) &:= \minmax^{\bbA}(n,\RtrainLp[2]; \inst,\Pnot,H), \\
      \minmax_{\cost}^{\bbA}(n;\inst,\Dist,H) &:= \minmax^{\bbA}(n,\Rsubcost; \inst,\Pnot,H) \label{eq:minimax_shorthand}
\end{aligned}
\end{equation}

While most of our lower bounds focus on restricted algorithm classes $\bbA$, some  lower bounds: they hold even without restriction to a particular class of algorithms.
\begin{definition}[Unrestricted Minimax Risk]\label{defn:unrestricted} We define the \emph{unrestricted} minimax risk $\minmax(n,\Risk; \inst,\Pnot,H)$ as  $\minmax^{\bbA_{\star}}(n,\Risk; \inst,\Pnot,H)$, where $\bbA_{\star}$ contains all IL algorithms $\Alg$ mapping $\Samp$ to (distributions over) policies $\pihat$. We even include in $\bbA_{\star}$ algorithms $\Alg$ which can return a $\pihat$ for which $\pihat$ may depend on time-step $t$ and past; i.e. $\pihat $ maps $(t,\bx_{1:t},\bu_{1:t-1})$ to distributions over $\bu_t$. 
We define the unrestricted minimax validation and evaluation risks $\minbctrain$ and $\minmax_{\cost}$ by direct analogy to \Cref{eq:minimax_shorthand}. 
\end{definition}
Lower bounds against unrestricted algorithm classes are often called \emph{information-theoretic}, in that they leverage the learners incomplete information about the ground-truth problem instance moreso than any algorithmic limitation imposed on the learner (or on the policies $\pihat$).

\subsection{In-Probability Minimax Risks}\label{ssec:in_prob_risks}

It may be objected that lower bounds on expected costs may be misleading, because compounding error may be large on rare events (as, for example, observed in the case of benevolent gambler's ruin in \Cref{sec:benevolent_ruin}). In what follows, we present a fixed-cost, in-probability risk, $\minprob$, which leads to more stringent lower bounds that rule out rare-event compounding error. We shall also show that this notion implies lower-bounds on the  $\minmax_{\cost}$ defined above.


It is most convenient to state our definition of in-probability risk for a $\cost: \Xspace^H \times \Uspace^H \to \R$  that vanishes on expert trajectories:

\begin{definition}\label{defn:c_vanish} We say a $\cost: \Xspace^H \times \Uspace^H \to \R$ ``vanishes on $(\cI,\Dist)$'' if for all $ (\pist,f) \in \cI$,  
\begin{align*}
\Pr_{\pist,f,D}[\cost(\bx_{1:H},\bu_{1:H}) = 0 ] = 1.
\end{align*}
 We define the set of such costs as $\Cvan(\cI,\Dist)$. 
\end{definition}
 We define a fixed-cost  ``in-probability risk'' as the  probability $p$ as the smallest $\epsilon$ such that the cumulative probability over exceeding $\epsilon$, under all randomness of validation and evaluation of the policy, is at most $p$.

\begin{definition}[In-Probability Risk]\label{def:in_prob_risk} Given $n \ge 1$ and $p \in (0,1]$, and a $\cost \in\Cvan(\cI,\Dist) $, we define the in-probability risk as 
\begin{align*}
\minprobcost^{\bbA}(n,\delta;\inst,\Dist, H) := \inf \left \{\epsilon : \inf_{\est \in \bbA}\sup_{(\pist,f) \in (\cI,\Dist)} \Eshorthand[\Pr_{\pihat,f,D}[\cost(\bx_{1:H},\bu_{1:H}) \ge \epsilon]] \le \delta\right\}.
\end{align*}
We define unrestricted minimax risks by analogy to \Cref{defn:unrestricted}.
\end{definition}

We remark that the above risks are  are equivalent to quantile risks considered in recent work in the statistical learning community \citep{el2024minimax,ma2024high}. However, while those works are concerned with establishes larger lower bounds for estimation with  high-probability guarantees, the focus in this work is simply showing that large error occurs with constant probability.

Note that, by Markov's inequality, it holds that (c.f.  \Cref{prop:risk_comparison}(c))
\begin{align}
\forall \cost \in \Cvan(\cI,\Dist), \quad \delta \cdot \minprobcost^{\bbA}(n,\delta;\inst,\Dist, H) \le \mincost^{\bbA}(n;\inst,\Dist, H),
\end{align}
Thus, a lower bound $\minprobcost$ suffices for lower bounds on $\mincost$. Further variants of the above risks are discussed in \Cref{sec:risk_comparison}. 



\subsection{Embedding Regression Problems}\label{sec:reg_minimax}
We will derive lower bounds on the minimax risk by embedding in more standard supervised regression problems over classes of functions $\cG$, which can be viewed as $1$-step IL problems. 

\begin{definition}[Supervised Learning Minimax Risks] \label{defn:minsl}  A-$\R^k$ regression problem family is a pair $(\cG,\Dreg)$ consisting of a distribution $\Dreg$ on $\R^{k}$ and a class of scalar-valued functions $\cG = \{g: \R^{k} \to \R\}$.  Given such a regression problem family $(\cG,\Dreg)$, its minimax risk is
\begin{align}
    \minsl(n;\Gclass,\Dreg) = \inf_{\estreg}\sup_{\gst \in \cG} \Exp_{\sampreg}\Exp_{\hat{g} = \estreg(\sampreg)}\left(\Exp_{\bz \sim \Dreg}[|\hat{g}(\bz) - \gst(\bz)|^2]\right)^{1/2}.
\end{align}
where $\Exp_{\sampreg}$ denotes expectation over samples $\sampreg = (\bz^{(i)},\gst(\bz^{(i)}))_{1 \le i \le n}$ for $\bz^{(i)} \iidsim \Dreg$, and $\estreg$ is any measurable function mapping $\sampreg$ to  functions $\ghat: \R^{k} \to \R$.
 Given $p \in (0,1]$, we define an in-probability risk 
\begin{align*}
\minslprob(n,\delta;\cG,\Dreg) := \inf \left\{\epsilon : \inf_{\estreg}\sup_{\gst \in \cG} \Exp_{\sampreg}\Exp_{\hat{g} = \estreg(\sampreg)}\Pr_{\bz \sim \Dreg,\bhaty \sim \ghat(\bz)}[|\bhaty - \gst(\bz)| > \epsilon] \le \delta\right\}.
\end{align*}  
\end{definition}
\begin{remark} Note that, in full generality, both $\estreg$ may be randomized, and the function $\hat g$ may be a stochastic function of its input: $\bhaty \sim \hat g(\bz)$. However, for the $L_2$ regression risk, Jensen's inequality implies that randomized regression estimators do not improve the minimax regression risk. 
\end{remark}
By a Chebyshev's  inequality argument,  we always have the inequality
\begin{align}
\minsl(n;\Gclass,\Dreg) \ge \sqrt{\delta}\minslprob(n,\delta;\cG,\Dreg). \label{eq:Minsl_bound_prob}
\end{align}

\subsection{Regularity and ``Typicality'' Conditions for Regression}
Because we consider imitation of an {deterministic expert}, the regression problems considered are noiseless. This is often referred to the \emph{interpolation setting} in statistical learning. For further discussion, see e.g. \cite{kohler2013optimal} and the references therein. 

First, we codify some more standard regularity conditions
\begin{definition}[Regular Regression Instances]\label{defn:regular_instances} We say $(\cG,\Dreg)$ is \emph{$R$-bounded} if with probability one over $\bz \sim \Dreg$, $\|\bz\| \le R$, and for all $\bz:\|\bz\| \le R$, $|g(\bz)| \le R$. We say $(\cG,\Dreg)$ is $(R,L,M)$-regular of if each $(g,\Dreg),g \in \cG$ is $R$-bounded, and $g$ are $L$-Lipschitz and $M$-smooth, and the class $\cG$ is closed under convex combination.  
\end{definition}

Next, recall that we must show compounding error occurs with good probability; otherwise, if large errors occur with low probability, then for the losses bounded in $[0,1]$  in $\Clip$, the contributions of these errors are insignificant in expectation. To this end, we introduce technical condition ensuring that if the minimax risk scales like $\beps_n$, then a similar  lower bound on the risk holds with constant probability as well. Because it is common to derive in-expectation lower bounds from in-probability ones (see, e.g. \cite{tsybakov1997nonparametric}), we denote this condition ``typical''-ity.


\begin{condition}[Typical Problem Class] \label{asm:conc} Let $\kappa, \delta \in (0,1)$. We say that $(\cG,\Dreg)$ is  $(\kappa,\delta)$-\typical if
\begin{align}
\minslprob(n,\delta;\cG,\Dreg) \ge \kappa \minsl(n;\cG,\Dreg), \quad \forall n \ge 1.
\end{align}
\end{condition}
Up to $\kappa$ and $\delta$, \Cref{asm:conc} is the converse of the inequality \Cref{eq:Minsl_bound_prob}. Finally, we say $\cG$ is \emph{convex} if $g_1,g_2 \in \cG$ implies $\alpha g_1 + (1-\alpha) g_2 \in \cG$ for any $\alpha \in [0,1]$. 
In \Cref{app:typical_true}, we verify that a large, classical families of regression problems are smooth, typical, and realize any desired fractional rate of estimation. Specifically, we establish the following result.
\begin{proposition}\label{prop:typical_true} For any integers $s \ge 2, k\ge 1$, there exist constants $\kappa,\delta \in (0,1)$ and $C,C' > 0$ depending only on $s$ and $k$, and an $\R^k$ -regression problem family $(\cG,\Dreg)$ which is $(\kappa,\delta)$-typical, $(1,1,1)$-regular, such that $\cG$ is convex and for all $n \ge 1$, 
\begin{align}
C \minsl(n;\cG,\Dist) \le n^{-s/k} \le C' \minsl(n;\cG,\Dist).
\end{align}
\end{proposition}




\section{Minimax Lower Bounds for Imitiation Learning}\label{sec:minmax_lb}
This section presents detailed statements of our lower bounds, stated in the language of minimax risks developed in \Cref{sec:minmax}. These results demonstrate that compounding error is a phenomenon that occurs independent of the statistical difficulty of minimizing the training risk, in the following sense that any typical statistical learning problem (\Cref{asm:conc}) can be embedded into a IL problem with exponential compounding erorr. 

More specifically,  we assume we are given an $\R^k$-regression problem family $(\cG,\Dreg)$ which is  $(\kappaconc,\deltaconc)$-\typical (\Cref{asm:conc}), and such that $\Dreg$ is $1$-bounded (\Cref{defn:regular_instances}). We use the following shorthand for the minimax risk of this regression problem
\begin{align*}
\beps_n := \minsl(n;\Gclass,\Dreg ).
\end{align*} 
For the first two results, we also assume $(\Gclass,\Dreg)$ is $(1,1,1)$-regular (\Cref{defn:regular_instances}) and $\Gclass$ is convex.
We will then show that  such classes can be embedded into Behavior Cloning problems such that 
\begin{itemize}
\item[(a)] the restricted and unrestricted minimax training risks coincide, and are close to the supervised learning minimax risk $\minbctrain^{\bbA}(n;\cI,\Dist,H) =\minbctrain(n;\cI,\Dist,H)\approx \beps_n$.
\item[(b)] There exists a $\cost \in \Clip$ such that the in-probability risks are considerably large. Specifically, $\minprobcost^{\bbA}(n;\cI,\Dist,H) \gg \beps_n$,  and often $\minprobcost(n,\Omega(\delta);\cI,\Dist,H) \ge \exp(\Omega(H))\beps_n$. 
\end{itemize}

We proceed to state three formal lower bounds. First, \Cref{thm:stable_detailed} (\Cref{sec:minmax_vanilla}) demonstrates that the class of simple IL algorithms (\Cref{defn:vanilla}) with smooth means and simply-stochastic noise incur exponential-in-$H$ compoundinger error. Next, \Cref{thm:stable_general_noise} (\Cref{sec:minmax_anticonc}) shows that exponential-in-$H$ compounding occurs even for a much larger class of algorithms with anti-concentrated noise (\Cref{defn:anti-concentrated_policy}), but this is capped to a rate of $\beps_n^{1-\Omega(1)}$. The illustrative benevolvent gambler's ruin policy in \Cref{sec:benevolent_ruin} provides weak evidence that non-simply stochastic may indeed be able to enjoy at most $\beps_n^{1-\Omega(1)}$ error due to clever randomization.  Finally, \Cref{thm:unstable_detailed} (\Cref{sec:minmax_unstable}), shows that   for problem families where the expert-dynamics pairs $(\pist,f)$ are closed-loop E-IISS, but the open-loop dynamics may be unstable, the unrestricted minimax rates exhibits exponential-in-$H$ compounding error.

Before continuing to the statements of these results, we describe some additional further features of the lower bounds that follow. 

\colorpar{Proper learning is optimal on the expert distribution.} In all results that follow, we show that proper learning is \emph{optimal} from the perspective of minimizing the loss under the distribution of the expert. Hence, while improperness may be of benefit when the policy is deployed, it confers no benefit when imitating expert data. The exact optimality of proper algorithms requires our consideration  of $L_2$ expert distribution error (see \Cref{rem:why_l_two} for discussion).

\colorpar{Controllability.} In addition to the all the regularity conditions (smoothness, boundedness, stability) promise above, we will also ensure that our constructions satisfy yet another desirable property: the dynamics $f \in \cF(\cI)$ are $1$-step controllable.

\begin{definition} Let $f:\Xspace \times \Uspace \to \Xspace$ be a dynamical map. We say that $f$ is $C$-one-step controllable if, for all $\bx,\bx' \in \Xspace$, there exists some $\bu \in \Uspace$ for which $f(\bx,\bu) = \bx'$, and $\|\bu\| \le C(1+\|\bx\| +\|\bx'\|)$. We say that $f$ is $O(1)$-one-step-controllable if the above holds for some universal constant $C = O(1)$.
\end{definition}
In fact, with a little additional effort, one can show that for the dynamics $f$ in our construction, the equation $\bx' = f(\bx,\bu)$ admits a unique solution $\bu^\star(\bx',\bx)$ for each $\bx',\bx \in \Xspace$, and $\bu^\star$ depends smoothly on $(\bx',\bx)$. This means that neither a lack of controllability, nor the an innability to control the system smoothly, are to blame for the lower bounds. 

\colorpar{Horizon scale invariance.} All the bounds that follow also hold when the cost function, $\cost$, is the maximum over time steps $H$ over $1$-Lipschitz costs, rather than the sum. This gives a normalization of the total cost which is horizon independent, whereas the sum of costs typically grows linearly in $H$. See  \Cref{sec:risk_comparison} for further discussion.

\colorpar{Longer horizon demonstrations do not help.} Each of the lower bounds hold in the regime where the learner has access to a sample $S_{n,H'}$, where $H' \ge H$ is any \emph{arbitrarily long} problem horizon (even infinitely long $H' = \infty$, measure-theoretic considerations permitting). This rules out the possibility that longer problem horizons may make the behavior cloning problem easier.

\subsection{Minimax Compounding Error for IL with Simple Policies}\label{sec:minmax_vanilla}
This section states our lower bound against simple IL algorithms ($\Avan$,  \Cref{defn:vanilla}), which we recall are those algorithms which return simply-stochastic policies with smooth and Lipschitz means. Our lower bounds follow from embedding regular, \typical regression problems satisfying the assumption that follows. 
\begin{restatable}{assumption}{asmreg}\label{asm:stable_reg} We assume that $(\cG,\Dreg)$ is $(1,1,1)$-regular (recall \Cref{defn:regular_instances}) and is  $(\kappaconc,\deltaconc)$-\typical (\Cref{asm:conc}), and that $\cG$ is convex. In particular, the classes of regression problems whose existence is guaranteed by \Cref{prop:typical_true} all satisfy this condition. 
\end{restatable}

We now state the main theorem:
\begin{thmmod}{thm:main_stable}{.A}{Lower Bound for Stable Systems, Detailed Version}\label{thm:stable_detailed}  
Let $0 < c \le 1 \le C$ be universal constants, let system dimension  $k \in \N$, and consider any $k$-dimensional  regression problem family $(\cG,\Dreg)$ satisfying \Cref{asm:stable_reg}. Then, for $d = k + 2$, there is a $(d,d)$-dimensional IL problem family $(\cI,\Dist)$ which is $\BigOh{1}$-regular (\Cref{def:regular}), and cost function $\cost \in \Clip \cap \Cvan(\cI,\Dist)$, such that for any $L,M \ge C$,  the the class of estimators $\bbA = \Avan(L,M)$  contains $\Aproper(\cI)$, and satisfies the following:
\begin{equation}
\begin{aligned}
    & \minbctrain(n;\inst,\Dist,H) =  \minbctrain^{\bbA}(n;\inst,\Dist,H) = \minbctrain^{\Aproper(\cI)}(n;\inst,\Dist,H) \\
    &\qquad\in \left[\frac{\tau}{2}\beps_n, \quad \beps_{n/3} + C e^{-cn}\right] \label{eq:bc_train_stable_detailed}
    \end{aligned}
    \end{equation}
    and
    \begin{align}
& \minprobcost^{\bbA}\left(n,c \delta ;\cI,\Dist,H\right) \ge c\min\left\{\beps_n  \cdot \kappa \left(\frac{17}{16}\right)^{H-2}  , ~ \frac{1}{L^2 M d} \right\}, \label{eq:bc_test_stable_detailed}
\end{align}
where $\tau,c,C$ can be chosen to be universal constants, and $\delta,\kappa$ are as in \Cref{asm:stable_reg}. Finally, for every $(\pi,f) \in \cI$, are both $f$ and $(\pi,f)$ are $(C,\rho)$-E-IISS, where $\rho \in (0,1)$ is a universal constant strictly less than $1$, and $f$ is $O(1)$-one-step-controllable.
\end{thmmod} 

The proof of \Cref{thm:stable_detailed} is given in \Cref{app:stable}, based on the high-level schematic in \Cref{sec:schematic}. The result consists of four statements. First, the minimax expert distribution minimax risk of the IL problem is, up to constants, exponentially small additive terms, and constant scalings of the sample size, the same as that of the embedded regression problem. Second, the minimax rates of proper IL algorithms, unrestricted IL algorithms, and simple algorithms are identical when measured in terms of the expert distribution (note: the equivalence of the first two implies equivalence to the third, due to $\Aproper(\cI) \subset \Avan(O(1)) \subset \{\text{unrestricted algorithms}\}$). The third is that the in-probability minimax risk of the IL problem is exponentially-in-$H$ larger.  

The final statement checks all desired regularity conditions. As mentioned above, $\cost$ and $\Dist$ are fixed for all $n$ and $H$; thus, neither unsupervised knowledge of the initial state distribution nor knowledge of the cost (as in, say, an offline RL framework) suffice to avoid exponentially compounding error. 

\begin{proof}[Deriving \Cref{thm:main_stable,thm:constant_prob}] \Cref{thm:main_stable,thm:constant_prob} are both readily derived from \Cref{thm:stable_detailed}. By \Cref{prop:typical_true}, for each $s,k$, we can take an $\R^k$ regression class $\cG$ for which $\beps_n = \epsilon^{-s/k}$, and  $\kappa$, $\delta$ to be constants depending only on $(s,k)$, and $d = O(k)$. Further, $\beps_n \gg \exp(-cn)$, but to constants. Thus, \Cref{eq:bc_train_stable_detailed} implies \iftoggle{arxiv}{\Cref{eq:RtrajL2p:bound}}{(a)} of \Cref{thm:main_stable}, whereas \Cref{eq:bc_test_stable_detailed} implies \Cref{thm:constant_prob}. Finally, \Cref{thm:constant_prob} implies \iftoggle{arxiv}{\Cref{eq:bad_lb}}{(b)} in \Cref{thm:main_stable} via the Markov's inequality statement, \Cref{eq:Minsl_bound_prob}.
\end{proof}

\subsection{Minimax Compounding for Smooth, Non-Simply-Stochastic Policies}\label{sec:minmax_anticonc}
Generalizing from simply-stochastic policies, we now establish lower bounds against algorithms which return policies that need not be simply stochastic, but satisfy a mild and broadly applicable anti-concentration condition. As noted above, the lower bound is somewhat weaker: compounding error occurs, but only up until an $\beps_n^{1-\Theta(1)}$ threshold. Moreover, compounding error is measured in $L_2$, which exascerbates the contribution of heavy-tailed errors. Specifically, for $\cost \in \Cvan(\cI,\Dist)$, we define an $L_2$-analogue of $\mincost$, namely:
\begin{equation}
\begin{aligned}
 \Rlpcost[2](\pihat;\pi^\star,f,\Pnot,H) &:= \Exp_{\pihat,f_{g,\xi},\Dist} \left[  |\cost(\bx_{1:H},\bu_{1:H})|^2\right]^{1/2}, \quad  \cost \in \Cvan(\cI,\Dist)\\
 \mincostlp[2](n;\cI,\Dist,H) &:= \minmax^{\bbA}\left(n,\Rlpcost;\cI,\Dist,H\right) \label{eq:rlp2}
 \end{aligned}
 \end{equation}

 These differences aside, our lower bounds shows that the benevolent gambler's ruin strategy of \Cref{sec:benevolent_ruin} is qualitatively unimprovable in general. Our lower bound pertains to algorithms which return policies statisfying a mild anti-concentration condition, stated first for general random variables.
\begin{restatable}[Quantitative Anti-Concentration]{definition}{defanticonc} Let $\alpha,p \in (0,1]$. We say that a scalar random variable $Z$ is $(\alpha,p)$-anti-concentrated if it satisfies
\begin{align}
\Pr[|Z-\Exp[Z]| \ge \alpha\Exp[|Z-\Exp[Z]|^{2}]^{1/2}] \ge p.
\end{align}
We say that a random vector $\bz \in \R^d$ if $(c,p)$-anti-concentrated if $\langle \bv, \bz \rangle$ is $(\alpha,p)$-anti-concentrated for any vector $\bv \in \R^d$ (equivalently, for any unit vector).  
\end{restatable}
Importantly, our definition of anti-concentration is relative to the random variable's own variance. In particular, \textbf{deterministic} random variables $(1,1)$-anti-concentrated according to the above definition.  Next, we extend our notion of anti-concentration to policies.

\begin{restatable}[Anti-Concentrated Policy]{definition}{defanticoncpol}\label{defn:anti-concentrated_policy} We say that a policy $\pi$ is $(\alpha,p)$ anti-concentrated if, for any $\bx,\bx' \in \R^d$, there exists a coupling  $P(\bx,\bx')$ of $\pi(\bx),\pi(\bx')$\footnote{Recall that a coupling of  $\pi(\bx),\pi(\bx')$ is a joint distribution over $(\bu,\bu')$ with marginals $\bu \sim \pi(\bx)$ and $\bu' \sim \pi(\bx')$.} such that if $(\bu,\bu') \sim P(\bx,\bx')$, the random vector $\bu - \bu'$ is $(\alpha,p)$-anti-concentrated. 
\end{restatable}
The ability to choose any coupling $P$ implies that anti-concentration holds for very general classes of policies, including: all simply-stochastic policies (in particular, determinisitic policies), all Gaussian policies $\pi(\bx) = \cN(\mu(\bx),\bSigma(\bx))$, and  policies which are mixtures of anti-concentrated policies (e.g. Gaussian mixtures or mixtures of deterministic policies) with components of constant-magnitude probability. In particular, the benevolent gambler's ruin policy (\Cref{sec:benevolent_ruin}) is anti-concentrated.  We verify these claims in \Cref{sec:ex:anti_concentrated}. 

\paragraph{Generalized Smooth Policies}
Motivated by these examples, we define the class of ``generalized smooth policies'' as those which are anti-concentrated, and which have Lipschitz and smooth means.
\begin{definition}[Generalized Smooth Policies]\label{defn:gen_smooth} Let $\Areason(L,M,\alpha,p)$ denote the class of algorithms which, with probability one, return stochastic, Markovian policies $\pi$ for which $\mean[\pi](\bx)$ is $L$-Lipschitz and $M$-smooth, and $\pi$ is $(\alpha,p)$-anti-concentrated.
\end{definition}

We are now ready to state our main result. Recall the $L_2$ minimax risks defined in \Cref{eq:rlp2} above. We also establish a convenient asymptotic notation.
\begin{definition}[$\polyost$-notation]\label{defn:polyost}
Given $b_1,b_2,\dots \le 1$, we use the notation $a = \polyost(b_1,b_2,\dots,b_k)$ to denote that $a \le c_1 (b_1\cdot b_2 \cdot b_k)^{c_2}$, $c_1$ is a sufficiently small universal constant, and $c_2$ a sufficiently large universal constant. 
\end{definition}

Our main theorem is as follows. 

\begin{thmmod}{thm:non_simple_stochastic}{.A}{Lower Bound for Non-Simply Stochastic Systems, Detailed Version}\label{thm:stable_general_noise} Consider the setting of \cref{thm:stable_detailed} with $d = k+2$, and let  $(\cI,\Dist)$ be the corresponding problem family from that theorem. Further, recall $\beps_n := \minsl(n;\cG,\Dreg)$.  For $L,M \ge 1$ and $\alpha, p \in (0,1]$, now consider the class of algorithms $\bbA = \Areason(L,M,\alpha,p)$. Then \Cref{eq:bc_train_stable_detailed} still applies to this choice of $\bbA$. Moreover, suppose that $\beps_n \le \polyost(1/L,1/M,1/d,\alpha,p,\kappa,\delta)$. Then, for all $n \ge 1$,
\begin{align}
\mincostlp[2]^{\bbA}(n;\cI,\Dist,H) \ge  c \kappa \cdot \delta \cdot \min\left\{ \beps_n \cdot 1.05^{H-2}, \beps_n^{1-\frac{1}{C'(1+\log(1/(\alpha p)))}}\right\}. \label{eq:anticonc_compounding}
\end{align}
\end{thmmod}

The proof of \Cref{thm:stable_general_noise} is given in \Cref{app:general_noise}, again based on the high-level schematic in \Cref{sec:schematic}. In words, this result shows that the same construction from \Cref{thm:stable_detailed} provides a challenging distribution for non-simply-stochastic, but exponential-in-$H$ compounding error occurs only up to a threshold which is $\beps_n^{1-\Omega(1)}$. Note that, because the construction is the same, \Cref{eq:bc_train_stable_detailed} with $\bbA = \Avan(L,M)$ implies the same for $\bbA = \Areason(L,M,\alpha,p)$, as the latter is a large algorithm class. 

\begin{proof}[Deriving \Cref{thm:non_simple_stochastic}] \Cref{thm:non_simple_stochastic} follows from \Cref{thm:stable_general_noise} exactly the same way as \Cref{thm:main_stable,thm:constant_prob} follow from \Cref{thm:stable_detailed}. That is, we by \Cref{prop:typical_true}, for each $s,k$, we can use an $\R^k$ regression class $\cG$ for which $\beps_n = \epsilon^{-s/k}$, and  $\kappa$, $\delta$ to be constants depending only on $(s,k)$, and $d = O(k)$. \Cref{thm:stable_general_noise}  gives an in-probabiliy bound, whilst \Cref{eq:Minsl_bound_prob} converts this to a bound in expectation. 
\end{proof}

\begin{remark}[Is anti-concentration necessary?] The anti-concentration requirement is a consequence of our choice to define policy smoothness in terms of its mean. Without this condition, policies which appear highly non-smooth with constant probability can be ``smoothed'' by adding low-probability, large-mass components to balance them out the means. We conjecture that by replacing mean-smoothness with a more careful notion of smoothness, based either on smoothness of densities (provided dominating measures exists), or based on classes of smooth test functions, the anti-concentration can be removed from the class $\Areason$. 
\end{remark}

\subsection{Minimax Compounding Error for Unstable Dynamics}\label{sec:minmax_unstable}
We round out the section by proving entirely unconditional lower bounds against compounding error when the dynamics are permitted to be smooth and Lipschitz, but unstable. 
\begin{thmmod}{thm:unstable}{.A}{Lower Bound with Unstable Dynamics, Detailed Version}\label{thm:unstable_detailed} Consider a  $(\kappa,\delta)$-\typical $\R^k$-regression problem family $(\cG,\Dreg)$, and let $\beps_n := \minsl(n;\cG,\Dreg)$. For any integer $d \ge k$, and any $\rho > 1$,
 there is an $(\R^d,\R^d)$-IL problem family $(\inst, \Dist)$ and $\cost \in \Clip \cap \Cvan(\cI,\Dist)$ , such that for all $2 \le H \le \frac{1}{2}e^{d(1-\rho^{-1})^2/2}$, 
\begin{align}
&\minbctrain(n;\inst,\Dist,H) = \beps_n  \label{eq:unstable_same_min}\\
&\minprobcost\left(n,\frac{\delta}{2};\inst,\Dist,H\right) \ge \min \left\{  \kappa \cdot \beps_n \cdot \rho^{(H-1)/2}, \, c_0 \right\} \label{eq:unstable_compounding}
\end{align}
Above, $c_0$ is a universal constant. 
Moreover, the construction ensures that each $(f,\pist) \in \cI$ is $(0,1)$-E-IISS,  and if $(\cG,\Dreg)$ is $(R,L,M)$-regular, then $(\cI,\Dist)$ is $(R,L',M')$-regular for $L' =  O(L+\rho)$ and $M' = O(M+L+\rho)$, and each $f \in \cF(\cI)$ is $O(L+\rho)$-one-step-controllable.  
\end{thmmod}
Again, the theorem is based on the schematic outlined in (\Cref{sec:schematic}), with the formal proof deferred to \Cref{app:unstable}. Note that, in the above theorem, $\minprob$ is the \emph{unrestricted minimax risk} (\Cref{defn:unrestricted}). That is, even history-dependent, non-smooth policies with arbitrary stochastic policies fail to elude the $\exp(H)$ compounding error.
\begin{proof}[Deriving \Cref{thm:unstable}] As with the proofs of \Cref{thm:main_stable,thm:constant_prob,thm:non_simple_stochastic} above , the result follows by instantiating \Cref{thm:unstable_detailed} with \Cref{prop:typical_true}. Details are the same as in the other cases. 
\end{proof}

\section{Proof Schematic}\label{sec:schematic}

All three lower bounds, \Cref{thm:stable_detailed,thm:stable_general_noise,thm:unstable_detailed}, all follow from the same schematic. We describe this schematic here, and then remark on how the arguments specialize at the end of the section. Throughout, fix  $H \in \N.$ Let $(\cG,\Dreg)$ be $\R^k$-regression problem family, and consider an $(\R^d,\R^m)$-IL problem families $(\cI,\Dist)$, where the instances take the form 
\begin{align}
\inst = \{(\pi_{g,\xi},f_{g,\xi}): g \in \cG, \xi \in \Xi\}, 
\end{align} 
indexed by $g \in \cG$, and auxilliary parameter $\xi$. The function $g \in \cG$ parameterizes a ``first-step'' of a regression problem that the learner needs to solve (as in \Cref{sec:proof_intuition}), and $\xi$ parameterizes some remaining residual uncertainty over the dynamics.

 We assume that each $(\pist,f) \in \cG$ are deterministic. However (for convenience), we consider a slight generalization of \Cref{sec:prelim} in which  $\pist(\bx,t)$ and $f(\bx,\bu,t)$ are allowed depend on a time argument $t$. Moreover, we allow $\pihat(\bx_1,t =1 )$ to depend on time and arbitrarily on the past $\pihat(\bx_{1:t},\bu_{1:t-1},t)$; indeed, the schematica arguments that follow hold for  time-varying, non-Markov policies. Rather, it is the \textbf{instantiation} of the schematic in the proofs of \Cref{thm:stable_detailed,thm:stable_general_noise} in which Markovianity plays an essential role.

Our results show that if the IL family $(\cI,\Dist)$ satisfies three key properties vis-a-vis the regression family $(\cG,\Dreg)$, then a general result template holds. These properties are as follows.

\begin{property}\label{defn:orthogonal} We say the $\tau$-\emph{orthogonal embedding property} holds if there exists a unit vector $\bv \in \R^d : \|\bv\| = 1$, a mapping $\pi_0: \Xspace \to \Uspace$, and mapping $\projj: \Xspace \to \cZ$, and a probability kernel $\kernn: \cZ \to \laws(\Xspace)$ such that 
\begin{itemize}
    \item The distribution of $\bz = \projj(\bx) \text{ under } \bx \sim \Dist$ is $\Dreg$, and the distribution of $\bx \sim \kernn(\bz) \text{ under } \bz \sim \Dreg$ is $\Dist$, and satisfies $\projj(\bx) = \bz$.
    \item With probability $1$ over $\bx \sim \Dist$, $\pi_{g,\xi}(\bx,1) = \pi_0(\bx) + \tau g(\projj(\bx))\bv$, where again $\pi_0$ is fixed across instances. In particular, $\pi_{g,\xi}(\bx,1)$ does not depend on $\xi$.
\end{itemize}
\end{property}
\begin{property}\label{defn:single_step} We say the \emph{single step property} holds if   if the conditional distribution $\Pr^{\pi,f,\Pnot}_{\traj_H \mid (\bx_1,\bu_1)}$ of the trajectory given $(\bx_1,\bu_1)$ is identical for all $(\pi,f) \in \inst$,  for all $(\bx_1,\bu_1)$ which are in the support of the distribution of $\Pr^{\pi,f,\Pnot}$. 
\end{property}

\begin{property}\label{defn:indistinguishable} We say the \emph{$\xi$-indistinguishable} property holds if,  under $\Pnot$ if $(\pi_{g,\xi},f_{g,\xi})$ and $(\pi_{g,\xi'},f_{g,\xi'})$  induces the same  distribution over trajectories for all $\xi,\xi'$ (notice $g$ is the same).  
\end{property}
Effectively, \Cref{defn:orthogonal} says that the first-step of behavior cloning in $(\cI,\Dist)$ is equivalent to the regression problems in $(\cG,\Dreg)$. \Cref{defn:single_step} says that all information about $g$ can be gleaned only from the $t=1$ time steps in the available sample $\Samp$, and \Cref{defn:indistinguishable} says that
$\Samp$ does not provide any information about the auxiliary vector $\xi$. 

Our lower bounds will all be established by checking \Cref{defn:indistinguishable,defn:orthogonal,defn:single_step}. Once verified, the following proposition can be invoked, whose proof is given in \cref{sec:proof:prop_redux}.

\begin{proposition}\label{prop:redux} Suppose $(\cI,\Dist)$ satisfy  \Cref{defn:indistinguishable,defn:orthogonal,defn:single_step} with parameter $\tau$ vis-a-vis $(\cG,\Dreg)$. Then,
\begin{itemize}
    \item[(a)] We have the equality
    \begin{align}
\minbctrain(n;\inst,\Pnot) = \mintrainone(n;\cI,\Dist) =  \tau\minsl(n;\cG,\Dreg),
\end{align} 
where 
\begin{align*}
\mintrainone(n;\cI,\Dist) := \inf_{\est} \sup_{(\pi,f)\in \inst} \Exp_{\Samp} \Exp_{\bx_1 \sim \Dist}\Exp_{\bu \sim \pihat(\bx,1)} \left[\| \pi(\bx_1,t=1) - \bu \|^2\right]^{1/2},
\end{align*}
 considers  the training minimax risk associated with errors at time step $t = 1$. 
\item[(b)] If $\cG$ is convex, and $\bbA \supseteq \Aproper(\cI)$, then 
\begin{align}
\minbctrain(n;\inst,\Pnot) = \minbctrain^{\bbA}(n;\inst,\Pnot) =   \minbctrain^{\Aproper(\cI)}(n;\inst,\Pnot)\label{eq:bbA_same}
\end{align}

\item[(c)]  Let $\bbA$ be a class of estimators satisfying \Cref{eq:bbA_same}, and let $\Pi_{\bbA}$ denote some class of policies such that  every $\est(\Samp) \in \bbA$ returns a policy $\pihat \in \Pi_{\bbA}$ with probability one, for any sample $\Samp$. 

Set $\beps_n := \minsl(n;\cG,\Dreg)$. Let $P$ be any distribution over $\xi$, and choose a risk $\Risk(\pihat;g,\xi)= \Risk(\pihat;\pi_{g,\xi},f_{g,\xi},\Dist,H)$ satisfying, for all $\pihat \in \Pi_{\bbA}$,   the inequality
\begin{align}
\Exp_{\xi \sim P} \Risk(\pihat;g,\xi) \ge K(\beps_n,H)\cdot \Pr_{\bx \sim \Dist, \bu \sim \pihat(\bx,t=1)}[|\langle \pi_{g,\xi_0}(\bx,t=1) - \bu, \bv \rangle| \ge \kappa \tau \beps_n], \label{eq:C_cond_prob_lb}
\end{align}
for some $K(H,\beps_n) > 0$, where  we note that the term on the right-hand side does depend on $\xi_0$, in view of \Cref{defn:indistinguishable}. Then
\begin{align}
\minmax^{\bbA}(n,\Risk;\cI,\Dist,H) = \inf_{\est \in \bbA} \sup_{g,\xi}\Exp_{\Samp \sim [\pi_{g,\xi},f_{g,\xi},\Dist]}\Exp_{\pihat \sim \est(\Samp)} \Risk(\pihat; g,\xi) \ge K(\beps_n,H) \delta.
\end{align}
\end{itemize}
\end{proposition}

Part (a) of the above proposition establishes equivalence of the minimax training risks and minimax regression risks, and shows both are equivalent to the risk incurred at the first time-step of the observed trajectories. Part (b) shows that, if $\bbA$ contains all proper algorithms, restricting to $\bbA$ does not worsen the IL training risk. 

The ``meat'' of the proposition is in part (c). The condition states if the condition \Cref{eq:C_cond_prob_lb} holds for some risk $\Risk$ of interest, then the minimax risk under $\Risk$ admits a lower bound.  \Cref{eq:C_cond_prob_lb} can be thought of a compounding error condition, which says that the average risk, over the uncertainty of the dynamics parameterized by $\xi \sim P$, scaled up by the compounding factor $K(\beps_n,H)$, is at least as large as the probability that the learner makes some mistake at time $t =1$. We note that we can simply just write $K = K(\beps_n,H)$ (we don't need any uniform quantification over $H$ and $\beps_n$), but the expressing $K(\beps_n,H)$ as a function of these terms clarifies its intended use. Lastly, we note that magnitude of the mistake considered inside the probability operator scales with $\tau \kappa \beps_n$, where again $\beps_n$ is the regression minimax risk, and the parameter $\kappa$ comes from \Cref{defn:anti-concentrated_policy}.

The key challenge in all of our lower bounds is to construct families of instances obeying \Cref{defn:orthogonal,defn:single_step,defn:indistinguishable}, and where there is enough variety over the dynamics (as parameterized by $\xi$) to force compondition \Cref{eq:C_cond_prob_lb} for $K(\beps_n,H) \approx \beps_n \cdot \exp(\Omega(H))$. We summarize here, deferring full proofs to the Appendix.
\begin{itemize}
	\item  \Cref{thm:stable_detailed} creates instances $(\pist,f)$ resembling the construction in \Cref{sec:proof_intuition}. Here, we use $\xi$ to encode whether or not the expert/dynamics are the system $(\bA_1,\bK_1)$ or $(\bA_2,\bK_2)$. As show in that section, uncertainty over these cases is enough to force error to compound exponentially in the horizon. The formal construction and proof are given in \Cref{app:stable}, which explains the other subtleties of the argument.
	\item \Cref{thm:stable_general_noise} uses the same construction as \Cref{thm:stable_detailed}. The main difference is that, for general anti-concentrated policies, only a weaker form of \Cref{eq:C_cond_prob_lb} can be established: namely, one of the form $K(\beps_n,H) \approx \min\{\exp(H)\beps_n,\beps_n^{1-\Theta(1)}\}$ The argument is delicate, and given in \Cref{app:general_noise}. In view of the benevolent gambler's ruin policy (\Cref{eq:gambler_dyn}), we cannot hope for a larger compounding error factor $K(\beps_n,H)$ when relaxing from simply-stochastic to general policies.
	\item \Cref{thm:unstable_detailed}, permitting unstable dynamics, uses bump-functions to embed a time-varying dynamical system, where the state-transition matrices are orthogonal matrices $\bO_t \in \bbO(d)$, scaled by a factor $\rho > 1$. When these are drawn from a uniform prior, there is no choice of control actions which can cancel the exponential growth, because any control action will be approximately orthogonal to a randomly rotated state with high probability. The use of rotation matrices in $d > 1$ is essential. Otherwise, if only scalar systems are considered, the ``non-simple'' policies of \Cref{sec:nonvanilla} can be used to thwart compounding error. The full proof is given in \Cref{app:unstable}.
\end{itemize}

\newpage
\appendix

\hypertarget{targ_appendix}{\nonumpart{Appendix}}

This appendix begins with statements and proofs of all fundamental technical tools in \Cref{app:technical}. \Cref{app:prelim,app:minmax} provided additional material for \Cref{sec:prelim,sec:minmax}, respectively. The remaining appendices are each dedicated to the proof of a single result. \Cref{sec:proof:prop_redux} establishes the general proof schematic, \Cref{prop:redux}, underying all results. \Cref{app:stable} proves the lower bounds for simple  policies (\Cref{thm:stable_detailed}, from which \Cref{thm:constant_prob,thm:main_stable} are stable). The proofs of \Cref{thm:stable_general_noise,thm:unstable_detailed}, from which  \Cref{thm:non_simple_stochastic,thm:unstable_detailed} are derived, are given in \cref{app:general_noise,app:unstable}, respectively. \Cref{label:app_non_vanilla} demonstrates how the use of non-simple policies provably overcomes our lower bound construction. Lastly, \Cref{app:ub_proofs} establishes the upper bounds (\Cref{thm:smoothgen}).


\section{Technical Tools}\label{app:technical}

This section outlines our technical tools. The most unique to this work are the first three sections. \Cref{sec:lem:exp_compounding} gives quantitative compounding error guarantees for smooth nonlinear dynamical systems with (Hurwitz)-unstable Jacobians. This generalizes arguments in \cite{jin2017escape} to non-symmetric Jacobians. \Cref{sec:stability} contains useful results regarding the stability of products of matrices.
Building on these, \Cref{sec:inc_stab_taylor} provides convenient sufficient conditions for incremental stability of nonlinear control systems. 

Moving to more standard results, \Cref{sec:anti_conc} recalls the seminal Paley-Zygmund and Carbery-Wright anti-concentration inequalities. These are applied to derive anti-concentration results for polynomials under the uniform distribution on the unit ball  in \Cref{sec:sphere_anti}. Finally, \Cref{sec:bump_functions} recalls the construction of bump functions, verifying that the construction allows their derivatives to have norms which do not grow with ambient dimension.

\subsection{Exponential Compounding in Unstable Systems}\label{sec:lem:exp_compounding}
\begin{restatable}{definition}{parmatrix}\label{defn:compound_matrix} Given parameters
$\gamma > 1,\mu \in (0,1), L \ge 1, r > 0$, we
 say $\bA$ is a \emph{$(\gamma,\mu,L,r)$-matrix} if $\bA$ admits the following block decomposition, where $\bY_1$ and $\bY_2$ are square matrices:
\begin{align*}
\bA = \begin{bmatrix} ~\bY_1 & \bW^\top \\
\tilde \bW & \bY_2
\end{bmatrix},
\end{align*}
where, for parameters $(\gamma,\mu,L,\nu)$, $\|\bY_2\|_{\op} \le 1 - \mu < 0$, $\|\tilde \bW\|_{\op} \le L$,  and $\sigma_{\min}(\bY_1) \ge 1+\gamma  > 1$, and  $\|\bW\|_{\op} \le r$.
\end{restatable}

\begin{restatable}[Exponential Compounding for $(\mu,\gamma,L)$-matrices]{proposition}{propexpcompound}\label{lem:exp_compound} Let $r > 0$, and let $F(\bx,t)$ be a time-varying, $M$-smooth dynamical map such that each 
\begin{align*}
\bA_t := \nabla_{\bx} F(\bx,t)\big{|}_{\bx = 0}
\end{align*} is a  $(\gamma,\mu,L,r)$-matrix with $\gamma \le 1$, with the same block structure across $t$, and where $r = \ost(L/\gamma\mu)$. Then, for any $\bx_1 \in \R^d$, then
\begin{align*}
\bx_{t+1} = F(\bx_t,t), \quad \bxtil_{t+1} = F(\bxtil_{t},t), \quad \bxtil_1 = \bx_1 \pm \epsilon \be_1
\end{align*}
then either
\begin{align}
\max_{1 \le t \le H} |\be_1^\top(\bx_t-\bxtil_t)| \ge \left(1+\frac{\gamma}{2}\right)^{H-1}\epsilon
\end{align} 
or
\begin{align}
\max_{1 \le t \le H}\max\{\|\bx_t\|,\|\bx_t'\|\} \ge \ost\left(\frac{1}{\mu \gamma\cdot L M}\right)
\end{align}
\end{restatable}

The proof of the above proposition is based on the following elementary recursion.
\begin{lemma}[Core Recursion]\label{lem:master_recursion} Let $\alpha_t,\beta_t$ be two sequences satisfying $\alpha_1 = \epsilon, \beta_1 = 0$ and, for $\gamma, \mu > 0, L,r \ge 0$:
\begin{align*}
\alpha_{t+1} \ge (1+\gamma)\alpha_t - r \beta_t, \quad \beta_{t+1} \le (1-\mu)\beta_t + L \alpha_t.
\end{align*}
Then, if $\eta = \frac{r L}{\gamma \mu} \le 1$, we have that $\alpha_{t+1} \ge (1+(1-\eta)\gamma)\alpha_t \ge (1+(1-\eta)\gamma)^t \epsilon$.
\end{lemma}
\begin{proof}[Proof of \Cref{lem:master_recursion}] We assume the inductive hypothesis that $t \mapsto \alpha_t$  is non-decreasing. Under this hypothesis, we have
\begin{align*}
\beta_{t+1} \le (1-\mu)\beta_t + L \alpha_t \le \underbrace{(1-\mu)^t\beta_1}_{= 0} +  \sum_{k=1}^t L (1-\mu)^{t-k}\alpha_k \le \frac{L}{\mu}\alpha_{t}.
\end{align*}
Then,
\begin{align*}
\alpha_{t+1} = (1+\gamma)\alpha_t - r \beta_t \ge (1 + \gamma(1 - \frac{r L}{\gamma \mu}) ) \alpha_t = (1+ \gamma(1-\eta))\alpha_t.
\end{align*}
which concludes the proof after recursing.
\end{proof}

We now turn to proving \Cref{lem:exp_compound}.
\begin{proof}[Proof of \Cref{lem:exp_compound}] Consider two sequences $\bx_t,\bxtil_t$ with $\delx = \bx_t - \bxtil_t$. Set $\nabla F(\bx)\big{|}_{\bx = \bx_0} = \bD$. Then, 
\begin{align*}
\|\delx_{t+1} - \bA_t \delx_{t}\| &= 
\|F(\bx_t,t) - F(\tilde \bx_t,t) - \nabla F(\bzero,t)\delx_t \| \\
&\le  \|F(\bx_t) - F(\tilde \bx_t) - \nabla F(\bx_t) \delx_t \| + \|\nabla F(\bx_t,t) - \nabla F(\bzero,t)\|\|\delx_t\|\\
&\le  \frac{M}{2}\|\delx_t\|^2 + \|\nabla F(\bzero,t) - \nabla F(\bx_t,t)\|\|\delx_t\|\\
&\le  \frac{M}{2}\|\delx_t\|^2 + M\|\bx_t\|\|\delx_t\|.
\end{align*} 
Assume $\max\{\|\bx_t\|, \|\delx_t\|\} \le r_0 := \frac{2r}{3M}$ for $1 \le t \le H$. Then, $\frac{M}{2}\|\delx_t\|^2 + M \|\bx_t\|\|\delx_t\| \le \frac{3Mr}{2}\|\delx_t\|$, so there exists a matrix $\bDel_t$ with $\|\bDel_t\| \le \frac{3Mr_0}{2} = r$ for which 
\begin{align}
\delx_{t+1} =  (\bA_t + \bDel_t)\delx_{t}.  \label{eq:delx_compound_rec}
\end{align}
Let $\bP$ denote the projection onto the coordinates contained in $(\bA_{t})_{[1]}$ (recall: we assume shared block-structure across $t$), and define $\alpha_t = \|\bP\delx_t\|$ and $\beta_t := \|(\eye - \bP)^\top \delx_t\|$. Then, using the block-structure of $\bA_t$ and conditions of  \Cref{defn:compound_matrix}, 
\begin{align*}
\alpha_{t+1} \ge (1+\gamma - r) \alpha_t  - 2r\beta_t, \quad \beta_{t+1} \le (1 - \mu + r)\beta_t + (r + L) \alpha_t.
\end{align*}
Let us make an inductive hypothesis that $\alpha_t$ is non-decreasing. Then, given $ r \le \min\{\mu/2,L,\gamma/4\}$, the above simplifies to 
\begin{align*}
\alpha_{t+1} \ge (1+\frac{3\gamma}{4} ) \alpha_t  - 2r\beta_t, \quad \beta_{t+1} \le (1 - \mu/2)\beta_t + 2L \alpha_t.
\end{align*}
The result now follows from \Cref{lem:master_recursion}, provided that
\begin{align}
\eta = \frac{2r\cdot 2L}{(\mu/2)(3\gamma/4)} \le \frac{1}{4},
\end{align}
which requires $r = \ost(1/\mu\gamma L)$.
\end{proof}

\subsection{Stability of Products of Matrices}\label{sec:stability}
\begin{definition} We say a sequence of matrices $(\bX_1,\bX_2,\dots)$ is  $(C,\rho)$-stable if, for any $n$, $\|\bX_n \cdot \bX_{n-1}\cdot \bX_{j}\|_{\op} \le C \rho^{n-j}$ for all $1 \le j \le n$. Recall $\bX$ is $(C,\rho)$-stable if the sequence $(\bX,\bX,\dots)$ is.
\end{definition}
\begin{lemma}\label{lem:product_stability} Let $(\bA_i)_{i \ge 1}$ be a $(C,\rho)$-stable sequence of matrices. Let $(\bX_i)_{i \ge 1}$ be a sequence of matrices such that, for each $i$, for which $\|\bX_i - \bA_i\| \le \epsilon$. Then, $(\bX_i)_{i \ge 1}$ is $(C,\rho+C\epsilon)$-stable. 
In particular, if $\rho = 1 - 2\gamma$ and $\epsilon = \frac{\gamma}{C}$, then $(\bX_i)_{i \ge 1}$  is $(C,1 - \gamma)$-stable. 
\end{lemma}
\begin{proof} Throughout, let $\||\cdot\|$ denote the operator norm.  First, let us prove our lemma in the case where for all $i$, we have $\|\bX_i - \nu_i\bA_i\| \le \epsilon$ for some $0 \le \nu_i \le 1$. 
\begin{align}
\bX_{n}\bX_{n-1}\dots\bX_1 = \sum_{S \subset [n]} \bT_S, \quad \bT_S :=  \prod_{i=1}^t (\I\{i \notin S\} \bA_i + \I\{i \in S\}\bDel_i).
\end{align}
For $|S| = k$, this means that there are at most $k_0 \le k+1$ (integer) subintervals of $[n]$, denoted whose endpoints we denote $a_j,b_j, 1 \le j \le k_0$, for which $a_j,a_{j+1},\dots,b_j \notin S$.  Furthermore, we must have $\sum_{j=1}^{k_0} (b_j - a_j) = n - k$. Lastly, we have that 
\begin{align}
\|\bA_{b_j} \cdot \bA_{b_j - 1} \dots \bA_{a_j}\| \le C\rho^{b_j - a_j}. \label{eq: thing}
\end{align}

Indeed, 
We therefore conclude that for $|S| = k$,   
\begin{align*}
\|\bT_S\| &= \|\prod_{i=1}^t (\I\{i \notin S\} \bA_i + \I\{i \in S\}\bDel_i)\|\\
&\le \prod_{i \in S}\|\bDel_i\|\prod_{j=1}^{k_0} \|\bA_{b_j} \cdot \bA_{b_j - 1} \dots \bA_{a_j}\| \\
&\le \epsilon^k C^{k_0} \prod_{j=1}^{k_0} \rho^{b_j - a_j}\\
&\le \epsilon^k C^{k_0} \rho^{n - k} \le C (C\epsilon)^k \rho^{n-k},
\end{align*}
where above we use $k_0 \le k+1$ and $\nu_i \in [0,1]$. 
Therefore,
\begin{align}
\|\bX_{n}\bX_{n-1}\dots\bX_1\| &= \sum_{S \subset [n]} \|\bT_S\| \le \sum_{S \subset [n]} C(C \epsilon)^{|S|} \rho^{n - |S|} = C(\rho + C \epsilon)^n.
\end{align}
\end{proof}

\subsection{Stability of Linearizations Implies Incremental Stability}
\label{sec:inc_stab_taylor}

\begin{lemma}\label{lem:stability_if_recursion} Let $\rho, \epsilon > 0$ and $\rho + \epsilon < 1$. Let $(\delx_t,\delu_t)$ be any sequence for which there exist a  $(C,\rho)$-strongly stable sequence $(\bA_t)_{t \ge 1}$, for which
\begin{align}
\|\delx_{t+1} - \bA_t\delx_{t}\| \le L \|\delu_t\| + \epsilon\|\delx_t\|. 
\end{align}
Then, 
\begin{itemize}
    \item[(a)] There exists matrices $(\bB_t)$ and $(\bX_t)$ with $\|\bB_t\| \le L$ and $\|\bX_t - \bA_t\| \le \epsilon$
    such that
\begin{align}
\delx_{t+1} = \bX_t\delx_{t}  + \bB_t \delu_t.
\end{align}
\item[(b)] We have
\begin{align*}
\|\delx_{t+1}\| &\le C(\rho + C\epsilon)^t + L \sum_{1 \le j \le t } C(\rho + C\epsilon)^{t-j} \|\delu_j\|.
\end{align*}
\end{itemize}
\end{lemma}
\begin{proof} We first prove part (a) at each time step $t$. If  $\delx_{t+1} - \bA_t\delx_{t} = 0$, this is holds for $\bX_t = \bA_t$ and $\bB_t = 0$. By similar reasoning, it suffices to prove the case when $\delx_t,\delu_t \ne 0$. Hence, let $\bz_t$ be a unit vector in the direction of $\delx_{t+1} - \bA_t\delx_{t}$, $\bv_t$ in the direction of $\delx_t$ and let $\bw_t$ a unit vector in the direction of $\delu_t$ (arbitrary if $\bu_t = 0$). Then, for some $\gamma_t \in [0,1]$, $\delx_{t+1} - \bA_t\delx_{t} = \|\delx_{t+1} - \bA_t\delx_{t}\|\bz_t \gamma_t L \|\delu_t\| + \epsilon\|\delx_t\| =  \gamma L \bz_t \bw_t^\top \delu_t + \gamma \epsilon \bz_t \bv_t^\top \delx_t$. Choosing $\bB_t = \gamma L\bz_t \bw_t^\top $ and $\bX_t - \bA_t = \gamma \epsilon \bz_t \bv_t^\top$ proves the claim. 

We now turn to part (b). Define $\bY_{t+1,s} := \bX_t \cdot \bX_{t-1} \dots \bX_s$. with the convention $\bY_{t+1,t+1} = \eye$.  
Part (a) implies
\begin{align}
\delx_{t+1} = \bY_{y+1,1}\delx_1 + \sum_{i=1}^t \bY_{t+1,i+1}\bB_i\delu_i
\end{align}
Taking the norm of each side and using Holder's inequality for $\ell_1$ and $\ell_{\infty}$, 
\begin{align}
\|\delx_{t+1}\| \le \|\bY_{t+1,1}\|\|\delx_1\| + \left(\sum_{i=1}^t \|\bY_{t+1,i+1}\|\|\bB_i\|\right) \|\delu_j\|. 
\end{align}
Using \Cref{lem:product_stability}, we have $\|\bY_{y+1,s}\| \le C(\rho + C\epsilon)^{t+1 -s}$, and by the above, $\|\bB_i\| \le L$. Hence, 
\begin{align*}
\|\delx_{t+1}\| &\le C(\rho + C\epsilon)^t + L \sum_{1 \le j \le t } C(\rho + C\epsilon)^{t-j} \|\delu_j\|.
\end{align*}
\end{proof}

\begin{restatable}{lemma}{incstabtaylor}\label{lem:inc_stab_taylor}  Let $\rho, \epsilon > 0$, $L \ge 1$,  and $\rho+C\epsilon < 1$. Suppose that there exists a $(C,\rho)$-stable matrix $\bA$ such that 
\begin{align*}\sup_{\bx,\bu}\|\nabla_{\bx} f(\bx,\bu) - \bA\| \le \epsilon, \quad \sup_{\bx,\bu}\|\nabla_{\bu} f(\bx,\bu)\| \le L.
\end{align*} 
Then,  $f(\bx,\bu)$ is $(C',\rho')$ stable such that $\rho' = \rho + C\epsilon$ and $C' = CL$.
\end{restatable}
\begin{proof}[Proof of \Cref{lem:inc_stab_taylor}] Let $(\bx_i,\bu_i)_{i \ge 1}$ and $(\bx_i',\bu_i')_{i \ge 1}$ be two sequences. Define $\delx_{t} := \bx_{t}' - \bx_{t}$ and $\delu_t$ similarly.
\begin{align*}
\delx_{t+1} &= \bx_{t+1}' - \bx_{t+1} = f(\bx'_t,\bu_t') - f(\bx_t,\bu_t) \\
&= \bA \delx_t + \underbrace{\int_{\alpha =0}^1 \nabla_{\bu} f(\bx_t, \alpha \bu_t' + (1-\alpha)\bu_t )\delu_t \rmd \alpha}_{\|\cdot\| \le L\|\delu_t} \\
&\quad+ \underbrace{\int_{\alpha=0}^t  (\nabla_{\bx} f( \alpha \bx_t' + (1-\alpha)\bx_t,\bu_t') - \bA_x)\delx }_{\|\cdot\| \le \epsilon \|\delx\|}.
\end{align*}
Thus, we obtain
\begin{align*}
\|\delx_{t+1} - \bA\delx_{t}\| \le L \|\delu_t\| + \epsilon\|\delx_t\|. 
\end{align*}
The result now follows from \Cref{lem:stability_if_recursion}.
\end{proof}

\subsection{Sufficient Conditions for One-Step Controllability}
\begin{lemma}\label{lem:one_step_controllable} Consider a control system with $\R^d= \R^m$ and dynamics $f(\bx,\bu) = \phi(\bx) + \bu + \psi(\bu,\bx)$, where (a) $\bx \mapsto \phi(\bx)$ is $L$-Lipschitz, (b) $\psi(\bu,\bx)\big{|}_{\bu = \bzero} = \bzero$ for all $\bu$, and for for some $\nu \in [0,1)$,  $\bu \mapsto \psi(\bx,\bu)$  for all $\bx$. Then, $f$ is $C := (1-\nu)^{-1}\max\{1,\phi(\bzero),L\}$-one-step controllable. The same also holds for dynamics $f(\bx,\bu) = \phi(\bx) - \bu + \psi(\bu,\bx)$. 
\end{lemma}
\begin{proof} Given $\bx,\bx'$, define $\bx^+ := \bx' - \phi(\bx)$. Consider $F(\bu) := \|\bx^+ - \psi(\bx,\bu) - \bu\|$. First, we have $F(\bu) \ge \|\bu\|(1-\nu) - \|\bx'\| $, and $F(\bzero) = \|\bx'\|$. Hence, $F(\bu)$ has all global minimizers in the set $U := \{\bu:\|\bu\| \le (1-\nu)^{-1}\|\bx^+\|\}$. Now let $\bu^\star$ be a global minizer of $F(\bu)$ If $F(\bu^\star) \ne 0$, then $\bv:= \bx^+ - \psi(\bx,\bu) - \bu \ne 0$. But then for $\eta \in(0,1)$,
\begin{align*}
F(\bu^\star + \eta \bv ) &:= \|\bx^+ - \psi(\bx,\bu^\star + \eta\bv) - \bu^\star - \eta \bv  \|\\
&= \|\bx^+ - \psi(\bx,\bu^\star) - \bu - \eta \bv  \| - \|\psi(\bx,\bu) - \psi(\bx,\bu^\star + \eta\bv)\|\\
&\le (1-\eta)\|\bv\|+ \eta \nu \|\bv\| \tag{$\nu$-Lipschitzness of $\psi$ in $\bu$}\\ 
&\le (1-(1-\nu)\eta)\|\bv\| > \|\bv\| = F(\bu^\star),
\end{align*}
contradicting optimal of $\bu^\star$. Hence, we find that $F(\cdot)$ has a global minimum for which $F(\bu^\star) = 0$, $\|\bu^\star\| \le (1-\nu)^{-1}\|\bx^+\| = (1-\nu)^{-1}\|\bx' - \phi(\bx)\| \le (1-\nu)^{-1}(\phi(\bzero) + L\|\bx\| + \|\bx'\|)$. This concludes the proof of $f(\bx,\bu) = \phi(\bx) + \bu + \psi(\bu,\bx)$, and it is easy to see the same argument holds for $f(\bx,\bu) = \phi(\bx) - \bu + \psi(\bu,\bx)$.
\end{proof}
\subsection{Anti-Concentration Tools}\label{sec:anti_conc}
\begin{lemma}[Paley-Zygmud Inequality]\label{lem:pz} Let $Z$ be a non-negative scalar. random varible. Then,
\begin{align*}
\Pr[Z \ge \theta \Exp[Z]] \ge (1-\theta)^2\frac{\Exp[Z]^2}{\Exp[Z^2]}, \quad \theta \in (0,1)
\end{align*}
\end{lemma}
\begin{lemma}[Carbery-Wright Inequality, \cite{carbery2001distributional}]\label{lem:cbw} Let $(B,\|\cdot\|)$ be a Banach Space (e.g. $B = \R$ and $\|\cdot\| = |\cdot|$), and let $P: \R^d \to B$ be a polynomial of degree at most $s$. Then, for any log-concave probability measure $\mu$ on $\R^d$, and any $0 \le r \le q < \infty$, we have
\begin{align*}
\Exp_{\bx \sim \mu}[\|P(\bx)\|^{\frac{q}{s}}]^{\frac{1}{q}} \le C \frac{\max\{q,1\}}{\max\{r,1\}} \Exp_{\bx \sim \mu}[\|P(\bx)\|^{\frac{r}{s}}]^{\frac{1}{r}}
\end{align*}
\end{lemma}
\begin{lemma}Let $\bx$ be uniformly distribution on the unit ball of radius $1$ in dimension $d$. Then $\Exp[\bx\bx^\top] \succeq \eye/3$. \label{lem:unit_ball_conc} 
\end{lemma}
\begin{proof} By rotation invariance,  we have $\Exp[\bx\bx^\top] = \Exp[\|\bx\|^2] \eye$. In one dimension, $\Exp[\|\bx\|^2] = \Exp_{U \sim [-1,1]}U^2 = \frac{1}{2}\int_{-1}^1 x^2 \rmd x  =  \frac{1}{3}$. In higher dimensions, a more involved computation shows this integral is larger than $1/3$: this is the concentration of measure phenomenon, where the $\|\bx\|^2$ concentrates more strongly around one in large dimensions.
\end{proof}
\subsection{Expectation-to-Uniform Bounds on the Sphere}\label{sec:sphere_anti}
\begin{lemma}\label{lem:sphere_diff} Let $\pibar:\R^d \to \R^m$  be $M$-smooth and deterministic. Let $\bx',\bx'$ be drawn i.i.d. on the ball of radius $\Delta$ supported on a subspace $V \subset \R^d$. Then, there exists a universal constant $c_{\star}$ such that, if 
\begin{align}
\Pr_{\bx,\bx'}[|\langle \bv, \bK (\bx' - \bx) - \mean[\pihat](\bx') - \mean[\pihat](\bx)\rangle| \ge M\Delta^2] \le c_{\star}
\end{align}
Then, $\|(\bK - \nabla \mean[\pihat](\bzero))\bP_{V}\|_{\op} \le 6M\Delta\sqrt{d}$. 
\end{lemma}
\begin{proof}[Proof of \Cref{lem:sphere_diff}] Let $\bK_0 =\nabla \pibar(\bzero)$. Then, by a Taylor expansion, we have $\|\mean[\pihat](\bx') - \mean[\pihat](\bx) - \bK_0(\bx'-\bx)\| \le M\Delta^2$. Hence, 
\begin{align}
\Pr_{\bx,\bx'}[|\langle \bv, (\bK - \bK_0) (\bx' - \bx)\rangle| \ge 2M\Delta^2] \le c_{\star}.
\end{align}
By the Paley-Zygmund and Carbery-Wright Inequalities (\Cref{lem:pz,lem:cbw}), $\Pr_{\bx,\bx'}[|\langle \bv, (\bK - \bK_0) (\bx' - \bx)\rangle| \ge \frac{1}{2}\Exp[|\langle \bv, (\bK - \bK_0) (\bx' - \bx)\rangle|^2]^{1/2} \ge c_0$ for some universal constant $c_0$. Hence if $c_\star \le c_{0}$, we must have that $\Exp[|\langle \bv, (\bK - \bK_0) (\bx' - \bx)\rangle|^2]^{1/2} \le 4M\Delta^2$. Because $\bx,\bx'$ are uniformly distributed on the unit ball restricted to $V$, by rescaling and invoking \Cref{lem:unit_ball_conc}, their covariances are at least $\Delta\bP_V/3d$. Adding these covariances of the independent variables, we find $\|\bv^\top(\bK - \bK_0)\bP_V\| \le 6\Delta\sqrt{d}$. Taking the supremum over $\bv$ concludes.
\end{proof}

\begin{lemma}[Consequence of Carbery-Wright]\label{lem:Carbury Wright} Let $\cD$ be a log-concave distribution on $\R^d$, and let $G(\bx)$ be a polynomial of degree at most $p$.  Then, for some universal constant $C \ge 1$,
\begin{align}
\Exp_{\bx \sim \cD}[\|G(\bx)\|^2] \le C^p\left(\Exp_{\bx \sim \cD}[\|G(\bx)\|]\right)^2
\end{align}
\end{lemma}
\begin{proof} The Carberry-Wright inequality, \Cref{lem:cbw},  with $q = 2p$ and $r = p$ yields
\begin{align}
\Exp_{\bx \sim \cD}[\|G(\bx)\|^2]^{1/(2p)} \le C\frac{\max\{p,1\}}{\max\{q,1\}}\Exp_{\bx \sim \cD}[\|G(\bx)\|]^{1/p} = 2C\Exp_{\bx \sim \cD}[\|G(\bx)\|]^{1/p},
\end{align}
where $C$ is a universal constant. Taking the $2p$-th power of both sides and multiplying $C$ by a factor of two concludes.
\end{proof}

\begin{lemma}[Derivative Bounds on the Ball]\label{lem:derivatives_on_the_ball} Let $p \ge 2$, and let $G$ be a function satisfying
\begin{align}
\|G(\delx) - \sum_{\ell = 0}^{p-1} \nabla^{(\ell)} G(\bzero) \circ (\delx^{\otimes \ell})\| \le \frac{M}{p!}\|\delx\|^p.
\end{align}
and suppose that
\begin{align}
\Exp_{\bx \sim\Delta\cdot \cB_d(1)}\|G(\bx)\| \le \epsilon. 
\end{align}
Then, for all $0 \le \ell \le d-1$, $\|\nabla^{(\ell)} G(\bzero)\|_{\fro} \le C^p(2 d)^{\ell/2} (\epsilon\Delta^{-\ell}  + \Delta^{p-\ell} /(p!))$. In particular, if $p$ is taken to be a universal constant, we have 
\begin{align}
\|\nabla^{(\ell)} G(\bzero)\|_{\fro} \le O(d^{\ell/2}(\epsilon\Delta^{-\ell}  + \Delta^{p-\ell}))
\end{align}
The result generalizes, up to universal constant multiplicative factors, to the case when $\bx$ has the distribution of $\bx^1- \bx^2$, where $\bx^1,\bx^2$ are drawn independently from $\Delta \cB_1(d)$. 
\end{lemma}
\begin{proof} The two facts about the distribution we use on the sphere are that it is log concave, enabling the use of Carbery-Wright, that the signs of each coordinate are independent and symmetric, and that its covariance has a particular form. For the final statement of the lemma, we note that $\bx^1 - \bx^2$ is log concave (log concavity is preserved under convolution), its coordinates still have independent signs, that and its covariance is equal to twice that of $\bx \sim \Delta \cB_1(d)$. Hence, we prove the statement only for the distribution of $\bx$.

Define the function 
\begin{align}
G_0(\bx) = \sum_{\ell = 0}^{p-1} \nabla^{(\ell)} G(\bzero) \circ (\delx^{\otimes \ell}),  \quad \epsilon_0  = \epsilon + \frac{M}{p!}\Delta^p.
\end{align}
Then, $G_0$ is a polynomial of degree at most $p-1$, and $\Exp[\|G_0(\bx)\|] \le \epsilon_0$. From \Cref{lem:Carbury Wright} and the fact that the uniform measure on the convex body $\cB_d(1)$ is log-concanve, it follows that $\Exp[\|G_0(\bx)\|^2] \le C^{p-1} \epsilon_0^2 \le C^p \epsilon_0^2$ for some universal $C \ge 1$. Notice further that if $\bx \sim \cB_d(1)$, $\Exp[\bx_i^2] \ge 1/(2d)$. Thus \Cref{lem:expectation_to_uniform} with $\nu = 2$ below implies that $\|\nabla^{(\ell)} G(\bzero)\|_{\fro} = \|\nabla^{(\ell)} G_0(\bzero)\|_{\fro} \le C^p\epsilon_0\Delta^{-\ell} (2 d)^{\ell/2} = C^p(2 d)^{\ell/2} (\epsilon\Delta^{-\ell}  + \Delta^{p-\ell} /(p!))$. This concludes the proof.
\end{proof}

\begin{lemma}\label{lem:high_probability_on_ball} Let $p \ge 2$, and let $G$ be a function satisfying
\begin{align}
\|G(\delx) - \sum_{\ell = 0}^{p-1} \nabla^{(\ell)} G(\bzero) \circ (\delx^{\otimes \ell})\| \le \frac{M}{p!}\|\delx\|^p.
\end{align}
and suppose that
\begin{align}
\Pr[\|G(\bx)\| \ge \epsilon] < \frac{1}{4C^{2p}},
\end{align}
where $C$ is the universal constant as in Carberry Wright. Then, the results of \Cref{lem:derivatives_on_the_ball} hold up to universal multiplicative constants. Note further that if $p$ is a universal constant, then we can take $ \frac{1}{4C^{2p}}$ to be as well. 
\end{lemma}
\begin{proof}[Proof of \Cref{lem:high_probability_on_ball}] Again 
\begin{align}
G_0(\bx) = \sum_{\ell = 0}^{p-1} \nabla^{(\ell)} G(\bzero) \circ (\delx^{\otimes \ell}),  \quad \epsilon_0  = \epsilon + \frac{M}{p!}\Delta^p.
\end{align}
Then $\Pr[\|G_0(\bx)\| \ge \epsilon_0] < \frac{1}{4C^{2p}},$. Now, suppose that $\Exp[\|G_0(\bx)\|] \ge 2\epsilon_0$. By Carbery Wright (\Cref{lem:Carbury Wright}), $\Exp[\|G_0(\bx)\|^2]^{1/2} \le C^p\epsilon_1$. By the Paley-Zygmud inequality (\Cref{lem:pz}),
\begin{align}
\frac{1}{4C^{2p}} > s \ge \Pr[\|G_0(\bx)\| \ge \epsilon_0] \ge  \frac{1}{4}\frac{\Exp[\|G_0(\bx)\|^2]}{\Exp[\|G_0(\bx)\|]^2} \ge \frac{1}{4C^{2p}}.
\end{align}
Hence, it must follow that in fact $\Exp[\|G_0(\bx)\|] \le 2\epsilon_0$. The bound now follows by repeating the arguments of \Cref{lem:derivatives_on_the_ball}.

\end{proof}

\begin{lemma}\label{lem:expectation_to_uniform} Let $G: \R^d \to \R^m$ be a polynomial satisfying 
\begin{align}
G(\delx) = \sum_{\ell = 0}^{p-1} \nabla^{(\ell)} G(\bzero) \circ (\delx^{\otimes \ell}).
\end{align}
where $\nabla^{(\ell)}$ is the $\ell$-th order derivative. Let $\cD$ be a  distribution supported on $\cB_d(1)$, such that $\Exp_{\bx \sim \cD}[ (\bx_i)^2] \ge 1/(\nu d)$, and such that the signs of each of its coordinates are symmetric and independent. Suppose further that
\begin{align}
\Exp_{\delx \sim \Delta \cD} \|G(\delx)\|^2 \le \epsilon^2.
\end{align}
Then, letting $\|\cdot\|_{\fro}$ denote the (tensor) Frobenius norm, 
\begin{align}
\sum_{\ell=0}^{p-1}\Delta^{2\ell} d^{-\ell} \|\nabla^{(\ell)} G(\bzero)\|_{\fro}^2 \le \epsilon^2.
\end{align}
From this, it follows that 
\begin{itemize}
\item For all $0 \le \ell \le p-1$, $\|\nabla^{(\ell)} G(\bzero)\|_{\fro} \le \epsilon\Delta^{-\ell} (\nu d)^{\ell/2}$ 
\item For all $\delx \in \Delta \cdot \cS^{d-1}$, $\|G(\delx)\| \le \epsilon \sum_{\ell=0}^{p-1}(\nu d)^{\ell/2}$.
\end{itemize}

\end{lemma}
\begin{proof} 
We have
\begin{align*}
\epsilon^2 \ge \Exp\left[\left\|\sum_{\ell = 0}^{p-1} \nabla^{(\ell)} G(\bzero) \circ (\delx^{\otimes \ell})\right\|^2\right] &= \Exp\left[\sum_{\ell,\ell' = 0}^{p-1} \left\langle (\nabla^{(\ell)} G(\bzero) \circ (\delx^{\otimes \ell})),  (\nabla^{(\ell')} G(\bzero) \circ (\delx^{\otimes \ell'})) \right\rangle \right]\\
&= \Exp\left[\sum_{\ell}^{p-1} \left\langle (\nabla^{(\ell)} G(\bzero) \circ (\delx^{\otimes \ell})),  (\nabla^{(\ell)} G(\bzero) \circ (\delx^{\otimes \ell})) \right\rangle \right]\\
&= \sum_{\ell}^{p-1}\Exp\left[ \|\nabla^{(\ell)} G(\bzero) \circ (\delx^{\otimes \ell}))\|^2  \right]\\
&= \sum_{\ell}^{p-1}\Exp\left[ \|\nabla^{(\ell)} G(\bzero) \circ (\delx^{\otimes \ell}))\|^2  \right]\\
&= \sum_{\ell}^{p-1}\sum_{s=1}^m \Exp\left[ \left(\be_s^\top (\nabla^{(\ell)} G(\bzero)) \circ (\delx^{\otimes \ell}))\right)^2  \right],
\end{align*}
For each $\ell,s$, $\be_s^\top\nabla^{(\ell)} G(\bzero))$ is some tensor $T$ with entries $T_{i_{1:\ell}}$, where $i_{1:\ell} = (i_1,\dots,i_{\ell}) \in [d]^\ell$. Denote the entries of $\delx$ by $\delx[j]$.  Then $\Exp[ (T \circ \delx^{\otimes \ell}))^2] = \sum_{i_{1:\ell}, i'_{1:\ell}}  (T_{i_{1:\ell}})(T_{i'_{1:\ell}})\Exp[\prod_{j = i_1,\dots,i_{\ell},i'_1,\dots,i'_{\ell}} \delx[j]]$. Because $\delx[j] \mid \delx[j'],j' \ne j$ is symmetric, each term $\Exp[\prod_{j = i_1,\dots,i_{\ell},i'_1,\dots,i'_{\ell}} \delx[j]]$ either vanishes, or is positive. Consequently, we can lower bound $\Exp[ (T \circ \delx^{\otimes \ell}))^2]$ by the sum over only terms where $i_{1:\ell} = i'_{1:\ell}$, and for these terms, $\Exp[\prod_{j = i_1,\dots,i_{\ell},i'_1,\dots,i'_{\ell}} \delx[j]] = \Exp[\prod_{j = i_1,\dots,i_{\ell}}\delx[j]^2] \ge \prod_{j = i_1,\dots,i_{\ell}}\Exp[\delx[j]^2] = \Delta^{2\ell} (\nu d)^{-\ell}$. We conclude that 
\begin{align}
\Exp[ (T \circ \delx^{\otimes \ell}))^2] \ge \Delta^{2\ell} (\nu d)^{-\ell} \sum_{i_{1:\ell}} (T_{i_{1:\ell}})^2. 
\end{align}
Thus, we conclude that
\begin{align}
\epsilon^2 \ge \sum_{\ell}^{p-1}\sum_{s=1}^m d^{-\ell} \sum_{i_{1:\ell}}(\be_s^\top \nabla^{(\ell)} G(\bzero))_{i_{1:\ell}}^2 = \sum_{\ell}^{p-1} \Delta^{2\ell} (\nu d)^{-\ell} \|\nabla^{(\ell)} G(\bzero)\|_{\fro}^2. 
\end{align}
The first consequence statement follows from the above, the fact that all summands are non-negative, and the elementary inequality $\sqrt{x+y} \le \sqrt{x}+\sqrt{y}$.
To prove the second statement of the lemma, we use the Taylor remainder bound of $G$ and
\begin{align*}
\|G(\delx)\| &= \left\|\sum_{\ell = 0}^{p-1} \nabla^{(\ell)} G(\bzero) \circ (\delx^{\otimes \ell})\right\| \\
&\le \sum_{\ell = 0}^{p-1} \left\|\nabla^{(\ell)} G(\bzero)\right\|_{\fro} \Delta^{\ell} \\
&\le \sum_{\ell = 0}^{p-1}  \sqrt{\left\|\nabla^{(\ell)} G(\bzero)\right\|_{\fro}\Delta^{\ell}(\nu d)^{-\ell})}(\nu d)^{\ell/2}\\
&\le \epsilon \sum_{\ell=0}^{p-1}(\nu d)^{\ell/2}.
\end{align*}
\end{proof}

\subsection{The existence of bump functions.}\label{sec:bump_functions}

\begin{restatable}[Existence of Bump Functions]{lemma}{lembump}\label{lem:bump} 
For any $k \in \N$, there exists a $C^{\infty}$ function $\bump_k(\bz): \R^k \to \R$, called an \emph{bump function}, sastisfying $\bump_k(\bz) = 1$ if and only if $\|\bz\| \le 1$, $\bump_k(\bz) = 0$  if and only if $\|\bz\| \ge 2$. And, for each $p \ge 1$, $\|\nabla^{p}\bump_k(\bz)\|_{\mathrm{op}} \le c_p$, where $\|\cdot\|_{\mathrm{op}}$ denotes the tensor-operator norm, and $c_p$ is a constant independent of $k$ but depending on $p$. Finally, $\nabla^p \bump_k(\bz) = 0$ for all $\bz:\|\bz\| \ge 2$.
\end{restatable}
\begin{proof}[Proof of \Cref{lem:bump}] The proof is standard, and included for completeness. Consider the function $\phi(u) = \exp(1 - \frac{1}{u})$ defined on $(0,1)$, and define
\begin{align}
\psi(u) =\begin{cases} 0 & u \le 0 \\
1 & u \ge 1\\
(1 - \phi(1 - u))\phi(u)  & u \in (0,1)
\end{cases}
\end{align}
We define 
\begin{align}
\bump_k(\bz) := \psi(2 - \|\bz\|^2). 
\end{align}
By construction $\bump_k(\bz) = 1$ if and only if $\|\bz\| \le 1$, $\bump_k(\bz) = 0$ if and only if $\|\bz\| \ge 2$.   For the second, clearly $\psi(u)$ is $C^{\infty}$ for $u > 0 , u < 0$ and $u \in (0,1)$. It is easy to check continuity at $u \in \{0,1\}$, and by using the fact that the derivatives of $\phi(u)$ take the form $g(1/u)\phi(u)$, where $g(u)$ is a polynomial, one can check that all derivatives of $\psi(u)$ vanish at $u \in \{0,1\}$; this establishes that $\psi$ is $C^{\infty}$. As $\bz \mapsto 2 - \|\bz\|^2$ is also $C^{\infty}$, we obtain that $\bump_k(\bz)$ is as well. 

To bound $\|\nabla^{p}\bump_k(\bz)\|_{\mathrm{op}}$, we observe that $\nabla^{p}\bump_k(\bz)$ is a symmetric $p$-tensor, and hence its operator norm is equal to the largest  value of $|\langle \nabla^{p}\bump_k(\bz), \bv^{\otimes p}\rangle|$ where $\bv \in \ballk(1)$. Note that $\langle \nabla^{p}\bump_k(\bz), \bv^{\otimes p}$ is just the order-$p$ directional derivative in the direction $p$, and thus 
\begin{align*}
\|\nabla^{p}\bump_k(\bz)\|_{\mathrm{op}} &\le \sup_{\bv \in \ballk(1)} \frac{\rmd }{\rmd s^p}(\psi  \circ (1-\|\bz + u \bv\|^2))\\
&\le \sup_{\bv \in \ballk(1)} \frac{\rmd }{\rmd s^p}(\psi(1 - \|\bz\|^2 + 2 u \langle \bz, \bv \rangle + u^2 \|\bv\|^2)).
\end{align*}
Using this expression, one can show that the maximial derivative does not depend on the dimension $k$. Note that it is also uniformly bounded because the derivatives of $\psi$ are. 
\end{proof}

\newcommand{\Diverg}{\mathrm{D}}
\section{Appendix for \Cref{sec:prelim}}\label{app:prelim}

\subsection{Trajectory Distance}\label{app:traj_dist}

We begin by defining a canonical coupling (joint distribution) between $\pihat$ and $\pist$ trajectories.

\begin{definition}[Canonical Coupling]\label{def:canoc_couple}Let $\pihat$ be arbitrary and $\pist,f$ be deterministic. We define the canonical coupling of $(\Pr_{\pihat,f,\Dist},\Pr_{\pist,f,\Dist})$, denoted by  $\Pr_{\hat\pi,\pist,f,\Dist}$ (resp.  $\Exp_{\hat\pi,\pist,f,\Dist}$) 
 as the distribution of  (resp. expectation over) the random variables $(\bx_{1:H}^\star,\bu_{1:H}^\star,\bhatx_{1:H},\bhatu_{1:H})$, where
 \begin{itemize}
 \item[(a)] Both trajectories have same initial state $\bx_1^\star = \hat \bx_1\sim \Dist$
 \item[(b)] Inputs $\bu^\star_t = \pist(\bx^\star_t)$ are chosen according to $\pist$, and inputs $\hat \bu_t \sim \hat \pi(\hat \bx_t)$ are chosen by $\pihat$ with independent randomness at each time step
 \item[(c)] Both $\bx^\star_{t+1} = f(\bx^\star_t,\bu^\star_t)$ and $\bhatx_{t+1} = f(\bhatx_t,\bhatu_t)$ evolve according to (deterministic) the system dynamics.  
\end{itemize}
 \end{definition}
 In terms of this, we define the $L_1$-trajectory risk as
 \begin{align}
\Rtraj(\pihat;\pi^\star,f,\Pnot,H) = \Exp_{\hat\pi,\pist,f,\Pnot}\left[ \sum_{t=1}^{H} \min\left\{\|\bx^\star_{t} - \hat\bx_{t}\| + \|\bu^\star_{t} - \bhatu_{t}\|,~ 1\right\}\right]. \label{eq:Rtrajj}
 \end{align}
 Above, we clip the expectation to a maximum of one to avoid pathologies of unbounded rewards. We now show that $\Rtraj \ge \sup_{\cost \in \Clip}\Rsubcost$. 
\begin{lemma}\label{lem:cost_traj_comparison} $\Rtraj(\pihat;\pist,f,\Dist,H) \ge \sup_{\cost \in \Clip}\Rsubcost(\pihat;\pist,f,\Dist,H)$. This bound holds even if we inflate $\Clip$ to include all time-varying costs of the form $\cost(\trajj) = \sum_h \cost_h(\bx_h,\bu_h)$, where each $\cost_h$ is $1$-Lipschitz and bounded in $[0,1]$.
\end{lemma}
\begin{proof} Suppose that $\cost(\trajj) = \sum_h \cost_h(\bx_h,\bu_h)$, where each $\cost_h$ is $1$-Lipschitz, and bounded in $[0,1]$. We have
\begin{align*}
\Rsubcost(\pihat;\pist,f,\Dist,H) &:= \Exp_{\pihat,f,\Dist}\left[\cost(\bx_{1:H},\bu_{1:H})\right]- \Exp_{\pist,f,\Dist}\left[\cost(\bx_{1:H},\bu_{1:H})\right]\\
&= \Exp_{\pihat,\pist,f,\Dist}\left[\cost(\bhatx_{1:H},\bhatu_{1:H}) - \cost(\bx^\star_{1:H},\bu^\star_{1:H})\right]\\
&= \sum_{h=1}^H\Exp_{\pihat,\pist,f,\Dist}\left[\cost_h(\bhatx_{h},\bhatu_{h}) - \cost(\bx^\star_{h},\bu^\star_{h})\right]\\
&= \sum_{h=1}^H\Exp_{\pihat,\pist,f,\Dist}\left[\min\{1,\cost_h(\bhatx_{h},\bhatu_{h}) - \cost_h(\bx_{h}^\star,\bu^\star_{h})\}\right] \tag{$\cost_h \in [0,1]$}\\
&= \sum_{h=1}^H\Exp_{\pihat,\pist,f,\Dist}\left[\min\{1,\|\bx^\star_{h}- \bhatx_h\| + \|\bu_{h}^\star - \bhatu_h\|\}\right] \tag{$\cost$ is $1$-Lipschitz}\\
&=: \Rtraj(\pihat;\pist,f,\Dist,H).
\end{align*}
\end{proof}

\subsection{Guarantees under $Q$-function regularity}\label{app:Q_functions}
Recall the definition of the $Q$-functions, 
\begin{align*}
\Qfun_{h;\pihat,f,\cost,H}(\bx,\bu) &:= \cost_h(\bx,\bu) + \sum_{h'> h}^H  \Exp_{\pihat,f}\left[\cost_{h'}\left(\bx_{h'},\bu_{h'}\right) \mid (\bx_h,\bu_h) = (\bx,\bu)\right].
\end{align*} 
In what follows, we fix $(\cost,\pihat,f,H)$, and adopt the shorthand $Q_h := \Qfun_{h;f,\pihat,\cost,H}$. The evaluation performance of a policy $\pihat$ can be evaluated via the celebrated \emph{performance difference lemma} \citep{kakade2003sample}:
\begin{lemma}\label{lem:perf_diff} Fix an additive cost $\cost(\bx_{1:H},\bu_{1:H}) = \sum_{h=1}^H \cost_h(\bx_h,\bu_h)$, and let $Q_h := \Qfun_{h;f,\pihat,\cost,H}$. Then,
\begin{align*}
\Rsubcost(\pihat;\pist,f,\Dist,H) = \sum_{h=1}^H \Exp_{\pist,f,\Dist}\Exp_{\bhatu_h \sim \pihat(\bx^\star_h)}\left[\Qfun_h(\bx^\star_h,\bhatu_h) - \Qfun_h(\bx^\star_h,\bu^\star_h) \right].
\end{align*}
\end{lemma}
We use the performance difference lemma to establish the claims of \Cref{sec:RL_perspective}.
\begin{lemma}\label{lem:perf_diff_Lipschitz} Suppose each $Q_h$ is $L$-Lipschitz. Then, 
\begin{align*}
\Rcost(\pihat;\pist,f,\Dist,H) \le L\cdot \RtrainLp[1](\pihat;\pist,f,\Dist,H) \le L \cdot \RtrainLp(\pihat;\pist,f,\Dist,H), \quad \forall p \ge 1.
\end{align*}
\end{lemma}
\begin{proof} From \Cref{lem:perf_diff}, 
\begin{align*}
\Rsubcost(\pihat;\pist,f,\Dist,H) &= \sum_{h=1}^H \Exp_{\pist,f,\Dist}\Exp_{\bhatu_h \sim \pihat(\bx^\star_h)}\left[\Qfun_h(\bx^\star_h,\bhatu_h) - \Qfun_h(\bx^\star_h,\bu^\star_h) \right]\\
&\le \sum_{h=1}^H \Exp_{\pist,f,\Dist}\Exp_{\bhatu_h \sim \pihat(\bx^\star_h)}\left[L\cdot\|\bhatu_h - \bu^\star_h\| \right] \tag{$Q_h$ is $L$-Lipschitz}\\
&= L \RtrainLp[1](\pihat;\pist,f,\Dist,H).
\end{align*}
The second inequality follows from Jensen's inequality.
\end{proof}
\begin{lemma}\label{lem:perf_diff_zero_one} Recall $\Risk_{\mathrm{train},\{0,1\}}(\pihat;\pist,f,H) := \sum_{h=1}^H\Exp_{\pist,f,\Dist} \Exp_{\bhatu \sim \pihat(\bstx_h)}I
\left\{\bstu_h \ne \bhatu_h\right\}$. Then, if each $Q_h \in [0,B]$, we have
\begin{align*}
\Rsubcost(\pihat;\pist,f,\Dist,H) \le  B \cdot \Risk_{\mathrm{train},\{0,1\}}(\pihat;\pist,f,\Dist,H).
\end{align*}
\end{lemma}
\begin{proof} Appealing again to the performance difference lemma,
\begin{align*}
\Rsubcost(\pihat;\pist,f,\Dist,H) &= \sum_{h=1}^H \Exp_{\pist,f,\Dist}\Exp_{\bhatu_h \sim \pihat(\bx^\star_h)}\left[\Qfun_h(\bx^\star_h,\bhatu_h) - \Qfun_h(\bx^\star_h,\bu^\star_h) \right]\\
&\le \sum_{h=1}^H \Exp_{\pist,f,\Dist}\Exp_{\bhatu_h \sim \pihat(\bx^\star_h)}\left[B\cdot\I\{\bhatu_h \ne \bu^\star_h\} \right] \tag{$Q_h \in [0,B]$}\\
&= B \cdot \Risk_{\mathrm{train},\{0,1\}}(\pihat;\pist,f,\Dist,H).
\end{align*}
\end{proof}
\subsection{Proof of \Cref{lem:EIISS}}\label{sec:proof_lem_EIISS}
\eiiss*
\begin{proof}Consider any $\bx_h$ and perturbation $\delta \bu$. Let $\bx_h' = \bx_h$ and,
\begin{align*}
\bx_{h + 1} = f(\bx_{h}, \pist(\bx_{h}^\star)), \quad &\bx'_{h+1} = f(\bx'_{h}, \pihat(\bx'_{h}) + \delta \bu), \\
\bx_{h' + 1}' = f(\bx'_{h'}, \pist(\bx'_{h'})), \quad &{\bx'}_{h'+1} = f(\bx'_{h'}, \pihat(\bx'_{h'})), \quad &\forall h' > h.
\end{align*}
Using that $\pihat$ is $(C, \rho)$-E-IISS we have
\begin{align*}
    \sum_{h' \geq h} \|\bx'_{h'} - \bx_{h'}\| &\leq \sum_{h' \geq h} C \rho^{h} \|\delta \bu\| \\
    &\leq \frac{C}{1 - \rho}\|\delta \bu\|.
\end{align*}
Then, provided $\pihat$ is $L_{\pihat}$-Lipschitz, expanding the definition of $\Qfun$ and applying triangle inequality,
\begin{align*}
    |\Qfun_{h;f,\pihat,\cost,H}(\bx,\bu + \delta\bu) - \Qfun_{h;f,\pihat,\cost,H}(\bx,\bu)| &\leq \|\delta \bu\| + \sum_{h'=h + 1}^{H} |\cost(\bx_h',\pihat(\bx_h')) - \cost(\bx_h, \pihat(\bx_h))| \\
    &\leq \|\delta \bu\| + (1 + L_{\pihat})\sum_{h' > h}\|\bx_h' - \bx_h\| \\
    &\leq \frac{C}{1 - \rho}(2 + L_{\pihat})\|\delta \bu\|.
\end{align*}
The result for $\Rsubcost$ follows by \Cref{lem:perf_diff_Lipschitz}.
\end{proof}
\subsection{Impossibility of Estimation in the $\{0,1\}$-Loss (\Cref{sec:rl_vs_control})}\label{sec:zero_one_loss}

\newcommand{\dhel}{\mathrm{d}_{\textsc{hel}}}
\newcommand{\dtv}{\mathrm{d}_{\textsc{tv}}}

\begin{remark}[Hellinger Distance, Total Variation Distnace,  and the $\{0,1\}$ loss]\label{rem:hellinger_tv}. Let $P,Q$ be probability distributions over the same probability space $\Omega$, with common densities $p,q$ with respect to a common base measure $\mu$. 
    We recall that the Hellinger and total variation distances, respectively,   are given by 
    \begin{align}
    \dhel(P,Q)^2 = \frac{1}{2}\int (\sqrt{p(\omega)} - \sqrt{q(\omega)})^2 \rmd \mu(\omega), \quad \dtv(P,Q) = \frac{1}{2}\int |p(\omega) - q(\omega)|\rmd \mu(\omega).
    \end{align}
    By the LeCam's inequality \citep[Lemma 2.4]{tsybakov1997nonparametric}, the above are qualitatively equivalent. 
    \begin{align}
        \frac{1}{2}\dhel(P,Q)^2 \le \dtv(P,Q) \le \dhel(P,Q) 
    \end{align}
    Moreover, in the special case where $P$ is a Dirac distribution supported on $\omega_p \in \Omega$, we have that 
    \begin{align}
        \dtv(P,Q) =  \Pr_{\omega_q \sim Q}[ \omega_q \ne \omega_p] = \Exp_{\omega_q \sim Q}[\I\{ \omega_q \ne \omega_p\}]
    \end{align}
    In the case where $P$ is the conditional distribution of $\pist(x)$ given $x$, which is a Dirac distribution supported at $\pist(x)$, we therefore see that the $\{0,1\}$ loss is precisely equal to the total variation distance $\Exp_{u \sim \pihat(x)}[\I\{ u \ne \pist(x)\}] = \dtv(\pihat(x),\pist(x))$. 
\end{remark}

\begin{proposition}[Impossibility of 0/1 and Information-Theoretic Estimation]\label{prop:cannot_est_TV} Let $\cG$ denote the class of $1$-Lipschitz functions from $[0,1] \to [-1,1]$. Then, and let $\Dreg$ denote the uniform distribution on $[0,1]$. Then, by \Cref{prop:asm_conc}, it holds that $\minsl(\cG,\Dreg) \lesssim \frac{1}{n}$. However, the minimax $\{0,1\}$-risk is:
\begin{align}
\forall n \in \N, \quad \inf_{\estreg}\sup_{\gst \in \cG} \Exp_{\sampreg}\Exp_{\hat{g} \sim \estreg(\sampreg)\Exp_{z \sim \Dreg}\Exp_{\hat{y} \sim \hat{g}(\cdot \mid z}}[\gst(z) \ne \hat{y}] = 1, 
\end{align}
where above, we permit randomized estimators $\hat{g}(\cdot \mid z)$
In particular, this means that 
\begin{align*}
&\forall n \in \N, \quad \inf_{\estreg}\sup_{\gst \in \cG} \Exp_{\sampreg}\Exp_{\hat{g} \sim \estreg(\sampreg)}\Exp_{z \sim \Dreg}\Diverg(\dirac_{\gst(z)},\hat{g}(\cdot \mid z))\\
&\quad=
\begin{cases} 1 & \Diverg = \text{Total Variation, Hellinger Distance} \\
\infty & \Diverg = \text{KL Divergence, Reverse KL Divergence},
\end{cases} 
\end{align*}
where above $\dirac_{\gst(z)}$ is the Dirac distribution supported at $y = \gst(z)$.
\end{proposition}
\begin{remark} The proof below generalizes to the case where $\cG = \Gsmooth(s,L;\ballk(1))$, $L > 0$ (see \Cref{defn:smoothclass}). Indeed, the hard functions below are convergent sums over cosine functions with exponentially decaying weights, so these can be renormalized to have all $s$ first derivatives bounded by any desired constant. In particular,  by \Cref{prop:asm_conc} and setting $k = 1$, the same result holds even when $\minsl(\cG,\Dreg) \lesssim C(s) n^{-s}$ for any integral exponent $s \in \N$.
\end{remark}
\begin{proof}Define the embedding $\phi(x) = (\cos(2\pi i z))_{1\le i \le D}: [0,1] \to \ell_1([D]) = \ell_2([D])$, and for  vectors $\bw \in  \ell_2([D])$,
\begin{align*}
g_{\bw}(z) = \langle \bw,  \phi(z)\rangle. 
\end{align*} 
We first establish a claim which states that $g_{\bw}(z)$ typically behaves like a continuous random variable.
\begin{claim}\label{claim:full_measure} Let $\bw = \bw_0 + \bw_1$, where $\bw_0$ is deterministic, and $\bw_1$ has Lebesgue density w.r.t. to a subspace of dimension $k \ge 1$. Then, for almost every $z \in [0,1]$, $g_{\bw}(z)$ has density with respect to the Lebesgue measure. 
\end{claim}
\begin{proof} Let $\bP$ denote the projection onto the subspace on which $\bw_1$ has density. Then, for $g_0(z) = \langle \bw_0, \phi(z)\rangle$, there is a random vector $\bw'$ with density with respect to the Lebesgue meaure on $\R^D$ such that  $g_{\bw}(z) = g_0(z) + \langle \bP\bw',  \phi(z)\rangle = g_0(z) + \langle \bw', \bP\phi(z)\rangle$. We claim that $\bP \phi(z)$ vanishes on a set of of measure zero. Indeed, $\|\bP \phi(z)\|^2$ is an analytic function, so if$\{z \in [0,1] : \phi(z)\}$ has positive measure, $\|\bP \phi(z)\|^2$ vanishes on all of $z \in [0,1]$. Yet, at the same time, $\int_{z=0}^1\|\bP \phi(z)\|^2\rmd z = \trace(\bP \int_{0}^1 \phi(z)\phi(z)^\top \rmd z) = \trace(\bP) > 0$, ans the entries of $\phi(z)$ are orthogonal on $[0,1]$, and their square expectation is nonvanishing. Finally, of all $z: \bP \phi(z) \ne \bzero$, we observe that $\langle \bw', \bP\phi(z)\rangle$ has density with respect to the Lebesgue measure.
\end{proof}


We now construct a Bayesian problem where $\bw$ is drawn from a prior (supported on $1$-bounded $1$-Lipschitz functions) such that, for any $\sampreg$, the posterior $\bw \mid \sampreg$ can be decomposed as in \Cref{claim:full_measure}. To do so, sample coordinates of $\bw$ independently  as  $w_i \unifsim [-1,1]/16^{-i}$, $1 \le i \le D$; we let $P([D])$ denote this prior. A simple computation reveals that 
\begin{align}
|\nabla g_{\bw}(z)|^2 = \sum_{i=1}^D w_i^2 |\nabla \cos(2\pi i z)| \le \sum_{i=1}^D \frac{(2\pi)^2 i^2}{16^2} \le 1,
\end{align}
so $g_{\bw}$ is supported on $1$-Lipschitz functions. Finally, we take $\Dreg = \mathrm{Unif}([0,1])$.  

Then, an sample $\sampreg = (z_i, g_{\bw}(z_i))_{1 \le i \le n}$ corresponds to taking $n$ measurements of $\bw$ with vectors $\phi(z_i) \in \R^D$. For $D > n$ it follows that, conditioned on $\sampreg$, the distribution of $\bw$ has density with respect to the Lebesgue measure supported on the subspace orthogonal to the  span of  $\phi(z_i) $, and is ortherwise deterministic on that subspace. Hence, by \Cref{claim:full_measure}, the posterior distribution of $g_{\bw}(z) \mid \sampreg$ has density with respect to the Lebesgue measure. Thus, for any conditional distribution $\mu(\hat y \mid z,\sampreg)$, 
\begin{align}
\Exp_{\bw \mid \sampreg}\Exp_{z \sim [0,1]}\Exp_{\hat y \sim \mu(\cdot \mid z,\sampreg)}\I\{y \ne g_{\bw}(x)\} = \int_{0}^1\left(\Exp_{\hat y \sim \mu(\cdot \mid z,\sampreg)}\underbrace{\Exp_{\bw \mid \sampreg}[y = g_{\bw}(x)]}_{=0 \text{ for almost all } z}\right)\rmd z = 0.
\end{align}
Hence, we have established a prior $P([D])$ such that
\begin{align*}
    &\inf_{\estreg}\sup_{\gst \in \cG} \Exp_{\sampreg}\Exp_{\hat{g} \sim \estreg(\sampreg)}\Exp_{z\sim \Dreg}\Exp_{\hat{y} \sim \hat{g}(\cdot \mid z)}[\gst(z) \ne \hat{y}]\\
    &\ge \inf_{\estreg}\sup_{D \in \N}\Exp_{\bw \sim P([D])} \Exp_{\sampreg}\Exp_{\hat{g} \sim \estreg(\sampreg)}\Exp_{z\sim \Dreg}\Exp_{\hat{y} \sim \hat{g}(\cdot \mid z)}[g_{\bw}(z) \ne \hat{y}]\\
    &= 1.
\end{align*}
\end{proof}
\subsection{Why the $L_2$ validation risk}\label{rem:why_l_two} Here, we justify our choice of focusing on the $L_2$ validation risk $\Rtrain$. There are three main reasons. First, $L_2$ validation risks are commonplace in both empirical and theoretical studies of regression problems. Moreover, as $L_2$ risks are large than $L_1$ risks by Jensen's inequality, showing that compounding error occurs \emph{even if} the $L_2$ validation risk is bounded yields a stronger lower bound than showing the same given a bound only on the $L_1$ analogue.

Second, we shall establish lower bounds with a convenient feature: restriction to the algorithm class $\bbA$ does not harm the validation risk, in the sense that the restricted an unrestricted minimax risks are identical: $\minbctrain^{\bbA} = \minbctrain$. Establishing this equality relies on the fact that a Pythagorean theorem holds in $L_2$ space, which renders proper estimators optimal (recall the definition of proper estimators in \Cref{def:proper}). We will show that the algorithm classes $\bbA$ of interest contain proper estimators, rendering the inequality $\minbctrain^{\bbA} = \minbctrain$.

Finally, our lower bounds for non-simply stochastic algorithms hold against an $L_2$-validation risk, defined in \Cref{sec:minmax_anticonc}. This is because the $L_2$-risk emphasizes the tails of the errors more significantly. For these results, an $L_2$ validation risk is preferrable for consistency. 

\iftoggle{arxiv}
{
\section{Appendix for \Cref{sec:minmax}}\label{app:minmax}}
{\section{Additional material for minimax formulations in \Cref{sec:minmax}}\label{app:minmax}}

\newcommand{\Dunif}{D_{\mathrm{unif}}}

\newcommand{\Sett}{\Omega}

\subsection{The necessity of the typical regresion classes, \Cref{asm:conc}}
Below, we provide an example of regression problem classes where optimal estimators make errors of magnitude at least $1$, but as $n \to \infty$, the probabilty of these errors decays as $1/n$, and consequently, the minimax risk still decays to $0$ as $n \to \infty$. 
\begin{example}[An example where \Cref{asm:conc} fails]\label{exp:bad_conc} Consider regression with a distribution $\Dreg$ supported on the dyadic set $\mathbb{D} := \{2^{-k}, k \in \N \cup \{0\}\}$, with $\Pr_{\Dreg}[z = 2^{-k}] \propto k^{-2}$. Consider $\cG$ to be the class of all binary functions $g: \mathbb{D} \to \{-1,1\}$. It is easy to check that the minimax optimal estimator predicts $g(z)$ for all $z \in \bbD$ seen in the sample, and predicts $z = 0$ otherwise; from this estimator, one can check that $\minsl(n,\cG;\cD) \propto 1/n$. On the other hand, \Cref{asm:conc} only holds for $c = 1/n$. This occurs because all errors have magnitude at least $1$, but are make with increasingly lower probability as $n$ grows larger. 
\end{example}

We now illustrate why \Cref{exp:bad_conc} is unsuitable for a compounding error construction. Consider  the following formulation of minimax risk for $C \ge 1, B > 0$:
\begin{align}
\minsl(n;\Gclass,\Dreg, [C,B]) = \inf_{\estreg}\sup_{\gst \in \cG} \Exp_{\sampreg}\Exp_{\hat{g} \sim \estreg(\sampreg)}\Exp_{\bz \sim \Dreg}[\min\{B,C|\gst(\bz) - \hat{g}(\bz)|^2\}].
\end{align}
This measures the risk, magnitude by a factor of $C \ge 1$, but clipped by $B$. Effectively, our arguments lower bound $\minsl(n;\Gclass,\Dreg, [C,B])$ when $C \sim \exp(\Omega(H))$. We observe that the mininimax risk of the problem in \Cref{exp:bad_conc} does not meaningfully change $B = 1$ and we increase  $C$, because all errors made are saturated at magnitude $\Omega(1)$. Hence, \Cref{exp:bad_conc} provides a problem which \emph{cannot} be embedded into a compounding error construction.

\subsection{Verifying \Cref{asm:conc} (Proof of \Cref{prop:typical_true})}\label{app:typical_true}

This section demonstrates that \Cref{asm:conc} holds for a natural class of non-parametric functions.  Before continuing, recall $\ballkr$ is the radius-$r$ ball in $\R^k$. Given a set $\Sett \subset \R^k$ of nonzero Lebesgue measure, we let $\Unif(\Sett)$ denote the uniform distribution on that set.
\begin{definition}[Smooth Functions] \label{defn:smoothclass}For $k,s \in \N$, and an open, bounded domain $\Omega \subset \R^k$, define $\Gsmooth(s,L;\Omega)$ as the set of functions $g: \Omega$ which are $s$-times continuously differentiable, and such 
\begin{align}
0 \le j \le s, \quad \| \nabla^j g(\bz) \|_{\op} \le L
\end{align}
where $\nabla^j$ is the $j$-th order derivative tensor (with $\nabla^0 g \equiv g$), and $\|\cdot\|_{\op}$ the tensor operator norm.   
\end{definition}

The above definition of smooth functions corresponds to the space of functions whose $L_{\infty}$, order-$s$ Sobolev norm (denoted $W_{\infty}^s(\Omega)$) is bounded, and in fact the results in this section extend to all $L_{p}$, order-$s$ Sobolev norms (the space $W_p^s(\Omega)$) for $p \ge 2$. We refer the reader to \cite{krieg2022recovery} for these generalizations. It is clear that the class $\Gsmooth(s,L;\Omega)$ (even if the domain $\Omega$ is nonconvex).

Our main result is a more constructive statement of \Cref{prop:typical_true}. 
\begin{proposition}\label{prop:asm_conc} For $k,s \in \N$, let $\cG_{s,k} := \Gsmooth(s,1;\ballk(1))$ denote the space of $s$-order $1$-smooth functions on the unit ball in $R^k$, and let $\Dist_k := \Dunif(\ballk(1))$. Note that this class is $(1,1,1)$-regular (recall \Cref{defn:regular_instances}) for $s \ge 2$. Then, there exists constants $C_1(s,k) > 0$ and $C_2(s,k) > 0 $, depending only on $s_k$, and a universal constant $c > 0$, such that for all $n \ge 1$, 
\begin{align*}
\minsl(n;\cG_{s,k}, \Dist_k) \le C_1(s,k) n^{-\frac{s}{k}}, \quad \minslprob(n,c^k;\cG_{s,k}, \Dist_k) \ge  C_2(s,k) n^{-\frac{s}{k}}
\end{align*}
In particular, there exists a constants $\kappa(s,k)$ and $\delta(k)$ such that $(\cG_{s,k}, \Dist_k)$ is $(\kappa(s,k),\delta(k))$-\typicall. Furthermore, in view of \Cref{eq:Minsl_bound_prob},the above bounds imply that there exists some other $C_3(s,k)$ for which 
\begin{align*}
\minsl(n;\cG_{s,k}, \Dist_k)  \ge  C_3(s,k) n^{-\frac{s}{k}}.
\end{align*}
\end{proposition}

The  upper bound is a direct consequence of \citet[Theorem 1]{krieg2022recovery}, taking $s$ as is, $d \gets k$, $p = \infty$ and $q = 2$. \footnote{Note that their normalization of the $(s,\infty)$-Sobolev norm is in fact slightly larger than ours (consult the third equation on \citet[Page 3]{krieg2022recovery}).} Here, we focus on the lower bounds. Again, the arguments are somewhat standard (see e.g. \citep{krieg2022recovery,tsybakov2009introduction,bauer2017nonparametric}).

\subsubsection{Proofs of \Cref{prop:asm_conc}}
As noted above, the upper bound follows from \citet[Theorem 1]{krieg2022recovery}. The lower bound follows from standard construction (see, e.g. \cite{kohler2013optimal,tsybakov1997nonparametric}, but where we take care to lower-bound the in-probability minimax risk, $\minslprob$.  
\begin{definition}[Packing, see e.g. Section 4 in \cite{vershynin2018high}] We say that $(\bz_1,\dots,\bz_m)$ forms an $\epsilon$-packing of a set $\Sett$ if each $\bz_i \in \Sett \subset \R^k$, and $\|\bz_i - \bz_j\| \ge \epsilon$ for $i \ne j$. 
\end{definition}
\begin{definition} Let $\cG$ be a function class supported on $\ballk(1)$. We say that a function $r(\cdot):(0,1) \to \R_{ > 0}$ is a $\epsilon_0$-bandwith function for $\cG$ if  for all $\epsilon \le \epsilon_0$, there is exists a function $g_{\epsilon}: \R^k \to \R_{\ge 0}$ for  which (a) $g_{\epsilon}(\bz) = 0$ for all $\bz: \|\bz\| \ge \epsilon$, (b) if $\bz_1,\bz_2,\dots,\bz_m$ are the centers of an $\epsilon$-packing of $\ballk(1)$, for $i \ne j$, that $\bz \mapsto \sum_{i} g_{\epsilon}(\bz - \bx_i) \in \cG$, and (c), if $ \|\bz\| \le \frac{\epsilon}{2}$, then $|g_{\epsilon}(\bz)| \ge 2r(\epsilon)$. 
\end{definition}

\begin{lemma}\label{lem:lower_bound_in_probability} There exists a universal constant $c \in (0,1)$ such that, for all $k \in \N$, the following is true. Let $\cG$ be a function class supported on $\ballk(1)$, and suppose that $r(\cdot):(0,1) \to \R_{\ge  0}$ is a $\epsilon_0$-bandwith function for $\cG$. Then, then, for $n \ge (\epsilon_0)^k$, we have
\begin{align}
\minslprob\left(n, c^k; \cG,\Unif(\ballk(1))\right)  \ge r \left(n^{-\frac{1}{k}}\right)
\end{align}
\end{lemma}
Before proving \Cref{lem:lower_bound_in_probability}, we prove the main results of this section. To instantiate the lower bound, recall the definition of bump functions:
\lembump*

We  use the following lemma to construct bandwidth functions.
\begin{lemma}\label{lem:bump_packing} Let $c_s$ denote the constant given in \Cref{lem:bump}, and define $c_s' := \max_{1 \le j \le s} c_j$. Given $\epsilon \in (0,1]$, define the function $\phi_{\epsilon,s}(\bz) = \frac{\epsilon^s}{2^sc_s'} \bump_k(2\bz/\epsilon)$. Then, $\phi_{\epsilon}(\bz) = 0$ for $\|\bz\| \ge \epsilon$, $\phi_{\epsilon}(\bz) \ge \frac{\epsilon^s}{2^sc_s'}$ for $\|\bz\| \le \epsilon/2$, and  for any $\epsilon$-packing $\bz_1,\dots,\bz_m$, the function
\begin{align*}
g(\bz) := \sum_{i=1}^m \phi_{\epsilon,s}(\bz - \bz_i)
\end{align*}
satisfies $g(\bz) \in [0,1]$, and $\sup_{\bm{\alpha} \in \N^k:|\bm{\alpha}| \le s}| \mathrm{D}^{\bm{\alpha}} g(\bz) |\le \max_{0 \le j \le s}  \|\nabla^j g(\bz)\| \le 1$.
\end{lemma}
\begin{proof} The inequalities $\phi_{\epsilon,s}(\bz) = 0$ for $\|\bz\| \ge \epsilon$ and $\phi_{\epsilon,s}(\bz) \ge \epsilon^s/c_s'$ for $\|\bz\| \le \epsilon/2$ follow directly from the definition of the bump function. In addition, we have that  $\nabla^j \phi_{\epsilon,s}(\bz) = 0$ for all $\bz:\|\bz\| \ge \epsilon$ as well. Given an   $\epsilon$-packing $\bz_1,\dots,\bz_m$,  there is at most one index $i_{\star}$ such that $\|\bz_{i_{\star}} - \bz\| < \epsilon$. If no such index $i_{\star}$ exists, than $g(\bz) := \sum_{i=1}^m \phi_{\epsilon,s}(\bz - \bz_i)$ and all its derivatives vanish. Otherwise, for $\epsilon \le 1$,
\begin{align*}
\|\nabla^j g(\bz)\| = \|\nabla^j \phi_{\epsilon}(\bz - \bz_{i_{\star}})\| = \frac{\epsilon^s}{2^sc_s'} \cdot (2/\epsilon)^j \|\nabla^j \bump_k (\bz')\big{|}_{\bz' = 2(\bz-\bz_{i_{\star}})/\epsilon}\| \le \frac{c_j}{c_s'} \cdot (2/\epsilon)^{j-s} \le 1.
\end{align*}
\end{proof}
Together, \cref{lem:bump_packing,lem:lower_bound_in_probability} imply the result.
\begin{proof}[Proof of lower bound in \Cref{prop:asm_conc}]  By \Cref{lem:bump_packing}, we see that $r(\epsilon) = \frac{\epsilon^s}{C}$ is ($\epsilon_0 = 1$)-bandwidth function for the class $\Gsmooth(s,1;\ballk(1))$ for some constant $C = C(s)$. The lower bound on $\minslprob$ now follows directly from \Cref{lem:lower_bound_in_probability}. 
\end{proof}

\subsubsection{Proof of \Cref{lem:lower_bound_in_probability}}
\begin{proof}  Let $N_{\epsilon}$ denote the maximal cardinality of an $\epsilon$-packing of $\cB_k(1)$. Pick a one such packing, and enumerate the center of the balls in the packing $\bz_1,\bz_2,\dots,\bz_{N_{\epsilon}}$ . 

Now, let us consider estimation against a prior over functions $\sum_{i=1}^{N_{\epsilon}} \xi_i g_i(\bz - \bz_i)$,where $\xi_i$ are i.i.d. Bernoulli random variables. For any estimator $\estreg$, the Bayesian probability of an error of magnitude $r(\epsilon)$, which lower bounds the worst-case probability, under this prior is 
\begin{align*}
&\inf_{\estreg}\Exp_{\xi}\Exp_{\sampreg}\Exp_{\bz \sim \cB_1}\Exp_{y \sim \estreg(\sampreg,\bz)}[\I\{|y -\sum_{i=1}^{N_{\epsilon}} \xi_i g_\epsilon(\bz - \bz_i)| > r(\epsilon)\}]. \\
&\ge \sum_{j=1}^{N_{\epsilon}}\frac{1}{\mathrm{vol}(\cB_k(1))}\inf_{\estreg}\Exp_{\xi}\Exp_{\sampreg}\int_{x_0 \in x_j + \bB_k(\epsilon)}\Exp_{y \sim \estreg(\sampreg,x_0)}[\I\{|y -\sum_{i=1}^{N_{\epsilon}} \xi_i g_{\epsilon}(\bz - \bz_i)| > \epsilon\}]\rmd x_0 \\
&= \sum_{j=1}^{N_{\epsilon}}\frac{1}{\mathrm{vol}(\cB_k(1))}\inf_{\estreg}\Exp_{\xi}\Exp_{\sampreg}\int_{x_0 \in x_j + \bB_k(\epsilon)}\Exp_{y \sim \estreg(\sampreg,x_0)}[\I\{|y - \xi_j g_\epsilon(\bz - \bz_j)| > \epsilon\}]\rmd x_0.
\end{align*}
Let $\cE_j$ be the probability that $\sampreg$ is contains no elements in the set $x_j + \bB_k(\epsilon)$; since samples are uniform on $\bB_k$, $p(n,\epsilon) := \Pr_{\sampreg}[\cE_j]$ is independent of $j$. On $\cE_j$, we have no information about $\xi_j$, so its posterior is uniform. Thus, for any chosen index $j$, and conditional measure $\mu(\cdot \mid x_0)$, the above is at least 
\begin{align*}
&\frac{N_{\epsilon}p(n,\epsilon)}{2\mathrm{vol}(\cB_k(1))} \int_{\bz\in \bz_j + \bB_k(\epsilon)}\Exp_{y \sim \mu(\cdot \mid x_0)}[\I\{|y - g_\epsilon(\bz - \bz_j)| > r(\epsilon)\} + \I\{|y| > r(\epsilon)\}]\rmd \bz\\
&\ge \frac{N_{\epsilon}p(n,\epsilon)}{2\mathrm{vol}(\cB_k(1))} \int_{\bz\in \bz_j + \bB_k(\epsilon)}\I\{| g_{\epsilon}(\bz - \bz_j)| > 2r(\epsilon)\}\rmd \bz\\
&= \frac{N_{\epsilon}p(n,\epsilon)}{2\mathrm{vol}(\cB_k(1))} \int_{\bz\in \bz_j + \bB_k(\epsilon)}\I\{| g_{\epsilon}(\bz )| > 2r(\epsilon)\}\rmd \bz\\
&> \frac{N_{\epsilon}p(n,\epsilon)}{2\mathrm{vol}(\cB_k(1))} \int_{\bz\in \bz_j + \bB_k(\epsilon)}\I\left\{ \|\bz\| < \frac{1}{2}\epsilon\right\}\rmd \bz \tag{ Definition of a bandwidth function}\\
&= \frac{(\mathrm{vol}(\ballk(\frac{\epsilon}{2})) N_{\epsilon}p(n,\epsilon)}{2\mathrm{vol}(\cB_k(1))}\\
&= \frac{\epsilon^k }{2^{k+1}}  \cdot N_{\epsilon}p(n,\epsilon).
\end{align*}
By a standard estimate (see, e.g. \citet[Section 4]{vershynin2018high}), $\epsilon^k N_{\epsilon} \gtrsim (c_1)^{k}$ for $\epsilon \le 1/4$ and universal, dimension-independent  constant $c_1 \in (0,1)$. Moreover, for $n \ge 1/\epsilon^k$, 
\begin{align*}
p(n,\epsilon) = 1 - \left(1 - \frac{\mathrm{vol}(\ballk(\epsilon))}{\ballk(1)}\right)^n  = 1 - (1 - \epsilon^k)^{n} \ge 1 - \exp(-n \epsilon^{-k}) \ge \frac{1}{2}.
\end{align*}
Hence, by setting $c = c_1/4$,
\begin{align*}
\inf_{\estreg}\Exp_{\xi}\Exp_{\sampreg}\Exp_{\bz \sim \cB_1}\Exp_{y \sim \estreg(\sampreg,\bz)}[\I\{|y -\sum_{i=1}^{N_{\epsilon}} \xi_i g_\epsilon(\bz - \bz_i)| > r(\epsilon)\}] \ge  c^{k}.
\end{align*}
By choosing $\epsilon = n^{-1/k}$, we conclude.
\end{proof}

\subsection{General (Minimax) Risks and Comparisons Between Them}\label{sec:risk_comparison}
In this section, we provide comparisons between general families of costs and their associated minimax risks. We recall the cannonical coupling between $(\Pr_{\pihat,f,\Dist},\Pr_{\pist,f,\Dist},)$ over random variables $(\bhatx_{1:H},\bhatu_{1:H}) \sim (\Pr_{\pihat,f,\Dist} $ and $(\bx^\star_{1:H},\bu^\star_{1:H}) \sim (\Pr_{\pist,f,\Dist})$ in \Cref{def:canoc_couple}.

We begin by defining general notions of $L_p$-style risks:
\begin{align}
\Rlpcost(\pihat;\pi^\star,f,\Pnot,H) &:= \Exp_{\pihat,\pist,f_{g,\xi},\Dist} \left[  |\cost(\bx_{1:H},\bu_{1:H}) - \cost(\bx_{1:H},\bu_{1:H})|^p\right]^{1/p}\nonumber.\\
\Rtrajlp(\pihat;\pi^\star,f,\Pnot,H) &:= \Exp_{\hat\pi,\pist,f,\Pnot}\left[ \sum_{t=1}^{H} \min\left\{\|\bx^\star_{t} - \hat\bx_{t}\| + \|\bu^\star_{t} - \bhatu_{t}\|,~ 1\right\}^p\right]^{1/p}. \label{eq:RtrajLp}
 \end{align}
 In the special case that $\cost$ vanishes on $(\cI,\Dist)$, $\Rlpcost$ takes a simpler form (coinciding with )
 \begin{align*}
 \Rlpcost[2](\pihat;\pi^\star,f,\Pnot,H) &:= \Exp_{\pihat,f_{g,\xi},\Dist} \left[  |\cost(\bx_{1:H},\bu_{1:H})|^p\right]^{1/p},\quad \cost \in \Cvan(\cI,\Dist),
 \end{align*}
 coinciding with \Cref{eq:rlp2}.

 We define the associated minimax risks
 \begin{align}
 \mincostlp(n;\cI,\Dist,H) := \minmax^{\bbA}\left(n,\Rlpcost;\cI,\Dist,H\right), \quad \mintrajlp := \minmax^{\bbA}\left(n,\Rtrajlp;\cI,\Dist,H\right).
 \end{align}
 Finally, for the sake of completeness, we propose a generalization of the in-probability risk for non-vanishing costs:
 \begin{align*}
&\minprobcost^{\bbA}(n,\delta;\inst,\Dist, H) \\
&\quad:= \inf \left \{\epsilon : \inf_{\est \in \bbA}\sup_{(\pist,f) \in (\cI,\Dist)} \Exp_{\Samp}\Exp_{\pihat \sim \est(\Samp)}\Pr_{\pihat,\pist,f,D}[\cost(\bhatx_{1:H},\bhatu_{1:H})-\cost(\bx^\star_{1:H},\bu^\star_{1:H}) \ge \epsilon] \le \delta\right\},
\end{align*}
which coincides with \Cref{def:in_prob_risk} in the case that $\cost \in \Cvan(\cI,\Dist)$. 
 
We now state a proposition consisting of elementary relations between the risks thus defined.
\begin{proposition}\label{prop:risk_comparison} Fix an IL problem class $(\cI,\Dist)$,  let  $n,\delta,H$ and algorithm class $\bbA$ be arbitrary. 
\begin{itemize}
	\item[(a)] \textbf{Monotonicty:} $\Risk(\pihat;\pist,f,\Dist,H) \ge \Risk'(\pihat;\pist,f,\Dist,H)$ for all $(\pist,f) \in \cI$ and all $\pihat$. Then $\minmax^{\bbA}(n;\Risk,\inst,\Dist,H) \ge \minmax^{\bbA}(n;\Risk',\inst,\Dist,H)$
	\item[(b)] \textbf{Markov's Inequality:} For any $\cost$, $\mincostlp^{\bbA}(n;\inst,\Dist, H) \ge \delta^{1/p}\minprobcost^{\bbA}(n,\delta;\inst,\Dist, H) $.
	\item[(c)] \textbf{Simplification for Vanishing Costs:} For any nonnegative $\cost \in \Cvan(\cI,\Dist)$ and  $(\pist,f) \in \cI$,  $\Rsubcost(\pihat;\pist,f,\Dist,H) = \Rlpcost[1](\pihat;\pist,f,\Dist,H)$. Thus, $ \mincost^{\bbA}(n;\inst,\Dist, H) = \mincostlp[1](n;\inst,\Dist, H)$.  In particular,  
	\begin{align*}
	\mincost^{\bbA}(n;\inst,\Dist, H) \ge \delta\minprobcost^{\bbA}(n,\delta;\inst,\Dist, H).
	\end{align*} 
	\item[(d)] \textbf{Trajectory Risk Dominates Lipschitz Cost:} For any $\cost \in \Clip$, $\Rtrajlp(\pihat;\pist,f,\Dist,H) \ge \Rlpcost(\pihat;\pist,f,\Dist,H)$ and
	 $\mintrajlp^{\bbA}(n;\inst,\Dist, H) \ge \mincostlp^{\bbA}(n,\delta;\inst,\Dist, H)$.
	 \item[(e)] \textbf{Monotonicity in $p$:} $\Rtrajlp,\Rlpcost,\mintrajlp$ and $\mincostlp$ are nondecreasing in $p$.
\end{itemize}
Point (c) also holds when $\Clip$ is replaced by the set of non-stationary additive costs, $\cost(\bx_{1:H},\bu_{1:H}) = \sum_{h=1}^H \cost_h(\bx_h,\bu_h)$, where each $\cost_h(\cdot,\cdot)$ is $1$-Lipschitz and bounded in $[0,1]$. 
\end{proposition}
\begin{proof} The points are straightforward to verify. Point (a) is immediate from the definition of minimax risk. Point $(b)$ uses the fact that, for a nonnegative random-variable $X$, $\Exp[X^p]^{1/p} \ge \epsilon \Pr[X \ge \epsilon]^{1/p}$ by Markov's inequality. Point (c) is simply uses $|x-y| = x$ for $y = 0$ and $x$ nonnegative. Point (d) directly generalizes the proof of \Cref{lem:cost_traj_comparison}. Finally, point (e) is Jensen's inequality: $\Exp[|X|^p]^{1/p} \le \Exp[|X|^q]^{1/q}$ for any $p \le q$.
\end{proof}

We conclude the section with a subtle point that may be of interest to experts. The class $\Clip$ considers additive costs, and hence typically will scale linearly in $H$. We may instead consider a class $\Cliptil$ of costs normalized by their maximum, which ensures total cost stays bounded.
\begin{definition}[$\max$-Lipschitz Cost Family]\label{defn:maxlipschitz} We define $\Cliptil := \{\cost(\bx_{1:H},\bu_{1:H}) = \max_{h \ge 1} \tilde\cost(\bx_h,\bu_h) : \tilde\cost \text{ is } 1-\text{Lipschitz  and takes values in } [0,1]\}$. 
\end{definition}
Lower bounds on $\Cliptil$ implies those on $\Clip$.
\begin{lemma}[An alternate horizon normalization] For each $\cost \in \Cliptil \cap \Cvan(\cI,\Dist)$, there exists a $\cost' \in \Clip \cap \Cvan(\cI,\Dist)$ such that $\minprobcost[\cost']^{\bbA}(n,\delta;\inst,\Dist, H) \ge \minprobcost[\cost]^{\bbA}(n,\delta;\inst,\Dist, H)$, and similarly for the $L_p$risks and minimax risks.
\end{lemma} 
\begin{proof} Let $\cost(\trajj) = \max_h \tilde \cost(\bx_h,\bu_h) \in \Cliptil \cap \Cvan(\cI,\Dist)$. It is straightforward to check that $\cost'(\traj) =\sum_h \tilde \cost(\bx_h,\bu_h)$ satisfies the desired conditions. Note that this argument extends to non-stationary costs as well.
\end{proof}
 


\newcommand{\tilC}{\tilde{\cC}}
\newcommand{\Cliptilmax}{\tilC_{\mathrm{lip},\mathrm{max}}}
Lastly, we remark that we can remove the restriction of the costs $\Cliptil$ to the range $[0,1]$ with the following trick.
\begin{remark}\label{lem:prob_min_clip}  Let $\Cliptilmax :=  \{\cost(\bx_{1:H},\bu_{1:H}) = \max_{h \ge 1} \tilde\cost(\bx_h,\bu_h) : \tilde\cost \text{ is } 1-\text{Lipschitz and nonnegative}\}$, by analogy to $\Cliptil$ but without the $[0,1]$ restriction. Then, if $\cost \in \Cliptilmax$, $\cost' = \min\{1,\cost\} \in \Cliptil$, and $\min\{1,\minprobcost^{\bbA}(n,\delta;\inst,\Dist, H)\} = \minprobcost[\cost']^{\bbA}(n,\delta;\inst,\Dist, H)$
\end{remark}


\newcommand{\tauinv}{\tau^{-1}}
\section{Proof of \Cref{prop:redux}}\label{sec:proof:prop_redux}
\newcommand{\Bernoulli}{\mathsf{Bernoulli}}

\newcommand{\sqrit}[1]{\left(#1\right)^{1/2}}

We recall
\begin{align*}
\mintrainone(n;\cI,\Dist) := \inf_{\est} \sup_{(\pi,f)\in \inst} \Exp_{\Samp} \Exp_{\bx_1 \sim \Dist}\Exp_{\bu \sim \pihat(\bx,1)} \left[\| \pi(\bx_1,t=1) - \bu \|^2\right]^{1/2}.
\end{align*}
\begin{definition}[Shorthand Notation]\label{defn:framework_shorthand}
We use the shorthand $ \sampreg \sim \lawtil(g)$ to denote the law of samples $\sampreg$ from the regression problem with ground truth $g \in \cG$. We let $\Samp \sim \law(g)$ to denote the law of samples $\Samp$ under the instance $(\pi_{g,\xi},f_{g,\xi})$. Notice that under the $\xi$-indistinguishable property (\Cref{defn:indistinguishable}), $\law(g)$ is well-defined  as it does not depend on $\xi$. Finally, for IL algorithms $\est$ and regression estimators $\estreg$, we define
\begin{align*}
\Rtrone(\est;g) &= \Exp_{\Samp \sim \law(g)}\Exp_{\hat \pi \sim \est(\Samp)} \sqrit{\Exp_{\bx \sim \Dist} \Exp_{\bu \sim \pihat(\bx,1)} \|\bu- \pi(\bx,t=1)\|^2}\\
\Rtr(\est;g) &= \Exp_{\Samp \sim \law(g)}\Exp_{\hat \pi \sim \est(\Samp)} \Rtrain(\pihat;\pi_{g,\xi},f_{\xi,g})\\
\Rsl(\estreg;g) &= \Exp_{\sampreg \sim \lawtil(g)}\sqrit{\Exp_{\ghat \sim \estreg(\Samp)}\Exp_{\bz \sim \Dreg}\|\hat g(\bz) - g(\bz)\|^2}
\end{align*}
Again, indistinguishability implies the above are well-defined (e.g. independent of $\xi$).
\end{definition}
\begin{lemma}[Reduction from regression  to IL]\label{lem:reg_to_BC_redux} For every  IL algorithm $\est$, there exists a regression algorithm $\estreg$ such that for all $g \in \cG$,
\begin{align}
\tau\Rsl(\estreg;g) \le \Rtrone(\est;g) \le \Rtr(\est;g)
\end{align}
\end{lemma}
\begin{proof}
Let us construct the transformation from regression algorithms $\estreg$ to IL algorithms $\est$. Our first step is to construct a (stochastic) transformation $\Phi$ from regression samples to IL trajectories defined by
\begin{align}
\Phi:  (\bz, y) \in (\R^{d'} \times \R)^n \mapsto (\bx, y\bv + \pi_0(\bx_1),\bx_2,\bu_2,\dots), \quad \bx \sim \kernn(\bz), \bx_2,\bu_2 \sim \Pr[\cdot \mid \bx_1,\bu_1],
\end{align}
where $\Pr$ is instance-independent measure governing the remainder of the trajectory conditioned on $(\bx_1,\bu_1)$. We extend $\Phi$ as mapping from samples $\sampreg$ of regression pairs to samples $\Samp$ of trajectories by independently applying $\Phi$ to each pair.  Further recall the definition $\mean[\pi](\bx,t) = \Exp_{\bu \sim \pi(\bx,t)}[\bu]$. Then, 
\begin{align}
\hat g(\bz;\pihat) = \frac{1}{\tau}\bv^\top\left(\Exp_{\bx \sim \kernn(\bz)}  \left[\mean[\pihat](\bx, t = 1) - \pi_0(\bx)\right]\right),
\end{align}
and finally define the regression estimator $\estreg(\sampreg)$ via
\begin{align*}
\Samp \sim \Phi(\sampreg), \quad \pihat \sim \est(\Samp), \quad \hat g = \hat g(\bz;\pihat).
\end{align*}
Let us now show that $\Rtr(\estreg;g) \le \Rsl(\est;g)$. For any $\xi$, we have
\begin{align*}
&\tau \Rsl(\estreg;g) \nonumber \\
&= \tau\Exp_{\sampreg \sim \lawtil(g)} \Exp_{\hat g \sim \estreg(\sampreg)} \sqrit{\Exp_{\bz \sim \Dreg} \|\hat g(\bz) -  g(\bz)\|^2}\\
&=  \Exp_{ \Samp \sim \law(g)}  \Exp_{\pihat \sim \est(\Samp)} \sqrit{\Exp_{\bz \sim \Dreg} \left\|\bv^\top\left(\Exp_{\bx \sim \kernn(\bz)}  \left[\mean[\pihat](\bx, t = 1) - \pi_0(\bx)\right]\right) -  g(\bz)\right\|^2}\\
&\le   \Exp_{ \Samp \sim \law(g)} \Exp_{\pihat \sim \est(\Samp)}\sqrit{\Exp_{\bz \sim \Dreg} \Exp_{\bx \sim\kernn(\bz)}\left\| \bv^\top\left(\mean[\pihat](\bx, t = 1) - \pi_0(\bx)\right) -  g(\bz)\right\|^2}\\
&\le  \Exp_{ \Samp \sim \law(g)} \Exp_{\pihat \sim \est(\Samp)} \sqrit{\Exp_{\bz \sim \Dreg} \Exp_{\bx \sim \kernn(\bz)}\left\|  \bv^\top \mean[\pihat](\bx, t = 1) - \bv^\top  \pi_{g,\xi}(\bx)\right\|^2}\\
&=   \Exp_{ \Samp \sim \law(g)} \Exp_{\pihat \sim \est(\Samp)} \sqrit{\Exp_{\bx \sim \Dist} \left\|  \bv^\top \mean[\pihat](\bx, t = 1) - \bv^\top  \pi_{g,\xi}(\bx)\right\|^2}
\end{align*}
The steps used in each line are as follows: definition of $ \Rsl$; the definition of the estimator $\estreg$ constructed from $\est$, and that $\Samp \sim \law(g)$ in that construction; Jensen's inequality; the formula for $\pihat_{g,\xi}$ given by \Cref{defn:orthogonal}; and the fact the pushforward of $\Dreg$ under $\kernn$ is $\Dist$. Continuing,
\begin{align*}
&\le \Exp_{ \Samp \sim \law(g)} \Exp_{\pihat  \sim \est(\Samp)} \sqrit{\Exp_{\bx \sim \Dist} \left\|  \mean[\pihat](\bx, t = 1) -  \pi_{g,\xi}(\bx)\right\|^2}\\
&\le  \Exp_{ \Samp \sim \law(g)} \Exp_{\pihat \sim \est(\Samp)} \sqrit{\Exp_{\bx \sim \Dist} \Exp_{\bhatu \sim \hat \pi(\bx,1)} \left\|  \bhatu -  \pi_{g,\xi}(\bx)\right\|^2}\\
&\le  \Rtrone(\est;g) \le  \Rtr(\est;g).
\end{align*}
where we use  that $\bv^\top (\cdot)$ is an orthogonal projection; Jensen's inequality again; the fact that the IL training risk is at least the $\ell_2$ loss on the first prediction.
\end{proof}
\begin{lemma}[Reduction from IL to regression]\label{lem:BC_to_reg_redux} For every  regression algorithm $\estreg$, there exists a regression algorithm $\est$ such that for all $g \in \cG$,
\begin{align}
\Rtr(\est;g) = \tau\Rsl(\estreg;g)
\end{align}
Moreover, if $\estreg$ is proper, i.e. with probability $1$ over its randomness $\ghat \sim \estreg$ lies in $\cG$, then so is $\est$. Lastly, it holds that if
\begin{align}
\inf_{\estreg}\sup_{g_{\star} \in \cG} \Exp_{\sampreg}\Exp_{\hat{g} \sim \estreg(\sampreg)}\Pr_{\bz \sim \Dreg, \y \sim \ghat(\by)}[|\gst(\bz) - \by| \ge \epsilon] \ge \delta, 
\end{align}  
then, for any IL algorithm $\Alg$ and any $\xi$, 
\begin{align}
\sup_{g \in \cG} \Exp_{\Samp \sim (\pi_{\xi,g},f_{\xi,g})}\Exp_{\hat{\pi} \sim \Alg(\Samp)}\Pr_{\bx \sim \Dist, \bu \sim \pihat(\bx)}[|\langle \pi_{g,\xi}(\bx) - \bu, \bv \rangle| \ge \tau\cdot\epsilon] \ge \delta, 
\end{align} 

\end{lemma}

\begin{proof} We prove the first statement; the ``Lasty,'' statement follows from a similar argument. Consider the map 
\begin{align*}
\Phi &: (\bx_1,\bu_1,\bx_2,\bu_2,\dots) \mapsto (\projj(\bx_1), \bv^\top \bu_1)\\
\hat \pi(\bx; \hat g) &= \bv \hat g(\projj(\bx)) + \pi_0(\bx).
\end{align*}
We let $\est$ be the algorithm which, given 
$\Samp \sim \law(g)$, construct $\sampreg = \Phi(\Samp)$ and selects $\hat g = \estreg(\sampreg)$, and returns $\hat \pi$ such that  $\pihat(\bx,1) = \hat \pi(\bx; \hat g)$.   By the definition of the one-step problem, we can define $\hat \pi(\bx,t)$ for $t > 1$ in a manner independent of the instance $(\pi,f) \in \cI$ and such that with probability one over $\bx_2,\bx_3,\dots$ under $\Pr_{\pi,f,\Dist}$,  $\hat \pi(\bx,t) = \pi(\bx,t)$. Doing so yields
\begin{align*}
\Rtrone(\est;g)^2 &= \Exp_{\bx \sim \Dist}\|\hat \pi(\bx,t = 1) - \pi_{g,\xi}(\bx,t = 1)\|^2 \\
&= \tau^2 \Exp_{\bx \sim \Dist}\|\bv \hat g(\projj(\bx)) + \pi_0(\bx) - (\bv g(\projj(\bx)) + \pi_0(\bx))\|^2\\
&= \tau^2 \Exp_{\bx \sim \Dist}\|\hat g(\projj(\bx)) -g(\projj(\bx))\|^2\\
&= \tau^2 \Exp_{\bz \sim \Dreg}\|\hat g(\bz) -g(\bz)\|^2 = \tau^2  \Rsl(\ghat;g)^2
\end{align*}
Hence,
\begin{align*}
\Rtr(\est;g) &= \Exp_{\Samp \sim \law(g)}\Exp_{\pihat \sim \est(\Samp)} \Rtrain(\hat \pi;g) \\
&= \tau\Exp_{\sampreg \sim \lawtil(g)}\Exp_{\ghat \sim \est(\sampreg)}  \Rsl(\ghat;g)  = \tau \Rsl(\estreg;g).
\end{align*}
Moreover, if $\estreg$ is proper, note that $\hat g \in \cG$ with probability one. Thus $\pihat(\bx;\ghat) = \bv \hat g(\projj(\bx)) + \pi_0(\bx)$ is equal to some $\pi_{\hat g,\xi}(\bx,t=1) \in \Pi_{\cI}$; moreover, by choosing $\hat \pi$ above to be equal to such a $\pi_{\hat g,\xi}(\bx,t=1)$, we can verify that $\pihat$ incurs no IL training error for $t > 1$, and $\Rtrain(\pihat;g) = \Rtrone(\pihat;g) = \tau\Rsl(\ghat;g)$,  The algorithm $\est$ constructed this way is now proper, and satisfies meaning that $\Rtr(\est;g)  = \tau\Rsl(\estreg;g)$.
\end{proof}

\begin{proof}[Proof of \Cref{prop:redux}]  We begin with \textbf{part (a)}. Recall the notation $\Rtr(\est;g)$ from \Cref{defn:framework_shorthand}, which reflects the fact that the IL training risk is the same for all instances of the form$(\pi_{g,\xi}, f_{g,\xi})$ for the same $g$ but differing $\xi$. By \Cref{lem:BC_to_reg_redux}, we have
\begin{align*}
\minbctrain(n;\inst,\Pnot) &= \inf_{\est}\sup_{g \in \cG}\Rtr(\est;g)\\
&\le \inf_{\estreg}\sup_{g \in \cG}\Rsl(\estreg;g) = \minsl(n;\cG,\Dreg).
\end{align*}
The reverse follows from from \Cref{lem:reg_to_BC_redux}, which establishes in fact that $\minsl(n;\cG,\Dreg) \le \mintrainone(n;\inst,\Pnot)$. As $\mintrainone(n;\inst,\Pnot) \le \minbctrain(n;\inst,\Pnot)$ (the former only considers loss from $h = 1$), we conclude that all three terms under consideration are equal. 

\paragraph{Part (b).} \Cref{lem:reg_to_BC_redux} gives
\begin{align*}
\minbctrain^{\bbA}(n;\inst,\Pnot) &= \inf_{\est \in \bbA}\sup_{g \in \cG}\Rtr(\est;g)\\
&\le \inf_{\est \text{ proper }}\sup_{g \in \cG}\Rtr(\est;g)\\
&\le \inf_{\estreg \text{ proper }}\sup_{g \in \cG}\Rsl(\estreg;g).
\end{align*}
When $\cG$ is convex, restriction to proper estimators does not change the minimax rate: 
\begin{align*}
\inf_{\estreg \text{ proper }}\sup_{g \in \cG}\Rsl(\estreg;g) = \inf_{\estreg}\sup_{g \in \cG}\Rsl(\estreg;g) = \minsl(n;\cG,\Dreg).
\end{align*} This follows because on can always project the estimated function $\ghat$ on $\cG$ in the metric $\|\cdot\|_{L_2(\Dreg)}$, which by the Pythagorean theorem and convexity of $\cG$ will never increase the loss. On the other hand, $\minbctrain^{\bbA}(n;\inst,\Pnot) \ge \minbctrain(n;\inst,\Pnot)$, and $\minbctrain(n;\inst,\Pnot) = \minsl(n;\cG,\Dreg)$ by the first statement of this lemma. Thus,
\begin{align*}
\minsl(n;\cG,\Dreg) =  \minbctrain(n;\inst,\Pnot)  \le \minbctrain^{\bbA}(n;\inst,\Pnot) =\minsl(n;\cG,\Dreg), 
\end{align*}
proving the desired equality.

\paragraph{Part (c).}  We start by using the fact that distribution of $\Samp$ does not depend the realization of $\xi$. Hence,  setting $\epsilon = \tau \kappa \beps_n$,
\begin{align*}
&\sup_{g,\xi}\Exp_{\Samp \sim (\pi_{g,\xi},f_{g,\xi})}\Exp_{\pihat \sim \est(\Samp)} \Risk(pi;g,\xi) \\
&\ge \sup_{g}\Exp_{\xi \sim P} \Exp_{\Samp \sim (\pi_{g,\xi},f_{g,\xi})}\Exp_{\pihat \sim \est(\Samp)} \Risk(\pi;g,\xi)\\
&\overset{(i)}{=} \sup_{g}\Exp_{\xi \sim P} \Exp_{\Samp \sim (\pi_{g,\xi_0},f_{g,\xi_0})}\Exp_{\pihat \sim \est(\Samp)} \Risk(\pi;g,\xi)\\
&\overset{(ii)}{=} \sup_{g} \Exp_{\Samp \sim (\pi_{g,\xi_0},f_{g,\xi_0})}\Exp_{\pihat \sim \est(\Samp)} \Exp_{\xi \sim P}\Risk(\pi;g,\xi),
\end{align*}
where $(i)$ uses the $\xi$-indistinguishability property (\Cref{defn:indistinguishable}) , and $(ii)$ is a consequence of Fubini's theorem.

Next, by \Cref{eq:C_cond_prob_lb}, we may lower bound $(ii)$ via 
\begin{align*}
&\inf_{\est \in \bbA}\sup_{g} \Exp_{\Samp \sim (\pi_{g,\xi_0},f_{g,\xi_0})}\Exp_{\pihat \sim \est(\Samp)} \Exp_{\xi \sim P}\Risk_{\beps_n \kappa \tau}(\pi;g,\xi) \\
&\quad\ge K \inf_{\est \in \bbA} \sup_{g} \Exp_{\Samp \sim (\pi_{g,\xi_0},f_{g,\xi_0})}\Exp_{\pihat \sim \est(\Samp)} \Pr_{\bx \sim \Dist, \bu \sim \pihat(\bx)}[|\langle \pi_{g,\xi_0}(\bx,t=1) - \bu, \bv \rangle| \ge \kappa \tau\beps_n]\\
&\quad\ge K \inf_{\estreg}\sup_{g_{\star} \in \cG} \Exp_{\sampreg}\Exp_{\hat{g} \sim \estreg(\sampreg)}\Pr_{\bz \sim \Dreg, \y \sim \ghat(\by)}[|\gst(\bz) - \by| \ge \kappa \beps_n],
\end{align*}
where the last line follows from  \Cref{lem:BC_to_reg_redux}, using convexity of $\cG$ and the fact that $\bbA$ contains all proper algorithms. Finally, \Cref{asm:conc} implies that the above is at least $K \delta$.
\end{proof}


\newcommand{\BigOhh}{O}
\section{Proof for Simple Policies, \Cref{thm:main_stable,thm:stable_detailed,thm:constant_prob}}\label{app:stable}

\newcommand{\epsZone}{\epsilon_{\{Z=1\}}}

In this section, we prove \Cref{thm:stable_detailed}. As noted below the statement of \Cref{thm:stable_detailed} in \Cref{sec:minmax_vanilla},  \Cref{thm:main_stable,thm:constant_prob} are direct consequences. Our aim is to make rigorous the intuitive proof sketched outlined in \Cref{sec:proof_intuition}, by carefully instantiating the reduction given in \Cref{prop:redux}. We encourage the review to review that proposition before continuing to read this section.
\newcommand{\Thetast}{\Theta_{\star}}

We recall our asymptotic notation: $a = \BigOh{b}$ to denote $a \le C b$ for some universal constant $C$, and  $a = \ost(b)$ to mean ``$a \le c \cdot b$ for $c$ sufficiently small.'' We will also write $a = \Theta(b)$ do denote that there exists universal constants $c_1,c_2 > 0$ such that $c_1 a \le b \le c_2 b$. Finally, we will use the notation $a = \Thetast(b)$ to denote $a = \ost(b)$ and $a = \Theta(b)$, that is, $a$ is smaller than a sufficiently small universal constant times $b$, but no more than a constant smaller. 

In what follows, \Cref{sec:lb_constr} provides the construction for the lower bound, explaining key simplifications and  motivations. \Cref{sec:stable_strategy} proceeds with a proof strategy. With this context, that section concludes by outlining the remainder of this Appendix, and describing the roles that the subsections that follow play in the overall proof.


\subsection{Lower Bound Construction}
\label{sec:lb_constr}

As in \Cref{sec:proof_intuition}, our lower bound forces the learner to make a single error at step $h = 1$, and shows that this error  compounds exponentially in $H$. To do so, we effectively ``patch together'' a region of space in which the learner needs to learn the embedded regression family $(\cG,\Dreg)$, and a region where the dynamics and optimal policy follow the linear construction detailed in \Cref{defn:chall_pair}.  We separate these regions via \emph{bump functions}, a construction widespread in mathematical analysis and statistical learning. We recall the salient properties of the bump function here.

\lembump*

Next, we introduce our formal construction. Recall that, in line with \Cref{prop:redux}, we parametrize our instance class with instances of the form $(\pi_{g,\xi},f_{g,\xi})$, where $g$ encodes the function to be estimated at the first time step, and $\xi$ parametrizes remaining uncertainty. 

\newcommand{\xoffs}{\bx_{\mathrm{offset}}}
\begin{construction}[Embedding Construction]\label{const:stable} Let $\tau,\Delta \in (0,1)$ be parameters to be chosen. We shall choose $\tau = \Thetast(1)$ and $\Delta = \Thetast\left(\frac{1}{ML \sqrt{d}}\right)$.  Define the matrices $\bbarA_i,\bbarK_i$ via
\begin{align*}
\quad \quad \bbarA_i := \begin{bmatrix} \bA_i & \bzero_{2\times d}  \\
    \bzero_{d \times 2} & \bzero_{d \times d} 
    \end{bmatrix}, \quad \bar\bK_i &= \begin{bmatrix} \bK_i & \bzero_{2\times d}  \\
    \bzero_{d \times 2} & \bzero_{d \times d}  \tag{Dynamical Matrices}
    \end{bmatrix},
\end{align*}
where above $\bA_i$ and $\bK_i$ are the matrices in \Cref{defn:chall_pair}, with $\mu \gets 1/4$. 
Furthermore,  let $\Proj_{\ge 3}$ denote the cannonical mapping from $\R^d$ to $\R^{d-2}$ which removes the first two coordinates.   Define the function $\restrict(\bx)$ and transformation $\cT[g]$ via
\begin{align*}
\restrict(\bx) := \bump_d(\bx - \xoffs), ~\xoffs := 3\be_3, \qquad \cT[g](\bx) :=  g(\Proj_{\ge 3} (\bx-\xoffs)), 
\end{align*} 
Let $\xi$ denote pairs $\xi = (i,\omega)$, where  $i \in \{1,2\}$, $\omega \in \{-1,1\}$. We define the instances $(\pi_{g,\xi},f_{g,\xi})$ via 
\begin{equation}\tag{Instance Class}
\begin{aligned}
\pi_{g,\xi}(\bx) &= \bbarK_{i} \bx + \tau\cdot \restrict(\bx)\cdot\cT[g](\bx)\be_1 \\
f_{g,\xi}(\bx,\bu) &= \bbarA_{i} \bx + \bu -  \tau\cdot \restrict(\bx)\cdot\cT[g](\bx)\be_1 \\
&\quad +  \omega\cdot\tau^2\cdot\restrict(\bx)\cdot \be_1\cdot\left(\cT[g](\bx) - \langle \be_1,\bu \rangle \bump_d(\bu)/\tau \right).
\end{aligned}
\end{equation}
Finally, given a $1$-bounded distribution $\Dreg$ on $\R^{d-2}$, define a distribution on $\R^{d}$ via,
\begin{align}
    \Pnot = \Pnot(\Dreg) \overset{d}{=}  \I\{Z=0\}\cdot (\xoffs + (0,0,\bz)) + \I\{Z=1\}(\Delta \cdot Y \cdot \bw) \tag{Initial State Distribution} \label{eq:smooth_Pnot}
\end{align}
where $Z \sim \mathrm{Bernoulli}(1/2)$,  $\bz  \in \R^{d-2} \sim \Dreg$,  $\bw$ is drawn uniformly on the unit ball supported on coordinates $2$-through-$d$: $\{\bw: \sum_{i =2}^d (\be_i^\top \bw)^2 \le 1\} $, $Y$ is a nonnegative random variable with $\Pr[Y = 1] = 1/2$ and $\Pr[ Y= 2^{-k}] \propto 1/k^2$ for $k \ge 1$, and where $(Z,Y,\bw,\bz)$  are independent random variables. 
\end{construction}

The difference between $Z = 0$ and $Z = 1$ cases is essential in the argument, which warrants us establishing a convenient shorthand. 
\begin{definition} We define the shorthand $\Disz$, $z \in \{0,1\}$ to denote the conditional distribution of $\Dist$ given $Z = z$. 
\end{definition} 

\paragraph{Explanation of \Cref{const:stable}}. The construction involves a number of daunting and complicated-seeming terms designed to carefully restrict the dynamics to ensure various global smoothness and stability properties, detailed in \Cref{ssec:stable_regular}. However, much is simplified by considering behavior of the dynamics at an initial state $\bx_1$ drawn from $\Dist$. 
\begin{claim}\label{claim:case_dynamics} Consider instance $(\pi_{g,\xi},f_{g,\xi})$ from \Cref{const:unstabl}, with $\xi = (i,\omega) \in \{1,2\} \times \{-1,1\}$. Let $\bx \sim \Dist$. If $Z = 0$ and $\bx = \xoffs + (0,0,\bz)$ for $\bz \in \R^{d-2}$, and let $\|\bu\| \le 1$. Then, 
\begin{align*}
\pi_{g,\xi}(\bx) &= \tau g(\bz)\be_1, \quad 
f_{g,\xi}(\bx,\bu) = \bu -(1 - \omega \tau)\be_1 \cdot (\pi_{g,\xi}(\bx) - \langle \be_1,\bu \rangle) .
\end{align*}
On the other hand, if $Z = 1$, then 
\begin{align*}
\pi_{g,\xi}(\bx) = \bbarK_i \bx, \quad f_{g,\xi}(\bx,\bu) =  \bu + \bbarA_i \bx.
\end{align*}  
\end{claim}
\begin{proof} When $Z = 0$, the $\restrict(\bx)$ term is equal to $1$, $\cT[g](\bx) = g(\bz)$. And when $\|\bu\| \le 1$, $\bump_d(\bu) = 1$. When $Z = 0$, the $\restrict(\bx)$ term is equal to zero. Applying these simplifications to \Cref{const:stable} establishes the claim.  
\end{proof}
It is now more transparent to see how \Cref{const:stable} implies the  plan described in \Cref{sec:proof_intuition}. The case $Z = 0$ is responsible for introducing statistical error which is to be compounded, and the case $Z = 1$ provides information about the linear regime of the expert and dynamics, but only along the subspace perpendicular to $\be_1$ (recall, $\bx_1 \mid Z = 1$ is distributed uniformly on the sphere on coordinates $2$-$d$). This forces the Jacobian of the mean of the learner policy to correspond to the $\bbarK_i$ matrices on that subspace. For the proof of the present theorem (\Cref{thm:stable_detailed}), we only leverage the $Y = 1$ subcase of $Z = 1$ to make this argument, but the $Y > 1$ cases are useful in the proof of \Cref{thm:stable_general_noise}, and to simplify the statements all theorems, we opted to allow the distribution $\Dist$ to be the same for both results.

Even with these simplifications, there is the additional parameter $\tau$, and indices $i$ and  $\omega$, that arise. The parameter $\tau$ is chosen to be sufficiently small that the nonlinear terms in the dynamics are overwhelmed by the linear terms. This ensures global exponentiall incremental stability. The indices $i \in \{1,2\}$ induce uncertainty over the challenging pair of dynamical system $(\bbarA_i,\bbarK_i)$, which embed the $(\bA_i,\bK_i)$ defined in \Cref{defn:chall_pair}. Finally, the parameter $\omega \in \{-1,1\}$ gives uncertainty over the sign of the error made along the $\be_1$ access when $Z = 0$.  Before continuing, we verfy that the matrices $\bbarA_i$ and $\bbarA_i+\bbarK_i$ are indeed stable.
\begin{lemma}\label{lem:Crho_Stab_constr} There exists some $C \ge 1$ and $\rho \in (0,1)$ such that both $\bbarA_i$ and $\bbarA_i+\bbarK_i$ are $(C,\rho)$ stable for $i \in \{1,2\}$.
\end{lemma}
\begin{proof} Recall that a matrix $\bA$ is $(C,\rho)$ stable if $\|\bA^k\|_{\op} \le C\rho^k$. As block-diagonal matrices are preserved under matrix powers, and operator norms decompose as maxima across blocks, we see that a block diagonal matrix $\bA$ is $(C,\rho)$-stable if and only if its blocks are. The top blocks of $\bbarA_i$ and $\bbarA_i + \bbarK_i$ are the matrices $\bA_i$ and $\bA_i + \bK_i$, whose stability is ensured by \Cref{lem:chall_pair}. The remaining block is the zero matrix, which is clearly $(C,\rho)$ for any $C , \rho \ge 0$. 
\end{proof}

\paragraph{Cost functions.}
\newcommand{\Vhard}{\overline{\cost}_{\mathrm{hard}}}

We will use a  \emph{single}, \emph{time-invariant} cost function which witnesses the separation between expert and imitator policies. The cost is constructed to carefully vanish on expert trajectories, whilst exposing large errors along the $\be_1$ direction.  In view of \Cref{lem:prob_min_clip}, we only need the cost to the maximum of costs which are nonnegative and Lipschitz, but not necessarily bounded above by $1$. 
\begin{construction}[Challenging Cost]\label{defn:challenging_cost} Let $C_{\Delta}$ be the universal constant in \Cref{lem:caseZ1_uninform}. We define
\begin{align*}
\chard(\bx,\bu) &= \Lcost|\langle\be_1, \bx\rangle|  + \Lcost\left(\|\bu - \bbarK_1\bx\| + \|\bu - \bbarK_2\bx\|\right)\bump\left(\frac{\bx}{2}\right)\\
&\quad+\Lcost \Delta\left(1-\bump(\bx - \xoffs)\right)\left(1-\bump\left(\frac{\bx}{C_{\Delta}\Delta}\right)\right)  \\
&\quad+\tau\Lcost(1 - \bump(\bu/\tau))\\
&\quad + \Lcost\left(\bump(\bx - \xoffs)\right)\|(\eye - \be_1 \be_1^\top)\bu\|.
\end{align*}
In terms of this, we we define
\begin{align*}
\Vhard(\bx_{1:H},\bu_{1:H}) := \max_{1 \le t \le H} \chard(\bx_t,\bu_t)  
\end{align*}
\end{construction}

We now show that the cost vanishes under the experts demonstration distribution, and is Lipschitz.
\begin{lemma}\label{eq:chard_vanish} For $\Delta = \ost(\tau), \tau = \ost(1)$,  it holds that $\Vhard \in \Cvan(\cI,\Dist)$, i.e., vanished on $(\cI,\Dist)$: 
\begin{align}
\sup_{(\pist,f) \in \cI} \Pr_{\pist,f,\Dist}[\Vhard(\trajj) \ne 0] = \sup_{(\pist,f) \in \cI} \Pr_{\pist,f,\Dist}[\exists t: \chard(\bx_t,\bu_t) \ne 0] = 0.
\end{align}
\end{lemma}
\begin{proof} Observe that  $\chard(\bx,\bu)$ with $\bu = \pi_{g,\xi}(\bx)$ vanishes whenever either $\bx$ is supported on coordinates $3,\dots,d$ (all the linear terms vanish) and in a unit ball around $\xoffs = 3\be_3$ (the bump function term vanishes, and on that ball, $\bu$ lies only in the $\be_1$ direction), or is supported on coordinates $2,\dots,d$ (the $\langle \be_1, \bx\rangle$ term vanishes) and lies in  of radius $\min\{C_{\Delta}\Delta,\ost(\tau)\}$ around the origin (the bump function terms vanish, and $\bbarK_i\bx$ is the same for both $i$, and $\|\bu\| \le \|\bbarK_i\bx\| \le \BigOhh(\tau)$ when $\|\bx\| \le \BigOhh(\tau) \le \ost(1)$). \Cref{lem:caseZ0} ensures the former situation under $Z = 0$ and time step $1$, and the latter under time steps $t > 1$, and \Cref{lem:caseZ1_uninform} ensures the latter under $Z = 1$ or $Z = 2$, for all timesteps.
\end{proof}

\begin{lemma}\label{lem:cost_Lipschitz} There is a choice of $\Lcost = \Theta(1)$ for which  $\chard$ is $1$-Lipschitz, and nonnegative. Hence, $\Vhard \in \Cliptilmax$. 
\end{lemma}
\begin{proof} This follows from the fact that bump functions are $\BigOhh(1)$-Lipschitz.
\end{proof}

\subsection{Overall Proof Strategy}\label{sec:stable_strategy} Recall that $\Disz
$ denoste the distribution of $\Dist$ conditioned on the event $\{Z=z\}$.  Our proof strategy is as follows:

\begin{itemize}
	\item The distribution $\Disz[1]$ forces any $\pihat$ with low error to satisfy $\pihat(\bx) \approx \pi_{g,\xi}(\bx) = \bbarK_i\bx$ on average over the $\Delta$-ball. By smoothness of $\pihat$, and by taking $\Delta$ to be of appropriate magnitude,  this forces the projection of the Jacobian of $\pihat$ along the directions spanned by the coordinates $\{2,3,\dots,d\}$ to match those of $\pist$. Technical details for this section are derived in \iftoggle{arxiv}{\Cref{ssec:Zone}}{\Cref{ssec:Z_cases}}, and rely on smoothness of $\mean[\pihat]$, as well as some convenient properties of expectations under the uniform distribution on the unit ball (notably, anti-concentration, which is slightly stronger than necessary for this argument, but ends up being useful in the proof of \Cref{thm:stable_general_noise}). This argument is similar in spirit to popular zero-order gradient estimators (see, e.g. \cite{flaxman2004online}), and role of the parameter $\Delta$ is to trade of between the quality of the Taylor approximation  (which improves for smaller $\Delta$) and effective variance of the Jacobian estimate (which, after appropriate  normalizatiom,  degrades with $\Delta$ small). 

	Fixing the Jacobian of $\mean[\pihat]$ ensures that $\nabla \mean[\pihat](\bzero)$ takes the form
	\begin{align}
	\nabla \mean[\pihat](\bzero) \approx \begin{bmatrix} \star & - c_{\mu} & \bzero \\
	\star & 0 & \bzero\\
	\star & \bzero & \bzero 
	\end{bmatrix} \label{eq:pihat_Taylor_approx}.
	\end{align}
	Following the argument of \Cref{lem:chall_pair}, this implies that, for at least one of $i \in \{1,2\}$ , 
	\begin{align}
	\bA_i + \nabla \mean[\pihat](\bzero) \approx \begin{bmatrix} a & 0 & \bzero \\
	\star & 1 - 2\mu & \bzero\\
	\star & \bzero & \bzero 
	\end{bmatrix},  |a| \ge 1 + \frac{\mu}{4}, \label{eq:bad_i_thing}
	\end{align}
	which is a matrix with single unstable eigenvector $\be_1$. 
	\item    The distribution $\Dist_{\{Z=0\}} = \Dist\big{|}_{Z=0}$ embeds the supervised learning problem associated with the class $\cG$. It is designed such that it conveys no further information about the parameters $\xi = (i,\omega)$. Moreover, by randomizing over the $\omega$, we force errors along the $\be_1$ direction. Specifically, for $\bx = (0,0,\bz) \sim \Disz[1]$ and for $\|\bu\| \le 1$ (otherwise, $\Vhard$ is large), \Cref{claim:case_dynamics} shows that 
	\begin{align}
	f_{g,\xi = (i,1)}(\bx,\bu) -f_{g,\xi = (i,-1)}(\bx,\bu) = 2\tau^2\cdot\left(g(\bz) - \langle \be_1,\bu \rangle \right)\be_1, \label{eq:starting_diff}
	\end{align}
	Hence, we make statistical errors along the $\be_1$ direction proportional to our mis-estimation of $g(\bz)$.

	Moreover, our construction ensures that the time step $t = 2$ is in the region in which the dynamics $f$ are given by the linear function. $f(\bx,\bu) = \bA_i + \bu$. These arguments are given in \iftoggle{arxiv}{\Cref{ssec:Zone}}{\Cref{ssec:Z_cases}}.

	\item We now invoke a quantitative variant of the unstable manifold theorem (\Cref{lem:exp_compound}), applying an argument similar to an efficient saddle-point escape introduced in \citep{jin2017escape} (but generalized to account for non-symmetric Jacobians and stripped of inessential details). This shows that  for the choice of $i$ for which \Cref{eq:bad_i_thing} holds, the autonomous dynamical system $F(\bx) = \bA_i\bx + \pihat(\bx)$ is exponentially unstable to perturbations along the $\be_1$ direction. Consequently, when $Z = 0$,  either the $(i,\omega = -1)$ and $(i,\omega = +1)$ dynamics divergence proportional to the estimation error of $g$, in view of \Cref{eq:starting_diff}. We emphasize that \Cref{lem:exp_compound} is the technical cornerstone of the entire lower bound argument. Building upon this argument, we establish a comprehensive statement of compounding error, whose presentation and proof are given in \Cref{ssec:compounding_error} 
	 
	\item To conclude, \Cref{ssec:stable_minimax} applies the reduction in \Cref{prop:redux} to show that,  the error at time step $t=1$ along $\be_1$ when $\{Z = 0\}$, which is proportional to $|\cT[g](\bx) - \langle \be_1,\bu \rangle |$, scales with the error of the embed regression problem, $\Omega(\minsl(n;\cG,\Dreg))$. We apply other ideas in that same reduction to relate the minimax risks under BC training, regression, and BC training restricted to estimators in $\bbA$.
	\item Finally, \Cref{ssec:stable_regular}  we verifies the various regularity conditions (smoothness, boundedness, stability). Here, the parameter $\tau$ plays a role in ensuring that the nonlinear terms do not overwhelm the stability guaranteed by the linear terms. 
\end{itemize}

\subsection{Analysis of the $Z \in \{0,1\}$  cases}\label{ssec:Z_cases} 

Here, we establish essential properties of the demonstration distribution, according to the value of the Bernoulli variable $Z$. 

\subsubsection{Case $Z=0$.}\label{ssec:Zzero} On these trajectories, the learner sees samples $\bz$ from the regression distribution $\Dreg$, embedded into dimension $d$ by appending two zero coordinates, and shifting by $3\be_1$.  These $\bx_1$ take the form $\bx_1= (0,0,\bz)$: their first two coordinates are zero, which implies that $\bbarK_{i} \bx_1 = 0$ and that expert policy selects $\pi_{g,\xi}(\bx) = \cT[g](\bx)\be_1 = g(\bz)\be_1$. Thus, the event $\{Z=0\}$ embeds the regression problem.

. Notice that the expert action is exactly canceled by the dynamics, as $\bbarA_{i} \bx_1  = 0$ (again, the first two coordinates of $\bx_1$ vanish). As $|g(\bz)| \le 1$, $\bu_1 = \pi_{g,\xi}(\bx_1)\be_1$ also satisfies $\bump_{d}(\bu_1) = 1$, and thus we find that for $\bx \gets \bx_1$ and $\bu \gets \pi_{g,\xi}(\bx_1)\be_1$,
\begin{align*}
f_{g,\xi}(\bx,\bu) &= \bbarA_{i} \bx + \underbrace{\bu -  \tau\cdot \restrict(\bx)\cdot\cT[g](\bx)\be_1}_{=0} \\
&\quad +  \omega\cdot\tau^2\cdot\restrict(\bx)\cdot \be_1\cdot\left(\underbrace{\cT[g](\bx) - \langle \be_1,\bu \rangle \bump_d(\bu/4)/\tau}_{=0} \right)\\
&= \bbarA_{i} \bx  = 0,
\end{align*}
where the last line uses the fact that $\bx$ is supported on the last $d-2$ coordinates, and the block structure of $\bbarA_i$. This establishes the following:
\begin{lemma} Conditioned on $Z = 0$, the expert trajectories take the form $\bx_1 = \xoffs + (0,0,\bz), \bz \sim \Dreg$, $\bu_1 = \tau \be_1 g(\bz)$, $\bx_h = \bu_h \equiv \bzero$ for $h > 1$. \label{lem:caseZ0}
\end{lemma}
In addition to characterizing the expert behavior on these trajectories, we also cheeck that unless $\Vhard$ grows large, the dynamics conditioned on $Z = 0$ are linear. 
\begin{lemma}\label{lem:value_lem_z_zero} Let $(\trajj)$ be a trajectory under dynamics $f_{g,i,\xi}$ for which $\bx_1 \in \mathrm{support}(\Dist \mid Z = 0)$. Suppose that 
\begin{align}
\Vhard(\trajj) \le \epsilon \le \ost(\tau),\label{eq:Vhard_bound}
\end{align}
Then, for all $2 \le t \le H$, we have 
\begin{align*}
\bx_{t+1} = \bbarA \bx_t + \bu_t, \quad \max\{\|\bx_t\|,\|\bu_t\|\} \le \BigOhh(\epsilon)
\end{align*}
\end{lemma}
\begin{proof} Assume \Cref{eq:Vhard_bound}. Define $\epsilon' = \epsilon/\Lcost \ge \epsilon$ . If $\epsilon' \le \tau \le 1$, then, from the definition of $\chard$, $\|\bu_1\| \le \tau$, so that (using $\bbarA_{i} \bx_1 =0$ when $\bx_1 \in  \mathrm{support}(\Dist \mid Z = 0)$)
\begin{align*}
\bx_2 &= \bbarA_{i} \bx_1 + \bu_1 + \omega\cdot\tau\cdot\restrict(\bx_1)\cdot \be_1\cdot\left(\tau\cT[g](\bx_1) - \langle \be_1,\bu_1 \rangle  \right), \\
&=  \bu_1 + \omega\tau\be_1\cdot\left(\tau\cdot\cT[g](\bx_1) - \langle \be_1,\bu_1 \rangle  \right),
\end{align*}
We also have $\|(\eye-\be_1\be_1^\top)\bx_2\| = \|(\eye-\be_1\be_1^\top)\bu_1\| \le \epsilon'$, and $|\langle \be_1,\bx_2 \rangle| \le \chard(\bx_2,\bu_2)/\Lcost \le \epsilon'$. Hence, $\|\bx_2\| \le 2\epsilon'$. Lastly, set $\|\delu_t\| = \|\bu_t -  \bbarK_i \bx_t\| \le \epsilon \le \epsilon'$. Then, if $\|\bx_2\|,\dots,\|\bx_t\|,\|\bu_2\|,\dots,\|\bu_t\| \le 1/2$ we have
\begin{align}
\bx_{t+1} &= (\bbarA_i + \bbarK_i) \bx_t + \delu_t  = \left(\sum_{i=2}^t((\bbarA_i + \bbarK_i))^{t-i}\delu_i\right) + (\bbarA_i + \bbarK_i))^{t-1}\bx_2.
\end{align}
By \Cref{lem:Crho_Stab_constr} $(\bbarA_i + \bbarK_i)$ is $(C,\rho)$-stable for universal $C,\rho \in (0,1)$. Thus,  bounding the geometric series and using the fact that the magnitude of  $\bx_2,\delu_i$ are atmost $2 \epsilon'$ and $\epsilon'$, respectively, we find
\begin{align}
\|\bx_{t+1}\| = \BigOhh(\epsilon')= \BigOhh(\epsilon),
\end{align}
where the $\BigOhh(\epsilon')$ hides a constant of $C/(1-\rho)$, not depending on $t$.  
Hence, for $\epsilon = \ost(1)$, we conclude that $\|\bx_{t}\| = \BigOhh(\epsilon)$ for all $2 \le t \le H$. Similarly, we have $\bu_t = \bbarK_i \bx_i + \delu_t = \BigOhh(\epsilon)$ for all $t$. Taking $\epsilon = \ost(1)$ enures that $\|\bx_{t+1}\|,\|\bx_{t+1} \le 1/2$, completing the induction. 
\end{proof}

\subsubsection{Case $Z = 1$ }
\label{ssec:Zone}

The purpose of the $Z = 1$ case is to force the Jacobian of the mean of the learner's policy to approximate $\bbarK_i$ on the subspace spanned by the cannonical basis vectors $\be_2,\dots,\be_d$. The following lemma makes this precise:
\begin{lemma}\label{lem:learn_Jacobian}Let $\Proj_{\ge 2}$ denote the projection onto coordinates $2$-through-$d$, and let $\pihat$ be any $M$-smooth simply-stochastic policy. Then, if
\begin{align}
\Pr_{\pihat,f_{g,(i,\omega)}}[\Vhard(\trajj) \ge M\Delta^2/2] \le \ost(1),
\end{align}
 we have the bound
$\|(\bhatK - \bbarK_i)\Proj_{\ge 2}\|_{\fro} \le 6M \Delta \sqrt{d}$. 
\end{lemma}
\newcommand{\Dzoneyk}[1][k]{\cD_{\{Z=1,Y=#1\}}}
\begin{proof}
Suppose $\Pr_{\pihat,f_{g,(i,\omega)},\Dist}[\Vhard(\trajj) \ge \epsilon] \le c_0$, where $c_0$ is a sufficiently small constant to be chosen. Let $\Dzoneyk$ is the distribution of $\bx \mid Z = 1, Y= k$. Because $\Pr[Z = 1, Y = 1] = 1/4$,
\begin{align}
\Exp_{\bx_1 \sim \Dzoneyk[1]}\Exp_{\bu \sim \pihat(\bx_1)} [\|\bbarK_i \bx_1 - \bu\| \ge \epsilon] \le 4c_0. 
\end{align}
By simple-stochasticity, there is a coupling $\Phat(\bx',\bx)$ over random inputs $\bu' \sim \pihat(\bx'),\bu \sim \pihat(\bx)$ where $\bu' - \bu = \mean[\pihat](\bx') - \mean[\pihat](\bx)$. Thus, by the triangle inequality and a union bound, we can symmetrize to obtain
\begin{align*}
&\Pr_{\bx,\bx' \iidsim \Dzoneyk[1]} [\|\bbarK_i (\bx' - \bx) - (\mean[\pi](\bx')-\mean[\pi](\bx))\| \ge 2\epsilon] \\
&= \Exp_{\bx',\bx \iidsim \Dzoneyk[1]}\Pr_{\bu',\bu \sim \Phat(\bx',\bx)} [\|\bbarK_i (\bx' - \bx) - (\bu' - \bu)\| \ge 2\epsilon]\\
&\le 2 \Exp_{\bx_1 \mid Z = 1}\Pr_{\bu \sim \pihat(\bx_1)} [\|\bbarK_i \bx_1 - \bu\| \ge \epsilon] \le 8c_0.
\end{align*}
Since $\Dzoneyk[1]$ has $\bx_1$ drawn from the uniform distribution on the unit ball over coordinates $2$-through-$d$, the result now follows from a technical \Cref{lem:sphere_diff} by taking $\epsilon = M\Delta^2/2$, and using the assumption that $\bx \mapsto \mean[\pi](\bx)$ is $M$-smooth.
\end{proof}

Whilst the  $Z = 1$ case forces the learner's policy to resemble $\bbarK_{i}$ on appropriate coordinate, it does so \emph{without conveying any information about the instance}.

\begin{lemma}\label{lem:caseZ1_uninform} There is a universal and dimension-independent constant $\Delta_0$ such that, if $\Delta \le \Delta_0$, the distribution of $(\bx_1,\dots,\bx_H)$ under $\Pr_{\pi_{g,\xi},f_{g,\xi},\Dist}[ \cdot \mid Z = 1]$ does not depend on $(g,\xi)$, and moreover, $\max_{t}\|\bx_t\| \le C_{\Delta} \cdot \Delta$, where $C_{\Delta}$ is a universal constant.

\end{lemma}
\begin{proof} The ``Moreover,'' part is clear from the construction. For the first part, there exists $(C,\rho)$ such that $(\bA_i + \bK_i)$ is $(C,\rho)$-stable for some $\rho < 1$ and $C < \infty$, and both of $i \in \{1,2\}$.  Using the block structure, this implies the same for $(\bbarA_i,\bbarK_i)$.  By inflating $C$ if necessary (note $\|\bbarK_i\| = \|\bK_i\|$ is dimension independent) we may ensure that $\sup_{n\ge 0}\{\|(\bbarA_{i}+\bbarK_{i})^n\bbarK_{i})\|, \|(\bbarA_{i}+\bbarK_{i})^n\bbarK_{i})\bbarK_{i}\| \le C$.  Hence, if we start at a state $\bx_1$ with $\|\bx_1\| \le 1/C$, we have that for either choosing of $i$, the linear dynamics $\btilu_h = \bbarK_i \btilx_h$, $\btilx_{h+1} = \bbarA_i \btilx_{h+1} + \btilu_h$, $\btilx_1 = \bx_1$ satisfy $\sup_{h \ge 1}\max\{\|\btilx_h\|,\|\btilu_h\|\} \le C\|\bx_1\| \le \Delta$. 
By construction, $\bx_{h+1} = \bbarA_i \bx_h + \bu_h$ and $\bu_h = \bK_i \bx_h$ obeys these same linear dynamics under the expert trajectory when $\max\{\|\bx_h\|,\|\bu_h\| \le 1\}$. In particular, when $\Delta$ is chosen to be less than $1/C$, we ensure these linear dynamics hold starting from $\bx_1 \mid Z = 1$. Note that such $\bx_1$ is also supported coordinates $2$-through-$d$, one can check inductively that $\bx_h, h > 1$ are also supported on these same coordinates, and that $(\bbarA_{i}+\bbarK_{i})\bx_h$ and $\bbarK_{i} \bx_h$ does not depend on $i$.
\end{proof}

\subsection{The compounding error argument.}  \label{ssec:compounding_error}

The goal of this section is to establish the following proposition. It establishes that, up to a threshold over $1/\poly(L,M,d)$, the probability of experiences exponential in $H$ compounding error is at least a constant times the probability that, under $\Disz[0]$, the learner makes a large mistake in the $\be_1$ direction. This nonlinear formalizes the heuristic argument given in \Cref{sec:proof_intuition}.
\begin{proposition}\label{cor:compound_max_v} Fix an $\epsilon_0$ and simply-stochastic, $L$-Lipschitz, $M$-smooth policy $\pihat$ (with $L,M \ge 1$). Suppose $\tau = \ost(1), \Delta = \Thetast\left(\frac{1}{ML \sqrt{d}}\right)$. Fix an $\epsilon > 0$ and $g \in \cG$. In terms of these, define
\begin{align*}
\epsilon_{\star} = \epsilon_{\star}(\epsilon_0) &:= \min\left\{ \ost\left(\frac{1}{L^2 M d}\right),\left(\frac{17}{16}\right)^{H-2} 2\tau\epsilon_0\right\}\\
p_{\star} = p_{\star}(\epsilon_0,g) &:= \Pr_{\bx_1 \sim \Disz[0],\bu \sim \pihat }\left[\left| \pist_{\xi_0,g}(\bx_1) - \left\langle \be_1,\bu \right\rangle\right| \ge \epsilon_0\right],
\end{align*}
where above we note that $p_{\star}$ does not depend on $\xi_0$ because $\pist_{\xi_0,g}(\bx_1)$ does not depend on $\xi_0$ when $\bx_1$ lies in the support of $\Disz[0]$.
Then,  
\begin{align}
\Exp_{(i,\omega) \sim P}\Pr_{\pihat,f_{g,(i,\omega)},\cD} \left[\Vhard(\trajj) \ge \epsilon_{\star} \right] \ge p_{\star}/C
\end{align} 
for some appropriate constant $C$. 
\end{proposition}

We prove \Cref{cor:compound_max_v} via a \Cref{lem:max_V_big} below. The remainder of the subsection will be dedicated to the proof of that lemma. 

\newcommand{\err}{\bm{\mathsf{err}}}
In what follows, we define two objects, parameterizerized in terms of the initial state $\bx_1$, and deviations $(\bzeta_t)$ from the mean. 
\newcommand{\traje}{\bm{\mathsf{traj}}}
\begin{definition} The trajectory induced by $\pihat$ conditioned on the random terms $\bzeta_{1:H}$:
\begin{align}
\traje_{g,(i,\omega)}(\bzeta_{1:H},\bx_1) = (\bx_{1:H},\bu_{1:H}), \quad \bu_h = \pibar(\bx_h) + \bzeta_h, \quad \bx_{h+1} = f_{g,(i,\omega)}(\bx_h,\bu_h)\label{eq:traj_dyns}
\end{align}
\end{definition}
\begin{definition}[First Stage Error] We define 
\begin{align*}\err(\bx_1,g,\bzeta_1) = \left| \tau\cdot\cT[g](\bx_1) - \langle \be_1,\pibar(\bx_1) + \bzeta_1 \rangle  \right|.
\end{align*}
\end{definition}

\begin{lemma}\label{lem:max_V_big} Let $\bx_1 \in \mathrm{support}(\Dist_{\{Z=0\}})$. Suppose $\tau = \ost(1), \Delta = \ost\left(\frac{1}{ML \sqrt{d}}\right)$, and 
\begin{align}
\max_{i,\omega} \Pr_{\pihat,f_{g,(i,\omega)}}[\Vhard(\trajj) \ge M\Delta^2/2] \le \ost(1), \label{eq:MDelsq}
\end{align}
 Then, for any choice of $g$ and any sequence $\bxi$, we have
\begin{align}
\max_{i,\omega}\Vhard(\traje_{g,(i,\omega)}(\bzeta_{1:H},\bx_1)) \ge \min\left\{ \ost(\min\{\tau,\sqrt{d}\Delta\}),\left(\frac{17}{16}\right)^{H-2} 2\tau\epsilon(\bx_1,g,\bzeta_1)\right\}
\end{align}
\end{lemma}
\begin{proof}[Proof of \Cref{cor:compound_max_v} assuming \Cref{lem:max_V_big}] Observe that under our parameter choices and $L, M \ge 1$, we can ensure
\begin{align*}
\epsilon_{\star} = \epsilon_{\star}(\epsilon_0,g) &= \min\left\{ \frac{M\Delta^2}{2}, \ost\left(\min\left\{\tau,\sqrt{d}\Delta\right\}\right),\left(1+\frac{\gamma}{2}\right)^{H-2} 2\tau\epsilon_0\right\}.
\end{align*}

We have two cases. First, let $c_{0} = \ost(1)$ be the constant implicit on the right hand side of \Cref{eq:MDelsq}.

\paragraph{Case 1:} $\Exp_{(i,\omega) \sim P} \Pr_{\pihat,f_{g,(i,\omega)}}[\Vhard(\trajj) \ge M\Delta^2/2] \ge c_0/4$. Then,  
\begin{align*}
\Exp_{(i,\omega) \sim P}\Pr_{\pihat,f_{g,(i,\omega)},\cD} \left[\Vhard(\trajj) \ge \epsilon_{\star}\right] \ge c_0/4 \ge p_{\star} c_0/4 \ge p_{\star}/C
\end{align*}
 for $C = 4/c_0$. 

\paragraph{Case 2:}  In the second case, we can assume that $\Exp_{(i,\omega) \sim P} \Pr_{\pihat,f_{g,(i,\omega)}}[\Vhard(\trajj) \ge M\Delta^2/2] \le c_0/4$, so that $\max_{(i,\omega)} \Pr_{\pihat,f_{g,(i,\omega)}}[\Vhard(\trajj) \ge M\Delta^2/2] \le c_0$.  Let us adopt the shorthand
\begin{align*}
\phi(\epsilon_0) =
\min\left\{ \ost(\min\{\tau,\sqrt{d}\Delta\}),\left(1+\frac{\gamma}{2}\right)^{H-2} 2\tau\epsilon_0\right\}
\end{align*}

It then follows from \Cref{lem:max_V_big} that, for any fixed $\bx_1 $ in the support of $\Disz[0]$ and noise sequence $\bzeta_{1:H}$ we have
\begin{align}
\Pr_{(i,\omega) \sim P} \left[\Vhard(\traje_{g,(i,\omega)}(\bzeta_{1:H},\bx_1)) \ge \phi(\epsilon_0)\right] \ge \frac{1}{4}  \text{whenever } \err(\bx_1,g,\bzeta_1) \ge \epsilon_0. \label{eq:key_eq}.
\end{align} 
\newcommand{\blank}{\cdot}
Thus, 
\begin{align*}
&\Exp_{(i,\omega) \sim P}\Pr_{\pihat,f_{g,(i,\omega)},\cD} \left[\Vhard(\trajj) \ge\phi(\epsilon_0)\right] \\
&\ge \frac{1}{2}\Exp_{(i,\omega) \sim P}\Pr_{\bx_1 \sim \Disz[0],\bzeta_{1:H}} \left[\Vhard(\traje_{g,(i,\omega)}(\bzeta_{1:H},\bx_1)) \ge \phi(\epsilon_0)\right] \\
&\ge \frac{1}{8}\Pr_{\bx_1 \sim \Disz[0],\bzeta_1}\left[\err(\bx_1,g,\bzeta_1) \ge \epsilon_0\right] \tag{by \Cref{eq:key_eq} and Bayes' rule}\\
&:= \frac{1}{8}\Pr_{\bx_1 \sim \Disz[0],\bu \sim \pihat }\left[\left| \tau\cdot\cT[g](\bx_1) - \langle \be_1,\bu \rangle\right|\ge\epsilon_0\right] \tag{definition of $\err(\bx_1,g,\bzeta_1)$}\\
&= \frac{p_{\star}(\epsilon_0,g)}{8} \tag{Definition of $p_{\star}$}
\end{align*}
Hence, the result holds for $C = 1/8$.
\end{proof}

\subsubsection{Proof of \Cref{lem:max_V_big}}
Our proof strategy is to adopt a quantitative variant of the stable manifold theorem, adapting an argument to due to \cite{jin2017escape}, and extending it to handle dynamical maps with non-symmetric gradients\footnote{In their paper \cite{jin2017escape}, the dynamical map in question arises from the gradient-descent update of a scalar-valued function say $h(\bx)$, so the gradient of the induced dynamical map is proportional the Hessian of $h(\bx)$, which is symmetric.} Informally, the stable manifold theorem considers a smooth dynamics map $F:\Xspace \to \Xspace$, whose gradient $\nabla F$ exhibits an eigenvalue strictly greater than one at the origin. The smoothness of the dynamics $F$ allow the approximation $F(\bx) \approx F(\bzero) + \nabla F(\bzero)\bx$ near the origin, which entails that the $k$-fold compositions $F^k(\bx)$ scale with $(\nabla F(\bzero))^k$ which, due to the unstable value, causes the state to grow exponentially.

For simply stochastic policies, we take (up to the additive noise $\bzeta_t$)
\begin{align}
F_{i}(\bx) := \bbarA_{i} \bx + \pibar(\bx), \label{eq:Fxi}
\end{align}
which is equal to $f_{\xi,g}(\bx,\pibar(\bx))$ is within the unit ball around $\bx = 0$. As $\bA := \nabla F_{i}(\bx)$ is a non-symmetric in general, its eigenvectors may be poorly conditioned. Thus, we argue that $\bA$ is approximately lower triangular,  with small top-right block, not-too-large bottom-left block, stable bottom-right block, and finally, an unstable (magnitude $>1$ entry) in the $(1,1)$-position. We define this structure as follows:

\newcommand{\radius}{\mathsf{radius}}

\parmatrix*

There is at least one choice of $i \in \{1,2\}$ for which $\nabla F_{i}(\bx)$ is of this form. 
\begin{claim}\label{claim:nice_form} Let $\bhatK := \nabla \pibar(\bx)\big{|}_{\bx = \bzero}$. 
Suppose that $\|(\bhatK - \bbarK_i)\Proj_{\ge 2}\|_{\fro} \le 6M \Delta \sqrt{d}$, which by \Cref{lem:learn_Jacobian} holds under the condition \Cref{eq:MDelsq}. Then, for $\Delta = \ost(1/LM\sqrt{d})$, there exists at least one of $i \in \{1,2\}$ for which $\nabla F_{i}(\bx)\big{|
}_{\bx =\bzero}$,  is a $(1/8,1/2,2L,r)$-matrix in \Cref{defn:compound_matrix}, where $r = \ost(1/L)$. 
\end{claim}
\begin{proof}[Proof of \Cref{claim:nice_form}] The argument generalizes the matrix argument given in \Cref{lem:chall_pair}.
Formally, set $\alpha = \be_1^\top \left(\nabla \pibar(\bx)\right)\be_1$ and set $\bm{\beta} = \bQ_{\ge 2} \nabla \pibar(\bx) \be_1$, where  $\bQ_{\ge 2}$ denote the projection from $\R^d \to \R^{d-1}$ which zeros out the first coordinate. Using the block structure of $\bbarA_{i},\bbarK_{i}$ and the computation in \Cref{lem:chall_pair}, we have
\begin{align*}
\nabla F_{i}(\bx) &= 
\bbarA_{i} + \bbarK_{i}\Proj_{\le 2} + \alpha \be_1\be_1^\top + \begin{bmatrix}0 \\ \bm{\beta} \\
 \end{bmatrix} +  \left(\nabla \pibar(\bx) - \bbarK_{i}\right)\Proj_{\ge 2}\\
&= 
\begin{bmatrix}(\bA_{i} + \bK_{i})_{11} + \alpha & \bm{\Delta}_{21}\\
\begin{bmatrix} (\bA_{i} + \bK_{i})_{21}\\
\bzero \end{bmatrix} + \bm{\beta} & \begin{bmatrix} 1 - 2\mu & \bzero \\
\bzero & \bzero \end{bmatrix} + \bm{\Delta}_{22}
\end{bmatrix},
\end{align*}
where 
\begin{align*}
&\max\left\{\|\bm{\Delta}_{21}\|, \|\bm{\Delta}_{22}\|\right\}  \le \|\left(\nabla \pibar(\bx) - \bbarK_{i}\right)\Proj_{\ge 2}\| \le 6M \Delta\sqrt{d}, \quad \\
&\|\bm{\beta}\| \le \|\nabla \pibar(\bx)\|_{\op} \le L.
\end{align*}
Arguing as in \Cref{lem:chall_pair}, we have that $(\bA_{i} + \bK_{i})_{11} \in \{1+ \mu, -(1- \frac{1}{4}\mu)\}$, so that there exists one of $i \in \{1,2\}$ for which $|(\bbarA_{i} + \bbarK_{i})_{11} + \alpha| \ge 1 + \frac{\mu}{4}$. Moreover, $|(\bA_{i} + \bK_{i})_{11}| \in \{|-c_{\mu}|,|0|\} \le |c_{\mu}| = \frac{3}{2}\mu \le 2 \mu$, and  $\|\bm{\beta}\| \le L$ when $\pibar(\bx)$ is $L$-Lipschitz. Finally, as $\mu \le 1$, assuming $6M\Delta \sqrt{d} \le \mu$ and $L \ge 2 \mu$, we conclude that$\nabla F_t(\bx)$ (which again, does not depend on $t$) admits the following block decomposition:
\begin{align*}
& \nabla F_{i}(\bx) := \begin{bmatrix} y_1 & \bW^\top \\
\tilde \bW & \bY_{[2]}
\end{bmatrix},
\end{align*}
where $|y_1| \ge 1 + \frac{\mu}{4}$, $\|\btilW\|_{\op} \le 2L$,  $\|\bY_{[2]}\|_{\op} \le 1 - \mu$ and $\|\bW\|_{\op} \le 6M\Delta \sqrt{d}$. We conclude by setting $\gamma \gets \mu/4$, $\mu \gets \mu$, and $L \gets 2L$ in the definition of a $(\gamma,\mu,L)$-matrix. To ensure the bound on $r$, we use \Cref{lem:learn_Jacobian}, which ensures that  $6 M \Delta \sqrt{d} \le \ost(1/L)$. 
\end{proof}

Next, we recall our major technical tool, which is a quantitative statement of an unstable-manifold theorem for dynamics whose Jacobian is a $(\mu,\gamma,L,\Delta)$-matrix.   
\propexpcompound*


\newcommand{\Dzzero}{\Dist_{\{Z=0\}}}
\newcommand{\atxzero}{\big{|}_{\bx = \bzero}}
We may now conclude the proof of \Cref{lem:max_V_big}.
\begin{proof}[Proof of \Cref{lem:max_V_big}] Fix $\bzeta_{1:H}$. We may assume that $\max_{i,\omega} \Vhard(\traj_{g,(i,\omega)}(\bzeta_{1:H},\bx_1)) \le \min\{\tau,\Delta\}\Lcost$, otherwise the result is  immediate. Let $(\bx_{i,\omega;t})$ denote the sequence of iterates given by the dynamics in \Cref{eq:traj_dyns}, namely by $\bu_h = \pibar(\bx_h) + \bzeta_h, \bx_{h+1} = f_{g,(i,\omega)}(\bx_h,\bu_h)$.

By \Cref{lem:value_lem_z_zero} and $\tau = \ost(1),\Delta = \ost(1/ML\sqrt{d}) = \ost(1)$, $\max_{i,\omega} \Vhard(\traj_{g,(i,\omega)}(\bzeta_{1:H},\bx_1)) \le \min\{\ost(\traj),\ost(\sqrt{d}\Delta)\}$  implies that 
\begin{align}\max_{2 \le t \le H}\|\bx_{i,\omega;t}\| \le \ost(\sqrt{d}\Delta). \label{eq:leDelta_thng}
\end{align} By \Cref{claim:nice_form}, we may choose an $i \in \{1,2\}$ for which $\bA_i + \nabla\pibar(\bx) \atxzero$ is a $(\gamma,\mu,L,r) = (1/8,1/2,L,\ost(1/L))$-matrix. Then, using the linearity of dynamics ensured by \Cref{lem:value_lem_z_zero},
\begin{align*}
\bx_{i,\omega;t+1} &= \bbarA_i + \pibar(\bx_{i,\omega;t}), \quad \omega \in \{-1,1\}\\
\bx_{i,1;2} - \bx_{i,-1;2} &= \pm 2\tau \epsilon(\bx_1,g,\bzeta_1), 
\end{align*}
where $\pm$ denotes an arbitrary choice of sign.
\Cref{lem:exp_compound} implies and the fact that $\bA_i + \nabla\pibar(\bx) \atxzero$ is a $(1/8,1/2,L,\ost(1/L))$-matrix  implies $2 \le t \le H$ for which either $\max_{\omega \in \{-1,1\}}|\be_1^\top\bx_{i,\omega;t+1}| \ge (\frac{17}{16})^{H-2} 2\tau \epsilon(\bx_1,g,\bzeta_1)$, or $\max_{\omega \in \{-1,1\}}\max_{2 \le t \le H} \|\bx_{i,\omega;t+1}\| \ge \Omega(\frac{1}{LM})$. By making $\Delta = \ost(1/LM\sqrt{d})$, the second case cannot occur without contradicting \Cref{eq:leDelta_thng}.
\end{proof}


\subsection{Proof of Minimax Risk Bounds in \Cref{thm:stable_detailed}}\label{ssec:stable_minimax}
In what follows, and in keeping with \Cref{const:stable}, we let again $\xi = \{i,\omega\}$ denote the hidden parameter, so policies and dynamics are of the form $f_{g,\xi}, \pi_{g,\xi}$. As established in \Cref{lem:caseZ1_uninform,lem:caseZ0}, the distribution over samples does not depend on the instance label $(g,\xi)$ for $Z = 1$, and one can verify that the properties, \Cref{defn:orthogonal,defn:single_step,defn:indistinguishable} all hold. We use these properties in what follows. Lastly, we remark that our arguments extends to the class of proper algorithms by noticing that $\bbA = \Asmooth(L,M)$ contains all proper estimators, provided $L \ge L_0$ and $M \ge M_0$; this is a consequence of the fact our construction uses deterministics policies and dynamics, and the smoothness/Lipschitzness computations of \Cref{ssec:stable_regular}.

\begin{lemma}\label{lem:modified_redux_stable} Let $\cI,\Dist$ be as in \Cref{const:stable}, and recall that $\Disz[0]$ denotes the conditional of $\Dist \mid \{Z=0\}$. Then, 
\begin{itemize}
\item[(a)] The BC problem class $(\cI,\Disz[0])$ satisfies the general reduction conditions of all part of \Cref{prop:redux}: namely \Cref{defn:indistinguishable,defn:orthogonal,defn:single_step}, with parameter $\tau$, the class $\cG$ is convex (by assumption), and $\bbA$ contains all estimation algorithms. 
\item[(b)] Let $(\pist,f) \in \cI$. Given a sample $\Samp^{(Z=0)}$ of $n$, length $H$ trajectories from $\Pr_{\pist,f,\Disz[0]}$, one can simulate a sample $\Samp$ of $n$, length $H$ trajectories from $\Pr_{\pist,f,\Dist}$.
\end{itemize}
\end{lemma}
\begin{proof} Part (a) can be easily checked from \Cref{lem:caseZ0} and going through the various conditions. The properness of $\bbA$ follows from \Cref{lem:full_regularity}.  Part (b) follows from the fact that, from initial states in the support of $\Disz[1]$, the distribution of the trajectories is identical for all $(\pist,f) \in \cI$ (\Cref{lem:caseZ1_uninform}). 
\end{proof}
\subsubsection{Lower bound on the training risks.} Let $\nhat\sim \mathrm{Binomial}(\frac{1}{2},n)$. Because we samples collect from  $\{Z = 1\}$-trajectories can be simulated without knowledge of the ground truth instance  (\Cref{lem:modified_redux_stable}(b)), we can decompose 
\begin{align*}
\minbctrain(n; \cI, \Dist,\cH) &= \frac{1}{2}\Exp_{\hat n}[\minbctrain(\hat n; \cI, \Dzzero,\cH) ]\\
\minbctrain^{\bbA}(n; \cI, \Dist,\cH) &= \frac{1}{2}\Exp_{\hat n}[\minbctrain(\hat n; \cI, \Dzzero,\cH) ], \\
\minbcevalhb^{\bbA}(n; \cI, \Dist,\cH) &= \frac{1}{2}\Exp_{\hat n}[\minbcevalhb^{\bbA}(n; \cI, \Dist,\cH)],
\end{align*}
which follows by conditioning on the number sampled trajectories for which $Z = 0$, which is distributed as $\hat n$. From \Cref{lem:modified_redux_stable}(a), the conclusion of \Cref{prop:redux} holds with parameter $\tau$. Invoking that proposition,
\begin{align*}
\forall \nhat \in \N, \quad \minbctrain(\nhat; \cI, \Dzzero,\cH)  = \minbctrain^{\bbA}(\nhat; \cI, \Dzzero,\cH) =  \tau\minsl(\nhat; \cG, \Dreg). 
\end{align*}
Taking an expectation over $\nhat$ and using the previous display proves the equality of $\minbctrain(n; \cI, \Dist,\cH) $ and $\minbcevalhb^{\bbA}(n; \cI, \Dist,\cH)$. The same also holds any $\bbA' \supseteq \Aproper(\cI)$, and in particular, for $\Aproper(\cI)$.

To upper bound the relevant terms, a chernoff bound on $\nhat \sim \mathrm{Binomial}(1/2,n)$ implies that that with probability $1 - e^{-c' n}$ for some $c' > 0$, we have $\nhat \ge n/3$. And, when $\nhat \ge n/3$, then because an estimator with fewer samples can always be simulated via an estimator with more samples, we have $\minbcevalhb^{\bbA}(\hat n; \cI, \Dzzero,\cH) \le \minbcevalhb^{\bbA}(n/3; \cI, \Dzzero,\cH)$ when $\nhat \ge n/3$. On the other hand, when $\nhat < n/3$, because all terms in \Cref{const:stable} remain uniformly bounded, there exists an estimator which makes at most constant error, say $C' > 0$. Hence, 
\begin{align*}
\Exp_{\nhat}[\minbctrain^{\bbA}(\nhat; \cI, \Dist,\cH]) &\le \Exp_{\nhat}[\minbctrain^{\bbA}(\nhat; \cI, \Dzzero,\cH]) \\
&\le \Pr_{\nhat}[\nhat \ge n/3] \minbctrain^{\bbA}(n/3; \cI, \Dzzero,\cH) + C'\Pr[\nhat < n/3] \\
&\le \minbctrain^{\bbA}(n/3; \cI, \Dzzero,\cH) + C' e^{- c'n}\\
&= \tau\minsl(n/3; \cG, \Dreg) + C' e^{- c'n}. 
\end{align*} 
For a lower bound on the training risk, we see that 
\begin{align*}
\minbctrain^{\bbA}(n; \cI, \Dist,\cH) &= \frac{1}{2} \Exp_{\hat{n}}[\minbctrain^{\bbA}(\nhat; \cI, \Dzzero,\cH]) \\
&\ge \frac{1}{2} \minbctrain^{\bbA}(n; \cI, \Dzzero,\cH) = \frac{\tau}{2} \minsl(n;\cG,\Dreg), \numberthis \label{eq:train_lb}
\end{align*}
where we use the fact that $\nhat \le n$ and that (because one can always neglect to use some samples) minimax risks are nonincreasing in $n$. We conclude this section by noting that $\tau = \Thetast(1)$ in \Cref{const:stable}.


\subsubsection{Lower bound on the evaluation risk.}\label{lb:eval_risk} 
\newcommand{\epscompound}{\epsilon_{\mathrm{compound}}}

We apply \Cref{cor:compound_max_v}  with $\epsilon_0 \gets \tau \kappa \beps_n$, where $\tau$ is as in the construction, and, $\kappa$ and $\beps_n$ are as in \Cref{asm:conc} on the problem class $(\cG,\Dreg)$. 
Let $P$ denote the uniform prior on $(i,\omega) \in \{1,2\} \times \{-1,1\}$Assume $\tau = \ost(1)$ and $\Delta = \ost\left(\frac{1}{ML \sqrt{d}}\right)$.  and for the terms
\begin{align*}
\epscompound &:= \min\left\{ \ost\left(\frac{1}{L^2 M d}\right),\left(1+\frac{\gamma}{2}\right)^{H-2} 2\tau^2 \beps_n \kappa \right\}\\
\end{align*}
Finally,  introduce the risk $\Risk_{\star}(\pihat,g,\xi) := \Pr_{\pihat,f_{g,\xi},\cD} \left[\Vhard(\bx_t,\bu_t) \ge \epscompound\right]$. Then, \Cref{cor:compound_max_v}  implies
\begin{align*}
\Exp_{\xi \sim P}\Risk_{\star}(\pihat,g,\xi) \ge \frac{1}{C} \Pr_{\bx_1 \sim \Disz[0],\bu \sim \pihat }\left[\left| \pihat_{g,\xi_0}(\bx_1) - \left\langle \be_1,\bu \right\rangle\right| \ge \tau \kappa \beps_n\right].
\end{align*} 
for some appropriate constant $C$. Above, we note $\pihat_{g,\xi_0}(\bx_1) = \tau\cdot\cT[g](\bx_1)$ for any choice of $\xi_0$ when $\bx_1 \sim \Disz[0]$. In light of \Cref{lem:modified_redux_stable}(a), we may apply \Cref{prop:redux} to the problem instance $(\cI,\Disz[0])$. This implies that 
\begin{align}
\inf_{\est}\sup_{g,\xi}\Exp_{\Samp \sim (\pi_{g,\xi},f_{g,\xi},\Disz[0])}\Exp_{\pihat \sim \est(\Samp)} \Risk_{\star}(\pi; g,\xi) \ge \delta/C.
\end{align}
Finally, from \Cref{lem:modified_redux_stable}(b), the $n$ samples from $\Disz[0]$ can simulate $n$ samples from the unconditioned distribution, $\Dist$. Thus, any estimator can do no better taking samples from $\Dist$:
\begin{align*}
\inf_{\est}\sup_{g,\xi}\Exp_{\Samp \sim (\pi_{g,\xi},f_{g,\xi},\Disz[0])}\Exp_{\pihat \sim \est(\Samp)} \Risk_{\star}(\pi; g,\xi) \ge \delta/C.
\end{align*}
Substituting in our definition of $\Risk_{\star}(\pihat,g,\xi) := \Pr_{\pihat,f_{g,\xi},\cD} \left[\Vhard(\bx_t,\bu_t) \ge \epscompound \right]$ and the definition of $\epscompound$ implies that 
\begin{align*}
\minprob\left(n,\frac{\delta}{C};\cI,\Dist,H\right) \ge \min\left\{ \ost\left(\frac{1}{L^2 M d}\right),\left(\frac{17}{16}\right)^{H-2} 2\tau^2 \beps_n \kappa \right\},
\end{align*}
Tuning $\tau = \Thetast(1)$ and absorbing constants concludes the demonstration.

\subsection{Regularity Conditions}\label{ssec:stable_regular}
The goal of this section is to establish the following:
\begin{lemma}\label{lem:full_regularity} Suppose $\cG$ satisfies \Cref{asm:stable_reg}. Then, provided $\tau$ is smaller than a universal constant, we have that $(\cI,\Dist)$ is ($\BigOhh(1),\BigOhh(1),\BigOhh(1)$)-regular, and for all $(\pi,f) \in \cI$, $f$ and $(\pi,f)$ are $(C',\rho')$-E-IISS for some $C' \ge 1,\rho' \in (0,1)$. In particular, for $\bbA = \Avan(L,M)$ for some sufficiently large constants $L, M \ge 1$, $\bbA$ is proper. Lastly, all $f$ as in \Cref{const:stable} are $O(1)$-one-step-controllable.
\end{lemma}
This result is a consequence of  the arguments that follow, with controllability deferred to \iftoggle{arxiv}{\Cref{sec:is_controllable}}{the end of this section}. 
Recall from \Cref{const:stable} the functions $[\cT(g)](\bx) :=  g(\Proj_{\ge 3} \bx)$ and $\restrict(\bx) := \bump_d(\bx - 3 \be_1)$
Lets introduce the shorthand 
\begin{align*}
\psi_g(\bx) := \restrict(\bx)\cdot\cT[g](\bx), \quad \psi_u(\bu,\bx) := \langle \be_1,\bu \rangle \bump_d(\bu/4)\restrict(\bx).
\end{align*} 
Then, we can write
\begin{equation}
\begin{aligned}
\pi_{g,\xi}(\bx) &= \bbarK_{i} \bx + \tau\psi_g(\bx)\be_1 \\
f_{g,\xi}(\bx,\bu) &= \bbarA_{i} \bx + \bu -  \tau\psi_g(\bx)\be_1 + \omega\cdot \be_1 (\tau^2\psi_g(\bx) - \tau \psi_u(\bu,\bx)).  \label{eq:const_summary}
\end{aligned}
\end{equation}
\begin{claim}\label{claim:psi_smooth} Suppose that each $g \in \cG$ is $L_0$-Lipschitz, $M_0$-smooth, and $1$-bounded on the ball of radius $2$ on $\R^{d-2}$.  Then, letting $\BigOhh(\cdot)$ hide universal constants, 
\begin{itemize}
	\item $\psi_g(\bx)$ is $\BigOhh(L_0+1)$-Lipschitz and $\BigOhh(1+L_0+M_0)$-smooth. 
	\item $\psi_{u}(\bx,\bu)$ is $\BigOhh(1)$-Lipschitz and $\BigOhh(1)$-smooth. 
\end{itemize}
\end{claim}
\begin{proof} Recall that the bump-functions have derivatives bounded by universal constants (\Cref{lem:bump}). Hence, the desired bounds follow from the product rule, and the fact that $\cT[g]$ inherits the smoothness/Lipschitzness of $\cG$, and the fact that $\restrict(\bx)$ constrains to a ball of radius $2$.  
\end{proof}

The following lemma gives us the desired regularity guarantee.
\begin{lemma}[Regularity]\label{lem:reg} Let $\tau \le 1$.  Suppose that each $g \in \cG$ is $L_0$-Lipschitz, $M_0$-smooth, and $1$-bounded on the ball of radius $2$ on $\R^{d-2}$. Then, every $(\pi,f) \in \cI$ are are $\BigOhh(L_0+1)$-Lipschitz and $\tau \cdot \BigOhh(1+L_0+M_0)$-smooth. Hence, $(\cI,\Dist)$ is ($\BigOhh(L_0+1),\BigOhh(1+L_0+M_0),\BigOhh(1)$)-regular.
\end{lemma}
\begin{proof} Follows from \Cref{eq:const_summary}, \Cref{claim:psi_smooth}, the chain rule, and the fact that $\|\bbarK_i\|,\|\bbarA_i\|$ are bounded by universal consants. The regularity statement requires further verifying that all trajectories remain bounded, which follows from \Cref{lem:caseZ0,lem:caseZ1_uninform}. 
\end{proof}

\subsubsection{Stability of the Construction}\label{ssec:stable_stability}

We use the following result, whose proof is deferred to \Cref{sec:inc_stab_taylor}.
\incstabtaylor*

\begin{lemma}\label{lem:stable_construction}
There exists universal constants $c' >0$, $C \ge 1$ and $\rho \in (0,1)$ such that, if each $g$ is $L_0$-Lipschitz, and $\tau \le c'\min\{1,1/L_0\}$, then for all $(\pi,f) \in \cI$, $f$ and $(\pi,f)$ are globally IISS with $\beta(r,k) = r\cdot C\rho^k$ and $\gamma(r) = C r$. 
\end{lemma}
\begin{proof} Let $(\pi,f) \in \cI$
\begin{align*}
f(\bx,\bu) &= \bbarA_{i} \bx + \bu - \tau\psi_g(\bx)\be_1 + \omega\cdot\tau\cdot \be_1 (\psi_g(\bx) - \psi_u(\bu,\bx))\\
f^{\pi}(\bx,\bu) &= (\bbarA_{i} +\bbarK_i)\bx  + \bu + \omega \be_1 (\tau^2\psi_g(\bx) - \tau\psi_u(\bbarK_i\bx + \tau \psi_g(\bx) \be_1 + \bu,\bx)).
\end{align*}
Following the proof of \Cref{lem:reg}, we surmise that 
\begin{align*}
\|\nabla_{\bx} f(\bx,\bu) - \bbarA_i\| \vee \|\nabla_{\bx} f^{\pi}(\bx,\bu) - (\bbarA_{i} +\bbarK_i)\| &\le \epsilon_{\nabla,\bx} =  \BigOhh(\tau(1+L_0))\\
\|\nabla_{\bu} f(\bx,\bu) \| \vee \|\nabla_{\bu} f^{\pi}(\bx,\bu)\| &\le L_{\nabla,\bu} = \BigOhh(1 + \tau) \le \BigOhh(1).
\end{align*}
The result now follows by observing that $\bbarA_i$ and $(\bbarA_{i} +\bbarK_i)$ are both $(C,\rho)$-stable for some $C \ge 1, \rho \in (0,1)$. Hence, chosing $\tau \le \ost(1/(1+L_0))$, we ensure that have $C\epsilon_{\nabla,\bx} \le (1 + \rho)/2 < 1$. The result now follows from \Cref{lem:inc_stab_taylor}.
\end{proof}

\subsubsection{Controllability}\label{sec:is_controllable}
For functions $\phi(\bx),\psi(\bx,\bu)$ different than those defined above, 
we can still  express $f(\bx,\bu)$ in \Cref{const:stable}, as $f(\bx,\bu) = \phi(\bx) + \psi(\bx,\bu) + \bu$, where $\phi(\bx)$ is $O(1)$-Lipschitz by the computations above, and $\psi(\bx,\bu) =  \omega\cdot\tau\cdot\restrict(\bx)\cdot \be_1\cdot\left( \langle \be_1,\bu \rangle \bump_d(\bu) \right)$. 

Note that $\omega \in \{-1,1\}$, and as $\bump_d(\bu)$ is $O(1)$-Lipschitz and  $\restrict(\bx)$  is $O(1)$-bounded, we can make $\psi(\bx,\bu)$, say, $1/2$-Lipschitz by taking $\tau = \ost(1)$. Moreover, we clearly also have $\psi(\bx,\bu = \bzero) = \bzero$. Hence,  the conditions of \Cref{lem:one_step_controllable} are met to ensure $O(1)$-one-step-controllability.



\newcommand{\Txs}{\mathscr{T}_{\bx}}
\newcommand{\Tys}{\mathscr{T}_{\by}}
\newcommand{\Cstab}{C_{\mathrm{stab}}}
\newcommand{\clip}{\mathrm{clip}}
\newcommand{\btily}{\tilde{\by}}
\newcommand{\Normal}{\mathrm{Normal}}
\section{Proof for Non-Simple Policies, \Cref{thm:non_simple_stochastic,thm:stable_general_noise}}\label{app:general_noise}

 In this section, we prove \Cref{thm:stable_general_noise}. As noted below the statement of \Cref{thm:stable_general_noise} in \Cref{sec:minmax_anticonc}, \Cref{thm:non_simple_stochastic} follows as as direct consequence.

We begin by recalling the asymptotic notation in \Cref{defn:polyost}. Given $b_1,b_2,\dots \le 1$, we use the notation $a = \polyost(b_1,b_2,\dots,b_k)$ to denote that $a \le c_1 (b_1\cdot b_2 \cdot b_k)^{c_2}$, $c_1$ is a sufficiently small universal constant, and $c_2$ a sufficiently large universal constant. We also recall that we consider the class $\bbA = \Areason(L,M,\alpha,p)$ (\Cref{defn:gen_smooth}) of algorithms which, with probability one, return stochastic, Markovian policies $\pi$ for which $\mean[\pi](\bx)$ is $L$-Lipschitz and $M$-smooth, and $\pi$ is $(\alpha,p)$-anti-concentrated.

\paragraph{Orgnization of the section.} In the section below, we give an overview of the proof of \Cref{thm:stable_general_noise}. We then give natural examples of anti-concentrated policies in \Cref{sec:ex:anti_concentrated}.  We then turn in to proving the truncation lemma, \Cref{ssec:trunc}, and establishing useful consequences. We then briefly generalize the Jacobian estimation lemma, \Cref{lem:learn_Jacobian}, in \Cref{sec:lem:learn_Jac_general}. Penultimately, we provide a statement and proof of  compounding error with anti-concentrated policies in \Cref{sec:prop_compounding_proof}. Finally, in  \Cref{sec:proof:general_noise}, we rigorously conclude the proof of \Cref{thm:stable_general_noise}. \Cref{thm:non_simple_stochastic} is a corollary of \Cref{thm:stable_general_noise}, as noted in \Cref{app:minmax}.

\subsection{Proof Overview}
The construction is identical to the \Cref{const:stable} used in the proof of \Cref{thm:stable_detailed}.  In particular, the regularity conditions all hold, as do the relations between $\minsl$, $\minbctrain$, and $\minbctrain^{\bbA}$ established in \Cref{thm:main_stable}. Our aim is to establish instead the compounding error guarantee, \Cref{eq:anticonc_compounding}, which we restate here for convenience. 
\begin{align}
\minmaxltwo^{\bbA}(n;\cI,\Dist,H) \ge  c \kappa \cdot \delta \beps_n\cdot \min\left\{ 1.05^{H-2}, (1/\beps_n)^{\frac{1}{C'(1+\log(1/(\alpha p)))}}\right\}. \tag{\Cref{eq:anticonc_compounding}}
\end{align}

To this end,  we need to modify the two arguments from the proof of \Cref{thm:stable_detailed} which required to simply-stochasticity. We instead replace these with arguments that rely on the more general anti-concentration condition (\Cref{defn:anti-concentrated_policy}). For convenience, we recall the relevant definitions here. 
\defanticonc*
\defanticoncpol*

To reiterate, there are  two arguments in need of ammending.  Both arguments appeal to the following property of anti-concentrated random variables, whose proof and useful consequences are deferred to \Cref{ssec:trunc}. This property states that if if a random variable $X'$ dominates in magnitude the sum of anti-concentrated random variable $Z$ and any constant offset, then the expectation of a sufficiently lenient truncation of $X'$ is still large in expectation.
\begin{lemma}[Truncation]\label{lem:truncc} Suppose that $Z$ is scalar, mean zero and $(\alpha,p)$-anti-concentrated random variable, $x$ a deterministic scalar, and  $X'$ a random scalar satisfying, with probability one,
\begin{align*}
|X'| \ge | x + Z|.
\end{align*}
Then, for any $\eta \in (0,1)$, setting $B(\eta)  =  \frac{5}{\eta \alpha^2 p^2}$, we have 
\begin{align*}
\Exp[\min\{B(\eta) |x|,|X'|\}] \ge (1 - \eta)|x|
\end{align*}
\end{lemma}

Next, the first argument to ammend is the one that forces $\nabla \mean[\pihat](\bzero)\Proj_{\ge 2} \approx \bbarK_i \Proj_{\ge 2}$ (\Cref{lem:learn_Jacobian}). Building on \Cref{lem:truncc},  it is  is straightforward to generalize this to the anti-concentrated setting, and this step is carried out by \Cref{lem:learn_Jac_general} in \Cref{sec:lem:learn_Jac_general}. 

The more challenging argument to generalize is the compounding error argument. Our new proof here generalizes mirrors the what occurs in the benevolent gamblers ruin example in \Cref{sec:benevolent_ruin}. Leveraging  \Cref{lem:truncc}, we carefully truncate the sequence $(\bx_1,\bx_2,\dots)$ to form a sequence $(\by_1,\by_2,\dots)$ such that $\by_t \equiv \bx_t$ with good probability, that $\Exp[|\langle \be_1, \by_{t+1}\rangle|] \ge \rho_1 |\langle \be_1, \by_{t}\rangle|$, and at the same time, $|\langle \be_1, \by_{t+1}\rangle| \le  \rho_2 |\langle \be_1, \by_{t}\rangle|$, where $1 < \rho_1 < \rho_2$. In particular, if $|\langle \be_1, \by_{t}\rangle| = \epsilon$, we must have
\begin{align}
\Exp|\langle \be_1, \by_{t+1}\rangle| \ge \rho_1^t \epsilon, \quad \text{and} \quad |\langle \be_1, \by_{t+1}\rangle| \le \rho_2^{t} \epsilon \text{ w.p. } 1. 
\end{align}

These two bounds imply that $|\langle \be_1, \by_{t+1}\rangle| \ge \rho_1^t \epsilon$ with some probability roughly $(\rho_1/\rho_2)^t$. By a Markov's inequality argument, this yields that $\Exp[|\langle \be_1, \by_{t+1}\rangle|] \ge \epsilon (\rho_1^2/\rho_2)^t$. Unfortunately, this argugment does not quite work as is because, in general $\rho_2 \gg \rho_1^2$. However, we show a careful modification applies, provided that we can instead lower bound $\Exp[|\langle \be_1, \by_{t+1}\rangle|^2]^{1/2}$, which can better take advantage of the heavy tails of $|\langle \be_1, \by_{t+1}\rangle|$. The argument is carried out in \Cref{sec:prop_compounding_proof}.

\subsection{Examples of Anti-Concentrated Policies}\label{sec:ex:anti_concentrated}

Before providing examples, we establish a few useful facts about anti-concentrated random variables. 
\begin{lemma}[Anti-Concentration via Tail Bounds]\label{lem:tail_anti_conc} Let $Z$ be a mean-zero scalar random variable satisfying $\Exp[Z^4] \le c\Exp[Z^2]^2$. Then, $Z$ is $(\frac{1}{\sqrt{2}},\frac{1}{4c})$-anti-concentrated. 
\end{lemma}
We note the next three lemmas use the `$Z$' notation we have been using for scalar random variables, but apply to vector-valued ones by taking projections along vector-directions.
\begin{proof} The Paley-Zygmund inequality (\Cref{lem:pz}) implies that $\Pr[ Z^2 \ge \theta\Exp[Z^2]] \ge (1-\theta)^2\frac{\Exp[Z^2]^2}{\Exp[Z^4]}]$ for any $\theta \in (0,1)$. Taking $\theta = 1/2$ proves the statement. 
\end{proof}
\begin{lemma}\label{lem:gaussian_anti_conc} Any Gaussian random vector is $(\frac{1}{\sqrt{2}},\frac{1}{12})$-anti-concentrated. 
\end{lemma}
\begin{proof} For Gaussian random vectors, it suffices to establish the case where $Z \sim \cN(0,1)$ (by taking vector directions, scaling, and re-centering). In this case, $3\Exp[Z^2]^2 = 3 = \Exp[Z^4]$, so \Cref{lem:tail_anti_conc} applies with $c = 3$.
\end{proof}
\begin{lemma}\label{lem:discrete_anti_conc} Let $Z$ be discretely distributed on set $\{z_1,z_2,\dots,z_m\}$, and let $p = \min_{1 \le i \le m} \Pr[Z = z_i]$. Then, $Z$ is $(1,p)$-anti-concentrated. In particular, a Dirac-delta is $(1,1)$-anti-concentrated.
\end{lemma}
\begin{proof} We may assume without loss of generality that $\Exp[Z] = 0$. For this recentering, let $i_{\star} := \argmax_{1 \le i \le m}|z_i|$. Then, $\Exp[Z^2]^{1/2} = \sqrt{\sum_i \Pr[Z = z_i] z_i^2} \le |z_{i_\star}|$, and $\Pr[Z = z_{i_{\star}}] \ge p$.
\end{proof}

\begin{lemma}\label{lem:mixture_anti_conc} 
Generalizing \Cref{lem:discrete_anti_conc}, let $Z$ be drawn from a discrete mixture of random variables $Z_i$ with mixture weights $p_i$, each satisfying $p_i \ge p_{\min}$, and which each $Z_i$ $(\alpha,p)$-anti-concentrated for some $\alpha \le 1$, and is either mean-zero, or symmetric about its mean.  Then, $Z$ is $(\alpha, p \cdot p_{\min}/2)$ anti-concentrated.
\end{lemma}
\begin{proof} Again, by taking projections along unit vectors, we may assume the variables are scalar and centered such that $Z$ has mean zero, and set $i_{\star} := \argmax_{i} \Exp[|Z_i|^2]$. Then, 
\begin{align}
\Exp[Z^2] \le \Exp[|Z_{i_{\star}}|^2].
\end{align}
If $Z_{i_{\star}}$ has mean zero, then $\Pr[| Z_{i_{\star}}| \ge \alpha\Exp[Z_{i_{\star}}^2]^{1/2}] \ge p $ as $Z_{i_{\star}}$ is $(\alpha,p)$ anti-concentrated. Otherwise, suppose without loss of generality that $\Exp[Z_{i_{\star}}] > 0$, let $\tilde Z_{i_{\star}} = Z_{i_{\star}} - \Exp[Z_{i_{\star}}]$. By the assume of the lemma, we may take $\tilde Z_{i_{\star}}$ to be symmetric. Then,  $\Pr[\tilde Z_{i_{\star}} \ge \alpha\Exp[\tilde Z_{i_{\star}}^2]^{1/2}] = \frac{1}{2}\Pr[|\tilde Z_{i_{\star}}| \ge \alpha\Exp[\tilde Z_{i_{\star}}^2]^{1/2}] \ge p/2$. Thus, 
\begin{align*}
\Pr[ Z_{i_{\star}} \ge \alpha \sqrt{\Exp[Z_{i_{\star}}^2]}] &= \Pr[ \tilde Z_{i_{\star}} + \Exp[Z_{i_{\star}}] \ge \alpha \sqrt{\Exp[\tilde Z_{i_{\star}}^2 + \Exp[Z_{i_{\star}}]^2]}] \\
&\ge \Pr[ \tilde Z_{i_{\star}} + \Exp[Z_{i_{\star}}] \ge \alpha \sqrt{\Exp[\tilde Z_{i_{\star}}^2]}+ \alpha|\Exp[Z_{i_{\star}}]|]  \tag{$\sqrt{x+y} \le \sqrt{x}+ \sqrt{y}$}\\
&\ge \Pr[ \tilde Z_{i_{\star}}  \ge \alpha \sqrt{\Exp[\tilde Z_{i_{\star}}^2]}]  \tag{$\alpha \le 1$, $\Exp[Z_{i_{\star}}] > 0$ by assumption}\\
&\ge p/2 \tag{established above}.
\end{align*}
In both cases, we obtain $\Pr[ Z_{i_{\star}} \ge \alpha \sqrt{\Exp[Z_{i_{\star}}^2]}]  \ge p/2$. Hence, $\Pr[Z \ge \sqrt{\Exp[Z^2]}] \ge \Pr[Z \ge \sqrt{\Exp[Z_{i_{\star}}^2]}] \ge \Pr[Z = Z_{i_{\star}}] \Pr[Z_{i_{\star}} \ge \sqrt{\Exp[Z_{i_{\star}}^2]}] \ge p_{\min} p /2$.
\end{proof}

We now list a number of examples of anti-concentrated properties, illustrating that the condition is natural and easy to meet.
\begin{example}[Simply Stochastic Policies] Any simply stochastic policy is $(1,1)$ anti-concentrated, because there exists a coupling $P$ of $\pi(\bx)$ and $\pi(\bx')$ under which $(\bu,\bu') \sim P(\bx,\bx')$ ensures $\bu - \bu'$ is deterministic. This is the coupling which sets  $\bu = \mean[\pi](\bx) + \bzeta$ and $\bu' = \mean[\pi](\bx') + \bzeta$, where $\bzeta$ is the noise distribution. Implicitly, this is the coupling we use in the proof of \Cref{thm:stable_detailed}. In particular, discrete policies are anti-concentrated.
\end{example}
\begin{example}[Gaussian Policies] Gaussian policies are also anti-concentrated. Consider any $\pi$ of the form $\pi(\bx) = \Normal(\mean[\pi](\bx),\Sigma(\bx))$, and let $P(\bx,\bx') = \pi(\bx) \otimes \pi(\bx')$ denote the independent coupling. Then, $(\bu,\bu') \sim P(\bx,\bx')$ is jointly Gaussian, and thus so is $\bu - \bu'$. Hence, it is $(\frac{1}{\sqrt{2}},\frac{1}{12})$-anti-concentrated by \Cref{lem:gaussian_anti_conc}.
\end{example}
\begin{example}[Benevolent Gambler's Ruin Policy]\label{exmp:BGR_pol} Recall the benevolent gambler's ruin policy from \Cref{sec:benevolent_ruin}. At each point, the policy is a mixture of two Dirac-distributions, each with probability $1/2$. Hence, under the independent coupling, $P(\bx,\bx') = \pi(\bx) \otimes \pi(\bx')$, $\bu - \bu'$ is a mixture of at most $4$ Dirac-deltas, each with probability at least $1/2$. Hence, it is $(1,1/4)$-anti-concentrated by \Cref{lem:discrete_anti_conc}
\end{example}
\begin{example}[Mixture of Gaussian Policies] If $\pi(\bx)$ is point-wise a mixture of Gaussians, with minimimal probability of each component $p$, then under the independent coupling $P(\bx,\bx') = \pi(\bx) \otimes \pi(\bx')$, $\bu - \bu'$ is a mixture of Gaussians with minimal component probability at least $p^2$. Moreover, each component distribution,  being a sum of two Gaussians, is Gaussian and thus both symmetric and $(\frac{1}{\sqrt{2}},\frac{1}{12})$-anti-concentrated by \Cref{lem:gaussian_anti_conc}. Thus, the mixture is $(\frac{1}{\sqrt{2}},\frac{p^2}{24})$-anti-concentrated by \Cref{lem:mixture_anti_conc}.
\end{example}

\subsection{The Truncation Lemma (\Cref{lem:truncc}) and Its Consequences}\label{ssec:trunc}
We prove the core truncation lemma, and then state and prove two useful corollaries. 
\begin{proof}[Proof of \Cref{lem:truncc}] Let $\Delta = \Var[Z]$, and assume $x > 0$ without loss of generality (the $x < 0$ follows by symmetry, and $x = 0$ case can be checked directly). We consider two cases.   First, assume $\Delta \ge C |x|$, where we pick $C  = \frac{2}{cp}$. Let $\cE = \{|Z| \ge \alpha Cx\}$. On $\cE$, we have
\begin{align*}
|X'| \ge |Z| - x - \epsilon \ge \alpha C x - (x) \ge (\frac{2}{p}x-x) \ge x/p.
\end{align*}
Therefore,
$\Exp[\min\{|X'|,x/p] \ge \Pr[E]x/p \ge x$.

Next, assume $\Delta \le \frac{2(1+\gamma) x}{\alpha p}$. Then, 
\begin{align*}
\Exp[\min\{|X'|,Bx + x\}] &= \Exp[\min\{| x + \sigma Z|,Bx + x\}]\\ 
&\ge \Exp[\I\{|Z| \le Bx]\min\{|  x + \sigma Z|,B + x\}]\\
&\ge \Exp[\I\{|Z| \le Bx]\} ((1+\gamma) x + \sigma Z)]\\
&= x + \sigma\Exp[\I\{|Z| \le Bx]\} Z]\\
&= x - \sigma\Exp[\I\{|Z| > Bx]\} Z]\\
&\ge (x - \Exp[\I\{|Z| > Bx]\} |Z|]. 
\end{align*}
We bound $\Exp[|Z|\I\{Z > Bx\}] \ge \int_{Bx}^{\infty} \Pr[|Z| \ge t] \le \int_{Bx}^{\infty}\frac{\Exp[Z^2]}{t^2} = \frac{\Exp[Z^2]}{Bx} \le \frac{\Delta^2}{Bx}$. Subsituting in $\Delta \le \frac{2 x}{\alpha p}$, we get
\begin{align*}
\Exp[|Z|\I\{Z > Bx\}] \le \frac{4 x}{\alpha ^2 p^2 B}.
\end{align*}
If we take $B = \frac{4}{\eta \alpha^2 p^2}$ for $\eta \le 1$, we get $\Exp[|Z|\I\{Z > B\}] \le \eta x$, and hence
\begin{align*}
\Exp[\min\{|X'|,Bx + x\}] \ge (1-\eta)x.
\end{align*} 
substituting $Bx +x \le \frac{5x}{\eta \alpha^2 p^2}$ concludes. 
\end{proof}
\begin{corollary}\label{cor:truncc} Suppose that $Z$ is a mean zero and $(c,p)$-anti-concentrated scalar random variable, $x$ a deterministic scalar, and $X'$ a scalar random variable. Suppose further that for $\gamma > 0$ and $\epsilon \ge 0$, the following holds with probability one:
\begin{align*}
|X'| \ge | x(1+\gamma) + Z| - \epsilon
\end{align*}
Then, we have
\begin{align*}
\Exp\left[\min\left\{\left(\frac{40\max\{\gamma,\gamma^{-1}\}}{\alpha^2p^2}\right)|x|,|X'|\right\}\right] \ge (1  + \gamma/2)|x|  \epsilon
\end{align*}
\end{corollary}
\begin{proof} By applying \Cref{lem:truncc} to the random variable $|X'| + \epsilon$ and setting $B \gets  \frac{5}{\eta \alpha^2 p^2}$, then 
\begin{align*}
\epsilon + \Exp[\min\{B(1+\gamma)|x|,|X'|\}]&=  \Exp[\min\{B(1+\gamma)|x| +\epsilon,|X'| +\epsilon\}] \\
&\ge \Exp[\min\{B(1+\gamma)|x|,|X'| + \epsilon\}] \ge (1+\gamma)(1-\eta)|x|,
\end{align*} 
or rearranging, 
\begin{align*}
\Exp[\min\{B(1+\gamma)|x|,|X'|\}]&=  \Exp[\min\{B(1+\gamma)|x| +\epsilon,|X'| +\epsilon\}] \ge  (1+\gamma)(1-\eta)|x| - \epsilon.
\end{align*} 
Take $\eta$ to be such that $(1+\gamma)(1-\eta) = (1+\gamma/2)$, or $\eta = 1 - \frac{1+\gamma/2}{1+\gamma} = \frac{\gamma}{2(1+\gamma)}$. Then, $B(1+\gamma) = \frac{10(1+\gamma)^2}{\alpha^2p^2\gamma} \le \frac{20(\gamma^{2}+1)}{\alpha^2p^2\gamma} = \frac{40\max\{\gamma,\gamma^{-1}\}}{\alpha^2p^2}$. 
\end{proof}
\begin{corollary}\label{lem:cor_trunc_hp} Suppose that $Z$ is a mean zero and $(c,p)$-anti-concentrated scalar random variable, $x$ a deterministic scalar, and $X'$ a scalar random variable. Furthers suppose that, with probability one,
\begin{align*}
|X'| \ge | x + Z| ,
\end{align*}
Then, $\Pr[|X'| \ge |x|/4] \ge \alpha^2 p^2/40$. 
\end{corollary}
\begin{proof} From \Cref{lem:truncc}, we have  $(1-\eta)|x| \le \Pr[|X'| \ge t|x|]$. We have $\Exp[\min\{B|x|,|X'|\}] \le B|x|\Pr[|X'| \ge t|x|] + t|x|\Pr[|X'| \ge t|x|] \le B|x|\Pr[|X'| \ge t|x|] + t |x|$. Setting $t = \eta$, we have 
\begin{align*}
(1-2\eta)|x| \le B|x|\Pr[|X'| \ge t|x|], \quad \Pr[|X'| \ge \eta|x|] \ge \frac{(1-2\eta)}{B} = \frac{(1-2\eta)\eta c^2 p^2}{5}.
\end{align*}
Taking $\eta = 1/4$, the above probability is at least $\alpha^2 p^2/40$.
\end{proof}
\subsection{Derivative Estimation under Anti-Concentration (Case $Z = 1$)}\label{sec:lem:learn_Jac_general}
In this section, we generalize the derivative estimation arguments of \Cref{lem:learn_Jacobian} from simply-stochastic policies to anti-concentrated ones. 
\begin{lemma}\label{lem:learn_Jac_general}Let $\Proj_{\ge 2}$ denote the projection onto coordinates $2$-through-$d$, and let $\pihat$ be any policy with  $M$-smooth which is $(\alpha,p)$ anti-concentrated (recall \Cref{defn:anti-concentrated_policy}) satisfying
\begin{align}
\Pr_{\pihat,f_{g,(i,\omega)}}[\Vhard(\trajj) \ge M(2^{-k}\Delta)^2/8] \le \ost(\alpha^2p^2/k^2),
\end{align}
 we have the bound
$\|(\bhatK - \bbarK_i)\Proj_{\ge 2}\|_{\fro} \le 8\sqrt{d}M \Delta 2^{-k} $. 
\end{lemma}
\begin{proof} Recall the distribution  $\Dzoneyk$ asthe distribution of $\bx \mid Z = 1, Y= k$. Because $\Pr[Z = 1, Y = k] \propto 1/k^2$, then if $\Pr_{\pihat,f_{g,(i,\omega)},\Dist}[\Vhard(\trajj) \ge \epsilon] \le c_0/k^2$. Then, arguing as in  \Cref{lem:learn_Jacobian}, we can start with
\begin{align*}
\left(\Exp_{\bx_1 \sim \Dzoneyk}\right)\Exp_{\bu \sim \pihat(\bx_1)} [\|\bbarK_i \bx_1 - \bu\| \ge \epsilon] \le O(c_0)
\end{align*}
Consider the coupling $(\bx,\bu)$ and $(\bx',\bu')$ with $\bx,\bx' \sim \Dzone$ and $\bu,\bu' \sim \pi(\bx),\pi(\bx')$ where $\bx,\bx'$ are independent and $\bu,\bu' \sim \Phat(\bx_1,\bx_1')$. By the triangle inequality and a union bound, we can symmetrize to obtain
\begin{align*}
\Exp_{\bx,\bx' \sim \Dzone}\Exp_{\bu',\bu \sim \Phat(\bx',\bx)} [\|\bbarK_i (\bx' - \bx) - (\bu'-\bu)\| \ge 2\epsilon] \le O(c_0)
\end{align*}
And thus, for all unit vectors $\bv$,
\begin{align*}
\Exp_{\bx,\bx' \sim \Dzone}\Exp_{\bu',\bu \sim \Phat(\bx',\bx)}[|\langle \bv, \bbarK_i (\bx' - \bx) - (\bu'-\bu)\rangle| \ge 2\epsilon] \le O(c_0)
\end{align*}
We may write $\bbarK_i (\bx' - \bx) - (\bu'-\bu) = \bbarK_i(\bx' - \bx) - \mean[\pihat](\bx') - \mean[\pihat](\bx) + \bz$, where $\langle \bv, \bz\rangle$ is $(\alpha,\rho)$ anti-concentrated. It follows from \Cref{lem:cor_trunc_hp}, a corollary of the main truncation lemma \Cref{lem:truncc},  that
\begin{align*}
\Exp_{\bx,\bx' \sim \Dzone}[|\langle \bv, \bbarK_i (\bx' - \bx) - \mean[\pihat](\bx') - \mean[\pihat](\bx)\rangle| \ge 8\epsilon] \le \BigOh{\frac{c_0}{\alpha^2 p^2}}.
\end{align*}
The result now follows by taking $\epsilon \le M(2^{-k}\Delta)^2/8$ and invoking \Cref{lem:sphere_diff}, whose conditions are met as soon as  $\frac{c_0}{\alpha^2 p^2} = \ost(1)$, i.e. $c_0 = \ost(\alpha^2 p^2)$.

\end{proof}

\subsection{The Compounding Error Argument (\Cref{prop:compounding_general_thing})}\label{sec:prop_compounding_proof}

This section establishes a general compounding error argument for anti-concentrated policies. We recall $\Vhard$ as the cost from \Cref{defn:challenging_cost} in \Cref{app:stable}. We show that the probability $\Vhard$ exceeds some threshold is sufficiently small (otherwise, of course, large error occurs), then we still observe a compounding error phenomenon.
\begin{condition}\label{condition:learn_general} Let $P$ be the uniform distribution over $\xi = (i,\omega) \in \{1,2\} \times \{-1,1\}$. For a given $g \in \cG$, we will assume that
\begin{align}
\Exp_{\xi \sim P} \Pr_{\pihat,f_{\xi,g},\Dist}\left[\Vhard(\bx_{1:H},\bu_{1:H}) \ge \epsilon^{.9}\right] \le \epsilon^{.18}/4. \label{eq:Vhard_case_one}
\end{align}
We will further assume that $\epsilon = \polyost(\alpha,p,1/L,1/M,\tau,1/d,\kappa,\delta)$ (recall: this means that $\epsilon$ is smaller than some polynomial of sufficiently high degree and with sufficiently small coefficients in these terms). 
\end{condition}
The goal of this section is to establish the following.
\begin{proposition}\label{prop:compounding_general_thing} Suppose \Cref{condition:learn_general} holds. Define 
\begin{align*}
K(\epsilon,H) := \min\left\{(1.05)^{H-2},\epsilon^{-\frac{1}{C'(1+\log(1/(\alpha p)))}} \right\}.
\end{align*} 
Then, we have 
\begin{equation}\label{eq:under_condition_compound}
\begin{aligned}
&  \Exp_{\xi \sim P}\Exp_{\pihat,f_{g,\xi},\Dist} \left[ \epsilon^{0.85} \wedge \Vhard(\bx_{1:H},\bu_{1:H})\right] \\
&\qquad\ge \frac{\Lcost}{4} K(\epsilon,H)\Exp_{\xi \sim P}\Exp_{\pihat,f_{g,\xi},\Dzzero}\left[\epsilon^{.9} \wedge 2\tau\left| \angs{\be_1}{\bu_1 - \pihat_{g,(\cdot)}(\bx_1)}\right|\right] - 4\epsilon^{1.03}.
\end{aligned}
\end{equation}
\end{proposition}

In what follows, for our given policy $\pihat$, we set 
\begin{align}
\bhatK := \nabla \mean[\pihat](\bz)\big{|}_{\bz = \bzero}.
\end{align}

\paragraph{Properties of the linearized closed-loop system.} We apply \Cref{lem:learn_Jac_general} with $k  = \log_2(6\sqrt{d}M\Delta/\epsilon^{.4})$. Using $\Delta = \Thetast\left(\frac{1}{ML \sqrt{d}}\right)$ from \Cref{const:stable},  taking $\beps = \polyost(\mathrm{problem}\,\mathrm{parameters})$ to be sufficiently small, and invoking \Cref{condition:learn_general},  we can make the  following hold:
\begin{claim}\label{claim:off_diag_small} Under \Cref{condition:learn_general}, we have that  
\begin{align*}
\|\be_1^\top (\bbarA_i + \bhatK)\Proj_{\le 2}\| \le \epsilon^{0.4}.
\end{align*}
\end{claim}

Following the proof of \Cref{claim:nice_form}, there exists an index $i$ for which the $(1,1)$-entry of the closed loop linearized system $(\bbarA_i + \bhatK)$ has magnitude greater than one. This will be the entry responsible for the large compounding error. 

\newcommand{\ibad}{i_{\mathrm{bad}}}
\begin{claim}\label{claim:one_one_entry} Under \Cref{condition:learn_general}, there exists an index $\ibad \in \{1,2\}$ for which $|\be_1^\top (\bbarA_{\ibad} + \bhatK)\be_1| :=  1+ \gamma$, where  $\gamma = 1/16$, and $1 + \gamma \le 2 + L$.
\end{claim}
\begin{proof} The first part follows from an argument as in \Cref{claim:nice_form}. We also notice that $(1+\gamma) \le |\bbarA_i[1]| + \|\nabla \mean[\pihat](\bx) \big{|}_{\bx = 0}\|_{\op} \le 2 + L$ by Lipschitzness of $\mean[\pihat]$.
\end{proof}

\newcommand{\dynmap}{F}
\paragraph{Trajectory Coupling.} The next step is to define a coupling of two trajectories generated by $\pihat$ on the $\{Z=0\}$ case, both under the dynamics associated with $\ibad$, but under a the different values of $\omega = \pm 1$.
\begin{definition}[The ``plus-and-minus'' sequence]\label{defn:plus_minus} Given index $i \in \{1,2\}$ chosen above, and $g \in \cG$ fixed, let $(\bx_t^{+},\bx_{t}^{-})$ denote a joint sequence defined as follows:
\begin{align*}
&\bx_1^+ \equiv \bx_1^- \sim \Dzzero, \quad\bu_1^+ \equiv \bu_1^- \sim \pihat(\bx_1^+), \quad  \bu_t^+,\bu_t^- \sim \Phat(\bx_t^+,\bu_t^+), t > 1\\
&\bx_{t+1}^+ = f_{g,(\ibad,\omega = +1)}(\bx_t^+,\bu_t^+), \quad \bx_{t+1}^- = f_{g,(\ibad,\omega = -1)}(\bx_t^-,\bu_t^-).
\end{align*}
We let $\Txs$ denote the random variable with distribution $(\bx_{2:H}^+,\bx_{2:H}^-)$. 
\end{definition}
The trajectories defined above make the same initial mistake at $t = 1$ but, due to differences in $\omega$, these mistakes are multiplied by opposite directions. See \Cref{const:stable} to that, when $\|\bu_1\| \le 1$, we have
\begin{align}
\langle \be_1, \bx_{2}^+ - \bx_{2}^-\rangle = 2\tau \langle \be_1, \bu_1 - \pihat_{g,(\cdot)}(\bx_1)\rangle, \label{eq:two_taust}
\end{align}
where $\bx_1 = \bx^+_1 \equiv \bx^-_1 $, $\bu_1 = \bu^+_1 \equiv \bu^-_1$, and where $(\cdot)$ above follows from the fact that, when $\bx_1 \sim \Dzzero$, $\pihat_{g,\xi}(\bx_1)$ does not depend on $\xi$. 

\paragraph{The truncated sequence.} We now introduce another stochastic process which serves as a surrogate for the coupled process defined in \Cref{defn:plus_minus}, but is truncated in such a way as to facillitate analysis. We will denote random variables from these truncated processd with the letter $\by$. To start, define the stochastic map
\begin{align}
 \dynmap(\by,\by') \overset{d}{=} (\bbarA_{\ibad} \bx + \bu, \bbarA_{\ibad} \by' + \bu'), \quad (\bu,\bu') \sim \Phat(\by,\by'),
\end{align}
where $\Phat(\by,\by')$ is the coupling between $\pihat(\by)$ and $\pihat(\by')$ for which $\bu - \bu'$ is $(\bu,\bu') \sim \Phat(\by,\by')$-anti-concentrated (\Cref{defn:anti-concentrated_policy}). Before continuing, let us introduce two bits of notation used throughout.  We define the clipping operator, which projects onto the ball of radius $B$:
\begin{align*}
\quad \clip_B(\bz) = \begin{cases} \bz & \|\bz\| \le B\\
B\frac{\bz}{\|\bz\|} & \|\bz\| \ge B_1
\end{cases}
\end{align*}

\newcommand{\Ctrunc}{C_{\mathrm{trunc}}}
\newcommand{\nextt}{\mathrm{next}}
\newcommand{\rhost}{\rho_{\star}}
\begin{definition}[Truncated Process Process]\label{definition_truncated}  We define the sequence $B_1 \le B_2 \le \dots$ as follows. For a constant $\Ctrunc$ defined in \Cref{lem:TV_coupling}, set
\begin{align}
 B_1 = 8\Ctrunc \epsilon^{0.9}, \quad B_{t+1} = \rho_{\star} B_t = 8\Ctrunc\rhost^t \epsilon^{0.9}, \quad \rhost = \frac{8\cdot 40\max\{\gamma,\gamma^{-1}\}}{\alpha^2p^2}. \label{eq:Bt}
\end{align}

 Let $(\bx_2^+,\bx_2^-)$ be as \Cref{defn:plus_minus}. Define the sequence   $\btily_1 = \clip_{B_1/8}(\bx_{2}^{+})$, and $\by_1 = \clip_{B_1/8}(\bx_2^{-})$. Further, define $(\btily_t^{\nextt}, \by_{t}^{\nextt}) \sim \dynmap(\btily_t,\by_t)$ as follows:
\begin{align*}
\by_{t+1} &= \clip_{ B_{t+1}/8}(\by_t^{\nextt})\\
 \btily_{t+1}[1] &= \by_{t+1}[1] + \clip_{ B_{t+1}/4}(\btily_t^{\nextt}[1] - \by_t^{\nextt}[1])\\
 \btily_{t+1}[2:d] &= \clip_{B_{t+1}/8}\btily_{t+1}^+[2:d] ,
\end{align*}
we use  following indexing conventions in popular programming languages such as NumPy, albeit with indexing starting at $1$.
Let $\Tys = (\btily_1,\dots,\btily_{H-1},\by_1,\dots,\by_{H-1})$. 
\end{definition}

\paragraph{Comparing the coupled sequence and its truncated analogue.} Because the {} coupled $\bx$-sequence and truncated $\by$-sequence differ only when $\by$ is subject to clipping, and clipping only arises when sequences exceed a certain magnitude, we can use \Cref{condition:learn_general} to control the TV-distance between $\Txs$ and $\Tys$.  

\begin{lemma}\label{lem:TV_coupling} There exists a constant $\Ctrunc$ such  that, for our definition $B_1 := 8\Ctrunc \epsilon^{0.9}$, we have under \Cref{condition:learn_general}
\begin{align}
\mathrm{TV}(\Txs, \Tys) \le \epsilon^{0.18}/2. 
\end{align}
\end{lemma}
\begin{proof}[Proof of \Cref{lem:TV_coupling}] From their definitions, we can couple together the $\Txs$ and $\Tys$ trajectory such that, when the clipping operation is never activated, we have
\begin{align}
\btily_t = \bx^+_{t+1}, \quad \by_t = \bx^-_{t+1}, \quad 1 \le t \le H-1. \label{eq:same_sequence}
\end{align}
The clipping operator is only ever activated when there is some $t$ for which $\max\{\|\btily_t\|,\|\|\by_t\|\} \ge B_{t+1}/8$ (the triangle inequality addresses $\|\btily_t - \by_t\|\ge B_{t+1}/4$). As $B_{t+1}/8 \ge B_t \ge B_1$, we that \Cref{eq:same_sequence} can fail at least only when $\max_{2 \le t \le H}\max\{\|\bx^+_{t}\|,\|\bx_t^-\|\} > B_1$. For $B_1 = \ost(\tau)$, \Cref{lem:value_lem_z_zero} ensures that there is a universal constant $\Ctrunc$ such that this occurs only on the event 
\begin{align*}
\cE = \{\Vhard(\bx_{1:H}^+,\bu_{1:H}^+)  \ge B_1/\Ctrunc\} \cup \{\Vhard(\bx_{1:H}^-,\bu_{1:H}^-) \ge B_1/\Ctrunc\}
\end{align*} 
Note that the condition $B_1 = \ost(\tau)$ is implied by $\epsilon^{.9} = \ost(\tau)$ as $\Ctrunc$ is universal. 

By a union bound, we can bound
\begin{align}
\mathrm{TV}(\Txs, \Tys) &= \inf_{\mathrm{couplings}} \Pr[\Txs \ne \Tys] \tag{variation representation of TV}\\
&\le \Pr[\cE]  \tag{argument above}\\
&\le \Pr[\{\Vhard(\bx_{1:H}^+,\bu_{1:H}^+)  \ge B_1/\Ctrunc\}] + \Pr[\{\Vhard(\bx_{1:H}^-,\bu_{1:H}^-)  \ge B_1/\Ctrunc\}] \tag{follows form a union bound}\\
&= \sum_{\omega \in \{+1,-1\}}\Pr_{\pihat,f_{g,(i,\omega)},\Dzzero}[\Vfun(\trajj) \ge B_1/\Ctrunc] \tag{construnction of coupled sequences, \Cref{defn:plus_minus}}\\
&\le \sum_{\omega \in \{+1,-1\}}\Pr_{\pihat,f_{g,(i,\omega)},\Dzzero}[\Vfun(\trajj) \ge \epsilon^{0.9}] \tag{Definition of $B_1$}\\
&\le \epsilon^{.18}/2 \tag{\Cref{condition:learn_general}}.
\end{align}

\end{proof}

\paragraph{Establishing compounding error of the truncated sequence. } The heart of the argument is now to establish compounding error on the $(\by_t,\btily_t)$ sequence. This is achieved by the following lemma, whose proof is deferred to \iftoggle{arxiv}{\Cref{sec:lem:yt_recurse}}{the following subsection} below. The key idea is to use show, via the truncation lemma \Cref{cor:truncc}, that in expectation, the magnitude of $\by_t - \btily_t$ along the $\be_1$ axis grows, even after the clipping operation is applied. The application of \Cref{cor:truncc} hinges crucially on the anti-concentration of the deviation of the policy $\pihat$ from its mean. We then use the clipping to ensure that $\by_t,\btily_t$ are small enough to ensure the Taylor approximation by the linear system, as well as a certain ``off-diagonal term'', remain controlled. These allow us to establish a one-step recursion which, when iterated yields the desired lemma.
\begin{lemma}\label{lem:yt_recurse} Suppose that $B_t \le \epsilon^{.8}$. Then, it holds that
	\begin{align*}
	(1+\gamma/2)| \angs{\be_1}{\by_t- \btily_t}| - \epssmall \le   \Exp[  |\angs{\be_1}{\by_{t+1} - \btily_{t+1}}| \midd \by_t, \btily_t].
	\end{align*} 
	In particular, by recursing, 
	\begin{align*}
	\Exp[ |\angs{\be_1}{\by_{t+1} - \btily_{t+1}}| ~\big{|}~ \by_1,\btily_1 ] \ge (1+\gamma/2)^{t}\left(| \angs{\be_1}{\by_1- \btily_1}| - \frac{\epssmall}{1 - (1/(1+\gamma/2))}\right).
	\end{align*}
	\end{lemma}

\paragraph{Establishing Compounding Error in the original sequence.} Proceeding from \Cref{lem:yt_recurse}, we establish compounding error on the $(\bx_t^+,\bx_t^-)$ sequence. 
\begin{lemma}\label{lem:main_recursion_lem} Suppose that $t$ is such that  $B_{t+1} \le  \epsilon^{.85}$ and $t \le H-2$.  Then for $\epsilon = \polyost(1/M)$. we have 
\begin{align*}
	\Exp\left[ \epsilon^{0.85} \wedge \absangs{\be_1}{\bx^+_{t+2}- \bx^-_{t+2}}\right] \ge (1+\gamma/2)^{t}\left(\Exp\left[B_1 \wedge  \absangs{\be_1}{\bx^+_2- \bx^-_2}\right]\right) - 3\epsilon^{1.03}.
	\end{align*}
\end{lemma}
\begin{proof}[Proof of \Cref{lem:main_recursion_lem}] Assume we have that as long as $B_t \le B_{t+1} \le \epsilon^{.8}$. 
	Taking expectations of \Cref{lem:yt_recurse}, we have
	\begin{align}
	\Exp\left[ \absangs{\be_1}{\by_{t+1} - \btily_{t+1}}\right] \ge (1+\gamma/2)^{t}\left(\Exp\left[ \absangs{\be_1}{\by_1- \btily_1}\right] - \frac{\epssmall}{1 - (1/(1+\gamma/2))}\right).
	\end{align}
	By \Cref{claim:y_bound}, we have $|\angs{\be_1}{\by_{t+1} - \btily_{t+1}}| \le B_{t+1} $ and $| \angs{\be_1}{\by_1- \btily_1}| \le B_1 $.  Hence, 
	\begin{align}
	\Exp[ B_{t+1} \wedge |\angs{\be_1}{\by_{t+1} - \btily_{t+1}}|] \ge (1+\gamma/2)^{t}\left(\Exp[B_1 \wedge  \absangs{\be_1}{\by_1- \btily_1}] - \frac{\epssmall}{1 - (1/(1+\gamma/2))}\right).
	\end{align}
	We may now perform a change-of-measure to the $\bx_t^+,\bx_t^-$ sequence of \Cref{defn:plus_minus}. This yields
	\begin{align*}
	&\Exp\left[ B_{t+1} \wedge \absangs{\be_1}{\bx^+_{t+2}- \bx^-_{t+2}}\right] \\
	&\quad\ge (1+\gamma/2)^{t}\left(\Exp[B_1 \wedge  \absangs{\be_1}{\bx^+_2- \bx^-_2}] - \frac{\epssmall}{1 - (1/(1+\gamma/2))}\right).\\
	&\quad -   B_{t+1}\left(\TV\left((\bx_{t+2}^+,\bx^-_{t+2}) , (\btily_{t+1},\by_{t+1})\right)  - B_1(1+\gamma/2)^t \TV\left((\bx_{2}^+,\bx^-_{2}) , (\btily_{1},\by_{1})\right)\right).
	\end{align*}
	Recall the definition of $\Txs = (\bx_{2:H}^+,\bx_{2:H}^-)$ and $\Tys = (\btily_{1:H-1}, \btily_{1:H-1})$. Then, both $\TV(\dots)$ terms in the above display are at most $\TV(\Txs,\Tys)$. Furthermore, examining the definition of the sequence $B_t$ (\Cref{eq:Bt}), we have $B_1(1+\gamma/2)^t \le B_{t+1}$. In therefore follows that 
	\begin{align*}
	&\Exp\left[ B_{t+1} \wedge \absangs{\be_1}{\bx^+_{t+2}- \bx^-_{t+2}}\right] \\
	&\quad\ge (1+\gamma/2)^{t}\left(\Exp\left[B_1 \wedge \absangs{\be_1}{\bx^+_2- \bx^-_2}\right] - \frac{\epssmall}{1 - (1/(1+\gamma/2))}\right) -  2B_{t+1}\TV\left(\Txs,\Tys\right). 
	\end{align*}
	Recalling $\epssmall := \epsilon^{1.2} + M \epsilon^{1.8}$, we have for $\epsilon = \polyost(1/M)$, that $\epssmall \le 2\epsilon^{1.2}$.
	Notice that $\gamma \ge 1/8$ (\Cref{claim:one_one_entry}), we have  $\frac{\epssmall}{1 - (1/(1+\gamma/2))} \le \frac{\epssmall}{1 - (16/17)} = 17\epssmall \le 34 \epsilon^{1.2} \le 5\epsilon^{.3} B_1$. And hence, $(1+\gamma/2)^t\frac{\epssmall}{1 - (1/(1+\gamma/2))} \le5\epsilon^{.3} B_{t+1}$. With this simplification,  and bounding $\TV\left(\Txs,\Tys\right) \le \epsilon^{.18}$, and using $\epsilon = \ost(1)$, we can bound the above by
	\begin{align*}
	&\Exp\left[ B_{t+1} \wedge \absangs{\be_1}{\bx^+_{t+2}- \bx^-_{t+2}}\right] \\
	&\quad\ge (1+\gamma/2)^{t}\left(\Exp\left[B_1 \wedge \absangs{\be_1}{\bx^+_2- \bx^-_2}\right]\right) -  \underbrace{B_{t+1}\left(2\TV\left(\Txs,\Tys\right)+5\epsilon^{.3}\right)}_{\le 3B_{t+1}\epsilon^{0.18}},
	\end{align*}
	By assumption, $B_{t+1} \le \epsilon^{.85}$, which concludes the proof.

\end{proof}

\paragraph{Concluding the proof of \Cref{prop:compounding_general_thing}.} Finally, we derive \Cref{prop:compounding_general_thing} from \Cref{lem:main_recursion_lem}. The key steps are to relate errors in the difference between the $(\bx^+,\bx^-)$ sequence to the magnitude of $\Vhard(\trajj)$, and to select $t$ as large as possibily so as to satisfy $B_{t+1} \le \epsilon^{.85}$.
\begin{proof}[Proof of \Cref{prop:compounding_general_thing}]  Let $P$ denote the uniform distribution on $(i,\omega) \in \{1,2\} \times \{-1,+1\}$. Then,
\begin{align*}
\Exp\left[ \epsilon^{0.85} \wedge \absangs{\be_1}{\bx^+_{t+2}- \bx^-_{t+2}}\right] &\le \Exp\left[ \epsilon^{0.85} \wedge \absangs{\be_1}{\bx^+_{t+2}}\right] + \Exp[ \epsilon^{0.85} \wedge \absangs{\be_1}{\bx^-_{t+2}}]\\
&\le \left(\Exp_{\pihat,f_{g,(\ibad,1)},\Dzzero} + \Exp_{\pihat,f_{g,(\ibad,-1)},\Dzzero}\right)[ \epsilon^{0.85} \wedge \absangs{\be_1}{\bx_{t+2}}]\\
&\le 2\Exp_{\xi \sim P}\Exp_{\pihat,f_{g,\xi},\Dzzero}[ \epsilon^{0.85} \wedge \absangs{\be_1}{\bx_{t+2}}]\\
&\le 4\Exp_{\xi \sim P}\Exp_{\pihat,f_{g,\xi},\Dist} \left[ \epsilon^{0.85} \wedge \absangs{\be_1}{\bx_{t+2}}\right]\\
&\le 4\Exp_{\xi \sim P}\Exp_{\pihat,f_{g,\xi},\Dist} \left[ \epsilon^{0.85} \wedge \frac{1}{\Lcost}\Vhard(\bx_{1:H},\bu_{1:H})\right] \tag{Definition of $\Vhard$ in \Cref{sec:lb_constr}}\\
&\le \frac{4}{\Lcost}\Exp_{\xi \sim P}\Exp_{\pihat,f_{g,\xi},\Dist} [ \epsilon^{0.85} \wedge \Vhard(\bx_{1:H},\bu_{1:H})] \tag{$\Lcost \le 1$}
\end{align*}
Moreover, from \Cref{eq:two_taust}, we have
\begin{align*}
\Exp\left[B_1 \wedge  \absangs{\be_1}{\bx^+_2- \bx^-_2}\right] &= \Exp_{\pihat,f_{g,(\cdot)},\Dzzero}\left[B_1 \wedge 2\tau \absangs{\be_1}{\bu_1 - \pihat_{g,(\cdot)}(\bx_1)}\cdot\I\{\|\bu_1\| \le 1 \}\right]\\
&\ge \Exp_{\pihat,f_{g,(\cdot)},\Dzzero}\left[B_1 \wedge 2\tau \absangs{\be_1}{\bu_1 - \pihat_{g,(\cdot)}(\bx_1)}\right] - B_1\Pr_{\pihat,f_{g,(\cdot)},\Dzzero}[\|\bu_1\| > 1 ]
\end{align*}
where above we used \Cref{const:stable} and where $(\cdot)$ denotes a lack of dependence on the $\xi$ argument in $f_{g,\xi},\pihat_{g,\xi}$. Thus, we have that $B_1\Pr_{\pihat,f_{g,(\cdot)},\Dzzero}[\|\bu_1\| > 1 ] = \inf_{\xi} B_1\Pr_{\pihat,f_{g,\xi},\Dzzero}[\|\bu_1\| > 1 ] \le \inf_{\xi} B_1\Pr_{\pihat,f_{g,\xi},\Dzzero}[\Vhard(\bx_{1:H},\bu_{1:H}) \ge \Lcost ] \le B_1\epsilon^{.18} $, where the last inequality uses \Cref{condition:learn_general}. Finally, we bound $B_1\epsilon^{.18} \le 8\Ctrunc \epsilon^{1.08} \le \epsilon^{.103}$ for $\epsilon = \ost(1)$. Thus, 
\begin{align*}
\Exp[B_1 \wedge  \absangs{\be_1}{\bx^+_2- \bx^-_2}] &= \Exp_{\pihat,f_{g,(\cdot)},\Dzzero}\left[B_1 \wedge 2 \absangs{\be_1}{\bu_1 - \pihat_{g,(\cdot)}(\bx_1)}\cdot\I\{\|\bu_1\| \le 1 \}\right]\\
&\ge \Exp_{\pihat,f_{g,(\cdot)},\Dzzero}\left[B_1 \wedge 2 \absangs{\be_1}{\bu_1 - \pihat_{g,(\cdot)}(\bx_1)}\right] -\epsilon^{1.03}\\
&= \Exp_{\xi \sim P}\Exp_{\pihat,f_{g,\xi},\Dzzero}\left[B_1 \wedge 2 \absangs{\be_1}{\bu_1 - \pihat_{g,(\cdot)}(\bx_1)}\right] -\epsilon^{1.03}.
\end{align*}
Finally, using $B_1 \ge \epsilon^{.9}$, and combining these results with \Cref{lem:main_recursion_lem} yields
\begin{align*}
&\frac{4}{\Lcost}\Exp_{\xi \sim P}\Exp_{\pihat,f_{g,\xi},\Dist} [ \epsilon^{0.85} \wedge \Vhard(\bx_{1:H},\bu_{1:H})] \\
&\quad\ge (1+\gamma/2)^t\Exp_{\xi \sim P}\Exp_{\pihat,f_{g,\xi},\Dzzero}[\epsilon^{.9} \wedge 2 \absangs{\be_1}{\bu_1 - \pihat_{g,(\cdot)}(\bx_1)}] - 4\epsilon^{1.03},
\end{align*}
again provided $B_{t+1} \le \epsilon^{0.85}$, as well as $t \le H-2$. For the first constraint on $t$, we require that $B_{t+1} = 8\epsilon^{.9}\Ctrunc \rhost^t \le \epsilon^{0.85}$, it suffices to take 
\begin{align}
t = \min\left\{H-2,\floor{\log\left(\frac{\epsilon^{-.05}}{8\Ctrunc}\right)/\log (\rhost)}\right\} &\ge \min\left\{H-2,\frac{1}{2}\log\left(\frac{\epsilon^{-.05}}{8\Ctrunc}\right)/\log (\rhost)\right\},
\end{align} 
where the inequality follows by checking that $\log\left(\frac{\epsilon^{-.05}}{8\Ctrunc}\right)/\log (\rhost) \ge 1$ for $\epsilon =\polyost(1/L,\alpha,p) = \polyost(1/\log(\rhost)) $. If the $H-2$ in the above minimum is the smaller term, then $(1+\gamma/2)^t = (1+\gamma/2)^{H-2} \ge (1.05)^{H-2}$. 

Otherwise,
\begin{align}
(1+\gamma/2)^t &\ge \exp\left( \frac{1}{2}\frac{\log(1+\gamma/2)}{\log (\rhost)} \cdot \frac{\epsilon^{-.05}}{8\Ctrunc}\right) .
\end{align}
As $\gamma \ge 1/8$, one can show that  $\log (\rhost) = \log(\gamma + \gamma^{-1}) + \log (\mathrm{const} \cdot 1/(\alpha^2 p^2)) \le \log(1+\gamma/2)+\log(\mathrm{const}) + \log(1/\alpha^2 p^2)$, and thus $\frac{\log(1+\gamma/2)}{\log (\rhost)} \ge \frac{1}{C + \log(1/\alpha^2p^2)}$ for an appropriately large constant $C$. Hence, for some other $C'$ (using $\Ctrunc = \BigOh{1}$), we find that 
\begin{align}
(1+\gamma/2)^t \ge \epsilon^{-\frac{1}{C'(1+\log(1/(\alpha p)))}}.
\end{align} 
This concludes the proof.
\end{proof}

\subsubsection{Proof of \Cref{lem:yt_recurse}}\label{sec:lem:yt_recurse}
Throughout, in view of \Cref{claim:off_diag_small}, we assume  $\|\be_1^\top (\bbarA_i + \bbarK) \Proj_{\ge 2}\| \le \epsilon^{0.4}$. Our first step is to show that the $(B_t)_{t \ge 1}$ sequences dominates the terms in $\|\by_t\|,\|\btily_t\|$ in magnitude.

\begin{claim}\label{claim:y_bound}  For all $t$, we have
\begin{align}
\|\by_{t}\| \vee \|\btily_{t}\| \vee \|\btily_t - \by_t\| \le B_t/2. 
\end{align}
\end{claim}
\begin{proof}[Proof of \Cref{claim:y_bound}] By construction, $\|\by_1\| \vee \|\btily_1\| \le B_1/8$.  In general, we have $\|\by_{t+1}\| \le B_{t+1}/8$, $|\btily_{t+1}[1]| \le B_{t+1}/4 + |\by_{t+1}[1]| \le B_{t+1}/4 + B_{t+1}/8 \le 3B_{t+1}/8$, and thus  $\|\btily_{t+1}\| \le |\btily_{t+1}[1]| + \|\btily_{t+1}[2:d]\| \le B_{t+1}/2$. The bound on $\|\btily_t - \by_t\|$ follows by noting $\|\btily_t - \by_t\| \le |\btily_t^{\nextt}[1] - \by_t^{\nextt}[1]| + \|\btily_{t+1}[2:d]\| + \|\by_{t+1}[2:d]\| \le \frac{B_t}{2}$.
\end{proof}

The next claim described shows that, on a single time-step, the magnitude of the distance between $(\by^{\nextt},\btily^{\nextt}) \sim G(\by,\btily)$ along the $\be_1$-axis increases, even if subject to truncation. 
	\begin{claim}\label{claim:G_y} Suppose that $\by,\btily \in \R^d$ satisfy $\|\by\|,\|\btily\| \le \epsilon^{0.8}$. Then, there exists a $(\alpha,p)$-anti-concentrated scalar random variable $Z$ such that  $(\by^{\nextt},\btily^{\nextt}) \sim G(\by,\btily)$ can satisfies the inequality
	\begin{align*}
	\Exp[\absangs{\be_1}{ \btily^{\nextt} - \by^{\nextt}}] 
	&\ge \Exp[|Z+  (1+\gamma)\angs{\be_1}{\by- \btily}|] - \epssmall.
	\end{align*}
	where $ \epssmall := \epsilon^{1.2} + M \epsilon^{1.8}$. In particular,  by \Cref{lem:truncc}
	\begin{align*}
	\Exp\left[\min\left\{  |\angs{\be_1}{\btily_t^{\nextt} - \by_t^{\nextt}  }|, ~ \frac{\rhost}{8}|\angs{\be_1}{\btily_t - \by_t}| \right\}\right] \ge \left(1+\gamma/2\right)| \angs{\be_1}{\btily - \by}|  - \epssmall.
	\end{align*}
	\end{claim}
	\begin{proof}[Proof of \Cref{claim:G_y}] Note that $\Exp\angs{\be_1}{\dynmap(\btily,\by)} = \angs{\be_1}{ \bbarA_{\ibad}(\btily-\by) + \mean[\pihat](\btily) - \mean[\pihat](\by) }$. Hence,
	\begin{align*}
	\angs{\be_1}{\btily_t^{\nextt} - \by_t^{\nextt}} &= \angs{\be_1}{ \bbarA_{\ibad}(\btily-\by) + \mean[\pihat](\btily) - \mean[\pihat](\by) } + Z,
	\end{align*}
	We recall from the construction of $\dynmap$,  $Z =  \angs{\be_1}{\dynmap(\btily,\by)}  - \Exp[\angs{\be_1}{\dynmap(\btily,\by)}] \overset{\mathrm{dist}}{=} \btilu - \bu - \Exp[\btilu - \bu]$ is $(\alpha,p)$-anti-concentrated under $(\btilu,\bu) \sim \Phat(\btily,\by)$.  .   

	 Recall $\bhatK = \nabla \mean[\pihat](\bzero)$, we get $\bbarA_{\ibad}(\btily - \by) + \mean[\pihat](\btily) - \mean[\pihat](\by) = \bbarK (\btily - \by) + \bw_0$, where $\bw_0$ is a Taylor remainder term, and where by $M$-smoothness of $\mean[\pihat]$, the remainder term is at most $\|\bw_0\| \le  \frac{M}{2}(\|\btily\|^2 + \|\by\|^2)\le M \epsilon^{1.6}$. 

	 Moreover, by assumption, $|\be_1^\top (\bbarA_{\ibad} + \bbarK) \Proj_{\ge 2}(\btily - \by)| \le  \epsilon^{0.4} \|\btily - \by\| \le  \epsilon^{1.2}$. Finally, by assumption, $|\be_1(\bbarA_{\ibad} + \bhatK)\be_1| := (1+\gamma)$.  Putting things together, 
	\begin{align*}
	\left|\angs{\be_1}{\btily_t^{\nextt} - \by_t^{\nextt}}\right| &=
	\left|\angs{\be_1}{\bbarA_{\ibad}(\btily-\by) + \mean[\pi](\btily) - \mean[\pi](\by)} + Z\right| \\
	&\ge \left|Z + \zeta(1+\gamma)\angs{\be_1}{\btily - \by}\right| - \left(\epsilon^{1.2} + M \epsilon^{1.8}\right)\\
	&\ge \left|Z' + (1+\gamma)\angs{\be_1}{\btily - \by}\right| - \underbrace{\left(\epsilon^{1.2} + M \epsilon^{1.8}\right)}_{=:\epssmall}
	\end{align*}
	where $\zeta$ is the sign of $\be_1(\bbarA_{\ibad} + \bhatK)\be_1$ and $Z' = \zeta Z$. Note that anti-concentration of $Z$ implies anti-concentration of any scaling of $Z$, and hence of $Z'$. Hence, \Cref{cor:truncc} implies:
	\begin{align*}
	\Exp\left[\min\left\{\underbrace{\left(\frac{40\max\{\gamma,\gamma^{-1}\}}{\alpha^2p^2}\right)}_{=:\rhost/8} \absangs{\be_1}{\btily - \by},\absangs{\be_1}{\btily_t^{\nextt} - \by_t^{\nextt}}\right\}\right] \ge (1 + \gamma/2)\absangs{\be_1}{\btily - \by} - \epssmall.
	\end{align*} 
	\end{proof}
	
	\begin{proof}[Proof \Cref{lem:yt_recurse}] Let $(\cF_t)$ denote the filtration generated by $(\by_t,\btily_t)$. 
	We have
	\begin{align*}
	&(1+\gamma/2) \absangs{\be_1}{\by_t- \btily_t}  - \epssmall \\
	&\le \Exp\left[\min\left\{ \absangs{\be_1}{\by_t^{\nextt} - \btily_t^{\nextt}}, \frac{\rho_{\star}}{8}\absangs{\be_1}{\by_t- \btily_t}\midd \cF_t \right\}\right] \tag{ \Cref{claim:G_y}}\\
	&\le \Exp\left[\min\left\{  \absangs{\be_1}{\by_t^{\nextt} - \btily_t^{\nextt}}, \frac{\rho_{\star}}{8}B_t\right\} \midd \cF_t\right] \tag{\Cref{claim:y_bound}}\\
	&= \Exp\left[\min\left\{  \absangs{\be_1}{\by_t^{\nextt} - \btily_t^{\nextt}}, \frac{B_{t+1}}{8}\right\}\midd \cF_t \right] \tag{Definition for $(B_t)$, \Cref{eq:Bt}}.
	\end{align*}
	Lets now relate $\by_t^{\nextt} - \btily_t^{\nextt}$ to $\by_{t+1} - \btily_{t+1}$.
	Now notice that one of two things may either $\absangs{\be_1}{\by_t^{\nextt} - \btily_t^{\nextt}} \ge \frac{B_{t+1}}{8}$, in which case $|\absangs{\be_1}{\by_{t+1} - \btily_{t+1}} \ge \frac{B_{t+1}}{8}$, or else $\absangs{\be_1}{\by_t^{\nextt} - \btily_t^{\nextt}} \le \frac{B_{t+1}}{8}$, which case $\absangs{\be_1}{\by_{t+1} - \btily_{t+1}} = \absangs{\be_1}{\by_t^{\nextt} - \btily_t^{\nextt}}$. Hence, we have that
	\begin{align*}
	(1+\gamma/2) \absangs{\be_1}{\by_t- \btily_t} - \epssmall \le   \Exp[  \absangs{\be_1}{\by_{t+1} - \btily_{t+1}}\midd \cF_t].
	\end{align*}
	By recursing, we then above
	\begin{align*}
	\Exp\left[  \absangs{\be_1}{\by_{t+1} - \btily_{t+1}} \midd \cF_1\right] &\ge 
	(1+\gamma/2)^{t} \absangs{\be_1}{\by_1- \btily_1} - \sum_{s=1}^t (1+\gamma/2)^{t-s} \epssmall\\
	&\ge 
	(1+\gamma/2)^{t}\left( \absangs{\be_1}{\by_1- \btily_1} - \frac{\epssmall}{1 - (1/(1+\gamma/2))}\right).
	\end{align*}
	\end{proof}



\subsection{Formal Proof of Lower bound \Cref{eq:anticonc_compounding} in \Cref{thm:stable_general_noise}}\label{sec:proof:general_noise}

\begin{proof}One of two cases can occur. Either, \Cref{eq:Vhard_case_one}  in  \Cref{condition:learn_general} holds, in which case \Cref{prop:compounding_general_thing} ensures
\begin{equation*}
\begin{aligned}
&  \Exp_{\xi \sim P}\Exp_{\pihat,f_{g,\xi},\Dist} \left[ \epsilon^{0.85} \wedge \Vhard(\bx_{1:H},\bu_{1:H})\right] \\
&\qquad\ge \frac{\Lcost}{4} K(\epsilon,H)\Exp_{\xi \sim P}\Exp_{\pihat,f_{g,\xi},\Dzzero}\left[\epsilon^{.9} \wedge 2\tau\left| \angs{\be_1}{\bu_1 - \pihat_{g,(\cdot)}(\bx_1)}\right|\right] - 4\epsilon^{1.03}.
\end{aligned}
\end{equation*}
Otherwise, \Cref{eq:Vhard_case_one} fails, so that
\begin{align*}
\Exp_{\xi \sim P} \Pr_{\pihat,f_{\xi,g},\Dist}[\min\{\epsilon^{.9}, \Vhard(\trajj)\}^2]^{1/2} &\ge \epsilon^{.9}\sqrt{\Exp_{\xi \sim P} \Pr_{\pihat,f_{\xi,g},\Dist}\left[\Vhard(\bx_{1:H},\bu_{1:H}) \ge \epsilon^{.9}\right]} \\
&\ge \epsilon^{.9}\epsilon^{.09}/2 = \epsilon^{.99}/2.
\end{align*} 
By Jensen's inequality 
\begin{align*}
\Exp_{\xi \sim P}\Exp_{\pihat,f_{g,\xi},\Dist} \left[ \epsilon^{0.85} \wedge \Vhard(\bx_{1:H},\bu_{1:H})\right] \le \Exp_{\xi \sim P}\Exp_{\pihat,f_{g,\xi},\Dist} \left[ \epsilon^{1.7} \wedge \Vhard(\bx_{1:H},\bu_{1:H})^2\right]^{1/2}.
\end{align*}
Therefore, we find that 
\begin{align*}
&{\Exp_{\xi \sim P}\Exp_{\pihat,f_{g,\xi},\Dist} \left[ \epsilon^{1.7} \wedge \Vhard(\bx_{1:H},\bu_{1:H})^2\right]^{1/2} + \Lcost \epsilon^{1.03}}\\\
&\quad \ge \min\left\{\epsilon^{.99}/2, \frac{1}{4}\Lcost K(\epsilon,H)\Exp_{\xi \sim P}\Exp_{\pihat,f_{g,\xi},\Dzzero}\left[\epsilon^{.9} \wedge 2\tau\left| \angs{\be_1}{\bu_1 - \pihat_{g,(\cdot)}(\bx_1)}\right|\right]\right\}.
\end{align*}
Moreover, for $\epsilon$ sufficiently small as in \Cref{condition:learn_general}, we have that  
\begin{align*}
\Exp_{\xi \sim P}\Exp_{\pihat,f_{g,\xi},\Dzzero}\left[\epsilon^{.9} \wedge 2\tau\left| \angs{\be_1}{\bu_1 - \pihat_{g,(\cdot)}(\bx_1)}\right|\right] \ge 2\tau^2 \kappa \epsilon \Exp_{\xi \sim P}\Exp_{\pihat,f_{g,\xi},\Dzzero}[ \angs{\be_1}{\bu_1 - \pihat_{g,(\cdot)}(\bx_1)} \ge \kappa\tau\epsilon],
\end{align*}
and by modifying the constant $C'$ in the term 
\begin{align*}
K(\epsilon,H) := \min\left\{(1.05)^{H-2},\epsilon^{-\frac{1}{C'(1+\log(1/(\alpha p)))}} \right\},
\end{align*} to be at least $C' \ge 100$ and using $\Lcost \le 1$ (see \Cref{defn:challenging_cost}), we may ensure that 
\begin{align*}
&\min\left\{\epsilon^{.99}/2, \frac{1}{4}\Lcost K(\epsilon,H)\Exp_{\xi \sim P}\Exp_{\pihat,f_{g,\xi},\Dzzero}\left[\epsilon^{.9} \wedge 2\tau\left| \angs{\be_1}{\bu_1 - \pihat_{g,(\cdot)}(\bx_1)}\right|\right]\right\}\\
&\qquad\ge  \Exp_{\xi \sim P}\Exp_{\pihat,f_{g,\xi},\Dzzero}\left[ \angs{\be_1}{\bu_1 - \pihat_{g,(\cdot)}(\bx_1)} \ge \kappa\tau\epsilon\right],
\end{align*}
Rearranging, 
\begin{align*}
&\Exp_{\xi \sim P}\underbrace{\left(\Exp_{\pihat,f_{g,\xi},\Dist} \left[ \epsilon^{1.7} \wedge \Vhard(\bx_{1:H},\bu_{1:H})^2\right]^{1/2} + \Lcost \epsilon^{1.03}\right)}_{=: \Risk(\pihat,g,\xi)}\\  &\ge \left(\frac{\Lcost K(\epsilon,H)\tau^2 \kappa \epsilon}{2}\right)\Exp_{\xi \sim P}\Exp_{\pihat,f_{g,\xi},\Dzzero}\left[ \angs{\be_1}{\bu_1 - \pihat_{g,(\cdot)}(\bx_1)} \ge \kappa\tau\epsilon\right].
\end{align*}
We now invoke \Cref{prop:redux}(d) with $\Risk(\pihat,g,\xi)$ as defined above, and following  the proof of  \Cref{thm:stable_detailed} given in \iftoggle{arxiv}{\Cref{lb:eval_risk}}{\Cref{ssec:stable_minimax}}. Here, recall $\epsilon = \beps_n = \minsl(n,\cG,\Dreg)$, and that $(\cG,\Dreg)$ is $(\kappa,\delta)$-\typical. From \Cref{prop:redux}(d) and a slight bit of rearranging,
\begin{align*}
\sup_{g \in \cG,\xi} \Exp_{\xi \sim P}\Exp_{\pihat,f_{g,\xi},\Dist} \left[ \epsilon^{1.7} \wedge \Vhard(\bx_{1:H},\bu_{1:H})^2\right]^{1/2}  \ge \left(\frac{\Lcost K(\epsilon,H)\tau^2 \kappa \epsilon}{2}\right)\delta - \Lcost \epsilon^{1.03}.
\end{align*}
Lastly, we recall that $K(\epsilon,H) \ge 1$, $\tau,\Lcost = \Omega(1)$, and and that for $\epsilon = \polyost(\kappa,\delta)$, there is some small universal constant $c$ such that $\left(\frac{\Lcost K(\epsilon,H)\tau^2 \kappa \epsilon}{2}\right)\delta - \Lcost \epsilon^{1.03} \ge K(\epsilon,H) \epsilon \kappa \delta$. Hence, 
\begin{align*}
\sup_{g \in \cG,\xi} \Exp_{\xi \sim P}\Exp_{\pihat,f_{g,\xi},\Dist} \left[ \epsilon^{1.7} \wedge \Vhard(\bx_{1:H},\bu_{1:H})^2\right]^{1/2} &\ge c \cdot\kappa \delta \cdot K(\epsilon,H) \epsilon\\
&=: c \cdot\kappa \delta \min\left\{\epsilon 1.05^{H-2}, \epsilon^{1-\frac{1}{C'(1+\log(1/(\alpha p)))}}\right\},
\end{align*}
as needed.
\end{proof}


\section{Proof for Unstable Dynamics, \Cref{thm:unstable_detailed,thm:unstable}}\label{app:unstable}

In this section, we establish exponential compounding error for unstable dynamical systems, \Cref{thm:unstable_detailed}. We note that \Cref{thm:unstable} follows as a direct consequence of  \Cref{thm:unstable_detailed}, as noted below the statement of the latter theorem in \Cref{sec:minmax_unstable}.


We  prove \Cref{thm:unstable_detailed} by first establishing a variant, \Cref{thm:unstable_detailed_tv}, which pertains to time-varying systems, and is proven in \Cref{sec:thm_tv} below. A time-varying dynamical system is just a dynamical system $f(\bx,\bu,t)$, which may depend on the arbitrarily on $t$. Similarly, we allow the expert  $\pist(\bx,t)$ to also depend on a $t$-argument. In \Cref{sec:remove_time_varying}, we proceed to establishing \Cref{thm:unstable_detailed} my modifying the construction to  hold for systems and policies which do  not vary with the argument $t$.

We now turn to the statement of \Cref{thm:unstable_detailed_tv}. Below $\bbO(d)$ denotes the  orthogonal group, that is, the set of  matrices in $\R^d$ with orthonormal columns, and $\proj_{ \le k}$ the projection onto the first $k$ cannonical basis elements. 

\begin{construction}\label{const:unstabl} Let $(\cG,\Dreg)$ be an $(k,\ell_2)$-regression family, and let $\rho > 2$. We define a $(d,d)$-IL family $(\cI,\Dist)$, where 
\begin{itemize}
\item[(a)] $\Dist$ draws $\bz \sim \Dreg$ and appends $d - k$ zeros to $\bz \in \R^k$ to form $\bx  = (\bz,\bzero) \in \R^d$. 
\item[(b)] Let $\xi = (\bO_2,\bO_2,\dots)$ denote sequences in $\bbO(d)$, and let $g \in \cG$. We take $\cI$ is the set of all instances $(\pist,f)$ of the following form:
\begin{align*}
    \pi_{g,\xi}(\bx,t) = \begin{cases} g(\proj_{ \le k}\bx)\be_1 &t =1 \\
    -\rho \bO_t \bx, & t > 1
    \end{cases}, \quad 
    f_{g,\xi}(\bx,\bu,t) = \bu - \pi_{g,\xi}(\bx,t)
\end{align*}
\end{itemize}
\end{construction}
The above construction follows the schematic of \Cref{prop:redux}, and the same proof plan sketched in \Cref{sec:proof_intuition}: the learner makes a mistake in the first step, due to uncertainty over the class $\cG$, and then must contend with uncertainty over the dynamics in the time steps that follow. Recall that we aim for lower bounds that hold in an \emph{unrestricted sense}, and apply even to learner's which select time-varying, history dependent policies $\pihat$. This renders simpler constructions that do not incorporate rotational uncertainty insufficient:
\begin{remark}[Scaled-Identity Dynamics do not suffice] One could imagine a simplified construction where either $\pi_{g,\xi}(\bx,t) = \sigma \rho \eye$ or  $\pi_{g,\xi}(\bx,t) = \sigma_t \rho \eye$ is the identity, scaled by $\rho$, and multipled by either a fixed sign $\sigma \in \{-1,1\}$ or a time varying sign $\sigma_t \in \{-1,1\}$. Noting that $\pm \eye \in \bbO(d)$, these constructions are a restriction of the class in \Cref{const:unstabl}. Unfortunately, these constuctions do yield unconditional lower bounds. For a fixed sign $\sigma$, a history-dependent learner can identify the dynamics, whereas for a time-varying sign $\sigma_t$, the benevolent gambler's ruin strategy (\Cref{sec:benevolent_ruin}) mitigates compounding error.
\end{remark}

\newcommand{\chardtv}{\cost_{\mathrm{hard},\mathrm{time\,var}}}

We consider the following {challenging} cost function:
\begin{align}
\chardtv(\bx_{1:H},\bu_{1:H}) = \max_{1 \le t \le H} \min\{1,\chardt(\bx_t,\bu_t)\}, \quad \chardt(\bx,\bu) = \I\{t \ge 1\}\|\bx\| 
\end{align}

A salient property of the construction and associated cost function is the following observation:
\begin{observation}\label{obs:vanishing_unstable} For any $(\pi,f) \in \cI$ and $\Dist$ as above, $\Pr_{\pi,f,\Dist}[\bx_t = \bzero \text{ and } \bu_t = \bzero, \quad \forall t \ge 2] = 1$. In particular, $\chardtv$ vanishies on $(\cI,\Dist)$. 
\end{observation}
This observation ensures that trajectories after time step $t \ge 2$ are uninformative. Using this fact, we will establish the following lower bound:

\begin{theorem}[Time-varing analogue of \Cref{thm:unstable_detailed}] \label{thm:unstable_detailed_tv}
For any $k,d \in \N$ with $k \le d$ and  $(k,\ell_2)$-regression family $(\cG,\Dreg)$ satisfying \Cref{asm:conc}, and $\rho \ge 1$, the construction above is such that
such that for all $2 \le H \le \frac{1}{2}\exp((1-\rho^{-1})^2 d/2)$, 
\begin{align}
&\minbctrain(n;\inst,\Dist,H) = \minsl(n;\Gclass,\Dreg ) =: \beps_n\label{eq:unstable_same_min_tv}\\
&\minmax_{\chardtv}(n,p;\inst,\Dist, H)\left(n, \frac{\delta}{2}; \inst,\Dist,H\right)  \ge \kappa \beps_n \rho^{(H-1)/2} \label{eq:unstable_compounding_tv}
\end{align}
\end{theorem}

\subsection{Proof of \Cref{thm:unstable_detailed_tv}} \label{sec:thm_tv}

We begin with a lemma which demonstrates that no control policy, even one which depends arbitrarily on history, can avoid compounding error when faced with time varying dynamics given by random rotation matrices: 

\begin{lemma}[Compounding Error with Orthonormal Matrices]\label{lem:error_amplification} Consider a stochastic dynamical system with 
\begin{align*}
\bx_{t+1} = \rho \bO_t \bx_t + \bu_t, \quad \bO_t \iidsim \bbO(d),
\end{align*}
Let $\bu_t$ be chosen by any (possibly stochastic) control policy such that the conditional distribution of $\bu_t$ depends only on $\bx_{1:t},\bu_{1:t-1}$. Then, for all $1 \le t \le H$, it holds that
\begin{align}
\Pr[ \forall 1 \le t \le H, \|\bx_{t+1}\| \ge (\sqrt{1-\alpha}\rho)^t\|\bx_1\| ] \ge 1- e^{\frac{d\alpha^2}{2}} H. \label{eq:high_probability_compounding}
\end{align}
In particular, $\alpha = 1 - 1/\rho$, we obtain
\begin{align}
\Pr[ \forall 1 \le t \le H, \|\bx_{t+1}\| \ge (\rho)^{t/2}\|\bx_1\| ] \ge 1- e^{\frac{d(1-\rho^{-1})^2}{2}} H. \label{eq:high_probability_compounding_two}
\end{align}

\end{lemma}
\begin{proof}

 By a union bound, it suffices to show that, for any $\bu_t$ conditioned on the past,
\begin{align*}
\Pr[\|\bx_{t+1}\| \ge \sqrt{1-\alpha}\rho \|\bx_t\|] \ge 1- e^{\frac{d\alpha^2}{2}}.
\end{align*}
We have that 
\begin{align*}
\|\bx_{t+1}\|^2 = \|\rho \bO_t \bx_t + \bu_t\|^2 = \rho^2 \|\bx_t\|^2 + \|\bu_t\|^2 + 2 \rho \|\bu_t\|\|\bx_t\| \cos \theta( \bO_t \bx_t, \bu_t), 
\end{align*}
where $\theta( \bO_t \bx_t, \bu_t)$ is the angle between the argument vectors. Using the elementary inequality $ab \le a^2 + b^2$, we can then lower bound the above by 
\begin{align}
\|\bx_{t+1}\|^2 \ge (1 - \cos \theta( \bO_t \bx_t, \bu_t)) \left(\rho^2 \|\bx_t\|^2 + \|\bu_t\|^2 \right) \ge \rho^2 (1 - \cos \theta( \bO_t \bx_t, \bu_t))\|\bx_t\|^2. \label{eq:theta_expansion}
\end{align}
Since $\bO_t \sim \bbO(d)$ and is independent of $\bu_t$, $\theta( \bO_t \bx_t, \bu_t)$ has the distribution of the angle between a fixed vector and a uniform vector on the sphere. 
A standard concentration inequality shows then that
\begin{align*}
\Pr\left[ \cos \theta(\bO_t\bx_t, \bu_t) \ge \frac{t}{\sqrt{d}}\right] \le e^{-t^2/2}
\end{align*}
Taking $t = \alpha \sqrt{d}$, we have that $\Pr[ \cos \theta(\bO_t\bx_t, \bu_t) \ge \alpha ] \le \exp( - d\alpha^2/2)$. On this event,  the \Cref{eq:theta_expansion} gives
\begin{align*}
\|\bx_{t+1}\| \ge \rho \sqrt{1 - \alpha} \|\bx_t\|,
\end{align*}
as needed.

\end{proof}

Continuing proof instantiates the arguments of the general schematic in \Cref{sec:schematic}; we encourage the reader to review that section before continuing to read the present. We instantiate \Cref{sec:schematic} by introduce the parameter $\xi = \{\bO_2,\bO_3,\dots,\bO_H\} \in \bbO(d)^{H-1}$.  We also let $P$ denote the uniform distribution of $\xi$, i.e., where $\bO_t$ are drawn i.i.d from the Haar measure on $\bbO(d)$.
A direct consequence of the previous lemma is as follows.
\begin{corollary}\label{cor:exp_compound}
For any arbitrary (even stateful, time-dependent) policy $\pihat$, and $P$ the uniform distribution over $\xi$, and any $g \in \cG$, we have 
\begin{align*}
\Exp_{\xi \sim P}\Pr_{\pihat,f_{g,\xi},\Dist}[\|\bx_H\| \ge \epsilon \cdot \rho^{(H-1)/2} \mid \|\bx_2\| \ge \epsilon ] \cdot  \Exp_{\xi \sim P}\\
\ge \left(1- H \exp\left(\frac{d(1-\rho^{-1})^2}{2}\right)\right) \cdot  \Exp_{\xi \sim P} \ge \frac{1}{2},
\end{align*} where the last inequality holds for our choice of $H$ in \Cref{thm:unstable_detailed_tv}.
\end{corollary}

We now turn to invoking \Cref{prop:redux}. To do so, we begin checking that its conditions hold.

\begin{lemma}\label{lem:const_unstable_good} The construction satisfies the three conditions of  \Cref{sec:schematic}, \Cref{defn:orthogonal,defn:single_step,defn:indistinguishable}, with $\tau = 1$. Hence, \Cref{prop:redux} applies (with $\tau =1$).
\end{lemma}
\begin{proof}  \Cref{defn:orthogonal,defn:single_step} can be checked directly. Here, we we prove that, $(\pi_{g,\xi}, f_{g,\xi})$ are on-policy indistinguishable under $\Dist$ (\Cref{defn:indistinguishable}). By \Cref{obs:vanishing_unstable},  $\bx_t,\bu_t$ vanish with probability one ujnder $\Pr_{\pi,f,\Dist}$ for all $(\pi,f) \in \cI$. Thus, all that remains is to show that $(\pi_{g,\xi}, f_{g,\xi})$ and $(\pi_{g,\xi'}, f_{g,\xi'})$ induces the same distribution over $\bx_1, \bu_1, \bx_2$. By construction, 
\begin{align*}
f_{g,\xi}(\bx,1) = f_{g,\xi'}(\bx, 1), \quad 
\pi_{g,\xi}(\bx, 1)= \pi_{g,\xi'}(\bx,1) \quad \forall \xi, \xi'.
\end{align*}
Therefore the distributions over $\bx_1, \bx_2, \bu_1$ are identical for all $\xi$ under a given $g, \Pnot$. This concludes the verification of \Cref{defn:indistinguishable}. 
\end{proof}

Directly from \Cref{prop:redux}, \Cref{eq:unstable_same_min} holds:  $\minbctrain(n;\inst,\Dist,H) = \minsl(n;\Gclass,\Dreg )$. Moreover, define the risk 
\begin{align}
\Risk_{\epsilon}(\pihat;g,\xi) = \Pr_{\pihat,f_{g,\xi},\Dist}[\chardtv(\bx_{1:H},\bu_{1:H}) \ge \epsilon \cdot \rho^{(H-1)/2}].
\end{align}
Letting $P$ be the uniform product distribution over $\xi = (\bO_t)_{t \ge 2}$ as in \Cref{lem:error_amplification}, we have that for any $g \in \cG$, and any fixed $\xi_0$,
\begin{align*}
&\Exp_{\xi \sim P} \Risk_{\epsilon}(\pihat;g,\xi) \\
&\ge \Exp_{\xi \sim P}\Pr_{\pihat,f_{g,\xi},\Dist}[\|\bx_H\| \ge \epsilon \cdot \rho^{(H-1)/2}] \tag{Defnition of $\chardtv$}\\
&= \Exp_{\xi \sim P}\Pr_{\pihat,f_{g,\xi},\Dist}[\|\bx_H\| \ge \epsilon \cdot \rho^{(H-1)/2} \mid \|\bx_2\| \ge \epsilon ] \cdot  \Exp_{\xi \sim P}\Pr_{\pihat,f,\Dist}[ \|\bx_2\| \ge \epsilon ]\\
&\ge \frac{1}{2}  \Exp_{\xi \sim P}\Pr_{\pihat,f_{g,\xi},\Dist}[ \|\bx_2\| \ge \epsilon ] \tag{\Cref{cor:exp_compound}}\\
&\ge \frac{1}{2} \cdot  \Pr_{\pihat,f_{g,\xi_0},\Dist}[ |\langle \be_1, \bhatu-  \pi_{g,\xi_0}(\bx)| \ge \epsilon ] \tag{\Cref{const:unstabl}, $\xi_0$ arbitrary}.
\end{align*}

To conclude, set $\beps_n = \minsl(\cG,\Dreg)$. By \Cref{lem:const_unstable_good}, and the fact that $(\cG,\Dreg)$ satisfies \Cref{asm:conc}, we may invokve \Cref{prop:redux}(c), from which it follows that
\begin{align*}
&\sup_{g,\xi}\Exp_{\Samp \sim (\pi_{g,\xi},f_{g,\xi})}\Exp_{\pihat \sim \est(\Samp)} \Pr_{\pihat,f_{g,\xi},\Dist}[\chardtv(\bx_{1:H},\bu_{1:H}) \ge \beps_n \kappa  \cdot \rho^{(H-1)/2}]\\
&:= \sup_{g,\xi}\Exp_{\Samp \sim (\pi_{g,\xi},f_{g,\xi})}\Exp_{\pihat \sim \est(\Samp)}\Risk_{\epsilon}(\pihat;f_{g,\xi},\Dist)\big{|}_{\epsilon = \kappa \beps_n} \ge \frac{\delta}{2}.
\end{align*}
Hence, the proof follows after recalling the definition of $\minprobcost[\chardtv]$ from \Cref{def:in_prob_risk}.

\subsection{Proof of \Cref{thm:unstable_detailed} from \Cref{thm:unstable_detailed_tv}}\label{sec:remove_time_varying}

For some universal constant radius $r_0 \in (0,1)$, a standard covering argument \citet[Section 4]{vershynin2018high} implies that there exists a set of points $\by_1,\dots,\by_N$, $N = \exp(d/2)$, such that $\cB(\by_i, 3r_0) \cap \cB(\by_j, 3r_0)$ are disjoint for any $1 \le i \ne j \le N$. We now define the function
\begin{align}
\psi_i(\bx) :=  \bump_{d}( (\bx - \by_i) / r_0),
\end{align}
where $\bump_d(\cdot)$ is the smooth bump function of \Cref{lem:bump}.

\begin{construction}\label{const:unstabl_tiv} Let $(\by_i)_{i \ge 1}$ be the packing centers, as above. Let $(\cG,\Dreg)$ be an $(k,\ell_2)$-regression family, such that $\Dreg$ is supported on a ball of radius $r_0$. Let $\rho > 2$. We define a $(d,d)$-IL family $(\cI,\Dist)$ via
\begin{itemize}
\item[(a)] $\Dist$ draws $\bz \sim \Dreg$ and appends $d - k$ zeros to form $\bx  = \by_1 + (r_0\bz,\bzero) \in \R^d$. 
\item[(b)] $\cI$ is the set of all instances of the following form:
\begin{align*}
	\pi(\bx) &= \psi_1(\bx)  g\left(\frac{\proj_{ \le k}\bx - \by_1}{r_0}\right)\be_1 + \sum_{t=2}^{N-1} -\rho \bO_t(\bx - \by_t) \psi_t(\bx) \\
	f(\bx,\bu) &= \bu -\pi(\bx).
\end{align*}
\end{itemize}
where $g \in \cG, \bO_t \in \bbO(d)$.
\end{construction}
\newcommand{\chardtiv}{\cost_{\mathrm{hard},\mathrm{tiv}}}
We also define a new hard cost
\begin{align*}
\chardtiv(\bx,\bu) := \frac{1}{\Lcost} \sum_{t = 2}^N \psi_t(\bx) \|\bx - \by_t\|.
\end{align*}
The following lemma establishes all relevant regularity conditions, including that $\cost(\bx_{1:H},\bu_{1:H}) := \max_{1 \le h \le H} \min\{1,\chardtiv(\bx,\bu)\} \in \Cliptil$, where we recall from \Cref{defn:maxlipschitz} that  $\Cliptil$ consists of all costs of the form $\max_{h \ge 1} \tilde\cost(\bx_h,\bu_h)$ for which $\tilde\cost \text{ is } 1-\text{Lipschitz  and takes values in } [0,1]$. 
\begin{lemma}\label{lem:regularity_conditions_tiv} Suppose $(\cG,\Dreg)$ is $(R,L,M)$-regular. For any dimension $d$, $\pi$ and $f$ are  $O(L+\rho)$-Lipschitz and $O(L+M+\rho)$-smooth. Similarly, $\chardtiv(\bx,\bu)$ is $1$ Lipschitz for some $\Lcost = O(1)$. Moreover, $(\pi,f)$ is $(1,0)$-EISS. Finally, each $f$ is $O(L+\rho)$-one-step-controllable.
\end{lemma}
\begin{proof}[Proof of \Cref{lem:regularity_conditions_tiv}]
 Since the $3r_0$-balls around each $\by_i$ are disjoint for differen $\by_i$, we have that for any $\bx$, either $\bx$ lies in exactly one $\cB(\by_i, 3r_0)$ for some $i$ and $\|\bx - \by_i\| \le 2.5 r_0$, lies in exactly one such ball but $\|\bx - \by_i\| \ge 2.5 r_0 $, or lies in no such ball. In the latter two cases, $\psi_i(\bx)$ definition of $G$ vanish at $\bx$, so $\nabla \psi_i(\bx), \nablatwo \psi_i(\bx)$ vanish. Hence, upper bounding the derivatives and Hessians of the above terms amounts to upper bounding the maximal contribution from any $i$. This is bounded because each $\|\by_t\| \le 1$,  $\|\bO_t\| = 1$, and and $\psi_i(\bx)$ is $O(1)$-Lipschitz and $O(1)$-smooth. The first claim now follows from the chain and product rules, using the fact that $g(\cdot)$ is $L$-Lipschitz and $M$-smooth by assumption.The guarantee for $\chardtiv$ is similar. 

 To see that $(\pi,f)$ is $(1,0)$-EISS, we observe that $f^{\pi}(\bx,\bu) = \bu$. For controllability, we invoke the special case of \Cref{lem:one_step_controllable} with $\phi(\bx) = \pi(\bx)$ being $O(L+\rho)$-Lipschitz, and $\psi(\bx,\bu) \equiv \bzero$.
\end{proof}

We can now prove \Cref{thm:unstable_detailed}. 
\begin{proof}[Proof of \Cref{thm:unstable_detailed}]
Let $(\cI,\Dist)$ be as in the time-invariant construction \Cref{const:unstabl_tiv}.
Consider the time-varying invertible rigid transformation
\begin{align}
G_t(\bx) = \bx - \by_t.
\end{align}
We can directly check that expert trajectories under the time-invariant  construction \Cref{const:unstabl_tiv} are equivalent to those under the time varying one \Cref{const:unstabl}, after applying $G_t$.  Hence, because we can invert each $G_t(\cdot)$, the  equivalence of the IL training risk and supervised learning risk, 
\begin{align}
\minbctrain(n;\cI,\cD,H) = \minsl(n;\cG,\Dreg), \label{eq:bctrain_sl_eq}
\end{align}
as in \Cref{thm:unstable_detailed_tv}, remains true for \Cref{const:unstabl_tiv}.

Moreover, with probability one $\Pr_{\pi,f,\Dist}$ for $(\pi,f)$, we have that $\bx_1$ lies in a ball of radius $r_0$ around $\by_1$, and $\bx_t = \by_t$ for all $t \ge 2$. Thus, with probability $1$, $\chardtiv$ vanishes on these trajectories. 

We also see that outside of the event $\{\max_{1 \le h \le H} \chardtiv(\bx_h,\bu_h) \ge \Lcost r_0\} $, the imitiator trajectories under \Cref{const:unstabl_tiv} and \Cref{const:unstabl} by are also related the transformation $G_t$. Hence, for $\epsilon \le \Lcost r_0$,
\begin{align*}
&\inf_{\est \in \bbA}\sup_{(\pist,f) \in (\cI,\Dist)} \Exp_{\Samp}\Exp_{\pihat \sim \est(\Samp)}\Pr_{\pihat,f,D,H}[\max_{1 \le t \le H}\chardtiv(\bx_{t},\bu_t) \ge \Lcost\epsilon] \\
&= \inf_{\est \in \bbA}\sup_{(\pist,f) \in (\cI',\Dist')} \Exp_{\Samp}\Exp_{\pihat \sim \est(\Samp)}\Pr_{\pihat,f,D',H}[\chardtv(\bx_{1:H},\bu_{1:H}) \ge \epsilon]
\end{align*}
where $\cI',\cD'$ are from the time-varying construction, \Cref{const:unstabl}. Applying the definition of $\minmax_{\cost}$ in $\dots$, the above implies
\begin{align*}
\minmax_{\cost}(n,\frac{\delta}{2};\cI,\Dist,H) &\ge \min\left\{\Lcost r_0, \minmax_{\cost}(n,\frac{\delta}{2};\cI',\Dist',H)\right\} \\
&\ge \min\left\{\Lcost r_0, \kappa \minsl(n;\cG,\Dreg)\right\}\\
&\ge \min\left\{\Lcost r_0, \kappa \minbctrain(n;\cI,\cD,H)\right\} \tag{\Cref{eq:bctrain_sl_eq}}.
\end{align*}
To conclude, recall (1) $\Lcost r_0$ is a universal constant and (2) from \Cref{lem:regularity_conditions_tiv},  $\cost$ is the maximum of $1$-Lipschitz costs, therefore lying in $\Cliptil$; hence, $\minmax_{\cost}(n,\frac{\delta}{2};\cI,\Dist,H) \le \minprob(n,\frac{\delta}{2};\cI,\Dist,H)$.  
\end{proof}


\section{Non-Simple Policies Circumvent the Construction in \Cref{thm:main_stable} }\label{label:app_non_vanilla}
\label{app:non_vanilla}

\newcommand{\Achunk}{\bbA_{\mathrm{chunk}}}
\newcommand{\Aperiod}{\bbA_{\mathrm{period}}}
This section shows that the ideas in \Cref{sec:nonvanilla} formally avoid exponential compounding error for the construction used in \Cref{thm:main_stable}.
\begin{definition}[Chunked-Policies]\label{defn:chunked} For $\ell \in \N$, let $\Achunk(L,M,3)$ consider the set of algorithms which return \emph{action-chunked} determinsitic policies of the form $\bx_{\ell h+1} \mapsto (\bu_{\ell h+1},\bu_{\ell h+2},\bu_{\ell h+2})$, which predicts sequences of $\ell$ control actions at each time $t = \ell h + 1$, each executed in open loop, for which the mapping $\bx_{\ell h+1} \mapsto (\bu_{\ell h+1},\bu_{\ell h+2},\bu_{\ell h+2})$ is $L$-Lipschitz and $M$-smooth. Note that $\Achunk(L,M,\ell = 1) = \Asmooth(L,M)$. 
\end{definition}
\begin{definition}[Periodic Time-Varying policies]\label{defn:periodic} Let $\Aperiod(L,M,3)$ denote the set of algorithms which return, with probability one, periodically time varying policies, $\pihat(\cdot,0),\dots,\pihat(\cdot,\ell-1)$, which select $\bu_t$ as $\bu_t \gets \pihat(\bx_t, (t-1)\mod \ell)$, and $\pihat(\cdot,i)$ is $L$-Lipschitz and $M$-smooth for $0 \le i \le \ell$. Note that $\Aperiod(L,M,\ell = 1) = \Asmooth(L,M)$. 
\end{definition}
In what follows, we bound a strong notion of minimax risk, $\Rtrajlp[2]$, which we recall satisfies $\Rtrajlp[2] \ge \Rtraj \ge \sup_{\cost \in \Clip}\Rsubcost$. 
\begin{proposition}\label{prop:benefits_not_vanilla}Consider the construction \Cref{const:stable} of \Cref{thm:stable_detailed,thm:main_stable}. Let $\cG$ be convex, and a  regular regression convex (\Cref{defn:regular_instances}), with $O(1)$ Lipschitzness and smoothness parameters. Finally, take $\beps_{n} := \minsl(n;\cG,\Dreg)$, where $\Dreg$ is the regression initial distribution form \Cref{const:stable}. Then,
\begin{itemize}
	\item[(a)] Let $\bbA = \Avan(L,\infty)$. Then, 
	\begin{align*}
	\minmax^{\bbA}(n, \Rtrajlp[2];\inst,\Dist,H) \lesssim \exp(-cn) + \beps_{n/3}
	\end{align*}
	\item[(b)] Let $\bbA = \Areason(L,M,1/4,1/4)$ for $L,M = O(1)$ and $\alpha, p = \Omega(1)$. Then, for some universal $q \in (0,1)$, 
	\begin{align*}
	\minmax^{\bbA}(n, \Rtrajlp[2];\inst,\Dist,H)  \lesssim \exp(-cn) + \beps_{n/3}^{1 - q}
	\end{align*}
	\item[(c)] Let $\bbA = \Achunk(L,M,3)$ or $\Aperiod(L,M,3)$, denoting the set of either $3$-action-chunked or periodic-with-period-$3$ policies defined in \Cref{defn:chunked,defn:periodic}, respectively. Then, 
	\begin{align*}
	\minmax^{\bbA}(n, \Rtrajlp[2];\inst,\Dist,H) \le \exp(-cn) + \beps_{n/3}
	\end{align*}
\end{itemize}
In particular, if we consider the regression classes of \Cref{prop:typical_true} (which are those used to instantiate \Cref{thm:main_stable}), then the above all hold with $\beps_{n/3} \gets n^{-s/k}$. This gives a form directly comparable with \Cref{thm:main_stable}.
\end{proposition}
\subsection{Proof Sketch of \Cref{prop:benefits_not_vanilla}}
\begin{proof}
 For brevity, we keep the proof slightly terser than the others in this paper; still, we make sure to provide all essential details.

For notational convience, we notate history-dependent, possibly stochastic policies $\pihat(\bx_{1:t},t)$ which may depend on past states, inputs, and the time index $t$. Note that this subsumes all classes of policies in the proposition we aim to prove. For example, $\Asmooth$ and $\Areason$ are attained by ignoring all but $\bx_t$, periodic policies depend on onlly $\bx_t$ and $t$. The case of chunked policies will require some minor-modifications, which we defer to the end of the proof. 

\newcommand{\Aregn}[1][n]{\Alg_{\mathrm{reg},#1}}
First, using the definition of minimax risk, and optimality of proper algorithms for convex classes, we can find for any $n$ a regression algorithm $\Aregn$, which is proper for $\cG$ such that, say,
\begin{align*}
\sup_{g \in \cG}\Exp_{\sampreg \sim (g,\Dreg)}\Exp_{\ghat \sim \Aregn}\Exp_{\bz \sim \Dreg}[|(\ghat - g)(\bz)|^2]^{1/2} \le 2\beps_n \label{eq:algregn_guarantee},
\end{align*} 
where the factor $2$ may be replaced by any constant strictly greater than $1$.

Now, let $\bump_d(\cdot)$ denote the bump function, and let $\pi_0$ be a ``base'' policy to be specified later. 
We then apply the following BC algorithm:
\begin{enumerate}
	\item Collect the sample of $n$-trajectories $\Samp$
	\item Count how many correspond to the $Z = 0$ case (this is possible since the initial states on each trajectory for $Z = 0$ and $Z = 1$ have disjoint support). Call this number $n_0$, and form the set $\sampreg[n_0] := \{(\Proj_{\le 3}(\bx_1^{(i)} - \xoffs),\frac{1}{\tau}\langle \be_1, \bu_1^{(i)}\rangle): \bx_1^{(i)} \text{ corresponds to the $Z=0$ case}\}$. Note that, conditioned on $n_0$, and for a BC instance indexed by a given $g \in \cG$,  $\sampreg[n_0]$ has the distribution of $n_0$ pairs $(\bz,\by)$ from the associated regression problem with regression function $g$ and initial distribution $\Dreg$.
	\item Call $\Aregn[n_0]$ on $\sampreg[n_0]$ to obtain $\ghat$. Note that by convexity of $\cG$, we may take $\ghat \in \cG$, so that $\ghat$ can be smooth.
	\item For the given base policy $\pi_0$ to be described (and specialized for each class of function), return
	\begin{align}
	\pihat(\bx_{1:t},\bu_{1:t-1},t) &= \bbarK_1 (1 - \be\be_1^\top)\bx_t \\
	&\quad+  \bump_d(\bx_t) \cdot \pi_0(\bx_{1:t},\bu_{1:t-1},t) \\
	&\quad+ \tau\cdot \restrict(\bx)\cdot\cT[\ghat](\bx)\be_1,
	\end{align}
	where $\restrict(\bx)$ and $\cT[\ghat]$ are as in \Cref{const:stable}.
\end{enumerate}
Recall the matrices $\bbarA_i$, $i \in \{1,2\}$ in the construction.
In each case, $\pi_0$ will select  some $\bbarA_i (\eye - \be_1 \be_1^\top)\bx$, for some appropriately chosen $i \in \{1,2\}$. 
  By examining \Cref{const:stable}, we see that for $\pi_0$ of this form,  then for initial states sampled on the event $Z = 1$, $\pihat$ perfectly matches the expect trajectories (see the proof of \Cref{lem:caseZ1_uninform}, which replices on the fact that $(\bbarK_1 -\bbarK_2)\bx = 0$ for $\bx$ perpendicular to $\be_1$). Hence, 
\begin{align}
\Rtrajlp[2](\pihat;\pist,\cD,\Dreg) \lesssim \Rtrajlp[2](\pihat;\pist,\cD,\Dzzero).
\end{align}
Let's turn to bounding the right-hand side. Let $(\bx^\star_t,\bu^\star_t)$ and $(\bhatx_t,\bhatu_t)$ denote random variables from the canonical coupling of trajectories from $(\pist,f,\Dzzero)$ and $(\pihat,f,\Dzzero)$, as in \Cref{def:canoc_couple}. Observe that $\bx^\star_1 = \bhatx_1$, and $\bu^\star_t \equiv \bx^\star_t = \bzero$ for $t > 1$. Hence, 
\begin{align}
\Rtrajlp[2](\pihat;\pist_{g,\xi},f_{g,\xi},\Dzzero) \lesssim \sqrt{\Exp[\min\{1,\|\bhatu_1-\bu^\star_1\|^2\}]} + \sum_{t \ge 2}\sqrt{\Exp[\min\{1,\|\bhatx_t\|^2 + \|\bhatu_t\|^2\}]}, 
\end{align}
where above $\Exp = \Exp_{\pihat,\pist_{g,\xi},f_{g,\xi},\Dzzero}$, and all random variables are as in the canonical coupling.

By using a similar argument to that of \Cref{ssec:stable_minimax} (where, with probability $1 - \exp(-\Omega(n))$, we have at least $n_0 \ge n/3$ samples) used for estimating $\hat{g}$. Let us call this event $\cE$ over the sampling. Conditioned on $n_0$ (and $\cE$), there are $n_0 \ge n/3$ i.i.d. samples from $\Dzzero$. Using the embedding of the regression problem into the control problem in \Cref{const:stable} and our choice of $\ghat$ estimator, we see that the error at time $t = 1 \mid \{Z=0\}$ is 
\begin{align}
\Exp_{\Samp \mid n_0 \ge n/3}\sqrt{\Exp_{\pihat,\pist_{g,\xi},f_{g,\xi},\Dzzero}[\|\bhatu_1 - \bu^\star_1\|^2]}  
\lesssim \Exp_{\sampreg[n_0] \mid n_0 \ge 3}\sqrt{\Exp_{\bz \sim \Dreg}|\ghat(\bz) - g(\bz)|^2} \lesssim \beps_{n/3}
\end{align}
where above we use the \Cref{const:stable} and the embedding of the $\cG$-regression problem.  By the same token, and again using the structure of the construction and the form of our policy $\pihat$ as above (for this, given $Z=0$, we have $\bhatx_2$ is in the $\be_1$-span, and $\|\bhatx_2\| \propto |\ghat(\bz) - g(\bz)|^2$)
\begin{align*}
\Exp_{\Dzzero \mid n_0 \ge n/3}\sqrt{\Exp_{\pihat,\pist_{g,\xi},f_{g,\xi},\Dzzero}[\|\bhatx_2\|^2]}  \lesssim \beps_{n/3}.
\end{align*}
Thus, it remains to show that, starting from $\bhatx_2$ satisfying the above expectation, each of the above policies will mitigate compounding error. We also note that, by Markov's inequality, we can assume that $\|\bhatx_2\|^2 \le 1/C$ for some sufficiently large $C$ which probability at least $1 - O(\beps_{n/3})$. Hence, using clipping of errors to $1$, we can bound (also accoutning for the $\exp(- \Omega(n))$ event where $n_0 \le n/3$)
\begin{align*}
&\Rtrajlp[2](\pihat;\pist_{g,\xi},f_{g,\xi},\Dzzero) \\
&\quad\lesssim \beps_{n/3} + \exp(-\Omega(n)) + \Exp_{\Dzzero \mid n_0 \ge n/3} \sum_{t \ge 2}\sqrt{\Exp[\min\{1,(\|\bhatx_t\|^2 + \|\bhatu_t\|^2\}\I\{\|\bhatx_2\| = \ost(1)\}]}.
\end{align*}
We now handle the various cases, again quite tersely.
\begin{itemize}
	\item[(a)] For \textbf{Part (a)}, we apply the same concentric stabilization trick along the $\be_1$ axis as in \Cref{sec:concentric_stabilization}, but now where in each interval in the $\be_1$ direction, we either play $\pi_0(\bx) = -\bbarA_i (\eye - \be_1\be_1)^\top\bx$ for $i \in \{1,2\}$.  As we can take  $\|\bhatx_2\|^2 \le 1/C$ to be small, state magnitudes grow at most by a constant on the first few steps, and 
	we see we still remain within the linear region of the construction. Then, within three at most $3$ time-steps, the $\be_1$ component becomes set to $0$, and remains at zero by the structure of the $(\bbarA_i,\bbarK_i)$ matrices. Finally, $\bbarK_1$ stabilizes either $\bbarA_i$ as long as states are orthogonal to the $\be_1$ direction, which keeps a constant compounding error. This establishes that
	\begin{align*}
	\Exp_{\Dzzero \mid n_0 \ge n/3} \sum_{t \ge 2}\sqrt{\Exp[\min\{1,(\|\bhatx_t\|^2 + \|\bhatu_t\|^2\}\I\{\|\bhatx_2\| = \ost(1)\}]} \lesssim \sqrt{\Exp_{\Dzzero \mid n_0 \ge n/3} \Exp[\|\bhatx_t\|^2]} \lesssim \beps_{n/3}.
	 \end{align*}
	\item[(b)] Rather than using concentric stabilization of $\pi_0$, we use the benevolent Gambler's Ruin construction to alternative $\pi_0$ between each of $\bbarA_i(\eye - \be_1 \be_1^\top)\bx$, $i \in \{1,2\}$ i.i.d. with probability $1/2$. One can show that, by mirroring the argument \Cref{sec:benevolent_ruin}, for some $q \in (0,1)$, 
	\begin{align*}
	&\Exp_{\Dzzero \mid n_0 \ge n/3} \sum_{t \ge 2}\sqrt{\Exp[\min\{1,(\|\bhatx_t\|^2 + \|\bhatu_t\|^2\}\I\{\|\bhatx_2\| = \ost(1)\}]} \\
	&\lesssim \sqrt{\Exp_{\Dzzero \mid n_0 \ge n/3} \Exp[\|\bhatx_t\|^{2(1-q)}]} \le \Exp_{\Dzzero \mid n_0 \ge n/3} \Exp[\|\bhatx_t\|^{2}]^{\frac{1-q}{2}}
	 \lesssim \beps_{n/3}^{1-q}.
	 \end{align*} 
	 We can check that this resulting policy is in $\Areason(L,M,\alpha,p)$ by noting all but the $\pi_0$ terms are deterministic and Lipschitz and smooth when $\cG$, and that the $\pi_0$ component has linear (and thus Lipschitz and smooth) mean, and that, being a mixture policy with even component probabilities, it satisfies the anti-concentration property with $\alpha,p = \Omega(1)$ by the same argument as in \Cref{exmp:BGR_pol}.
	 \item[(c)] With  alternating or history dependent policies, we altnerate between  $\bbarA_i(\eye - \be_1 \be_1^\top)\bx$, $i \in \{1,2\}$. This kills the $\be_1$ direction with at most $3$ steps, and as in part (c), we stabilize the system for the remaining part of the trajectory. This yields the same qualitative bound as in (a).
\end{itemize}

\end{proof}

\newtheorem{manualtheoreminner}{Theorem}
\newenvironment{manualtheorem}[1]{%
  \IfBlankTF{#1}
    {\renewcommand{\themanualtheoreminner}{\unskip}}
    {\renewcommand\themanualtheoreminner{#1}}%
  \manualtheoreminner
}{\endmanualtheoreminner}

\section{Proof of Upper Bounds, \Cref{thm:smoothgen}}\label{app:ub_proofs}

\newcommand{\fcl}[1]{f_{\mathrm{cl}}^{#1}}

\begin{remark}\label{rem:explore_general} The careful reader may notice that we assume that both the expert distribution is well-spread, but also that the expert policy $\pist$ is deterministic. Both appear to be in tension because, e.g. $\pist$ cannot undertake its own exploration. However, our result can be easily extended to more realistic settings with two modifications:
\begin{enumerate}
    \item If we assume a static, stationary expert policy $\pist(\bx)$, then we need only assume (up to possibly polynomial factors in horizon $H$) that the mixture measure over all time-steps $h$, defined as 
    \begin{align}
    \Pr^{\mathrm{mix}}_{\pist,f,D} = \frac{1}{H} \sum_{h=1}^H\Pr_{\pist,f,D}[ \bx_h^\star \in \cdot ]
    \end{align}
    is  well spread. This requires that only sufficient exploration can be provided in aggregate over timesteps $h$, and can therefore better take advantage in randomness from the initial conditions $\bx_1 \sim D$. 
    \item Our argument should be able to be generalized to settings where the learner is given observations of pairs $(\bx,\bu)$, where $\bx$ from a sufficiently ``well-spread'' distribution that covers the expert distribution in an appropriate sense, and $\bu = \pist(\bx)$ are perfect expert actions. This covers the \textsc{Dart} algorithm due to \cite{laskey2017dart}, but we defer formal details to future work.
\end{enumerate}
\end{remark}

\paragraph{Supporting Lemmas.} We begin by proving the following supporting lemmas, which give bounds on various relevant properties of well-spread distributions. 

The first lemma, \Cref{lem:change_of_measure} uses the properties of well-spread distributions to upper bound the expectation of $f(x + \sigma w)$ in terms of $f(x)$ for bounded $f$, where $x, w$ sampled from a well-spread distribution $P$ and a unit-balled supported distribution $D$, respectively. This allows us to upper bound the effect of injecting noise on top of any well-spread distribution.

The second supporting lemma, \Cref{lem:smooth_functions}, shows that for second-order-smooth functions (i.e. bounded hessian), we can bound the expectation under $P$ with $\sigma$-magnitude adversarial perturbations in terms of $P$ perturbed some $\sigma$-scaled noise distribution $D$. The combination of this with \Cref{lem:change_of_measure} yields a powerful result upper bounding the adversarial error.

We then combine these results with the adversarial bound of Proposition 3.1 of \citet{pfrommer2022tasil} (restated in a specialized form in \Cref{lem:tasil}) to yield our final guarantees.

\begin{lemma}[Change of Measure for Well-Spread Distributions]\label{lem:change_of_measure} Let $\cD$ be any distribution supported on the unit ball in $\R^d$. If $P$ is $(L,\epsilon,\sigma_0)$-well-spread (\Cref{def:dist_smooth}), then for all $\sigma \le \sigma_0$, and all bounded, nonnegative, measurable functions $f:\R^d \to [0,B]$, 
\begin{align}
    \Exp_{\bx \sim \cP}\Exp_{\bw \sim \cD}[f(\bx + \sigma \bw)] \le e^{L\sigma}\Exp_{\bx \sim P}[f(\bx)] + \epsilon B. 
\end{align}
\end{lemma}

\begin{proof} Let $\cK_0 := \{\bx: \dist(\bx,\cK^c) \le \sigma_0\}$.  We have
\begin{align*}
    &\Exp_{\bx \sim P}\Exp_{\bw \sim \cD}[f(\bx + \sigma \bw)] \\
    &\quad= \underbrace{\Exp\left[\I\{\bx \in \cK_0\}\Exp_{\bw \sim \cD}[f(\bx + \sigma \bw)]\right]}_{T_1}  +  \underbrace{\Exp\left[\I\{\bx \notin \cK_0\}\Exp_{\bw \sim \cD}[f(\bx + \sigma \bw)]\right]}_{T_2}
\end{align*}
As $|f| \le B$,  we have $T_2 \le B\Pr[\bx \notin \cK_0] \le B \epsilon$. Thus, we turn to upper bounding the first term. Note that if $\bx \in \cK_0 \subset \cK$, then $\bx + \sigma \bw \in \cK$, as $\|\sigma \bw\| = \sigma \|\bw\| \le \sigma \le \sigma_0$ (recall $\cD$ is supported on the unit ball). Thus, the first term is equal to 
\begin{align*}
    T_1 &= \Exp_{\bx \sim P}\Exp_{\bw \sim \cD}\left[\I\{\{\bx , \bx + \sigma \bw\} \subset \cK\}f(\bx + \sigma \bw)\right] \\
    &= \Exp_{\bw \sim \cD}\left[\underbrace{\Exp_{\bx \sim P}\left[\I\{\{\bx , \bx + \sigma \bw\} \subset \cK\}f(\bx + \sigma \bw)\right]}_{:= T_1(\bw)}\right],
\end{align*}
where we use that $f$ is non-negative, measurable to apply Tornelli's theorem. Gathering our current progress,
\begin{align}
    \Exp_{\bx \sim P}\Exp_{\bw \sim \cD}[f(\bx + \sigma \bw)] \le \epsilon B + \Exp_{\bw \sim  \cD}[T_1(\bw)] \label{Eq:current_progress}
\end{align}
Via a change of variables, we have that  the quantity $T_1(\bw)$ above is equal to  
\begin{align*}
    \int_{\bx 
 \in \R^d} \I\{\{\bx , \bx + \sigma \bw\} \subset \cK\} f(\bx + \sigma \bw) p(\bx) \rmd \bx =  \int_{\bu 
 \in \R^d} \I\{\{\bu - \sigma \bw , \bu\} \subset \cK\} f(\bu) p(\bu - \sigma \bw) \rmd \bu
\end{align*}
Now notice that (i) $\cK$ is convex, (ii) $\{\bu - \sigma \bw , \bu\} \subset \cK$ and (iii) $\log p(\cdot)$ is $L$-Lipschitz on $\cK$. This gives that for any $\bu,\bw$ for which $\I\{\{\bu - \sigma \bw , \bu\} \subset \cK\} = 1$,  we have
\begin{align}
    |\log p(\bu - \sigma \bw) - \log p(\bu )| \le L \sigma \|\bw\| \le L \sigma, 
\end{align}
and thus
\begin{align}
    p(\bu- \sigma \bw) \le e^{L\sigma} p(\bu). 
\end{align}
It follows then that we can bound 
\begin{align}
    T_1(\bw) &= \int_{\bu 
 \in \R^d} \I\{\{\bu - \sigma \bw , \bu\} \subset \cK\} f(\bu) p(\bu - \sigma \bw) \rmd \bu \\
 &\le e^{L\sigma}\int_{\bu 
 \in \R^d} \I\{\{\bu - \sigma \bw , \bu\} \subset \cK\} f(\bu) p(\bu) \rmd \bu\\
 &\le e^{L\sigma}\Exp_{\bu \sim P}[f(\bu)]
\end{align}
Since the above bound holds for all $\bw:\|\bw\| \le 1$, combining the above display with \eqref{Eq:current_progress} concludes the demonstration. 
\end{proof}

\begin{lemma}[Smooth Functions]\label{lem:smooth_functions} Suppose $\hat{\pi},\pi^\star: \R^d \to \R^m$ are $M$-second-order-smooth. Then, for $f(\bx) := \|\hat\pi(\bx) - \pi^\star (\bx)\|^2$, zero-mean distribution $\cD$ supported on the unit ball, and with $\nu = 1/\lambda_{\min}(\Exp_{\bw \sim \cD}[\bw \bw^\top])$, we have
\begin{align}
     \sup_{\bw \in \Ball_d}\|(\hat \pi - \pi^\star)(\bx + \sigma \bw)\|^2 \le 8\nu \Exp_{\bw \sim \cD}\|(\hat \pi - \pi^\star)(\bx + \sigma \bw)\|^2 + 16 \nu M^2 \sigma^4. \label{eq:smoothly_supremal}
\end{align}
Consequently, for any $P$ which is $(L,\epsilon,\sigma_0)$-well-spread, and if $\max_{\bx} \|\hat \pi(\bx) - \pi^\star(\bx)\|^2 \le B$, then for all $\sigma \le \min\{\sigma_0,1/L\}$, 
\begin{align*}
\Exp_{\bx \sim P}\left[\sup_{\bw \in \Ball_d}\|(\hat \pi - \pi^\star)(\bx + \sigma \bw)\|^2\right] \leq 8\nu \Exp_{\bx \sim P}[\|(\pihat - \pist)(\bx)\|^2] + 8 \nu B \epsilon + 16 \nu M^2 \sigma^4.
\end{align*}
Specializing to the intermediate distribution $\cD = S^{d-1}$ yields $\nu = d$ and the relation:
\begin{align}
\Exp_{\bx \sim P}\left[\sup_{\bw \in \Ball_d}\|(\hat \pi - \pi^\star)(\bx + \sigma \bw)\|^2\right] \leq 8d \Exp_{\bx \sim P}[\|(\pihat - \pist)(\bx)\|^2] + 8 d B \epsilon + 16 d M^2 \sigma^4.
\end{align}
\end{lemma}

\begin{proof}
To simplify matters, it suffices to study a function $\pi(\bx) = \hat \pi - \pi^\star$ which is $2M$-second order smooth. We shall also prove the more general statement for arbitrary $\cD$. Define $\nu = 1/\lambda_{\min}(\Exp_{\bw \sim \cD}[\bw \bw^\top])$; note that in the case where $\cD$ is uniform on the sphere, $\nu = d$, recovering the desired bound. We have
\begin{align*}
    \sup_{\bw \in \Ball_d} \|\pi(\bx+ \sigma \bw)\|^2 &\le \sup_{\bw \in \Ball_d} 2\|\pi(\bx+\sigma \bw) - \pi(\bx) - \sigma \nabla \pi(\bx) \cdot \bw\|^2 + 2\|\pi(\bx) - \sigma \nabla \pi(\bx) \cdot \bw\|^2\\
    &\le2\|M \sigma^2 \bw\|^2  + 2 \sup_{\bw \in \Ball_d}  \|\pi(\bx) - \sigma \nabla \pi(\bx) \cdot \bw\|^2\\
    &\le2M^2 \sigma^4  + 4\|\pi(\bx)\|^2 +  4 \sup_{\bw \in \Ball_d}  \|\sigma \nabla \pi(\bx) \cdot \bw\|^2\\
    &=2M^2 \sigma^4  + 4\|\pi(\bx)\|^2 +  4 \sigma^2 \|\nabla \pi(\bx) \|_{\op} \numberthis \label{eq:supw_thing}
\end{align*}
On the other hand, using the elementary inequality $\|\bx + \bx'\|^2 \ge \frac{1}{2}\|\bx\|^2 - \|\bx'\|^2$, we have
\begin{align*}
   \Exp_{\bw \sim \cD} \|\pi(\bx+ \sigma \bw)\|^2 &\ge \frac{1}{2}  \Exp_{\bw \sim \cD} \|\pi(\bx) - \sigma \nabla \pi(\bx) \cdot \bw\|^2\\
   &\quad - \Exp_{\bw \sim \cD}\|\pi(\bx+\sigma \bw) - \pi(\bx) - \sigma \nabla \pi(\bx) \cdot \bw\|^2 
\end{align*}
Using the same smoothness argument as above, the second term on the right hand side contributes at most $(\frac{1}{2} \cdot 2M \sigma^2)^2 = M^2\sigma^4$. Moreover, using that $\Exp[\bw] = 0$ and $\Exp[\bw \bw^\top] = \frac{1}{\nu}$ by definition, we have 
\begin{align*}
    \Exp_{\bw \sim \cD} \|\pi(\bx) - \sigma \nabla \pi(\bx) \cdot \bw\|^2 = \|\pi(\bx)\|^2 + \frac{\sigma^2}{\nu}\trace(\nabla\pi(\bx)) \ge \frac{1}{\nu}(\|\pi(\bx)\|^2 + \sigma^2\|\nabla \pi(\bx)\|_{\op}). 
\end{align*}
Hence, we have,
\begin{align*}
    \Exp_{\bw \sim \cD} \|\pi(\bx+ \sigma \bw)\|^2 &\ge \frac{1}{2\nu}(\|\pi(\bx)\|^2 + \sigma^2\|\nabla \pi(\bx)\|_{\op}) - M^2 \sigma^4 \\
    &= \frac{1}{8\nu}(4\|\pi(\bx)\|^2 + 4\sigma^2\|\nabla \pi(\bx)\|_{\op}) - M^2 \sigma^4 \\
    &\ge \frac{1}{8\nu}\left(\sup_{\bw \in \Ball_d}\|\pi(\bx+\sigma \bw)\|^2 - 2M^2 \sigma^4\right) - M^2 \sigma^4 \tag{by \eqref{eq:supw_thing}}\\
    &\ge \frac{1}{8\nu}\left(\sup_{\bw \in \Ball_d}\|\pi(\bx+\sigma \bw)\|^2\right) - 2M^2 \sigma^4
\end{align*}
Rearranging, 
\begin{align*}
     \sup_{\bw \in \Ball_d}\|\pi(\bx+\sigma \bw)\|^2 \le 8\nu \Exp_{\bw \sim \cD} \|\pi(\bx+ \sigma \bw)\|^2 + 16\nu M^2 \sigma^4.
\end{align*}
\end{proof}

For compatibility with \Cref{lem:tasil}, we use Markov's and rearrange the above bound to upper bound the probability of the exceeding a given threshold value.

\begin{lemma}\label{lem:smooth_shrinking_radius} Suppose that $\hat{\pi}, \pi^\star$ are $M$-second-order-smooth, $B$-bounded, and $P$ is $(L,\epsilon, \sigma_0)$-well-spread. Let $\kappa := \sqrt{\Exp_{x \sim P}[\|\pihat(\bx) - \pist(\bx)\|^2]}$, $\kappa_1 := \max\{\kappa, \epsilon^2\}, \kappa_2 := \max\{\kappa, \sqrt{\epsilon}\}$. Provided $\kappa_1 \leq \rho_0^2, 1/L^2$, for any $K \geq 0$,
\begin{align*}
    \Pr_{x \sim P}[\sup_{\bw \in \Ball_d}\|(\pihat - \pist)(x + \sqrt{\kappa_1} \bw)\| \geq K \sqrt{\kappa_1}] \leq \frac{d (8 + 16B^2 + 16M^2)}{K^2} (\kappa + \epsilon).
\end{align*}
\end{lemma}
\begin{proof}Let $\sigma := \sqrt{\kappa_1}$. Note that $\epsilon \leq \sigma \leq \min\{\rho_0, 1/L\}$. Since $\hat{\pi}, \pi^\star$ are $M$-second-order-smooth, $B$-bounded and $P$ is $(L,\epsilon,\sigma_0)$-well-spread with $\sigma < 1/L, \sigma_0$,
\begin{align*}
    \Exp_{\bx \sim P}\left[\sup_{\bw \in \Ball_d}\|(\hat \pi - \pi^\star)(\bx + \sigma \bw)\|^2\right] &\leq 8 d\kappa^2 + 16B^2d \epsilon + 16M^2\sigma^4d \\
    &\leq 8 d \kappa_2^2 + 16 B^2 d \kappa_2^2 + 16 M^2 d \kappa_2^2 \\
    &\leq d(8 + 16 B^2 + 16 M^2) \kappa_2^2.
\end{align*}
By Markov's inequality and using that $\frac{k_2^2}{k_1} \leq k_2 \leq (k + \sqrt{\epsilon})$,
\begin{align*}
    \Pr_{x\sim P}[\sup_{\bw \in \Ball_d} \|(\pihat - \pist)(x  + \sqrt{k_1} \bw)\| \geq K \sqrt{\kappa_1} ] \leq \frac{d (24 + 16B^2 + 16M^2)}{K^2} (\kappa + \sqrt{\epsilon}).
\end{align*}
\end{proof}



\begin{lemma}[TaSIL,\,\cite{pfrommer2022tasil}]\label{lem:tasil} Let $(\pist, f)$ be deterministic and $\pist$ be 
$(C,\rho)$-E-IISS.
For any deterministic policy $\pihat$ and initial state $\bx_1$, let $\hat{\bx}_1 = \bx^\star_1 := \bx_1$ and $\bx^\star_{t+1} := \fcl{\pist}(\bx_t^{\star}), \hat{\bx}_{t+1} := \fcl{\pist}(\hat{\bx}_t) \, \forall t \geq 2$. Then for any $\epsilon > 0, t > 0$,
\begin{align*}
\max_{1 \leq k \leq t - 1}\sup_{\bw \in \Ball_d} \|(\pihat - \pist)(\bx^{\star}_k + \epsilon \bw)\| \leq \frac{1 - \rho}{C}\epsilon \Longrightarrow  \max_{1 \leq k \leq t} \|\hat{\bx}_k - \bx^{\star}_k\| \leq \epsilon
\end{align*}
\end{lemma}
\begin{proof}This is a simple proof using induction. The base case $t = 1$ is true by construction as $\bx^\star_1 = \hat{\bx}_1$. For $t \geq 2$, we assume the statement holds for $t - 1$. Then, it follows by the induction hypothesis that
\begin{align*}
    &\max_{1 \leq k \leq t - 1}\sup_{\bw \in \Ball_d} \|(\pihat - \pist)(\bx^{\star}_k + \bw)\| \leq \frac{1 - \rho}{C}\epsilon \\
    &\Longrightarrow  \max_{1 \leq k \leq t - 1} \|\hat{\bx}_k - \bx^{\star}_k\| \leq \epsilon \textrm{ (from induction hypothesis)}\\
    &\Longrightarrow  \max_{1 \leq k \leq t - 1} \|\pihat(\hat{\bx}_k) - \pist(\hat{\bx}_k)\| \leq \max_{1 \leq k \leq t - 1} \sup_{\|\delta\| \leq \epsilon} \|\pihat(\bx^\star_k + \delta) - \pist(\bx^\star_k + \delta)\|.
\end{align*}
We now recall the following property of $(C, \rho)$ incrementally input-to-state-stabilizng policies:
\begin{align*}
    \|\hat{\bx}_t - \bx_t^\star\| \leq C\sum_{s=1}^{s}\rho^{t-s}\|\hat{\pi}(\hat{\bx}_s) - \pi^\star(\hat{\bx}_s)\| \leq \frac{C}{1 - \rho}\left(\max_{1 \leq k \leq t - 1} \|\hat{\pi}(\hat{\bx}_s) - \pi^\star(\hat{\bx}_s)\|\right).
\end{align*}
This yields the desired bound,
\begin{align*}
    \|\hat{\bx}_t - \bx_t^\star\| &\leq \frac{C}{1 - \rho} \left(\max_{1 \leq k \leq t - 1} \|\hat{\pi}(\hat{\bx}_s) - \pi^\star(\hat{\bx}_s)\|\right) \\
    &\leq \frac{C}{1 - \rho}\left(\max_{1 \leq k \leq t - 1} \sup_{\|\delta\| \leq \epsilon} \|\pihat(\bx^\star_k + \delta) - \pist(\bx^\star_k + \delta)\|\right) \\
    &\leq \epsilon.
\end{align*}

\end{proof}

\paragraph{Main smoothness result:}

The main result of \Cref{thm:smoothgen} follows by a straightforward combination of \Cref{lem:smooth_shrinking_radius}, combined with \Cref{lem:tasil}. At a high level, \Cref{lem:tasil} provides performance bounds for the learned policy given a bound on the adversarial error, whie \Cref{lem:smooth_shrinking_radius} gives precisely a bound on the probability of a small adversarial error occuring for well-spread distributions. 

\begin{manualtheorem}{5}[Smooth Training Distribution]
Consider any $(d,m)$-BC instance $(\inst, \Pnot)$.
Provided for any $(\pist, f) \in \inst$, $h \in [H]$, the distribution $\Pr_{\pi,f,\Pnot}$ is $(L,\epsilon, \sigma_0)$-well-spread (\Cref{def:dist_smooth}) for $h > 1$ and $\pist,\pihat$ are deterministic, $M$-second-order-smooth, $L_{\pi}$-Lipschitz, and $B$-bounded, and $\pist$ is $(C, \rho)$ incrementally input-to-state stablizing (\Cref{def:diss}), the following holds. Then, provided that $\Rtrain(\pihat; \pist, f, \Pnot, H) \leq \min\{\rho_0, 1/L\}$,
\begin{align*}
\Rcost(\pihat;\pist, f, D, H) &\leq 
     c H d \frac{C^2}{(1 - \rho)^2}\left[\Rtrain(\pihat; \pist, f, \Pnot, H) + \sqrt{\epsilon}\right].
\end{align*}
where $c := 16 d(1 + 2B^2 + 2M^2)$.
\end{manualtheorem}

\begin{proof}Let $\bx_1^\star = \hat{\bx}_1 \sim \Pnot, \bx_{t+1}^\star = \fcl{\pist}(\bx_t^\star), \hat{\bx}_{t+1} = \fcl{\pihat}(\hat{\bx}_t)$ and define $\kappa := \Rtrain(\pihat; \pist, f, \Pnot, H)$, $\kappa_1 := \max\{\kappa, \epsilon^2\}, \kappa_2 := \max\{\kappa, \epsilon\}$.
We note that since $\cost$ is 1-Lipschitz and $\pist,\pihat$ are $L_{\pi}$-Lipschitz, we can rewrite,
\begin{align*}
    \Rcost(\pihat; \pist, f, D, H) &\leq \Exp_{\hat{\pi},\pi^\star,f,D}\left[\sum_{h=1}^{H}(\min\{\|\bu_h^\star - \hat{\bu}_h\| + \|\bx_h^\star - \hat{\bx}_h\|, 1\})\right]  \\
    &\leq (1 + 2L_{\pi})\Exp_{\pihat,\pist, f,\Pnot}\left[\sum_{h=1}^{H}\min\{\|\bx_h - \bx^\star_h\|,1\}\right] \\
    &\leq (1 + 2L_{\pi})H \Exp_{\pihat,\pist, f,\Pnot}\left[\max_{1 \leq h \leq H} \min\{ \|\bx_h - \bx^\star_h\|, 1\} \right] \\
    &= (1 + 2 L_{\pi})H \int_{0}^{1} \Pr_{\pihat,\pist, f,\Pnot}\left[\max_{1 \leq h \leq H} \|\hat{\bx}_h - \bx^\star_h\| \geq \eta \right] d\eta \\
    &\leq (1 + 2 L_{\pi})H \Bigg(\int_{0}^{\sqrt{\kappa_1}} \Pr_{\pihat,\pist, f,\Pnot}\left[\max_{1 \leq h \leq H} \|\hat{\bx}_h - \bx^\star_h\| \geq \eta \right] d\eta \\
    &\hspace{5em} + \Pr_{\pihat, \pist, f, \Pnot} \left[\max_{1 \leq h \leq H} \|\hat{\bx}_h - \bx^\star_h\| \geq \sqrt{\kappa_1} \right] \Bigg).
\end{align*}
We use \Cref{lem:tasil} and \Cref{lem:smooth_shrinking_radius} to bound the tail probability:
\begin{align*}
    \Pr_{\pihat,\pist, f, \Pnot}\left[\max_{1 \leq h \leq H} \|\hat{\bx}_h - \bx^\star_h\| \geq \sqrt{\kappa_1} \right] &\leq \Pr_{\pist, f, \Pnot}\left[\max_{1 \leq h \leq H} \sup_{\bw \in \Ball_d} \|(\pihat - \pist)(\bx_h + \sqrt{\kappa}_1\bw)\| \geq \frac{1 - \rho}{C}\sqrt{\kappa}_1 \right] \\
    &\leq \frac{C^2}{(1 - \rho)^2}d(8 + 16B^2 + 16M^2)(\kappa + \sqrt{\epsilon}).
\end{align*}
We can similarly bound the probability over the bulk,
\begin{align*}
    \int_{0}^{\sqrt{\kappa_1}} \Pr_{\pihat,\pist, f,\Pnot}\left[\max_{1 \leq h \leq H} \|\hat{\bx}_h - \bx^\star_h\| \geq \eta \right] d\eta &\leq \int_{0}^{\sqrt{\kappa_1}}\Pr_{\pist, f,\Pnot}\left[\max_{1 \leq h \leq H} \sup_{\bw \in \Ball_d} \|(\pihat - \pist)(\bx_x + \eta \bw)\| \geq \frac{1 - \rho}{C}\eta \right]d\eta \\
    &\leq \int_{0}^{\sqrt{\kappa_1}}\Pr_{\pist, f,\Pnot}\left[\max_{1 \leq h \leq H} \sup_{\bw \in \Ball_d} \|(\pihat - \pist)(\bx_h + \sqrt{\kappa_1} \bw)\| \geq \frac{1 - \rho}{C}\eta \right]d\eta \\
    &\leq \frac{C}{1 - \rho}\Exp\left[\max_{1 \leq h \leq H} \sup_{w \in \Ball_d} \|(\pihat - \pist)(\bx_h + \sqrt{\kappa_1}\bw)\|\right] \\
    &\leq \frac{C}{1 - \rho}[4\sqrt{d} \kappa + 4B \sqrt{d} \sqrt{\epsilon} + 4 \sqrt{d M} \kappa].
\end{align*}
Combining these bounds,

\begin{align*}
    \Rcost(\pihat; \pist, f, D, H) &\leq 16 H d \frac{C^2}{(1 - \rho)^2}(1 + 2 L_{\pi})(1 + 2B^2 + 2M^2)\left[\Rtrain(\pihat; \pist, f, \Pnot, H) + \sqrt{\epsilon}\right] \\
    &= c H d \frac{C^2}{(1 - \rho)^2}\left[\Rtrain(\pihat; \pist, f, \Pnot, H) + \sqrt{\epsilon}\right].
\end{align*}
\end{proof}
\section{Experimental Details}\label{app:experiments}

\paragraph{Dynamics Details.} For the experiments in \Cref{sec:experiments}, we use the construction \Cref{const:stable}, with $d = 4$ and visualize the performance on the $A_1, K_1$ matrices. We use $\mu = 1/8$ (instead of $1/4$) to slightly reduce the instability of the system so that we can visualize the effect of larger $H$. This does not affect the key properties of the construction beyond slightly reducing the instability.

For the nonlinear perturbation function $g$ used in the construction of the dynamics and expert of \Cref{const:stable}, we used a randomly initialized 3-layer MLP with 16 hidden units in each layer and tanh activations. The weights and biases were initialized using a truncated normal and a uniform distribution over $[-1,1]$, respectively.

\paragraph{Model Details.} The behavior cloning policies were parameterized by 4-layer MLPs of similar design to the $g$ network to ensure feasibility of the learning problem. For all diffusion policy experiments, we used a 3-layer MLP with 16 hidden units with FiLM conditioning \cite{perez2018film}. We used a $256$-dimensional sinusoidal time embedding, concatenated with the observation, as an input to the FiLM embedding.

\paragraph{Training Details.} We used a batch size of $512$ for the behavior cloning and $128$ for the diffusion policy. All policies were trained for $10,000$ iterations using $N=8192$ training trajectories. For all expertments we use the AdamW optimizer \citep{loshchilov2017decoupled} with a cosine decay schedule \citep{loshchilov2016sgdr}. For the behavior cloning experiments, we use an initial learning rate and weight decay of $1 \times 10^{-3}$ and for diffusion policy we use an initial learning rate of $1 \times 10^{-4}$ and weight decay of $1 \times 10^{-5}$.

\paragraph{Evaluation Details.} All models were evaluated over $16$ initial conditions across $5$ different training seeds (for a total of $80$ unique $\bx_1$). For the action chunking experiments, we trained models with chunk lengths $h  \in [1, 2, 4, 8]$. For the replica noising experments, we used a noise parameter of $\sigma = 0.1$. We show the performance of the different policies over rollouts of length $H \in [2, 4, 8, 12, 20, 26, 32]$.

\end{document}